\documentclass[a4paper]{article}

\usepackage[english]{babel}
\usepackage[utf8x]{inputenc}
\usepackage[T1]{fontenc}
\usepackage{verbatim}
\usepackage{bbm}
\usepackage{bm}

\usepackage[letterpaper,top=3cm,bottom=2cm,left=3cm,right=3cm,marginparwidth=1.75cm]{geometry}

\usepackage{amsmath}
\usepackage{amsfonts}
\usepackage{amsthm}
\usepackage{dsfont}
\usepackage{enumerate}
\usepackage{sectsty}
\usepackage{wrapfig}
\usepackage{graphicx}
\usepackage{url}
\usepackage[colorinlistoftodos]{todonotes}
\usepackage[colorlinks=true, allcolors=blue]{hyperref}

\usepackage[backend=biber, maxnames=5, minnames=4]{biblatex}
\allowdisplaybreaks

\newtheorem{theorem}{Theorem}[section]
\newtheorem{lemma}[theorem]{Lemma}
\newtheorem{remark}[theorem]{Remark}
\newtheorem{proposition}[theorem]{Proposition}
\newtheorem{hypothesis}[theorem]{Hypothesis}

\newtheorem{assumption}[theorem]{Assumption}
\newtheorem{definition}[theorem]{Definition}

\newcounter{casenum}

\numberwithin{equation}{subsection}

\allowdisplaybreaks

\title{Scaling Effects and Uncertainty Quantification in Neural Actor Critic Algorithms \footnote{This article is part of the project ``DMS-EPSRC: Asymptotic Analysis of Online Training Algorithms in Machine Learning: Recurrent, Graphical, and Deep Neural Networks'' (NSF DMS-2311500).}}

\author{Nikos Georgoudios\footnote{Department of Mathematics \& Statistics, Boston University, E-mail: \url{ngeor@bu.edu}.}, Justin Sirignano\footnote{Mathematical Institute, University of Oxford, E-mail: \url{Justin.Sirignano@maths.ox.ac.uk}.}, and Konstantinos Spiliopoulos\footnote{Department of Mathematics \& Statistics, Boston University, E-mail: \url{kspiliop@bu.edu}.}
}

\begin{document}
\maketitle

\begin{abstract}
We investigate the neural Actor-Critic algorithm, utilizing shallow neural networks for both the Actor and Critic models. The focus of this work is twofold: (a) compare, under various scaling schemes, the convergence properties of the network outputs as the number of hidden units $N$ and the number of training steps tend to infinity, and (b) provide precise control of the error of the approximations for the different scalings. Previous work has shown that convergence to ODEs with random initial condition occurs with a $1/\sqrt{N}$ scaling. In this work we shift the focus from merely analyzing convergence speed to developing a more comprehensive statistical understanding of the algorithm's output, seeking to quantify  the uncertainty of the neural actor-critic algorithm. More specifically we study the more general $1/N^{\beta}$ scaling, where $\beta \in (1/2, 1)$ is treated as a hyperparameter. We derive an asymptotic expansion of the network outputs, viewing them as statistical estimators, in order to elucidate their structure. We show that, to leading order in $N$, the variance scales as $N^{\frac{1}{2}-\beta}$, implying improved statistical robustness as $\beta$ approaches $1$. This behavior is supported by numerical experiments, which also suggest faster convergence for this choice of scaling. Finally, our results provide concrete guidelines for selecting algorithmic hyperparameters—such as learning rates and exploration rates—as functions of 
$N$ and $\beta$, ensuring provably favorable statistical behavior.
\end{abstract}

\section{Introduction}

Neural networks have profoundly shaped the field of reinforcement learning, enabling substantial advances in a wide array of sequential decision-making problems. Notable recent achievements of deep reinforcement learning include solving complex robotic manipulation tasks \cite{rubicscube, rt2} and enhancing the training of large language models through reinforcement learning from human feedback \cite{llm1, llm2}. Among the most prominent algorithms in this domain are deep Q-learning and the neural Actor-Critic algorithm. Q-learning was originally introduced as a tabular method in \cite{qlearn}, while the integration of neural networks into Q-learning, effectively using them to approximate the state-action value function, or $Q$-function, was pioneered in \cite{td_tesauro}. The actor-critic framework, introduced in \cite{konda1999actor}, extends this paradigm by employing two neural networks: an actor, which learns an optimal policy, and a critic, which estimates the corresponding Q-function. This dual-network architecture balances policy improvement and value estimation, offering a scalable and flexible approach for tackling high-dimensional reinforcement learning tasks.

In this work we focus on the online neural network actor-critic reinforcement learning algorithm where the neural network parameters are trained via stochastic gradient descent (SGD). In particular, we model the actor and the critic neural networks as
\begin{align*} 
P_{\theta}^{N}(x,a) &= \frac{1}{N^{\beta}} \sum_{i=1}^N B^i \sigma(U^i\cdot (x,a))\\
Q_{\omega}^{N}(x,a) &= \frac{1}{N^\beta} \sum_{i=1}^N C^i\sigma(W^i \cdot (x,a))    
\end{align*}
respectively, where $(x,a)\in \mathcal{X}\times \mathcal{A}$ is an state-action pair, $N$ is the number of hidden units, $\theta=(B^{i},U^{i})_{i=1}^{N}, \omega=(C^{i},W^{i})_{i=1}^{N}$ are the parameters to be trained via SGD, becoming time dependent $\theta_{t}$ and $\omega_{t}$, and  $\beta \in \left(\frac{1}{2},1\right)$ acts as a scaling parameter. Details on assumptions and the algorithms are in Section \ref{assumptions_sec}. Note that $\beta=1/2$ corresponds to the square root scaling, often called the neural tangent kernel regime, whereas $\beta=1$ corresponds to the mean field scaling, often called the feature learning regime (details follow below).

The goal of this paper is twofold. Firstly, we analyze the limiting behavior of the pair $(P_{\theta_{t}}^{N}, Q_{\omega_{t}}^{N})$ as $N\rightarrow\infty$ and  $t\rightarrow\infty$ for the different scalings $\beta \in \left(\frac{1}{2},1\right)$. Secondly,  we obtain, in the form of an asymptotic expansion in $N\rightarrow\infty$, a fine-grain analysis of the asymptotic behavior of the pair $(P_{\theta_{t}}^{N},Q_{\omega_{t}}^{N})$. This exactly characterizes its bias-variance decomposition to leading order in $N$.

The main result of the paper is Theorem \ref{main_result_theorem}, where, pointwise in the scaling parameter $\beta \in \left(\frac{1}{2},1\right)$, we rigorously establish an asymptotic expansion in $N\rightarrow\infty$. In particular, for $n\in \mathbb{N}$ and $\beta \in \left( \frac{2n-1}{2n}, \frac{2n+1}{2n+2} \right)$ we prove that as  $N\rightarrow\infty$
         \begin{align*}
         P_t^N &\approx P_t^{(0)} + N^{\beta-1}P_t^{(1)} + N^{2\beta-2}P_t^{(2)}+...+N^{(n-1)(\beta-1)}P_t^{(n-1)}+N^{\frac{1}{2}-\beta}P_t^{(n)},\\
     Q_t^N &\approx Q_t^{(0)} + N^{\beta-1}Q_t^{(1)} + N^{2\beta-2}Q_t^{(2)}+...+N^{(n-1)(\beta-1)}Q_t^{(n-1)}+N^{\frac{1}{2}-\beta}Q_t^{(n)},
    \end{align*}
where, $P_t^{(j)}(x,a)$ and $Q_t^{(j)}(x,a)$ are both $N,\beta-$independent and are deterministic quantities for $j<n$ and random quantities when $j=n$. An expansion like this provides us with useful and fine grain information on the structure and convergence behavior of the networks. In particular, it is providing us with the explicit form of the bias and the variance of those estimators to leading order in $N$. Moreover, it allows us to rigorously optimize with respect to the scaling parameter $\beta$, at least in the sense of controlling the bias and the variance of the neural network estimators and achieving an optimal balance between them. In Theorem \ref{large_t_conv_thm} we establish that $Q_t^{(0)}$ converges to the state-action value function and  $P_t^{(0)}$ to a stationary point of the corresponding objection function as $t\rightarrow\infty$.

These expansions suggest that the leading-order bias term scales as $N^{\max{\{\beta - 1, \frac{1}{2} - \beta\}}}$, which is minimized with respect to $N$ when $\beta = \frac{3}{4}$. On the other hand, the asymptotic expansion also suggests that to, leading order in $N$, the variance is minimized at $\beta \to 1$. This implies potentially greater stability of the algorithm when $\beta=1$. In fact, the numerical experiments of Section \ref{S:NumericalResults} reveal that the fastest convergence occurs around $\beta = 1$, which also minimizes the variance. 
 Both numerical observations are consistent with the asymptotic expansions of the actor and the critic,  and suggest that the non-random terms decay quickly in training time.

The mathematical innovation of our work on the Actor-Critic algorithm lies in shifting the focus from merely analyzing convergence speed to developing a more comprehensive statistical understanding of the algorithm's output, including a rigorous characterization of both the bias and the variance of the neural network estimates.

The practical contribution of our work is that a higher value of the scaling parameter $\beta$ leads to potentially more stable estimators for the actor, which is what one is interested in practice during deployment, and to higher rewards. As we shall also discuss in the sequel, the value of the scaling factor $\beta$  influences the convergence rate and the statistical fluctuations of the algorithm with more benefits for large $N$, as $\beta$ gets larger.

Reinforcement learning and Q-learning has evolved a lot in recent years. Classical texts such as \cite{BertsekasTsitsiklis,KushnerYin,SuttonBarto1998},  and surveys like \cite{Arulkumaran2017} give a good overview of the topic.  Reinforcement learning algorithms are typically based on some variation of  policy gradient methods or Q-learning algorithm, see \cite{policygradient1999} and \cite{qlearn,WatkinsDayan,Tsitsiklis1994}. 
Reinforcement learning combined with neural networks (i.e., using Q-networks) was proposed in \cite{MnihEtAll_videoGames_2015}, further developed in \cite{Mnih2016} and in many other works soon thereafter.

Mathematical analysis of these algorithms though is still in its infancy. The works that are perhaps closest to ours are \cite{arlnn} and \cite{nac}. In \cite{arlnn} the authors study the limit point of the Q-learning algorithm with shallow neural networks as $N,t\rightarrow\infty$ in the pure-exploration case for $\beta=1/2$, i.e., they determine $Q^{(0)}_{t}$ in the pure exploration case.  In \cite{nac}  the authors study the limit point of the actor-critic reinforcement learning algorithm with shallow neural networks as $N,t\rightarrow\infty$, i.e., they determine $(P^{(0)}_{t}, Q^{(0)}_{t})$, also for $\beta=1/2$. Note that in this work, not only we address the limit for all $\beta\in(1/2,1)$, but we also provide the fine-grade analysis of the behavior of the output through its asymptotic expansion mentioned above. Other related work includes \cite{shallow_conv_wang}, \cite{multilayer_conv}, and \cite{cayci2022finitetimeanalysisentropyregularizedneural}. The convergence of shallow and deep neural actor–critic models is established in \cite{shallow_conv_wang} and \cite{multilayer_conv}, respectively, while \cite{cayci2022finitetimeanalysisentropyregularizedneural} examines the impact of regularization techniques on the bias and variance of the algorithm.

Our emphasis on the scaling factor $\beta$ is motivated by its role in dictating the values of the learning rate and exploration rate hyperparameters through the relations in \eqref{hyperparamneter_cond_eq}. Even though the value of $\beta$ does not affect the limit as $N$ (the number of neurons) and $T$ (training time) tend to $\infty$, as characterized in Theorems \ref{main_result_theorem} and \ref{large_t_conv_thm}, it does affect the rest of the terms in the asymptotic series expansion leading to a meaningful bias-variance decomposition of the algorithm. The value of $\beta$ also affects the choice of the learning rate hyperparemeter, which we also concretely specify. In short the value of the scaling factor $\beta$  influences the convergence rate and the statistical fluctuations of the algorithm. Specifically, the form of the expansion reveals that the variance strictly decreases as $\beta \to 1$, while numerical experiments suggest faster convergence for the same value of $\beta$, implying that, for long training times, the bias that the random fluctuations induce dominate the convergence errors .

A similar line of investigation into the effect of the scaling factor on the uncertainty of shallow neural networks appears in \cite{lln}, \cite{clt}, and \cite{normeff}, where neural network regression (i.e., not the reinforcement learning setting that we study here) is studied under different scaling regimes. It is found in those works that in the regression case, for $\beta\in(1/2,1)$ and to leading order in $N$, the asymptotic variance of the algorithm also scales like $N^{\frac{1}{2} - \beta}$. These results suggest that uncertainty due to initialization decreases as $\beta$ approaches $1$, while numerical studies also indicate improved convergence speed.

This behavior has been further generalized to training of deep neural networks in the regression problem in \cite{normdeep}, which extends the analysis of the $\beta \in \left(\frac{1}{2}, 1\right)$ regime to multiple layers. The findings reveal a similar trend: uncertainty diminishes as $\beta$ increases, with the strongest influence observed in the top layers and diminishing effects in lower layers.

Compared to the neural network regression problem, there are additionally three major technical challenges arising in our analysis. The first concerns the sampling mechanism of the algorithm: the state-action pairs used for training are generated by a non-homogeneous Markov chain. This inhomogeneity stems from the fact that the actions are sampled from a distribution determined by the actor network, whose parameters evolve over time. The averaging of slowly varying functionals over such non-stationary chains is addressed in Section \ref{mc_sec_app} of the Appendix.

The second challenge relates to the characterization of the pre-limit stationary distribution.   By fixing the actor parameters at a specific iteration $k$, the sampling process yields a  Markov chain with transition matrix $\mathds{P}_k^N$ and stationary distribution $\pi_k^N$. The error terms in our expansion depend in part on deviations of $\pi_k^N$ from its limiting counterpart. In particular, a major challenge lies in the identification of the fluctuations behavior (i.e., the behavior beyond the first order limit) of the pre-limit stationary distribution of objects like $\pi_k^N$. We identify a novel way that yields a unique representation of the fluctuations error terms leading to a well defined limit. 

The third challenge is that we are tasked to find the limiting and fluctuations behavior and do error analysis of both the actor and the critic neural networks accounting for the interactions. 

The rest of the paper is organized as follows.  In Section \ref{assumptions_sec} we precisely state the assumptions and the algorithmic framework for the actor-critic reinforcement algorithm studied in this paper. In Section \ref{main_results_sec} we present our main results, Theorems \ref{main_result_theorem} and \ref{large_t_conv_thm} and the supporting lemmas. Section \ref{S:NumericalResults} includes our numerical studies. 
The rest of the sections in the paper include the mathematical proofs of the main results. In particular, in Section \ref{leading_order_conv_chap} we prove the leading order convergence of actor and critic to $P^{(0)}_{t}$ and $Q^{(0)}_{t}$ respectively while Section \ref{asymptotics_sec} establishes the limit of  $P^{(0)}_{t}$ and $Q^{(0)}_{t}$ as $t\rightarrow\infty$. Section \ref{first_order_error_sec}
contains the proofs of the first order correction terms $P_t^{1,N}$ and $Q_t^{1,N}$ defined as
    \begin{equation*}
        P_t^{1,N} = N^\phi \left( P_t^N-P_t^{(0)} \right), \quad Q_t^{1,N} = N^\phi \left( Q_t^N - Q_t^{(0)}\right), 
    \end{equation*}
where $\phi$ is the appropriate normalization.

In the appendix we prove technical lemmas used throughout the paper (Appendices \ref{prelim_sec_app} and \ref{first_order_proofs_app}) as well as completing via an induction argument the rest of the asymptotic expansion beyond the first order correction term (Appendix \ref{higher_order_error_sec}). In Appendix \ref{mc_sec_app} we study the fluctuations of slowly varying functionals, averaged over an inhomogeneous but slowly varying Markov chain, around their mean. These generic results are of independent interests and are used in the paper to control several martingale terms that appear in derivation of the limit of the different error terms in the asymptotic expansion.

\section{The Actor-Critic Neural Network Model and Assumptions} \label{assumptions_sec}

Throughout this article, we will study and compare the convergence behavior of the Actor-Critic Algorithm, \cite{ac_orig}, using neural networks for the Actor and Critic models under different scaling schemes. We rigorously establish control of fluctuation corrections and characterization of the statistical error to leading order $N$.  The Markov Decision Process we study consists of the following: 
\begin{itemize}
    \item A state space $\mathcal{X}$ and an action space $\mathcal{A}$.
    \item A reward function $r: \mathcal{X}\times \mathcal{A} \to \mathbb{R}$ that assigns each state-action pair $(x,a)\in \mathcal{X}\times \mathcal{A}$ to a real valued reward.
    \item Transition probabilities $p(x'|x,a)$ that provide a distribution over the next state $x' \in \mathcal{X}$ for every current state-action pair $(x,a)\in \mathcal{X}\times \mathcal{A}$.
    \item An initial distribution $\rho_0$ over the state-action pairs $(x,a)\in \mathcal{X}\times \mathcal{A}$.
    \item A discount factor $\gamma \in (0,1)$.
\end{itemize}

We will tackle the problem of finding an optimal policy $\pi(a|x)$, which consists of a distribution over the actions $a\in \mathcal{A}$ for each current state $x\in \mathcal{X}$. 
We start by introducing some necessary assumptions on the Markov Decision Process defined above.
\begin{assumption} \label{MDP_assumptions}
    The following hold true for our underlying Markov Decision Process 
    \begin{enumerate}
        \item The state and action spaces $\mathcal{X}$ and $\mathcal{A}$ respectively have finite size and satisfy $\mathcal{X}\subset \mathbb{R}^{d_x}$, $\mathcal{A}\subset \mathbb{R}^{d_a}$, where $d_x$ and $d_a$ are positive integers. We also define $d'=d_x+d_a$.
        \item The reward function $r$ is bounded in $[-1,1]$.
    \end{enumerate}
\end{assumption}

Assuming a fixed policy $f$, we define the Markov Chains
\begin{equation} \label{Markov_Chains_def_eq}
    \begin{aligned}
        (\mathcal{M}, f) &: (x_0,a_0) \sim\rho_0 \xrightarrow{{p}(\cdot |{x_0}, {a_0})} {x_1} \xrightarrow{\pi(\cdot |x_1)}{a_1} \xrightarrow{{p}(\cdot |{x_1}, {a_1})} {x_2} \xrightarrow{\pi(\cdot |x_2)}{a_2} \rightarrow ...\\
        (\mathcal{M}, f)_{\text{aux}} &: (\tilde{x}_0,\tilde{a}_0) \sim\rho_0 \xrightarrow{\tilde{p}(\cdot |{\tilde{x}_0}, {\tilde{a}_0})} {\tilde{x}_1} \xrightarrow{\pi(\cdot |\tilde{x}_1)}{\tilde{a}_1} \xrightarrow{\tilde{p}(\cdot |{\tilde{x}_1}, {\tilde{a}_1})} {\tilde{x}_2} \xrightarrow{\pi(\cdot |\tilde{x}_2)}{\tilde{a}_2} \rightarrow ...
    \end{aligned}
\end{equation}

The difference between those chains lies in the transition probabilities $p(x'|x,a)$ for the first case, where in the second case we define  $\tilde{p}(x'|x,a) =\gamma p(x'|x,a)+(1-\gamma)\rho_0(x')$  for all $(x',a,x)\in \mathcal{X}\times \mathcal{A}\times\mathcal{X}$, which ensures that every state is chosen with positive probability at each step whenever $\rho_0$ does so. With a slight abuse of notation, $\rho_0(x)$ denotes the marginal distribution of $x$ under $\rho_0$. We now introduce an ergodicity assumption on those chains.

Both chains $(\mathcal{M}, f)$ and $(\mathcal{M}, f)_{\text{aux}}$ are ergodic whenever the policy $f$ chooses every action $a\in \mathcal{A}$ with positive probability for every state $x\in \mathcal{X}$. As a result, both have a unique stationary distribution (see section 1.3.3 of \cite{lawler2018introduction}).
\begin{assumption}\label{Markov_Chain_ergodicity_assumption}
       Let $\pi^f$ and $\sigma_{\rho_0}^f$ be the (assumed) unique stationary measures of the chains $(\mathcal{M}, f)$ and $(\mathcal{M}, f)_{\text{aux}}$ respectively. We shall assume that they both have well defined densities and, for convenience of notation, if no confusion arises, we will refer to them as measures and densities interchangeably. Assume that $\pi^f$ and $\sigma_{\rho_0}^f$ are Lipschitz continuous with respect to the Total Variation distance, i.e there exists a constant $L>0$ such that for any two policies $f$ and $f'$, the following holds
\begin{equation}\label{measure_lipschitz_assumption_eq}
        \begin{aligned}
            \max\left\{d_{TV}(\pi^f, \pi^{f'}), d_{TV}(\sigma_{\rho_0}^f, \sigma_{\rho_0}^{f'})\right\} \leq Ld_{TV}(f,f'),
        \end{aligned}
    \end{equation}
    where for two distributions $p_1$, $p_2$ on the same domain $D$, the total variation distance is defined as 
\begin{equation}\label{TV_dist_def_eq}
        d_{TV}(p_1,p_2) = \frac{1}{2} \sum_{x \in D}\left|p_1(x)-p_2(x) \right|.
    \end{equation}
\end{assumption}

Two important quantities for the Actor-Critic algorithm, which depend on the policy $\pi$, are the state- and state-action value functions. We will denote both as $V$, but the first will take a state $s\in \mathcal{S}$ as input, while the latter is defined over state-action pairs $(x,a)\in \mathcal{X}\times \mathcal{A}$. 
Specifically, for a given policy $f$, the state value function $V^f:\mathcal{X}\to \mathbb{R}$ is given by
\begin{equation} \label{state_value_function_def_eq}
    V^f (x) = \mathds{E}\left[ \sum_{k=0}^\infty \gamma^k r(x_k,a_k) \big| x_0=x \right]],
\end{equation}
while the state-action value function $V^f:\mathcal{X}\times \mathcal{A} \to \mathbb{R}$ is given by
\begin{equation} \label{state_action_value_function_def_eq}
    V^f (x,a) = \mathds{E}\left[ \sum_{k=0}^\infty \gamma^k r(x_k,a_k) \big| x_0=x, a_0=a \right].
\end{equation}

In both cases $(x_k,a_k)_k$ are a sample trajectory of $(\mathcal{M}, f)$, which justifies the dependency on the policy $f$. Intuitively, those two quantities are the discounted sums of future rewards when starting at state $x$, or, respectively, state $x$ and action $a$, and following the policy $f$.

We now introduce the actor and critic models, which will be used to obtain an approximation for the optimal policy, which maximizes the discounted sum of future rewards $J(f)$, given by
\begin{equation}\label{total_reward_policy_eq}
    J(f) = \mathds{E}\left[ \sum_{k=0}^\infty \gamma^k r(x_k,a_k) \right] =\sum_{x\in \mathcal{X}}\rho_0(x)V^f(x)= \sum_{(x,a)\in \mathcal{X}\times \mathcal{A}}\rho_0(x,a)V^f(x,a),
\end{equation}
and the unknown state-action value function of the optimal policy, respectively. We choose a shallow neural network for each of those models. Specifically

\textbullet \hspace{5pt} The \textit{actor model}, acting as an approximation of an optimal policy, defined as
\begin{equation} \label{actor_model_eq}
    f_{\theta}^{N}(x,a) = \text{Softmax}(P_{\theta}^{N}(x,a)) = \frac{\exp{(P_{\theta}^{N}(x,a))}}{\sum_{a' \in \mathcal{A}} \exp{(P_{\theta}^{N}(x,a'))}}, \hspace{10pt} 
\end{equation}
where $P_{\theta}^{N}(x,a)$ is the \textit{actor network}:
\begin{equation} \label{actor_network_eq}
    P_{\theta}^{N}(x,a) = \frac{1}{N^{\beta}} \sum_{i=1}^N B^i \sigma(U^i\cdot (x,a))
\end{equation}
parameterized by the parameters $\theta = (B^1,...,B^N, U^1,...,U^N)$ where $B^i \in \mathbb{R}$ and $U^i \in \mathbb{R}^{d'}$.

\textbullet \hspace{5pt} The \textit{critic model}, acting as an approximation of the unknown state-action value function for the optimal policy (approximated by the actor model), is the critic network
\begin{equation} \label{critic_network_eq}
    Q_{\omega}^{N}(x,a) = \frac{1}{N^\beta} \sum_{i=1}^N C^i\sigma(W^i \cdot (x,a)),
\end{equation}
parameterized by the parameters $\omega = (C^1,...,C^N, W^1,...,W^N)$, where $C^i \in \mathbb{R}$ and $W^i \in \mathbb{R}^{d'}$. In both \eqref{actor_network_eq} and \eqref{critic_network_eq}, $\beta \in \left(\frac{1}{2},1\right)$ acts as a scaling parameter.

\begin{assumption} \label{actor_critic_models_assumption} At this point, we make the following assumptions on the actor and critic models: 
    \begin{enumerate}
        \item At  initialization, $(C_0^i, W_0^i) \sim v_0^{(0)}$ and $(B_0^i, U_0^i) \sim \mu_0^{(0)}$ for all $i \in [N]=\{1,2,...,N\}$, where $v_0$ and $\mu_0$ have zero mean, i.e. $\mathds{E}[B_0^i] = \mathds{E}[C_0^i] =0,  \forall i \in [N]$,  and have bounded support, so that $|C_0^i|, |B_0^i| \leq C<\infty$ and $\mathbb{E}\|W_0^i\|, \mathbb{E}\|U_0^i\| \leq C<\infty$ for some constant $C$ and for all $i\in [N]$. The random variables $C_0^i, W_0^i, B_0^i, U_0^i$ for $i \in [N]$ are assumed to be independent from each other.
        \item The scalar function $\sigma(\cdot) : \mathds{R} \rightarrow \mathds{R}$, known as the \textit{activation function}, is assumed to be bounded, three times continuously differentiable (i.e in $C_b^3(\mathds{R})$) and with bounded derivatives. For convenience and without loss of generality, we will assume $\sigma$ and its first three derivatives to be bounded by $1$, and slowly increasing, such that for any $a>0$,
        \begin{equation*}
        \lim_{x \rightarrow \pm \infty} \frac{\sigma(x)}{x^a} = 0.
        \end{equation*}
    \end{enumerate} 
\end{assumption}

An example satisfying Assumption \ref{actor_critic_models_assumption} would be the standard sigmoid function $\sigma(x) = (1+e^{-x})^{-1}$.

The vectors $\theta_k = (B_k^1,...,B_k^N,U_k^1,...,U_k^N)$ and $\omega_k = (C_k^1,...,C_k^N,W_k^1,...,W_k^N)$ represent the trained parameters of the actor and critic network after $k$ training updates. We also define $P_k^{N} := P_{\theta_k}^{N}, f_k^{N} := \text{Softmax}(P_k^{N})$, and $Q_k^{N} :=Q_{\omega_k}^{N}$.

The critic network's parameters are updated via temporal difference learning \cite{qlearning}, in which the loss function is defined as 
\begin{equation}
    L^{\theta_k}(\omega_k) := \sum_{(x,a) \in \mathcal{X}\times\mathcal{A}} \left[ Y_k(x,a) - Q_k^N(x,a) \right]^2\pi^{f_k^N}(x,a),
\end{equation}
where $\pi^{f_k^N}(x,a)$ is the stationary distribution of the chain $(\mathcal{M}, f_k^N)$ and the target $Y_k$ is defined as 
\begin{equation}
    Y_k(x,a) = r(x,a) +\gamma \sum_{(x',a')\in \mathcal{X}\times \mathcal{A}} Q_k^N(x',a')f_k^N(x',a')p(x'|x,a).
\end{equation}

As computing the stationary distribution $\pi^{f_k^N}(x,a)$ is expensive,  we will estimate $L^{\theta_k}(\omega_k)$ via
\begin{equation}
    l^{\theta_k}(\omega_k) := \left[ \hat{Y}_k(x_k,a_k) - Q_k^N(x_k,a_k) \right]^2,
\end{equation}
where an estimate of $Y_k(x_k,a_k)$ is
\begin{equation}
    \hat{Y_k}(x_k,a_k) := r(x_k,a_k) + \gamma Q_k^N(x_{k+1},a_{k+1}).
\end{equation}

The actor model is updated using the discounted sum of future rewards $J(f_\theta)$ as the objective function to maximize. The policy gradient theorem \cite{pgthm}, states that if a policy $f_{\theta}$ is parameterized by $\theta$ then 
\begin{equation} \label{critic_loss_eq}
    \nabla_\theta V^{f_\theta}(x) = \sum_{x\in \mathcal{X}} \left( \sum_{k\geq 0}\mathbb{P}[x_k=x|x_0]\right) \sum_{a\in \mathcal{A}} \nabla_\theta f_\theta(x,a)V^{f_\theta}(x,a),
\end{equation}
which as shown in \cite{nac} yields
\begin{equation}
    \nabla_\theta J(\theta) = \sum_{(x,a)\in \mathcal{X}\times \mathcal{A}} \sigma_{\rho_0}^{f_\theta}(x,a) \nabla_\theta (\log{f_\theta(x,a)})V^{f_\theta}(x,a).
\end{equation}

The latter is estimated by $\nabla_\theta (\log{f_\theta(\tilde{x}_k,\tilde{a}_k)})V^{f_\theta}(\tilde{x}_k,\tilde{a}_k)$ because computing the stationary distribution $\sigma_{\rho_0}^{f_\theta}$ of $(\mathcal{M}, f_k^N)_{\text{aux}}$ is expensive, see also \cite{nac}.

Training the Critic network using temporal difference learning, and the Actor network using the policy gradient theorem gives the following stochastic gradient descent parameter update equations:  
    \begin{align}
        &C_{k+1}^i = C_k^i + \frac{\alpha^N}{N^{\beta}}\left(r(x_k, a_k)+\gamma Q_k^{N}(x_{k+1}, a_{k+1})- Q_k^{N}(x_k, a_k)\right)\sigma(W_k^i\cdot (x_k, a_k))\nonumber\\
        &W_{k+1}^i = W_k^i + \frac{\alpha^N}{N^{\beta}}\left(r(x_k,a_k)+\gamma Q_k^{N}(x_{k+1},a_{k+1})- Q_k^{N}(x_k,a_k)\right)C_k^i\sigma'(W_k^i\cdot (x_k,a_k)\cdot (x_k,a_k)\nonumber\\
        &B_{k+1}^i = B_k^i + \frac{{\zeta_k^N}}{N^{\beta}} Q_k^{N}(\tilde{x}_k, \tilde{a}_k) \left( \sigma(U_k^i \cdot (\tilde{x}_k, \tilde{a}_k) - \sum_{a' \in \mathcal{A}} f_k^{N}(\tilde{x}_k,a')\sigma(U_k^i \cdot (\tilde{x}_k, a')) \right)\label{param_update}\\
        &U_{k+1}^i = U_k^i + \frac{{\zeta_k^N}}{N^{\beta}} Q_k^{N}(\tilde{x}_k, \tilde{a}_k) \left(B_k^i \sigma'(U_k^i \cdot (\tilde{x}_k, \tilde{a}_k))\cdot (\tilde{x}_k, \tilde{a}_k) - \sum_{a' \in \mathcal{A}} f_k^{N}(\tilde{x}_k,a')B_k^i\sigma'(U_k^i \cdot (\tilde{x}_k, a')) \cdot (\tilde{x}_k, a') \right),\nonumber
    \end{align}
where $\alpha^N$, ${\zeta_k^N}$ are learning rates. 

The tuples $(x_k,a_k), (\tilde{x}_k, \tilde{a}_k)$ are sample paths from the critic and actor processes respectively, where\\
\textbullet \hspace{5pt} The "critic" process is defined as 
\begin{equation} \label{critic_process_def_eq}
    (\mathcal{M}, \text{Cr}) : ({x_0}, {a_0}) \sim \rho_0 \xrightarrow{{p}(\cdot |{x_0}, {a_0})} {x_1} \xrightarrow{g_0^N({x_1},\cdot)} {a_1} \xrightarrow{{p}(\cdot |{x_1}, {a_1})} {x_2} \xrightarrow{g_1^N({x_2},\cdot)} {a_2},...
\end{equation}
\textbullet \hspace{5pt} The "actor" process is defined as 
\begin{equation}\label{actor_process_def_eq}
    (\mathcal{M}, \text{Ac}) : (\tilde{x_0}, \tilde{a_0}) \sim \rho_0 \xrightarrow{\tilde{p}(\cdot |\tilde{x_0}, \tilde{a_0})} \tilde{x_1} \xrightarrow{g_0^N(\tilde{x_1},\cdot)} \tilde{a_1} \xrightarrow{\tilde{p}(\cdot |\tilde{x_1}, \tilde{a_1})} \tilde{x_2} \xrightarrow{g_1^N(\tilde{x_2},\cdot)} \tilde{a_2},...
\end{equation}

Notice that each action is not sampled directly from the actor model $f_k^N$ at each step. Instead, an exploration policy is used, $g_k^N$, which for every $(x,a) \in \mathcal{X}\times \mathcal{A}$ is defined as 
\begin{align} \label{g_k^N_eq}
    g_k^{N}(x,a) &= \frac{\eta_k^N}{|\mathcal{A}|}+(1-\eta_k^N)f_k^{N}(x,a),
\end{align}
where $(\eta_k^N)_{k\geq 0}$ is a sequence of exploration rates such that $0< \eta_k^N\leq 1$ and $\eta_k^N \xrightarrow{k\rightarrow \infty}0$. This ensures that each action in $\mathcal{A}$ is selected with probability at least $\eta_k^N/|\mathcal{A}| >0$. By assumption \ref{Markov_Chain_ergodicity_assumption}, this implies that for every $k$, the Markov chains $(\mathcal{M}, g_k^N)$ and $(\mathcal{M}, g_k^N)_{\text{aux}}$, are both ergodic, and their stationary distributions $\pi^{g_k^{N}}$ and $\sigma_{\rho_0}^{g_k^{N}}$ are well defined (exist and are unique). Notice that in contrast to those chains, where the exploration policy is constant throughout each step, the chains $(\mathcal{M}, \text{Ac})$ and $(\mathcal{M}, \text{Cr})$ are inhomogeneous and do not have a stationary distribution. However, the chains $(\mathcal{M}, g_k^N)$ and $(\mathcal{M}, g_k^N)_{\text{aux}}$ as well as their stationary distributions will be helpful in the analysis of the actor and critic processes $(\mathcal{M}, \text{Ac})$ and $(\mathcal{M}, \text{Cr})$.

As we shall see throughout this article, for a fixed value of the scaling hyperparameter $\beta \in \left(\frac{1}{2},1\right)$, the learning and exploration rates $\alpha_k^N$, $\zeta_k^N$, $\eta_k^N$ and the limit points $N^{2-2\beta}\zeta_{\lfloor Nt \rfloor}^N \rightarrow \zeta_t$ and $\eta_{\lfloor Nt \rfloor}^N \rightarrow \eta_t$ as $N\to \infty$ need to satisfy the following conditions
\begin{equation} \label{hyperparamneter_cond_eq}
    \begin{aligned}
        \alpha_k^N, \zeta_k^N = \Theta(N^{2\beta - 2}), \quad \int_0^\infty \zeta_t dt = \infty, \quad \int_0^\infty \zeta_t^2dt < \infty, \quad \int_0^\infty \zeta_t \eta_t <\infty,\\ \lim_{t\to \infty}\frac{\zeta_t}{\eta_t} =0, \quad
         \eta_t = \Omega(1/t^a) \quad \forall a>0, \quad (\zeta_k^N)_k, (\eta_k^N)_k \text{ decreasing}.
    \end{aligned}
\end{equation}

Here $ \zeta_t$ and $ \eta_t$ are the time rescaled Actor network learning rate and exploration policy. We recall that $\eta_t = \Omega(1/t^a)$ means that there are  $C,t_{0}<\infty$  so that $C\frac{1}{t^{\alpha}}\leq \eta_t$ for $t\geq t_{0}$.

To be specific, a choice of hyperparameters satisfying the conditions above is given by
\begin{equation}\label{learning_rates_def_eq}
    \begin{aligned}
        \alpha^N_{k} = \frac{\alpha}{N^{2-2\beta}}, \quad \zeta_k^N = \frac{1}{N^{2-2\beta}\left(1+\frac{k}{N} \right)}, \quad \eta_k^N = \frac{1}{1+\log^2\left( 1+\frac{k}{N}\right)},
    \end{aligned}
\end{equation}
where $0<\alpha<\infty$ is a positive constant. We will make use of these choices in this paper. We also define the empirical measures 
\begin{equation} \label{empirical_measures_eq}
    \mu_k^{N} = \frac{1}{N} \sum_{i=1}^{N} \delta_{B_k^i, U_k^i}, \hspace{10pt} v_k^{N} = \frac{1}{N} \sum_{i=1}^{N} \delta_{C_k^i, W_k^i}.
\end{equation}

In addition, we define the following time-rescaled processes for any $ (x,a) \in \mathcal{X} \times \mathcal{A}$
    \begin{align}
        P_t^{N}(x,a) &= P_{\lfloor Nt\rfloor}^{N}(x,a), \hspace{10pt} f_t^{N}(x,a) = f_{\lfloor Nt\rfloor}^{N}(x,a), \hspace{10pt} g_{t}^{N}(x,a) = g_{\lfloor Nt\rfloor}^{N}(x,a)\label{discrete_to_cont_eq}\\
        Q_{t}^{N}(x,a) &= Q_{\lfloor Nt\rfloor}^{N}(x,a), \hspace{10pt} \mu_{t}^{N} = \mu_{\lfloor Nt\rfloor}^{N}, \hspace{10pt} v_{t}^{N} = v_{\lfloor Nt\rfloor}^{N}, \hspace{10pt} \zeta_t^N = \zeta_{\lfloor Nt \rfloor}^N, \hspace{10pt}\eta_t^N = \eta_{\lfloor Nt \rfloor}^N\nonumber
    \end{align}

Our goal is to derive an asymptotic expansion in the limit as $N\rightarrow\infty$ for the outputs of the actor and critic networks, for different values of the scaling parameter $\beta \in \left(\frac{1}{2},1 \right)$, that will have the form
    \begin{align}
    P_t^N(x,a) &\approx P_t^{(0)}(x,a) +N^{-\phi_1}P_t^{(1)}(x,a)+N^{-\phi_2}P_t^{(2)}(x,a)+...+N^{-\phi_n}P_t^{(n)}(x,a), \label{expansion_form_eq}\\
        Q_t^N(x,a) &\approx Q_t^{(0)}(x,a) +N^{-\phi_1}Q_t^{(1)}(x,a)+N^{-\phi_2}Q_t^{(2)}(x,a)+...+N^{-\phi_n}Q_t^{(n)}(x,a),\nonumber
    \end{align}
for appropriate, to be specified, exponents $\phi_{j}>0$. Here, $P_t^{(j)}(x,a)$ and $Q_t^{(j)}(x,a)$ are non-random for $j<n$ and random when $j=n$. An expansion like this will provide us useful and fine-grain information on the structure and convergence behavior of the networks. In particular, it is providing us with the explicit form of the bias and the variance of those estimators to leading order in $N$. Moreover, it allows us to rigorously optimize with respect to the scaling parameter $\beta$, at least in the sense of addressing the bias-variance trade-off and minimizing the variance of the neural network estimators.

\section{Main Results} \label{main_results_sec}
We now present the main result of this article. Note that $\bigcup_{1\leq n\in \mathds{N}}\left( \frac{2n-1}{2n}, \frac{2n+1}{2n+2} \right) = \left( \frac{1}{2},1\right)$, so that theorem \ref{main_result_theorem} provides an expansion of the neural network outputs for each given $\beta \in \left(\frac{1}{2},1\right)$.
\begin{theorem} \label{main_result_theorem} Let the assumptions in section \ref{assumptions_sec} hold. 
    Let $n\in \mathbb{N}$ and let $\beta \in \left( \frac{2n-1}{2n}, \frac{2n+1}{2n+2} \right)$. Then the following expansions hold for the neural network outputs $P_t^N$ and $Q_t^N$, as $N\rightarrow\infty$
        \begin{align}
           P_t^N &\approx P_t^{(0)} + N^{\beta-1}P_t^{(1)} + N^{2\beta-2}P_t^{(2)}+...+N^{(n-1)(\beta-1)}P_t^{(n-1)}+N^{\frac{1}{2}-\beta}P_t^{(n)},\nonumber\\
            Q_t^N &\approx Q_t^{(0)} + N^{\beta-1}Q_t^{(1)} + N^{2\beta-2}Q_t^{(2)}+...+N^{(n-1)(\beta-1)}Q_t^{(n-1)}+N^{\frac{1}{2}-\beta}Q_t^{(n)},\nonumber
        \end{align}   
    where the functions $P_t^{(i)}$ and $Q_t^{(i)}$ are solutions of ODEs with zero initial condition for $i<n$, see \eqref{Q_P_higher_order_error_terms_limits_eq}-\eqref{Q_P_higher_order_error_terms_limits_eqInitialCondition0}, and $\hat{P}_t^{(n)}$ and $\hat{Q}_t^{(n)}$ are solutions to ODEs with random initial conditions and are given by \eqref{Q_P_higher_order_error_terms_limits_eq}-\eqref{Eq:RandomInitialConditions} if $\beta=\frac{2n+1}{2n+2}$ and by \eqref{Q_P_higher_order_error_terms_limits_eq_lastTerm}-\eqref{Eq:RandomInitialConditions} if $\beta<\frac{2n+1}{2n+2}$.
\end{theorem}

\begin{remark}
The expansion above allows us to make conclusions about the bias and the variance of the neural network outputs. Specifically, in theorem \ref{large_t_conv_thm} we show that $P_t^{(0)}$ and $Q_t^{(0)}$ converge to a solution of the Bellman Equation and the corresponding policy as $t\to \infty$, which is the desired limit behavior. The non-random terms $P_t^{(j)}(x,a)$ and $Q_t^{(j)}(x,a)$ for $j<m$ can be thought of as bias corrections. The last terms $P_t^{(n)} $ and $Q_t^{(n)}$ capture the randomness that is due to the random initialization and can be thought of as the  variance correction, quantifying the uncertainty to leading order in $N$ uncertainty. The expansion tells us how each bias and variance term scales with respect to $N$ for given values of $\beta$. In particular, an immediate conclusion of the asymptotic expansion is that to leading order in $N$ the smallest variance is achieved when $\beta=1$.
\end{remark}

We will break down the proof of Theorem \ref{main_result_theorem} into smaller pieces. We start by studying the leading order convergence of the actor and critic model outputs. We define the limit candidates to be $Q_t^{(0)}$ and $P_t^{(0)}$, which are continuous on $[0,T]$ and satisfy 
\begin{equation} \label{Q_t_P_t_def_eq}
    \begin{aligned}
            Q_t^{(0)}(x,a) &= \alpha \int_0^t \sum_{(x',a',x'',a'') \in \mathcal{X}\times \mathcal{A}\times \mathcal{X}\times \mathcal{A}} \Bigg( r(x',a') + \gamma Q_s^{(0)}(x'',a'')- Q_s^{(0)}(x',a') \Bigg)\\ 
            &\hspace{30pt}\langle B_{(x,a),(x',a')}, v_0^{(0)} \rangle \pi^{g_s^{(0)}}(x',a')g_s^{(0)}(x'',a'')p(x''|x',a') ds\\
            P_t^{(0)}(x,a) &= \int_0^t \sum_{({x}',{a}', a'') \in \mathcal{X}\times\mathcal{A}\times \mathcal{A}}\zeta_s Q_s^{(0)}({x}',{a}') \left( \langle B_{({x},{a}),({x}',{a}')}, \mu_0^{(0)} \rangle - \langle B_{({x},{a}), ({x}', {a}'')},\mu_0^{(0)} \rangle \right)\\ 
            &\hspace{30pt} f_s^{(0)}({x}', {a}'')\sigma_{\rho_0}^{g_s^{(0)}}({x}',{a}')ds\\
            P_0^{(0)}({x},{a}) &= 0, \hspace{10pt} Q_0^{(0)}(x,a) = 0
        \end{aligned}
\end{equation}
where
\begin{equation} \label{B_def_eq}
        B_{(x,a), (x'.a')} = \sigma(w\cdot (x,a))\sigma(w\cdot (x',a')) + c^2\sigma'(w\cdot (x,a))\sigma'(w\cdot (x',a')) (x,a) \cdot (x',a').
\end{equation}

This also involves the limit candidates for the actor model, the policy, the actor network learning rate, and the policy exploration rate, which we define as:
\begin{equation}\label{actor_policy_learning_rates_limit_eq}
    \begin{aligned}
        f_t^{(0)}(x,a) &= \text{Softmax}(P_t^{(0)}(x,a)), \quad g_t^{(0)}(x,a) = \frac{\eta_t}{|\mathcal{A}|}+(1-\eta_t)f_t^{(0)}(x,a)\\
        \zeta_t &= \lim_{N\to \infty}\zeta_t^N = \frac{1}{1+t}, \quad \eta_t = \lim_{N \to \infty}\eta_t^N = \frac{1}{1+\log^2{\left( 1+t\right)}}.
    \end{aligned}
\end{equation}

In the large time limit, the following behavior of the neural network outputs is shown below.
\begin{theorem} \label{large_t_conv_thm}
    The limit critic network $Q_t^{(0)}$ and limit actor network $P_t^{(0)}$ converge to the solution of the Bellman Equation $V^{f_t^{(0)}}$, defined over the state-action pairs $(x,a)\in \mathcal{X}\times \mathcal{A}$, and to a stationary point respectively as $t\to \infty$, in the sense that the following holds
    \begin{equation*}
        \begin{aligned}
            \| Q_t^{(0)}-V^{f_t^{(0)}}\| &= O(\eta_t)\\
            \left\| \frac{\partial J(f_t^{(0)})}{\partial P^{(0)}_{t}(x,a)} \right\| &\to 0
        \end{aligned}
    \end{equation*}
\end{theorem}

Notice that this result allows us to control the speed of convergence of the limit critic network $Q_t^{(0)}$ through the choice of the exploration rate schedule $\eta_k^N$. The proof of Theorem \ref{large_t_conv_thm} is in Section \ref{asymptotics_sec}.

\subsection{Leading order convergence: the zeroth order term} \label{leading_result_sec}
In this subsection we summarize the leading order convergence results. But before doing so, we need to define convergence in $L_p$ for the stochastic processes that we will study.
\begin{definition} \label{L_p_conv_def}
    We will say that a stochastic process $X_t^N : E\to \mathbb{R}$, with $E$ a finite size domain and  $t\in[0,T]$ with $T$ finite, converges in $L_p$ to $X_t : E\to \mathbb{R}$ if for every $\epsilon>0$ and $p\in \mathbb{N}$ there is a finite constant $N_{T,p}$ independent of $N$ such that for $N>N_{T,p}$ the following bound holds
    \begin{equation*}
        \sup_{t\in [0,T]} \mathds{E}\left[\max_{e\in E}\left|X_t^N(e)-X_t(e) \right|^p \right] \leq \epsilon.
    \end{equation*}
\end{definition}
\begin{proposition} \label{leading_order_conv_prop}
    Let $T<\infty$ and $p\in \mathbb{N}$. Let $\delta_t^N$ be any of the processes $f_t^N$, $g_t^N$, $\pi^{g_t^N}$, and $\sigma_{\rho_0}^{g_t^N}$. Let $\delta_t^{(0)}$ be their corresponding limit candidate $f_t^{(0)}$, $g_t^{(0)}$, $\pi^{g_t^{(0)}}$, or $\sigma_{\rho_0}^{g_t^{(0)}}$ respectively. There are constants $C_{T,p}$ and $N_{T,p}$ independent of $N$  such that $\delta_t^N$ converges in $L_p$ to $\delta_t$, and the following convergence rate holds for $N>N_{T,p}$
    \begin{equation*}
    \begin{aligned}
        \sup_{t \in [0,T]}\mathds{E}\left[\max_{(x,a) \in \mathcal{X}\times \mathcal{A}}|\delta_t^N(x,a)-\delta_t^{(0)}(x,a)|^p \right] &\leq C_{T,p} \left[N^{p(\beta-1)}+N^{p\left(\frac{1}{2}-\beta\right)}\right].
    \end{aligned}
    \end{equation*}
\end{proposition}

The proof of proposition \ref{leading_order_conv_prop} is in Section \ref{SS:MeasureProcessConv}.

\begin{proposition} \label{empirical_measures_conv_prop}
    Let $T<\infty$ and $p\in \mathbb{N}$. There are constants $C_{T,p}$ and $N_{T,p}$ independent of $N$ such that the empirical measures $v_t^N$ and $\mu_t^N$ converge in $L_p$ to $v_0^{(0)}$ and $\mu_0^{(0)}$ respectively, in the sense that for any fixed $h \in C_b^2(\mathds{R})$ the following convergence rates hold $N>N_{T,p}$
    \begin{equation*}
        \begin{aligned}
            \sup_{t\in [0,T]}\mathds{E}\left[ |\langle h, v_t^N \rangle - \langle h, v_0^{(0)}\rangle |^p\right] +\sup_{t\in [0,T]}\mathds{E}\left[ |\langle h, \mu_t^N \rangle - \langle h, \mu_0^{(0)}\rangle |^p\right]&\leq C_{T,p} N^{p(\beta-1)}.
        \end{aligned}
    \end{equation*}
\end{proposition} 

\begin{theorem} \label{leading_order_conv_th}
     Let $T<\infty$ and $p\in \mathbb{N}$. There are constants $C_{T,p}$ and $N_{T,p}$ independent of $N$ such that the actor and critic network outputs $P_t^N$ and $Q_t^N$ converge in $L_p$ to $P_t^{(0)}$ and $Q_t^{(0)}$ respectively, and the following convergence rates hold for $N>N_{T,p}$
    \begin{equation*}
    \begin{aligned}
    \sup_{t \in [0,T]}\mathds{E}\left[\max_{(x,a) \in \mathcal{X}\times \mathcal{A}}|P_t^N(x,a)-P_t^{(0)}(x,a)|^p \right] &\leq C_{T,p} \left[N^{p(\beta-1)}+N^{p\left(\frac{1}{2}-\beta\right)}\right],\\
        \sup_{t \in [0,T]}\mathds{E}\left[\max_{(x,a) \in \mathcal{X}\times \mathcal{A}}|Q_t^N(x,a)-Q_t^{(0)}(x,a)|^p \right] &\leq C_{T,p} \left[N^{p(\beta-1)}+N^{p\left(\frac{1}{2}-\beta\right)}\right].
    \end{aligned}
    \end{equation*}
\end{theorem}

The proofs of proposition \ref{empirical_measures_conv_prop} and theorem \ref{leading_order_conv_th} are in section \ref{leading_order_conv_sec}.

\subsection{Next order convergence: first order fluctuations}
The result of the convergence theorem \ref{leading_order_conv_th} motivates the study of the corresponding error terms, since the randomness from the random initialization is not yet captured in the limit terms $P_t^{(0)}$ and $Q_t^{(0)}$. Specifically, we define the scaled leading order error terms as
\begin{equation} \label{first_order_network_output_error_terms}
    P_t^{1,N}(x,a) = N^{\phi}\left( P_t^{N}(x,a)-P_t^{(0)}(x,a)\right), \hspace{30pt}    Q_t^{1,N}(x,a) = N^{\phi}\left( Q_t^{N}(x,a)-Q_t^{(0)}(x,a)\right), 
\end{equation}
\noindent where $\phi = \min{\{1-\beta, \beta-\frac{1}{2}\}}$ or, equivalently, $\phi = 1-\beta$ for $\beta \in \left(\frac{3}{4},1 \right)$ and $\phi = \beta-\frac{1}{2}$ for $\beta \in \left(\frac{1}{2},\frac{3}{4}\right]$.

In order to study the convergence of those error terms, we will need to determine the leading error terms of the actor model output $f_t^N$, the policy $g_t^N$, the stationary distributions $\pi^{g_t^N}$ and $\sigma_{\rho_0}^{g_t^N}$, and the empirical measures $v_t^N$ and $\mu_t^N$. We define the fluctuation processes,
\begin{equation} \label{leading_error_terms_def_eq}
    \begin{aligned}
        f_t^{1,N}(x,a) &= N^{\phi}\left(f_t^N(x,a)-f_t^{(0)}(x,a)\right)\\
        g_t^{1,N}(x,a) &= N^{\phi}\left(g_t^N(x,a)-g_t^{(0)}(x,a)\right)\\
        {\pi}^{g_t^{1,N}}(x,a) &= N^{\phi}\left(\pi^{g_t^N}(x,a)-\pi^{g_t^{(0)}}(x,a)\right)\\
        {\sigma}_{\rho_0}^{g_t^{1,N}}(x,a) &= N^{\phi}\left(\sigma_{\rho_0}^{g_t^N}(x,a)-\sigma_{\rho_0}^{g_t^{(0)}}(x,a)\right)\\
        v_t^{1,N} &= N^{\phi}\left(v_t^N - v_t^{(0)} \right)\\
        \mu_t^{1,N} &= N^{\phi}\left(\mu_t^N - \mu_t^{(0)} \right).
    \end{aligned}
\end{equation}

In order to define the limit candidates of the above processes, we first define the time-evolving matrices $\mathds{P}_t^N$, $\Pi_t^N$ and their limits $\mathds{P}_t$ and $\Pi_t$ as the square matrices of dimension $|\mathcal{X}\times \mathcal{A}|$ with elements
 \begin{equation} \label{stationary_measure_transition_matrices_def_eq}
    \begin{aligned}
        &\mathds{P}^N_{t,(x,a),(x',a')} = g_t^N(x',a')p(x'|x,a), \quad \Pi^N_{t,(x,a),(x',a')} = g_t^N(x',a')\tilde{p}(x'|x,a)\\
        &{\mathds{P}}_{t,(x,a),(x',a')}^{(0)} = g_t^{(0)}(x',a')p(x'|x,a), \quad {\Pi}_{t,(x,a),(x',a')}^{(0)} = g_t^{(0)}(x',a')\tilde{p}(x'|x,a).
    \end{aligned}
\end{equation}

 Notice that by Proposition \ref{leading_order_conv_prop},  $\mathds{P}_t^{(0)}$ and $\Pi_t^{(0)}$ are indeed the limits of $\mathds{P}_t^N$ and $\Pi_t^N$ respectively. We also define the rank-1 matrices $W_v$ as follows.
\begin{definition} \label{W_v_def}
    Let $v$ be a vector of dimension $|v|$. We define $W_v$ to be the rank-1 matrix of dimension $|v|\times|v|$ where all the rows of which are copies of $v$.
\end{definition}

We now define the limit candidates for the first order error terms in \eqref{first_order_network_output_error_terms} and \eqref{leading_error_terms_def_eq}. We start by defining $Q_t^{(1)}$ and $P_t^{(1)}$ to be the bounded functions that satisfy
\begin{align}\label{Q_t^{(1)}_def_eq}
            Q_t^{(1)}(x,a) &= Q_0^{(1)}(x,a)+ \alpha \int_{0}^t \sum_{(x',a',x'',a'') \in \mathcal{X}\times \mathcal{A} \times \mathcal{X} \times \mathcal{A}} \left(\gamma Q_s^{(1)}(x'',a'') - Q_s^{(1)}(x',a') \right)
             \langle B_{(x,a),(x',a')}, v_0^{(0)} \rangle\nonumber\\
        &\hspace{40pt} g_s^{(0)}(x'',a'') \pi^{g_s^{(0)}}(x',a') p(x'',a''|x',a')ds\nonumber\\
            &+ \mathds{1}\left\{\beta \geq \frac{3}{4}\right\}\alpha \int_{0}^t \sum_{(x',a',x'',a'') \in \mathcal{X}\times \mathcal{A} \times \mathcal{X} \times \mathcal{A}} \left(r(x',a') + \gamma Q_s^{(0)}(x'',a'') - Q_s^{(0)}(x',a') \right)\nonumber\\
            &\hspace{40pt} \langle B_{(x,a),(x',a')}, v_s^{(1)}\rangle g_s^{(0)}(x'',a'') \pi^{g_s^{(0)}}(x',a') p(x'',a''|x',a')ds\nonumber\\
            &+ \alpha \int_{0}^t \sum_{(x',a',x'',a'') \in \mathcal{X}\times \mathcal{A} \times \mathcal{X} \times \mathcal{A}} \left(r(x',a') + \gamma Q^{(0)}_s(x'',a'') - Q^{(0)}_s(x',a') \right)
             \langle B_{(x,a),(x',a')}, v_0^{(0)} \rangle\nonumber\\
            &\hspace{40pt} g^{(1)}_s(x'',a'') \pi^{g_s^{(0)}}(x',a') p(x'',a''|x',a')ds\nonumber\\
            &+ \alpha \int_{0}^t \sum_{(x',a',x'',a'') \in \mathcal{X}\times \mathcal{A} \times \mathcal{X} \times \mathcal{A}} \left(r(x',a') + \gamma Q^{(0)}_s(x'',a'') - Q^{(0)}_s(x',a') \right)
             \langle B_{(x,a),(x',a')}, v_0^{(0)} \rangle\nonumber\\
            &\hspace{40pt}  g_s^{(0)}(x'',a'') \pi^{g_s^{(1)}}(x',a') p(x'',a''|x',a')ds,      
\end{align}
\begin{align}\label{P_t^{(1)}_def_eq}
            P_t^{(1)}(x,a) &= P_0^{(1)}(x,a)+  \int_{0}^t \sum_{(x',a',x'') \in \mathcal{X}\times \mathcal{A} \times \mathcal{X}} \zeta_sQ^{(1)}_s(x,a)
             \left(\langle B_{(x,a),(x',a')}, v_0^{(0)} \rangle - \langle B_{(x,a),(x',a'')}, v_0^{(0)}\rangle  \right)\nonumber\\
            &\hspace{40pt} f_s^{(0)}(x',a'')\sigma_{\rho_0}^{g_s^{(0)}}(x',a')ds\nonumber\\
            &+\mathds{1}\left\{\beta \geq \frac{3}{4}\right\} \int_{0}^t \sum_{(x',a',x'') \in \mathcal{X}\times \mathcal{A} \times \mathcal{X}} \zeta_sQ_s^{(0)}(x,a)
             \left(\langle B_{(x,a),(x',a')}, \mu_s^{(1)} \rangle - \langle B_{(x,a),(x',a'')}, \mu_s^{(1)} \rangle  \right)\nonumber\\
            &\hspace{40pt}  f_s^{(0)}(x',a'')\sigma_{\rho_0}^{g_s^{(0)}}(x',a')ds\nonumber\\
            &+\int_{0}^t \sum_{(x',a',x'') \in \mathcal{X}\times \mathcal{A} \times \mathcal{X}} \zeta_s Q^{(0)}_s(x,a)
             \left(\langle B_{(x,a),(x',a')}, v_0^{(0)} \rangle - \langle B_{(x,a),(x',a'')}, v_0^{(0)}\rangle  \right)\nonumber\\
            &\hspace{40pt}  f_s^{(1)}(x',a'')\sigma_{\rho_0}^{g_s^{(0)}}(x',a')ds\nonumber\\
            &+ \int_{0}^t \sum_{(x',a',x'') \in \mathcal{X}\times \mathcal{A} \times \mathcal{X}} \zeta_sQ_s^{(0)}(x,a)
             \left(\langle B_{(x,a),(x',a')}, v_0^{(0)} \rangle - \langle B_{(x,a),(x',a'')}, v_0^{(0)}\rangle  \right)\nonumber\\
            &\hspace{40pt}  f_s^{(0)}(x',a'')\sigma_{\rho_0}^{g_s^{(1)}}(x',a')ds,
        \end{align}
        and that also satisfy the following initial conditions, depending on the value of $\beta$:\\

    \textbullet \hspace{5pt} Case 1: $\beta \in \left(\frac{1}{2}, \frac{3}{4}\right]$
    \begin{equation} \label{K_t_L_t_init_eq_case1o}
        \begin{aligned}
            Q_0^{(1)}(x,a) &\sim \mathcal{G}(x,a) = \mathcal{N}\left(0, \mathds{E}_{v_0}\left[ |c\sigma(w\cdot (x,a))|^2\right]\right)\\
            P_0^{(1)}(x,a) &\sim \mathcal{H}(x,a) = \mathcal{N}\left(0, \mathds{E}_{\mu_0}\left[ |b\sigma(u\cdot (x,a))|^2\right]\right),
        \end{aligned}
    \end{equation}
    for every $(x,a) \in \mathcal{X}\times \mathcal{A}$. Note that the variances of the normal distributions defined above are finite by assumption \ref{actor_critic_models_assumption}.

    \textbullet \hspace{5pt} Case 2: $\beta \in \left(\frac{3}{4}, 1\right)$
    \begin{equation} \label{K_t_L_t_init_eq_case1}
        \begin{aligned}
            Q_0^{(1)}(x,a) = P_0^{(1)}(x,a) = 0 \hspace{15pt}\forall (x,a) \in \mathcal{X}\times \mathcal{A}
        \end{aligned}
    \end{equation}

In equations \eqref{Q_t^{(1)}_def_eq} and \eqref{P_t^{(1)}_def_eq} we used the limit candidates ${f}_t^{(1)}$, ${g}_t^{(1)}$, ${\pi}^{g_t^{(1)}}$, ${\sigma}_{\rho_0}^{g_t^{(1)}}$, $v_t^{(1)}$ and $\mu_t^{(1)}$ of ${f}_t^{1,N}$, ${g}_t^{1,N}$, ${\pi}^{g_t^{1,N}}$ , ${\sigma}_{\rho_0}^{g_t^{1,N}}$, $v_t^{1,N}$ and $\mu_t^{1,N}$ respectively, which are defined as follows
{\allowdisplaybreaks        \begin{align}\label{leading_order_error_terms_limits_eq}
            {f}_t^{(1)}(x,a) &= \sum_{a' \in \mathcal{A}} f_t^{(0)}(x,a) \left( \mathds{1}\{a=a'\} - f_t^{(0)}(x,a')\right) P_t^{(1)}(x,a')\nonumber\\
            {g}_t^{(1)}(x,a) &= (1-\eta_t){f}_t^{(1)}(x,a)\nonumber\\
            {\pi}^{g_t^{(1)}} &= -\pi^{g_t^{(0)}}\mathds{P}_t^{(1)} \left(I-\mathds{P}_t^{(0)} + W_{\pi^{g_t^{(0)}}}\right)^{-1}\nonumber\\
            {\sigma}_{\rho_0}^{g_t^{(1)}} &= -\sigma_{\rho_0}^{g_t^{(0)}}{\Pi}_t^{(1)} \left(I-\Pi_t^{(0)} + W_{\sigma_{\rho_0}^{g_t^{(0)}}}\right)^{-1}\nonumber\\
            \mathds{P}^{(1)}_{t,(x,a),(x',a')} &= g^{(1)}_t(x',a')p(x'|x,a) \\
            \Pi^{(1)}_{t,(x,a),(x',a')} &= g^{(1)}_t(x',a')\tilde{p}(x'|x,a)\nonumber\\
            \langle h, v_t^{(1)} \rangle &= \alpha \int_0^t \sum_{(x',a'),(x'',a'') \in \mathcal{X}\times\mathcal{A}} \left(r(x',a')+\gamma Q_s^{(0)}(x'',a'')- Q_s^{(0)}(x',a')\right) \nonumber\\
            &\hspace{30pt}\langle C_{(x',a')}^h, v_0^{(0)} \rangle\pi^{g_s^{(0)}}(x',a')g_{s}^0(x'',a'')p(x''|x',a')ds\nonumber\\
            \langle h, \mu_t^{(1)} \rangle &= \alpha \int_0^t \sum_{(x',a',a'') \in \mathcal{X}\times\mathcal{A}\times \mathcal{A}} \frac{1}{1+s} Q_s^{(0)}(x',a') \left(\langle C_{(x',a')}^{h}, \mu_0^{(0)} \rangle -\langle C_{(x', a'')}^{{h}}, \mu_0^{(0)} \rangle\right) \nonumber\\
               &\hspace{50pt}f_s^{(0)}(x', a'')\sigma_{\rho_0}^{g_s^{(0)}}(x',a')ds,\nonumber
        \end{align}
    }
     where
    \begin{equation} \label{C^f_eq}
        C_{(x,a)}^h = \sigma(w\cdot (x,a))\partial_c h(c,v) + c\sigma'(w\cdot (x,a))\partial_w h(c,w) \cdot (x,a),
    \end{equation}
    and we consider the equations for $v_t^{(1)}$ and $\mu_t^{(1)}$ to hold true for any $h \in C_b^3(\mathbb{R})$. The matrices $W_{\pi^{g_t^{(0)}}}$ and $W_{\sigma_{\rho_0}^{g_t^{(0)}}}$ are defined via definition \ref{W_v_def}, and we used matrix notation in the definition of the limit of the error terms of the stationary distributions $\tilde{\pi}^{g_t^{(0)}}$ and $\tilde{\sigma}_{\rho_0}^{g_t^{(0)}}$ in \eqref{leading_order_error_terms_limits_eq}. 
\begin{remark}
    We highlight that the definition of ${\pi}^{g_t^{1,N}}$ , ${\pi}^{g_t^{(1)}}$, ${\sigma}_{\rho_0}^{g_t^{1,N}}$, and ${\sigma}_{\rho_0}^{g_t^{(1)}}$, as well as their corresponding higher order terms that will be defined later on, differs from the definition of  ${\pi}^{g_t^{N}}$ , ${\pi}^{g_t^{(0)}}$, ${\sigma}_{\rho_0}^{g_t^{N}}$, and ${\sigma}_{\rho_0}^{g_t^{(0)}}$. While the latter are stationary distributions of previously defined Markov chains, this definition is not true for their error terms. For example, ${\pi}^{g_t^{1,N}}$ is defined to be the process satisfying the third equation in \eqref{leading_error_terms_def_eq}, and is not the stationary distribution of $(\mathcal{M},g_t^{1,N})$.
\end{remark}

We now summarize the convergence results for the first order error terms. Note that Case 1 in propositions \ref{first_order_error_conv_prop}, \ref{empirical_measures_leading_order_error_conv_prop} and theorem \ref{actor_critic_networks_leading_error_conv_thm} are proven in section \ref{first_order_error_tersms_sec_3},
 whereas Case 2 of those results is shown in Section \ref{first_order_error_tersms_sec_2}.
\begin{proposition} \label{first_order_error_conv_prop}
    Let $T<\infty$ and $p\in \mathbb{N}$. Let $\delta_t^{1,N}$ be any of the processes $f_t^{1,N}$, $g_t^{1,N}$, $\pi^{g_t^{1,N}}$ and $\sigma_{\rho_0}^{g_t^{1,N}}$. Let $\delta_t^{(1)}$ be the corresponding limit $f_t^{(1)}$, $g_t^{(1)}$, $\pi^{g_t^{(1)}}$ or $\sigma_{\rho_0}^{g_t^{(1)}}$ respectively. Then $\delta_t^{1,N}$ converges to $\delta_t^{(1)}$. Specifically, we consider the following cases for $\beta\in\left(\frac{1}{2},1\right)$:\\

    \textbullet  \hspace{5pt} Case 1: $\beta \in \left( \frac{1}{2}, \frac{3}{4}\right]$: The convergence is the weak sense in $D_{\mathbb{R}^{d'}}([0,T])$.

    \textbullet \hspace{5pt} Case 2: $\beta \in \left( \frac{3}{4}, 1\right)$: The convergence is in $L_p$ and there exist constants $C_{T,p}$ and $N_{T,p}$ independent of $N$ such that the following convergence rates hold for $N>N_{T,p}$:
    \begin{equation*}
        \begin{aligned}
            \sup_{t\in [0,T]}\mathds{E}\left[ \max_{(x,a)\in \mathcal{X}\times \mathcal{A}}\left|{\delta}_t^{1,N}(x,a)-{\delta}_t^{(1)}(x,a)\right|^p \right] &\leq C_{T,p} \left[N^{p(\beta-1)}+N^{p(\frac{3}{2}-2\beta)}\right].
        \end{aligned}
    \end{equation*}
\end{proposition}

\begin{proposition}\label{empirical_measures_leading_order_error_conv_prop}
    Let $T<\infty$ and $p\in \mathbb{N}$. The convergence of the first order errors of the empirical measures $v_t^{1,N}$ and $\mu_t^{1,N}$ is described as follows, depending on the value of  $\beta\in\left(\frac{1}{2},1\right)$:\\

    \textbullet  \hspace{5pt} Case 1: $\beta \in \left( \frac{1}{2}, \frac{3}{4}\right)$: Both error terms vanish in $L_p$, in the sense that for any $h\in C_b^3(\mathds{R})$ there exist constants $C_{T,p}$ and $N_{T,p}$ independent of $N$, such that the following bounds hold for $N>N_{T,p}$:
    \begin{equation*}
        \begin{aligned}
            \sup_{t\in [0,T]}\mathds{E}\left[ \left|\langle h, v_t^{1,N} \rangle\right|^p\right] +\sup_{t\in [0,T]}\mathds{E}\left[ \left|\langle h, \mu_t^{1,N} \rangle \right|^p\right]&\leq C_{T,p}N^{p(2\beta-\frac{3}{2})}.
        \end{aligned}
    \end{equation*}

    \textbullet  \hspace{5pt} Case 2: $\beta \in \left[ \frac{3}{4}, 1\right)$: The error terms converge in $L_p$ to $v_t^{(1)}$ and $\mu_t^{(1)}$ respectively ,in the sense that for any fixed $h \in C_b^3(\mathds{R})$  there exist constants $C_{T,p}$ and $N_{T,p}$ independent of $N$ such that the following convergence rates hold for $N>N_{T,p}$
    \begin{equation*}
        \begin{aligned}
            \sup_{t\in [0,T]}\mathds{E}\left[ \left|\langle h, v_t^{1,N} \rangle - \langle h, v_t^{(1)}\rangle \right|^p\right]+\sup_{t\in [0,T]}\mathds{E}\left[ \left|\langle h, \mu_t^{1,N} \rangle - \langle h, \mu_t^{(1)}\rangle \right|^p\right] &\leq C_{T,p}N^{p(\beta-1)}.
        \end{aligned}
    \end{equation*}
\end{proposition}
\begin{theorem}\label{actor_critic_networks_leading_error_conv_thm}
    Let $T<\infty$ and $p\in \mathbb{N}$. The first order error terms $P_t^{1,N}$ and $Q_t^{1,N}$ of the actor and critic networks converge to $P_t^{(1)}$ and $q_t^{(1)}$ respectively. Specifically, we consider the following cases for $\beta\in\left(\frac{1}{2},1\right)$:\\

    \textbullet  \hspace{5pt} Case 1: $\beta \in \left( \frac{1}{2}, \frac{3}{4}\right]$: The convergence is in the weak sense in $D_{\mathbb{R}^{d'}}([0,T])$.\\

    \textbullet  \hspace{5pt} Case 2 : $\beta \in \left( \frac{3}{4}, 1\right)$: The convergence is in $L_p$ and there exist constants $C_{T,p}$ and $N_{T,p}$ independent of $N$, such that the following convergence rates hold for $N>N_{T,p}$:
        \begin{align}
        \sup_{t\in [0,T]}\mathds{E}\left[ \max_{(x,a)\in \mathcal{X}\times \mathcal{A}}\left|P_t^{1,N}(x,a)-P_t^{(1)}(x,a)\right|^p \right] &\leq  C_{T,p} \left[N^{p(\beta-1)}+N^{p(\frac{3}{2}-2\beta)}\right]\nonumber\\
            \sup_{t\in [0,T]}\mathds{E}\left[ \max_{(x,a)\in \mathcal{X}\times \mathcal{A}}\left|Q_t^{1,N}(x,a)-Q_t^{(1)}(x,a)\right|^p \right] &\leq  C_{T,p} \left[N^{p(\beta-1)}+N^{p(\frac{3}{2}-2\beta)}\right]\nonumber
        \end{align}
\end{theorem}
\subsection{Completing the asymptotic series expansion.}
Theorem \ref{actor_critic_networks_leading_error_conv_thm} shows us that the randomness from the initialization is captured in the initial condition of the leading order error term whenever $\beta \leq \frac{3}{4}$. On the other hand, when $\beta \in \left( \frac{3}{4},1\right)$ the $L_p$ convergence to a nonrandom process suggests that this randomness will be found in a higher order expansion term. This pattern continues as follows: We fix $\beta \in \left( \frac{2n-1}{2n}, \frac{2n+1}{2n+2}\right]$ and inductively define the higher order scaled error processes as follows:
    \begin{align}
        Q_t^{m,N}(x,a) &= N^{\phi_m}\left(Q_t^{m-1,N}(x,a)-Q_t^{(m-1)}(x,a) \right)\nonumber\\
        P_t^{m,N}(x,a) &= N^{\phi_m}\left(P_t^{m-1,N}(x,a)-P_t^{(m-1)}(x,a) \right)\nonumber\\
        f_t^{m,N}(x,a) &= N^{\phi_m}\left(f_t^{m-1,N}(x,a)-f_t^{(m-1)}(x,a)\right)\nonumber\\
        g_t^{m,N}(x,a) &= N^{\phi_m}\left(g_t^{m-1,N}(x,a)-g_t^{(m-1)}(x,a)\right)\nonumber\\
        {\pi}^{g_t^{m,N}}(x,a) &= N^{\phi_m}\left(\pi^{g_t^{m-1,N}}(x,a)-\pi^{g_t^{(m-1)}}(x,a)\right)\label{higher_error_terms_def_eq}\\
        {\sigma}_{\rho_0}^{g_t^{m,N}}(x,a) &= N^{\phi_m}\left(\sigma_{\rho_0}^{g_t^{m-1,N}}(x,a)-\sigma_{\rho_0}^{g_t^{(m-1)}}(x,a)\right)\nonumber\\
        v_t^{m,N} &= N^{\phi_m}\left(v_t^{m-1,N} - v_t^{(m-1)} \right)\nonumber\\
        \mu_t^{m,N} &= N^{\phi_m}\left(\mu_t^{m-1,N} - \mu_t^{(m-1)} \right),\nonumber
    \end{align}
where $m\leq n$ and the scalings $\phi_m$ of the error terms satisfy
\begin{equation*}
    \phi_m = \begin{cases}
        1-\beta,  \hspace{42pt} \text{for } 1\leq m \leq n-1\\
        \frac{1}{2} -n(1-\beta), \quad \text{for } m=n\\
        \end{cases}
\end{equation*}

We now define the limit candidates for $\beta \in \left( \frac{2n-1}{2n}, \frac{2n+1}{2n+2}\right]$ for a given $n\in \mathbb{N}$. We start with the actor model error terms for $m\leq n$

\begin{equation} \label{actor_model_higher_order_error_terms_limits_eq}
    \begin{aligned}
        f_t^{(m)}(x,a) = \begin{cases}
            D^m(f_t^{(0)}(x,a)) : \left[P_t^{(1)}(x,a) \right]^{\otimes m} & \text{if}\quad m<n , \text{ or }  \beta = \frac{2n+1}{2n+2} \\
            0  &\text{if} \quad m=n, \text{ and }  \beta < \frac{2n+1}{2n+2}.
        \end{cases}
    \end{aligned}
\end{equation}

Here $D^j(f_t^N(x,a))$ denoted the $j-$th derivative Tensor of a function with vector input\footnote{ In Einstein notation we have $(a \otimes b)_{ij} = a_ib_j$ and we define inductively $a^{\otimes 1} = a$ and $a^{\otimes j} = a \otimes a^{\otimes j-1}$. The inner product id denoted as  $A:B = A_{ijk...}B_{ijk...}$ where the summation is over all indices.}. The inner product of Tensors, corresponding to element-wise multiplication and summation of the results, is denoted with the operator $:$, while $v^{\otimes j}$ denotes the $j-$fold outer product of a vector $v$ with itself. We assume that $x$ is fixed in \eqref{actor_model_higher_order_error_terms_limits_eq}.

The limit candidates for the higher order policy errors are then defined as
\begin{equation} \label{policy_higher_order_error_terms_limits_eq}
    g_t^{(m)}(x,a)  = (1-\eta_t)f_t^{(m)}(x,a), \quad m\leq n.
\end{equation}

Next we define the limit candidates for the $m-$th order error terms of the stationary distributions, where $m\leq n$, as
 
{\small     \begin{align} \label{stationary_measure_higher_order_error_terms_limits_eq}
        \pi^{g_t^{(m)}} &= \begin{cases}
            \left[ \pi^{g_t^{(0)}}{\mathbb{P}}_t^{(m)}+  \sum_{j=1}^{m-1}\left( \pi^{g_t^{(j)}}\mathbb{P}_t^{(m-j)} -  \pi^{g_t^{(j)}}W_{\pi^{g_t^{(m-j)}}}\right)\right] \left( I-\mathbb{P}_t^{(0)}+W_{\pi^{g_t^{(0)}}}\right)^{-1},  &\text{ if } m<n \text{ or } \beta= \frac{2n+1}{2n+2}\nonumber\\
            0,  &\hspace{-10pt} \text{ if } m=n \text{ and } \beta< \frac{2n+1}{2n+2}\\
            \end{cases}\nonumber\\
        \sigma_{\rho_0}^{g_t^{(m)}} &= \begin{cases}
            \left[ \sigma_{\rho_0}^{g_t^{(0)}}{\mathbb{P}}_t^{(m)}+  \sum_{j=1}^{m-1}\left( \sigma_{\rho_0}^{g_t^{(j)}}\Pi_t^{(m-j)} -  \sigma_{\rho_0}^{g_t^{(j)}}W_{\sigma_{\rho_0}^{g_t^{(m-j)}}}\right)\right] \left( I-\Pi_t^{(0)}+W_{\sigma_{\rho_0}^{g_t^{(0)}}}\right)^{-1},  &\text{ if } m<n \text{ or } \beta= \frac{2n+1}{2n+2}\\
            0,  &\hspace{-10pt} \text{ if } m=n \text{ and } \beta< \frac{2n+1}{2n+2}
            \end{cases}
    \end{align}}
    where
    \begin{equation} 
        \begin{aligned}
            \mathbb{P}^{(j)}_{t, (x,a),(x',a')} &= g_t^{(j)}(x',a')p(x'|x,a) \\
            \Pi^{(j)}_{t, (x,a),(x',a')} &= g_t^{(j)}(x',a')\tilde{p}(x'|x,a)
        \end{aligned}
    \end{equation}

    and the matrices $W_{\pi^{g_t^{j}}}$ and $W_{\sigma_{\rho_0}^{g_t^{j}}}$ satisfy the definition \ref{W_v_def} for $j\leq n$.

    We now define the limit candidates for the higher order empirical measure error terms. For any $h \in C_b^3(\mathds{R})$ and $m\leq n$ we define $v_t^{(m)}$ and $\mu_t^{(m)}$ as the measures satisfying
        \begin{align}
            &\langle h, v_t^{(m)} \rangle = \sum_{\substack{m_1,m_2,m_3,m_4 \in \{0,1,...,m-1\} \\ m_1+m_2+m_3+m_4=m-1}} \alpha \int_0^t \sum_{(x',a'),(x'',a'') \in \mathcal{X}\times\mathcal{A}} \Big(r(x',a')\mathds{1}\{m_1=0\}\nonumber\\
            &\hspace{60pt}+\gamma Q_s^{(m_1)}(x'',a'')- Q_s^{(m_1)}(x',a')\Big)\langle C_{(x',a')}^h, v_s^{(m_2)} \rangle g_{s}^{m_3}(x'',a'')\pi^{g_s^{(m_4)}}(x',a')p(x''|x',a')ds\nonumber\\
            &\langle h, {\mu}_t^{(m)} \rangle = \sum_{\substack{m_1,m_2,m_3,m_4 \in \{0,1,...,m-1\} \nonumber\\ m_1+m_2+m_3+m_4=m-1}} \alpha \int_0^t \sum_{(x',a'),(x'',a'')\in \mathcal{X}\times \mathcal{A}} \frac{1}{1+s} Q_s^{(m_1)}(x',a') \nonumber\\
                & \hspace{60pt} \left( \langle C_{(x',a')}^h, \mu_s^{(m_2)} \rangle+\langle C_{(x'',a'')}^h, \mu_s^{(m_2)} \rangle\right) f_s^{(m_3)}(x',a'') \sigma_{\rho_0}^{g_s^{(m_4)}}(x',a')ds,\label{empirical_measure_higher_order_error_terms_limits_eq}
        \end{align}
    for $m<n$ or $\beta = \frac{2n+1}{2n+2}$, and
    
    \begin{equation}
        \langle h, v_t^{(n)} \rangle = \langle h, \mu_t^{(n)} \rangle = 0,
    \end{equation}
    for $\beta < \frac{2n+1}{2n+2}$.
    
    Lastly, we define the limit candidates for the higher order Critic and Actor network output errors. For $m<n$ we have 
    \begin{align} 
\label{Q_P_higher_order_error_terms_limits_eq}
            &Q_t^{(m)}(x,a) = Q_0^{(m)}(x,a)\\
            &+  \sum_{\substack{m_1,m_2,m_3,m_4 \in \{0,1,...,m\} \\ m_1+m_2+m_3+m_4=m}} \alpha \int_{0}^t \sum_{(x',a'),(x'',a'')\in \mathcal{X}\times \mathcal{A}} \left[r(x',a')\cdot \mathds{1}\left\{m_1=0 \right\} +\gamma Q_s^{(m_1)}(x'',a'')-Q_s^{(m_1)}(x',a') \right] \nonumber\\
                &\hspace{60pt}\times\langle B_{(x,a),(x',a')}, v_s^{(m_2)}\rangle g_s^{(m_3)}(x'',a'') \pi^{g_s^{(m_4)}}(x'',a'')p(x''|x',a')ds\nonumber\\
        &{P}_t^{(m)}(x,a) = {P}_0^{(m)}(x,a)\nonumber\\
            &+ \sum_{\substack{m_1,m_2,m_3,m_4 \in \{0,1,...,m\} \\ m_1+m_2+m_3+m_4=m}}  \int_{0}^t \sum_{(x',a',a'')\in \mathcal{X}\times \mathcal{A}\times \mathcal{A}} \frac{1}{1+s} Q_s^{(m_1)}(x',a') \left(\langle B_{(x,a),(x',a')}, \mu_s^{(m_2)}\rangle+\langle B_{(x,a),(x',a'')}, \mu_s^{(m_2)}\rangle \right)\nonumber\\
                &\hspace{60pt}\times f_s^{(m_3)}(x',a'') \sigma_{\rho_0}^{g_s^{(m_4)}}(x',a')ds,\nonumber
        \end{align}
    with initial conditions
    \begin{equation}
        \begin{aligned}
            Q_t^{(m)}(x,a)=0, \quad P_t^{(m)}(x,a)=0,\label{Q_P_higher_order_error_terms_limits_eqInitialCondition0}
        \end{aligned}
    \end{equation}
    for all $(x,a)\in \mathcal{X}\times\mathcal{A}$.

    For the $n-$th order correction term we again distinguish two cases. For $m=n$ and $\beta =\frac{2n+1}{2n+2}$, the pair $(Q_t^{(n)}(x,a),P_t^{(n)}(x,a))$ satisfies \eqref{Q_P_higher_order_error_terms_limits_eq} (with $m=n$), with the random initial conditions  \begin{equation}\label{Eq:RandomInitialConditions}       \begin{aligned}
            Q_0^{(n)}(x,a)\sim\mathcal{G}(x,a), \quad P_0^{(n)}(x,a) \sim \mathcal{H}(x,a),
        \end{aligned}
    \end{equation}
    for all $(x,a)\in \mathcal{X}\times\mathcal{A}$ and where $\mathcal{G}$ and $\mathcal{H}$ are Gaussian and defined in \eqref{K_t_L_t_init_eq_case1o}, see also \eqref{initialization_dist_conv_eq}.

   If $m=n$ and $\beta<\frac{2n+1}{2n+2}$, then we have
    \begin{equation}\label{Q_P_higher_order_error_terms_limits_eq_lastTerm}
        \begin{aligned}
            &Q_t^{(n)}(x,a)  = Q_0^{(n)}(x,a)\\
            &+ \alpha \int_{0}^t \sum_{(x',a'),(x'',a'')\in \mathcal{X}\times \mathcal{A}} \left[\gamma Q_s^{(n)}(x'',a'') - Q_s^{(n)}(x',a')\right]\langle B_{(x,a),(x',a')},  v_0^{(0)}\rangle g_s^{(0)}(x'',a'') \pi^{g_s^{(0)}}(x'',a'')p(x''|x',a')ds\\
            &{P}_t^{(n)}(x,a) = P_0^{(n)}(x,a)\\
            &+\int_{0}^t \sum_{(x',a',a'')\in \mathcal{X}\times \mathcal{A}\times \mathcal{A}}\frac{1}{1+s}  Q_s^{(n)}(x',a')\left[ \langle B_{(x,a),(x',a')}, \mu_0^{(0)}\rangle+\langle B_{(x,a),(x',a'')}, \mu_0^{(0)}\rangle\right]f_s^{(0)}(x',a'')\sigma_{\rho_0}^{g_s^{(0)}}(x',a')ds,
        \end{aligned}
    \end{equation}
 with initial conditions being random and given by \eqref{Eq:RandomInitialConditions}.   
    The following convergence results hold.
    \begin{proposition} \label{higher_order_error_terms_conv_prop}
        Let $T<\infty$, $p\in \mathds{N}$, $m\leq n$ and $\beta \in \left( \frac{2n-1}{2n}, \frac{2n+1}{2n+2} \right]$. Let $\delta_t^{m,N}$ be any of the processes $f_t^{m,N}$, $g_t^{m,N}$, $\pi^{g_t^{m,N}}$ and $\sigma_{\rho_0}^{g_t^{m,N}}$. Let $\delta_t^{(m)}$ be the corresponding limit candidate $f_t^{(m)}$, $g_t^{(m)}$, $\pi^{g_t^{(m)}}$ or $\sigma_{\rho_0}^{g_t^{(m)}}$ respectively. There exist constants $C_{T,p,m}$ and $N_{T,p,m}$ independent of $N$, such that for $N>N_{T,p,m}$ the following bound is satisfied
    \begin{equation*}
        \begin{aligned}        \sup_{t\in[0,T]}\mathds{E}\left[\max_{(x,a)\in \mathcal{X}\times\mathcal{A}}\left|\delta_t^{m,N}(x,a)-\delta_t^{(m)}(x,a) \right|^p \right] &\leq C_{T,p,m}N^{p(\beta-1)}.
            \end{aligned}
        \end{equation*}
    \end{proposition}

    \begin{proposition} \label{empirical_measures_intermediate_terms_conv_prop}    Let $T<\infty$, $p\in \mathds{N}$, $m\leq n$ and $\beta \in \left( \frac{2n-1}{2n}, \frac{2n+1}{2n+2} \right]$. For any $h\in C_b^3(\mathds{R})$ there exist constants $C_{T,p,m}$ and $N_{T,p,m}$ depending on $T$, $p$ and $m$ only  and independent of $k$ and $N$, such that for $N>N_{T,p,m}$ the $m-$th order empirical measure error terms $v_t^{m,N}$ and $\mu_t^{m,N}$ satisfy the following bounds
  \begin{equation*}
        \begin{aligned}
            \sup_{t\in [0,T]}\mathds{E}\left[ \left|\langle h, v_t^{m,N} \rangle - \langle h, v_t^{(m)}\rangle \right|^p\right] + \sup_{t\in [0,T]}\mathds{E}\left[ \left|\langle h, \mu_t^{m,N} \rangle - \langle h, \mu_t^{(m)}\rangle \right|^p\right] &\leq C_{T,p}N^{p(\beta-1)}, &&m\leq n-1 \\
            \sup_{t\in [0,T]}\mathds{E}\left[ \left|\langle h, v_t^{m,N} \rangle - \langle h, v_t^{(m)}\rangle \right|^p\right]+ \sup_{t\in [0,T]}\mathds{E}\left[ \left|\langle h, \mu_t^{m,N} \rangle - \langle h, \mu_t^{(m)}\rangle \right|^p\right] &\leq C_{T,p,m}N^{p\left(\frac{1}{2} - m(1-\beta)\right)}, &&m= n \\
        \end{aligned}
    \end{equation*}
    \end{proposition}
    \begin{theorem} \label{Q_P_intermediate_error_terms_conv_th}
        Let $T<\infty$, $p\in \mathds{N}$, $m<n$ and $\beta \in \left( \frac{2n-1}{2n}, \frac{2n+1}{2n+2} \right]$. There exist constants $C_{T,p,m}$ and $N_{T,p,m}$ depending on $T$, $p$ and $m$ only  and independent of $k$ and $N$, such that for $N>N_{T,p,m}$ the $m-$th order Actor and Critic network output error terms $P_t^{m,N}$ and $Q_t^{m,N}$ satisfy the following bounds
    \begin{equation*}
        \begin{aligned}                \sup_{t\in[0,T]}\mathds{E}\left[\max_{(x,a)\in \mathcal{X}\times\mathcal{A}}\left|P_t^{m,N}(x,a)-P_t^{(m)}(x,a) \right| \right] &\leq C_{T,p,m}N^{p(\beta-1)}, \quad &&m<n-1\\
                \sup_{t\in[0,T]}\mathds{E}\left[\max_{(x,a)\in \mathcal{X}\times\mathcal{A}}\left|Q_t^{m,N}(x,a)-Q_t^{(m)}(x,a) \right| \right] &\leq C_{T,p,m}N^{p(\beta-1)}, \quad &&m<n-1\\
                \sup_{t\in[0,T]}\mathds{E}\left[\max_{(x,a)\in \mathcal{X}\times\mathcal{A}}\left|P_t^{m,N}(x,a)-P_t^{(m)}(x,a) \right| \right] &\leq C_{T,p,m}N^{p\left(\frac{1}{2} - m(1-\beta)\right)}, \quad &&m=n-1\\
                \sup_{t\in[0,T]}\mathds{E}\left[\max_{(x,a)\in \mathcal{X}\times\mathcal{A}}\left|Q_t^{m,N}(x,a)-Q_t^{(m)}(x,a) \right| \right] &\leq C_{T,p,m}N^{p\left(\frac{1}{2} - m(1-\beta)\right)}, \quad &&m=n-1              
            \end{aligned}
        \end{equation*}
    \end{theorem}
    \begin{theorem} \label{higher_order_conv_th}
        Let $T<\infty$ and $\beta \in \left( \frac{2n-1}{2n}, \frac{2n+1}{2n+2} \right]$. Then the $n-$th order Actor and Critic network output error terms $P_t^{n,N}$ and $Q_t^{n,N}$  converge weakly to $P_t^{(n)}$ and $Q_t^{(n)}$ respectively in $D_{\mathbb{R}^{d'}}([0,T])$.
    \end{theorem}

    The proofs of propositions \ref{higher_order_error_terms_conv_prop} and \ref{empirical_measures_intermediate_terms_conv_prop} and of theorems \ref{Q_P_intermediate_error_terms_conv_th} and \ref{higher_order_conv_th} are in Appendix \ref{higher_order_error_sec}.

\section{Numerical Results}\label{S:NumericalResults}

\subsection{Experimental Study of the Bias}

In order to compare the algorithm's convergence speed experimentally for the different scaling hyperparameters $\beta \in \left[ \frac{1}{2},1\right]$, we tested it on the forest management MDP implemented in python's \textit{Markov Decision Process  (MDP) Toolbox}, using the default setting and a discount factor of $\gamma=0.7$ \cite{pymdptoolbox}. 

We are interested in the behavior of the actor, which dictates the policy, and in the rewards. 
\begin{enumerate}
\item{ The \textit{Actor MSE Loss}, defined as the Mean Square Error (MSE) between the Actor model $f_t^N$ and the true optimal policy $\pi^*$ and given by
    \begin{equation}
        \text{Ac}_{\text{MSE}} = \frac{1}{|\mathcal{X}\times\mathcal{A}|} \sum_{(x,a) \in \mathcal{X}\times\mathcal{A}}\left(f_t^N(x,a) - \pi^*(x,a)\right)^2
    \end{equation}}

\item{The \textit{Reward}, defined as the expected future discounted rewards obtained when following the policy $f_t^N$, and given by
    \begin{equation}
        \rho^{f_t^N} = \mathds{E}_{f_t^N}\left[\sum_{k=0}^\infty \gamma^k r(x_k,a_k) \right] =  \sum_{(x,a) \in \mathcal{X}\times\mathcal{A}} V^{f_t^N}(x,a) \rho_0(x,a),
    \end{equation}
    where $\rho_0$ is the initial distribution of the MDP.}
\end{enumerate}

The trained neural networks were $N=10000$ neurons wide and were trained for time $T=100$. The experiment was performed $100$ times for each $\beta$ and the  MSE loss and reward were averaged along them. The parameters were initialized from truncated normal distributions.
\begin{figure}[ht!]
\includegraphics[width=1\textwidth]{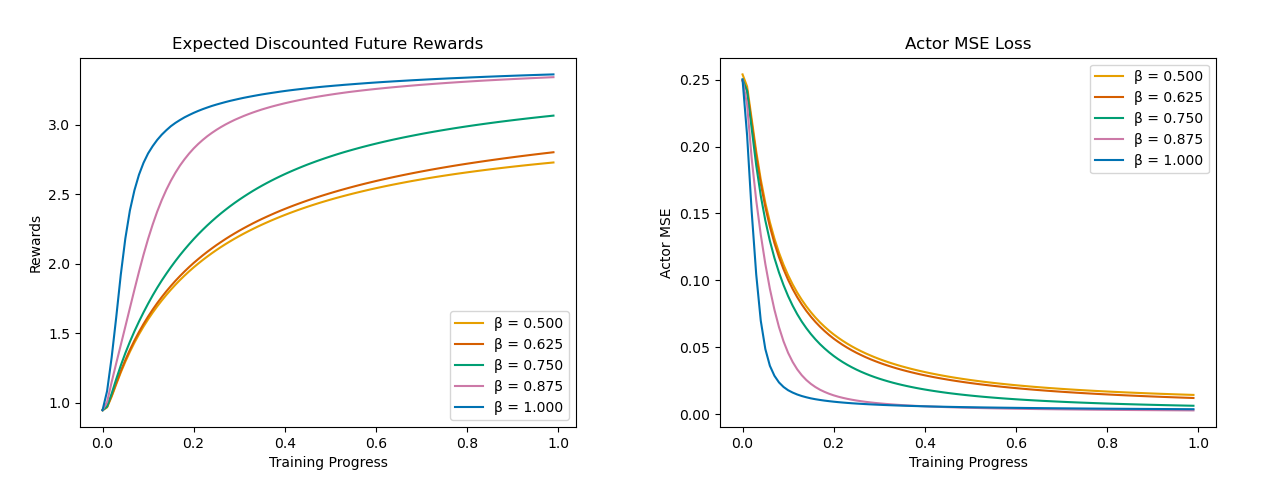}
\caption{The Reward and Actor MSE Loss as a function of training time. }
\label{F:Figure1}
\end{figure}

The comparison-relevant numerical results are presented in figure \ref{F:Figure1} and are consistent with our theory, and suggest that the non-random terms in the expansions of the neural networks in theorem \ref{main_result_theorem} decay quickly as $T$ grows large, increasing the importance of the terms $(n)^{th}$ (the random terms) which capture the random fluctuations induced by the initialization.

\subsection{Experimental Study of the Variance}

Variance in the network outputs and estimates arises from two sources: the random initialization and the Markov chains. The variance induced by the Markov chains is captured in the martingale terms (see section \ref{leading_order_conv_chap}) and scales as $O(N^{-1/2})$, as shown in section \ref{mc_sec_app} of the Appendix. To leading order in $N$, the variance scales as $N^{\tfrac{1}{2}-\beta}$ and is driven by the random initialization. Consequently, we expect the variance to decrease as $\beta$ increases from $\tfrac{1}{2}$ to $1$.

All involved neural networks have width $N=10000$ and are trained for $T=100$ iterations. Each network is trained $100$ times to estimate the standard deviations of the following quantities: the policy $\pi$ induced by the Actor network and the discounted future rewards corresponding to that policy, computed using Q-iteration.

\begin{figure}[ht!]
\includegraphics[width=1\textwidth]{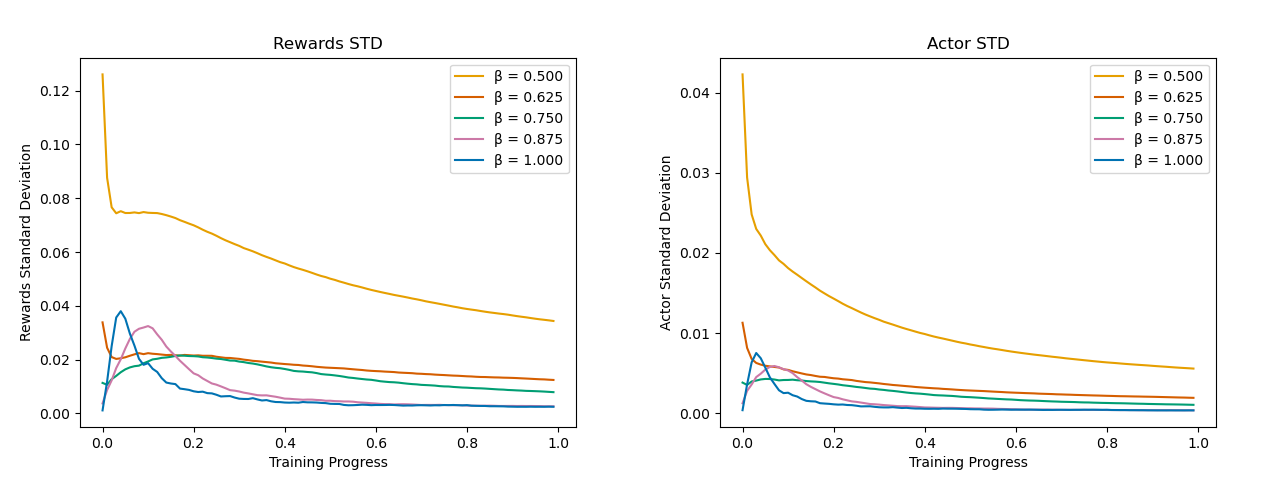}
\caption{Standard deviation Monte Carlo estimates for the Actor, the Critic, and the Rewards as a function of training progress for different values of the scaling parameter $\beta$.}\label{F:Figure2}
\end{figure}

The results of our experiment, as seen in figure \ref{F:Figure2}, align clearly with the theoretical predictions. Increasing $\beta$ leads to lower variance in both the rewards and the induced policy, implying a more robust algorithm.

\section{Leading Order Convergence} \label{leading_order_conv_chap}
In this section we will prove theorem \ref{leading_order_conv_th} along with propositions \ref{leading_order_conv_prop} and \ref{empirical_measures_conv_prop}. In what follows we will often use the notation $\xi = (x,a)$, $\xi_k=(x_k,a_k), \tilde{\xi}_k=(\tilde{x}_k, \tilde{a}_k)$ for convenience. Throughout this section, unless otherwise stated, we consider $0<T<\infty$ fixed.
\begin{definition}\label{O_L_p_def}
    We will denote with $O_{L_p}(N^\beta)$ any sequence $X_t^N$ of random processes that satisfies 
    \begin{equation*}
\sup_{t\in[0,T]}\mathds{E}\left[\left|X_t^N \right|^p \right] \leq C N^{p\beta},
    \end{equation*}
whenever $N$ is large enough and for constant $C<\infty$ that does not depend on $N$.
\end{definition}

\subsection{Pre-Limit Equations}

In this section we will provide the pre-limit evolution equations of the Actor and Critic network processes $P_t^N$ and $Q_t^N$ and the empirical measure processes $v_t^N$ and $\mu_t^N$ along with some useful results. Details regarding the derivation of those equations can be found in section \ref{prelim_sec_app} of the Appendix. 

We will start by providing two bounds that will be useful for our analysis. The proof of lemma \ref{parameter_bounds_apriori_lemma} can be found in section \ref{prelim_sec_app} of the Appendix while lemma \ref{Q_P_L_4_Bound} follows immediately from lemmas \ref{parameter_bounds_apriori_lemma} and the definitions of the actor and critic networks \eqref{actor_network_eq} and \eqref{critic_network_eq} respectively. 

\begin{lemma}\label{parameter_bounds_apriori_lemma}
        For any $k\leq TN$, and $i \in [N]=\{1,...,N\}$ there exists a constant $C_T < \infty$ that does not depend on $N$ and $k$ such that 
        \begin{equation}
            \begin{aligned}
                &\bullet \max \left(|C_k^i|, |B_k^i|, \mathds{E}\left[\|W_k^i\|\right], \mathds{E}\left[\|U_k^i\| \right]\right) \leq C_T\\
                &\bullet \max \left(|C_{k+1}^i-C_{k}^i|, \|W_{k+1}^i-W_k^i\|\right) \leq \frac{C_T}{N}\\
                &\bullet \max \left(|B_{k+1}^i-B_{k}^i|, \|U_{k+1}^i-U_k^i\|\right) \leq \frac{C_T}{N}
            \end{aligned}
        \end{equation}
    \end{lemma}

    \begin{proof}[Proof of lemma \ref{parameter_bounds_apriori_lemma}.]
    The first two bounds can be proven in the same way as the first two bounds of lemma 4.2 in \cite{nac}. We provide a proof for the third bound. We first observe that
    
    \begin{equation} \label{Q_P_absolute_bound_eq}
        \max_{\xi \in \mathcal{X}\times \mathcal{A}}|Q_k^N(\xi)| \leq C_TN^{1-\beta}, \quad \max_{\xi \in \mathcal{X}\times \mathcal{A}}|P_k^N(\xi)| \leq C_TN^{1-\beta}
    \end{equation}
    
    which follows from the definition \eqref{critic_network_eq} and the first bound of this lemma.
    
   We now have
        \begin{equation*}
            |B_{k+1}^i-B_k^i| \leq N^{-\beta}\zeta_k^N \max_{\xi \in \mathcal{X}\times \mathcal{A}}|Q_k^N(\xi)|\left(\max_{\xi \in \mathcal{X}\times \mathcal{A}}|\sigma(U_k^i\cdot \xi)| +|\mathcal{A}|\max_{\xi \in \mathcal{X}\times \mathcal{A}}|f_k^N( \xi)||\sigma(U_k^i\cdot \xi)|\right) \leq \frac{C_T}{N}, 
        \end{equation*}
        where we used \eqref{Q_P_absolute_bound_eq} and the bounds
        \begin{equation*}
            \begin{aligned}
                \max_{\xi \in \mathcal{X}\times \mathcal{A}}|f_k^N( \xi)| &\leq 1, \quad
                \max_{\xi \in \mathcal{X}\times \mathcal{A}}|\sigma(U_k^i\cdot \xi)| &\leq 1.\quad \zeta_k^N = \frac{1}{N^{2-2\beta}\left(1+\frac{k}{N} \right)}
            \end{aligned}
        \end{equation*}

       The first of which follows from the uniform bound of the softmax function, and the second follows from assumption \ref{actor_critic_models_assumption}. Similarly, we have
        \begin{equation*}
            \|U_{k+1}^i-U_k^i\| \leq N^{-\beta}\zeta_k^N \max_{\xi \in \mathcal{X}\times \mathcal{A}}|Q_k^N(\xi)|\left(\max_{\xi \in \mathcal{X}\times \mathcal{A}}|\sigma(U_k^i\cdot \xi)| \|\xi\| +|\mathcal{A}|\max_{\xi \in \mathcal{X}\times \mathcal{A}}|f_k^N( \xi)||\sigma(U_k^i\cdot \xi)\|\xi\| |\right) \leq \frac{C_T}{N}, 
        \end{equation*}
        where we used the same bounds along with the assumption \ref{MDP_assumptions} on the finiteness of $\mathcal{X}\times \mathcal{A}$ which implies $\max_{\xi \in \mathcal{X}\times \mathcal{A}}\|\xi\| < \infty$.
        \end{proof}

We now provide the pre-limit equation for the Critic Network, which is decomposed into a drift term, martingale terms, and remainder terms. 
    \begin{equation}\label{Q_t^N_eq}
        \begin{aligned}
            Q_t^N(\xi) &=  \alpha \int_{0}^{t} \sum_{\xi',( x'', a'') \in \mathcal{X}\times \mathcal{A}} \left( r(\xi') +\gamma Q_s^N(x'',a'') - Q_s^N(\xi')\right) \langle B_{\xi, \xi'}, v_s^N \rangle \pi^{g_s^N}(\xi')g_{s}^N(x'',a'')p(x''|\xi')ds\\
            &\hspace{20pt}+N^{1-\beta}\langle c\sigma(w\cdot \xi), v_0^N\rangle + M_t^{1,N}(\xi) + M_t^{2,N}(\xi)+ M_t^{3,N}(\xi)+R_{Q,t}^N,
        \end{aligned}
    \end{equation}
    where the martingale terms are
    \begin{equation} \label{M_t^N_eq}
        \begin{aligned}
            M_t^{1,N}(\xi) &= \frac{\alpha}{N}\sum_{k=0}^{\lfloor Nt \rfloor-1} r(\xi_k) \langle B_{\xi, \xi_k}, v_k^N \rangle - \frac{\alpha}{N}\sum_{k=0}^{\lfloor Nt \rfloor-1} \sum_{\xi' \in \mathcal{X}\times\mathcal{A}} r(\xi') \langle B_{\xi, \xi'}, v_k^N \rangle \pi^{g_k^N}(\xi')\\
            M_t^{2,N}(\xi) &= \frac{\alpha}{N}\sum_{k=0}^{\lfloor Nt \rfloor-1}\gamma Q_k^N(\xi_{k+1}) \langle B_{\xi, \xi_k}, v_k^N \rangle \\
            &\hspace{20pt}- \frac{\alpha}{N}\sum_{k=0}^{\lfloor Nt \rfloor-1} \sum_{(\xi', x'',a'') \in \mathcal{X}\times\mathcal{A}\times \mathcal{X}\times \mathcal{A}}\gamma Q_k^N(x'',a'') \langle B_{\xi, \xi'}, v_k^N \rangle \pi^{g_s^N}(\xi')g_{s}^N(x'',a'')p(x''|\xi')\\
            M_t^{3,N}(\xi) &= -\frac{\alpha}{N}\sum_{k=0}^{\lfloor Nt \rfloor-1} Q_k^N(\xi_k) \langle B_{\xi, \xi_k}, v_k^N \rangle + \frac{\alpha}{N}\sum_{k=0}^{\lfloor Nt \rfloor-1} \sum_{\xi' \in \mathcal{X}\times\mathcal{A}} Q_k^N(\xi') \langle B_{\xi, \xi'}, v_k^N \rangle \pi^{g_k^N}(\xi'),
        \end{aligned}
    \end{equation}
    and the remainder satisfies
    \begin{equation} \label{R_{Q,t}^N_bounds_eq}
    \begin{aligned}
        \sup_{t\in[0,T]}\left| R_{Q,t}^N\right| \leq C_TN^{-\beta},  \quad
        \left| R_{Q,t}^N\right| =O_{L_p}(N^{-1}),
    \end{aligned}
    \end{equation}
    where the class of processes $O_{L_p}$ is defined in Definition \ref{O_L_p_def}. 
    For the Actor model, the prelimit equation reads
    \begin{equation} \label{P_t^N_eq}
        \begin{aligned}
            P_t^N(\xi) &=\int_{0}^{t} \sum_{({x}', {a}', a'') \in \mathcal{X}\times\mathcal{A}\times \mathcal{A}}  \left( \frac{1}{1+s} Q_s^N({x}', {a}')\right) \Bigg( \langle B_{\xi, ({x}', {a}')}, \mu_s^N \rangle +  \langle B_{\xi, ({x}',{a}'')}, \mu_s^N \rangle \Bigg) \\
            &\hspace{40pt} f_s^N({x}',{a}'')\sigma_{\rho_0}^{g_s^N}({x}',{a}')ds + N^{1-\beta}\langle b\sigma(u\cdot {\xi}), \mu_0^N\rangle + \tilde{M}_t^N(\xi) + R_{P,t}^N,
        \end{aligned}
    \end{equation}
     where the martingale term is
        \begin{align}
            \tilde{M}_t^{N}(\xi) &= \frac{1}{N}\sum_{k=0}^{\lfloor Nt \rfloor-1} \left( \frac{1}{1+\frac{k}{N}} Q_k^N(\tilde{\xi}_k)\right) \left( \langle B_{\xi, \tilde{\xi}_k}, \mu_k^N \rangle + \sum_{a' \in \mathcal{A}}f_k^N(\tilde{x}_k,a') \langle B_{\xi, (\tilde{x}_k,a')}, \mu_k^N \rangle \right)\nonumber\\
            &- \frac{1}{N}\sum_{k=0}^{\lfloor Nt \rfloor-1} \sum_{({x}', {a}',a'') \in \mathcal{X}\times\mathcal{A}\times \mathcal{A}}  \left( \frac{1}{1+\frac{k}{N}} Q_k^N({x}', {a}')\right) \Bigg( \langle B_{\xi, ({x}', {a}')}, \mu_k^N \rangle\nonumber\\            &\hspace{40pt}+  \langle B_{\xi, ({x}',{a}'')}, \mu_k^N \rangle \Bigg) f_k^N({x}',{a}'')\sigma_{\rho_0}^{g_k^N}({x}',{a}'),\label{M_tilde_t^N_eq}
        \end{align}
    and the remainder satisfies (recall definition \ref{O_L_p_def})
   \begin{equation} \label{R_{P,t}^N_bounds_eq}
    \begin{aligned}
        \sup_{t\in[0,T]}\left| R_{P,t}^N\right| \leq C_TN^{-\beta}, \quad
        \left| R_{P,t}^N\right| =O_{L_p}(N^{-1}).
    \end{aligned}
    \end{equation}

    Using the pre-limit equation, we can derive $L_p$ bounds for the actor and critic network outputs. The proof can be found in section \ref{prelim_sec_app} of the Appendix
    \begin{lemma} \label{Q_P_L_4_Bound}
        There exist constants $C_{T,p} < \infty$ and $N_{T,p}$ that do not depend on $N$ and $k$, such that for all $k$ with $k \leq \lfloor NT \rfloor-1$ and $N>N_{T,p}$ the following $L^p$ bounds hold
        \begin{equation*}
            \begin{aligned}
                \mathds{E}\left[ \max_{(x,a) \in \mathcal{X}\times \mathcal{A}} \left|Q_k^N(x,a)\right|^p + \max_{(x,a) \in \mathcal{X}\times \mathcal{A}} \left|P_k^N(x,a)\right|^p \right] &\leq C_{T,p}
            \end{aligned}
        \end{equation*}       
    \end{lemma}

    Note that the difference between  \eqref{Q_P_absolute_bound_eq} and lemma \ref{Q_P_L_4_Bound} is that  \eqref{Q_P_absolute_bound_eq} provides an a.s. bound while lemma \ref{Q_P_L_4_Bound} provides a bound on expectation.

    We now provide the pre-limit evolution equations of the empirical measure processes $v_k^N$ and $\mu_k^N$. For any $h\in C_b^2(\mathbb{R}^{1+d'})$ the process $v_t^N$ satisfies
    \begin{equation} \label{v_t^N_eq}
        \begin{aligned}
            \langle h, v_{t}^N \rangle - \langle h,v_0^N \rangle &= \frac{\alpha}{N^{2-\beta}} \sum_{k=0}^{\lfloor NT\rfloor -1} \left(r(\xi_k)+\gamma Q_k^N(\xi_{k+1})- Q_k^N(\xi_k)\right) \langle C_{\xi_k}^h, v_k^N \rangle + R_{v,t}^N,
        \end{aligned}
    \end{equation}
    where
    \begin{equation} \label{C_xi^f_eq}
        C_{\xi}^h = \sigma(w\cdot \xi)\partial_c h(c,v) + c\sigma'(w\cdot \xi)\partial_w h(c,w) \xi,
    \end{equation}
    and the remainder satisfies
    \begin{equation} \label{R_{v,t}^N_bounds_eq}
    \begin{aligned}
        \sup_{t\in[0,T]}\left| R_{v,t}\right| &\leq C_TN^{-1}.
    \end{aligned}
    \end{equation}

Similarly, for any $h\in C_b^2(\mathbb{R}^{1+d'})$ the process $\mu_t^N$ satisfies
    \begin{equation} \label{mu_t^N_eq}
        \begin{aligned}
            \langle \tilde{h}, \mu_{t}^N \rangle - \langle \tilde{h},\mu_0^N \rangle  &= \frac{\alpha}{N^{2-\beta}} \sum_{k=0}^{\lfloor NT\rfloor -1} \frac{1}{1+\frac{k}{N}} Q_k^N(\tilde{\xi_k}) \left(\langle C_{\xi_k}^{\tilde{h}}, \mu_k^N \rangle -\sum_{a' \in \mathcal{A}}f_k^N(\tilde{x}_k, a')\langle C_{(\tilde{x}_k, a')}^{\tilde{h}}, \mu_k^N \rangle\right) + R_{\mu,t}^N,
        \end{aligned}
    \end{equation}
    where the remainder satisfies
    \begin{equation} \label{R_mu^N_bounds_eq}
    \begin{aligned}
        \sup_{t\in[0,T]}\left| R_{\mu,t}^N\right| &\leq C_TN^{-1}.
    \end{aligned}
    \end{equation}

    \subsection{Analysis of the Martingale Terms and of the Initial Condition} \label{martingale_init_sec}
    We now argue that the martingale terms $M_t^{1,N}$, $M_t^{2,N}$, $M_t^{3,N}$ and $\tilde{M}_t^N$,  and the initial conditions $Q_0^N(\xi) = N^{1-\beta}\langle c\sigma (w\cdot \xi) , v_0^N\rangle$ and $P_0^N(\xi) = N^{1-\beta}\langle b\sigma (u\cdot \xi) , \mu_0^N\rangle$ are sufficiently small so that they do not affect the convergence results. Their bounds are summarized in the following lemmas, the proofs of which can be found in section \ref{mc_sec_app} of the Appendix.
    \begin{lemma} \label{M_bound_lemma}
        For the martingale components defined in \eqref{M_t^N_eq} and \eqref{M_tilde_t^N_eq}, there exist uniform constants $C_{T,p}$ and $N_{T,p}$ that do not depend on $N$ and are uniform over $t\in [0,T]$, such that for $N>N_{T,p}$ the following $L_p$ bounds are satisfied
    \begin{equation*}
        \begin{aligned}
            \sup_{t \in [0,T]} \mathds{E}\left[ |M_t^{1,N}|^p+|M_t^{2,N}|^p+|M_t^{3,N}|^p+|\tilde{M}_t^{1,N}|^p\right] &\leq C_{T,p} N^{-\frac{p}{2}}
        \end{aligned}
    \end{equation*}
    \end{lemma}

    Let us now consider the initial conditions in the pre-limit equations \eqref{Q_t^N_eq}, \eqref{P_t^N_eq} corresponding to the terms $Q_0^N(\xi) = N^{1-\beta}\langle c\sigma(w\cdot \xi), v_0^N\rangle$ and $P_0^N(\xi)=N^{1-\beta}\langle b\sigma(u\cdot \xi), \mu_0^N\rangle$. By the central limit theorem for measures and assumption \ref{actor_critic_models_assumption}, we immediately get   
    \begin{equation} \label{initialization_dist_conv_eq}
        \begin{aligned}
           N^{\frac{1}{2}} \langle c\sigma(w\cdot \xi), v_0^N\rangle &\xrightarrow{d} \mathcal{G}(\xi) \sim \mathcal{N}\left(0, \langle \left(c\sigma(w\cdot \xi)\right)^2,v_0^{(0)} \rangle \right)\\
            N^{\frac{1}{2}}\langle b\sigma(u\cdot \xi), \mu_0^N \rangle &\xrightarrow{d} \mathcal{H}(\xi) \sim \mathcal{N}\left( 0,\langle \left(b\sigma(u\cdot \xi)\right)^2,\mu_0^{(0)} \rangle \right).
        \end{aligned}
    \end{equation}

    Here assumption \ref{actor_critic_models_assumption} assures that the above variances are finite. Notice also that by assumption \ref{actor_critic_models_assumption} the summands implied in $\langle c\sigma(w\cdot \xi), v_0^N\rangle$ and $\langle b\sigma(u\cdot \xi), \mu_0^N\rangle$ by definition \eqref{empirical_measures_eq} are i.i.d. and bounded. Hoeffding's inequality then implies that
    \begin{equation}
        \begin{aligned}
           \mathds{P}\left[N^{\frac{1}{2}} \langle c\sigma(w\cdot \xi), v_0^N\rangle \geq c \right] &\leq \exp{\left(\frac{c^2}{C_T}\right)}\\
            \mathds{P}\left[N^{\frac{1}{2}} \langle b\sigma(u\cdot \xi), \mu_0^N\rangle \geq c \right] &\leq \exp{\left(\frac{c^2}{C_T}\right)},
        \end{aligned}
    \end{equation}    
    for any positive constant $c$. Uniform integrability follows and thus the central limit theorem also implies convergence of the moments. For the outputs of the actor and critic networks at the initialization, we thus get the following asymptotic behavior (lemma \ref{prelim_initial_condition_convergence}) and moment bounds (lemma \ref{prelim_initial_condition_L_1_L_2_bounds}):
    \begin{lemma} \label{prelim_initial_condition_convergence}
        For every $\xi \in \mathcal{X}\times \mathcal{A}$,\begin{equation*}
        \begin{aligned}
                N^{\beta-\frac{1}{2}}P_0^N(\xi) &\xrightarrow{d} \mathcal{H}(\xi),\quad, N^{\beta-\frac{1}{2}}Q_0^N(\xi) \xrightarrow{d} \mathcal{G}(\xi),
            \end{aligned}
        \end{equation*}
        where $\mathcal{H}(\xi)$ and $\mathcal{G}(\xi)$ are as defined in \eqref{initialization_dist_conv_eq}.
    \end{lemma}

 We also get the following bounds
    \begin{lemma} \label{prelim_initial_condition_L_1_L_2_bounds}
        Let $p\in \mathds{N}$. There are uniform constants $C_{T,p}$ and $N_{T,p}$ that do not depend on $N$, such that at the initialization, the actor and critic network outputs satisfy the following $L_p$ bounds as $N>N_{T,p}$
        \begin{equation*} \label{prelim_initial_condition_L_1_L_2_bound_eq}
        \begin{aligned}
            \mathds{E}\left[ \max_{\xi \in \mathcal{X}\times \mathcal{A}}|Q_0^N(\xi)|^p + \max_{\xi \in \mathcal{X}\times \mathcal{A}}|P_0^N(\xi)|^p\right] &\leq C_{T,p}N^{p\left(\frac{1}{2}-\beta\right)}.
        \end{aligned}
    \end{equation*}       
    \end{lemma}
    \subsection{Analysis of the Actor Model Output, Policy, and Stationary Measure Convergence} \label{actor_model_policy_stationary_measure_conv_sec}

    We now show that the pre-limit processes $f_t^N$, $g_t^N$, $v_t^N$ and $\mu_t^N$ converge in $L_p$ to the processes $f_t$, $g_t$, $v_t$, $\mu_t$ respectively, where $L_p$ convergence is as defined in definition \ref{L_p_conv_def}.   We start with providing a convergence rate for the exploration policy parameter $\eta_t^N$ that will be necessary for our further analysis.   \begin{lemma}\label{eta_t^N_convergence_error_bound}
        The sequence of exploration rates $\eta_t^N$ satisfies
            $\sup_{t\in [0,T]} \left|\eta_t^N - \eta_t\right| \leq \frac{1}{N}$.
    \end{lemma}

    \begin{proof}
        Considering their definitions in \eqref{learning_rates_def_eq}, \eqref{discrete_to_cont_eq} and \eqref{actor_policy_learning_rates_limit_eq} we write
            \begin{align}
                |\eta_t^N - \eta_t| &= \left|\frac{1}{1+\log^2{\left(1+\frac{\lfloor Nt\rfloor}{N} \right)}} -\frac{1}{1+\log^2{\left(1+t\right)}} \right|\nonumber\\
                & \leq \frac{1}{\left|\left(1+\log^2{\left(1+\frac{\lfloor Nt\rfloor}{N} \right)} \right) \left( 1+\log^2{\left(1+t\right)}\right)\right|}\left| \log^2{\left(1+t\right)}-\log^2{\left(1+\frac{\lfloor Nt\rfloor}{N} \right)}\right|\nonumber\\
                & \leq \left| \log^2{\left(1+t\right)}-\log^2{\left(1+\frac{\lfloor Nt\rfloor}{N} \right)}\right|\\
                & \leq \left|t-\frac{\lfloor Nt\rfloor}{N} \right|\nonumber\\
                &\leq \frac{1}{N},\nonumber
        \end{align}
        because the function $x\rightarrow \log^2{\left( 1+x\right)}$ is Lipschitz continuous with constant $L=1$ for $x\geq 0$.
    \end{proof}
    The next lemma proves the main result of this section. The proof of convergence for the actor model relies on a Taylor expansion of $f_t^N$ around $f_t$. For the policy $g_t^N$ we then make use of the exploration parameter convergence rate derived above. For the stationary measures $\pi^{g_t^N}$ and $\sigma_{\rho_0}^{g_t^N}$ we then use assumption \ref{measure_lipschitz_assumption_eq} on their Lipschitz property with respect to the policy $g_t^N$.
\begin{lemma}\label{g_s^N_error_bound_lemma}
        Let $p\in \mathds{N}$. There is a uniform constant $C$ such that for every $(x,a) \in \mathcal{X} \times \mathcal{A}$ and every $t \in [0,T]$ the sequence of actor model and exploration policy outputs $f_t^N(x,a)$ and $g_t^N(x,a)$ satisfy 
        \begin{align*}
            \left| f_t^N(x,a) - f_t^{(0)}(x,a) \right| &\leq \sum_{a'\in \mathcal{A}}|P_t^N(x,a')-P_t^{(0)}(x,a')|\\
            \left| g_t^N(x,a) - g_t^{(0)}(x,a) \right| &\leq \frac{2}{N} + \sum_{a'\in \mathcal{A}}|P_t^N(x,a')-P_t^{(0)}(x,a')|\\
            \left| \pi^{g_t^N}(x,a) - \pi^{g_t^{(0)}}(x,a) \right|+\left| \sigma_{\rho_0}^{g_t^N}(x,a) - \sigma_{\rho_0}^{g_t^{(0)}}(x,a) \right| &\leq \frac{2}{N} + \sum_{a'\in \mathcal{A}}|P_t^N(x,a')-P_t^{(0)}(x,a')|.  
        \end{align*}
    \end{lemma}

    \begin{proof}
    By the Taylor expansion of the Softmax function, considering $x$ fixed, we obtain 
        \begin{align}
            \Big| f_t^N(x,a) &- f_t^{(0)}(x,a)\Big| = \Big|\text{Softmax}(P_t^N(x,a)) - \text{Softmax}(P_t^{(0)}(x,a))\Big|\nonumber\\
            &= \left| \sum_{a' \in \mathcal{A}} \text{Softmax}(P_t^{N,*}(x,a))\left( \mathds{1}\{a=a'\} - \text{Softmax}(P_t^{N,*}(x,a'))\right) \left( P_t^{N}(x,a') - P_t^{(0)}(x,a')\right)\right|\nonumber\\
            &\leq \sum_{a'\in \mathcal{A}} \left| P_t^{N}(x,a') - P_t^{(0)}(x,a') \right|,\label{softmax_taylor_1_eq}
        \end{align}
where $P_t^{N,*}(x,a')$ lies between $P_t^{N}(x,a')$ and $P_t^{(0)}(x,a')$ for every $a' \in \mathcal{A}$ and we used the fact that the softmax function maps each element of the input vector to a number in the interval $[0,1]$.
        
        From equations \eqref{g_k^N_eq} and \eqref{discrete_to_cont_eq} we get
        \begin{align}
            \left|g_t^N(x,a) - g_t^{(0)}(x,a)\right| &= \left|\frac{1}{|\mathcal{A}|}(\eta_t^N - \eta_t) + (1-\eta_t^N)\left(f_t^N(x,a) - f_t^{(0)}(x,a)\right) - (\eta_t^N-\eta_t)f_t^{(0)}(x,a) \right|\nonumber\\
            &\leq \frac{1}{|\mathcal{A}|}|(\eta_t^N - \eta_t)| + |f_t^N(x,a) - f_t^{(0)}(x,a)| + |\eta_t^N - \eta_t|\nonumber\\
            &\leq \frac{1}{|\mathcal{A}| N} + \left|\text{Softmax}(P_t^N(x,a)) - \text{Softmax}(P_t^{(0)}(x,a))\right| + \frac{1}{N},\label{g_t_error_bound_eq}
        \end{align}
         where we used lemma \ref{eta_t^N_convergence_error_bound} and the fact that $\eta_t^N, f_t^{(0)}(x,a) \in [0,1]$. Substituting \eqref{softmax_taylor_1_eq} into \eqref{g_t_error_bound_eq} gives the bound on $\left| g_t^N(x,a) - g_t^{(0)}(x,a)\right|$. The error bounds for the stationary processes then follow from assumption \ref{Markov_Chain_ergodicity_assumption}.
    \end{proof}

    \subsection{Analysis of the Empirical Measure Process Convergence}\label{SS:MeasureProcessConv}

    In this section we will show that the empirical measures $v_t^N$ and $\mu_t^N$ converge to the initialization measures $v_t$ and $\mu_t$ respectively, in the sense of inner productg convergence with a class of test functions to be specified. We start by observing that $C_\xi^h$ as defined in \eqref{C_xi^f_eq} is bounded. This follows from assumption \ref{actor_critic_models_assumption} on the boundedness of $\sigma$, assumption \ref{MDP_assumptions} on the finite size of $\mathcal{X}\times \mathcal{A}$, and lemma \ref{parameter_bounds_apriori_lemma}.   \begin{lemma}\label{C_bound_lemma}
         For any $h\in C_b^2(\mathbb{R}^{1+d'})$ and any $\xi \in \mathcal{X}\times\mathcal{A}$, there is a  constant $C_T$ that does not depend on $N$ such that
      \begin{equation*}
            |C_\xi^h| \leq C_T.
        \end{equation*}

        Consequently, we get
        \begin{equation*}
            \begin{aligned}
                \sup_{t \in [0,T]}\left|\langle C_\xi^h, v_t^N \rangle \right| +\sup_{t \in [0,T]}\left|\langle C_\xi^h, \mu_t^N \rangle \right| \leq C_T.
            \end{aligned}
        \end{equation*}
    \end{lemma}

    We now derive a bound on $v_t^N$ and $\mu_t^N$ relative to their initializations $v_0^N$ and $\mu_0^N$ respectively.  \begin{lemma}\label{empirical_measures_stationarity_lemma}
    Let  $p\in \mathds{N}$. For an arbitrary but fixed $h \in C_b^2(\mathbb{R})$, there exist uniform constants $C_T$ and $C_{T,p}$ that do not depend on $N$ and are uniform over $t\in[0,T]$, such that the following bounds hold for the empirical measure processes $v_t^N$ and $\mu_t^N$
    \begin{equation*}
        \begin{aligned}
             \sup_{t \in [0,T]}|\langle h, \mu_t^N -\mu_0^N \rangle| +\sup_{t \in [0,T]}|\langle h, v_t^N -v_0^N \rangle| &\leq C_T\\
            \sup_{t \in [0,T]}\mathds{E}\left[|\langle h, \mu_t^N -\mu_0^N \rangle|^p\right] +\sup_{t \in [0,T]}\mathds{E}\left[|\langle h, v_t^N -v_0^N \rangle|^p\right] &\leq C_{T,p}N^{p(\beta-1)}.
        \end{aligned}
    \end{equation*}
        
    \end{lemma}

    \begin{proof}
        The constants $C_T$ and $C_{T,p}$ may change throughout the proof but remain independent of $N$ and uniform over $t\in[0,T]$. The almost sure bound follows immediately from the pre-limit equations \eqref{mu_t^N_eq} and \eqref{v_t^N_eq}, equation \eqref{Q_P_absolute_bound_eq}, assumption \ref{MDP_assumptions} on the boundedness of $r$ and lemma \ref{C_bound_lemma}.
        To show the bound in expectation, we raise \eqref{v_t^N_eq} to the power of $p$ and use the power-mean inequality to get
        \begin{equation*}
            \begin{aligned}
                |\langle h, v_t^N &-v_0^N \rangle|^p = \left| \frac{\alpha}{N^{2-\beta}} \sum_{k=0}^{\lfloor Nt\rfloor -1} \left(r(\xi_k)+\gamma Q_k^N(\xi_{k+1})- Q_k^N(\xi_k)\right) \langle C_{\xi_k}^h, v_k^N \rangle+ R_{v,t}^N \right|^p\\
                &\leq \frac{C_{T,p}}{N^{p(2-\beta)}}\left|  \sum_{k=0}^{\lfloor Nt\rfloor -1} \left(r(\xi_k)+\gamma Q_k^N(\xi_{k+1})- Q_k^N(\xi_k)\right) \langle C_{\xi_k}^h, v_k^N \rangle \right|^p + C_{T,p}\left| R_{v,t}^N\right|^p\\
                &\leq \frac{C_{T,p}}{N^{p(2-\beta)}}\lfloor Nt \rfloor^{p-1} \sum_{k=0}^{\lfloor Nt\rfloor -1} \left(r(\xi_k)+\gamma Q_k^N(\xi_{k+1})- Q_k^N(\xi_k)\right)^p + C_{T,p} \left|R_{v,t}^N\right|^p.
            \end{aligned}
        \end{equation*}

        Taking expectations and using lemma \ref{Q_P_L_4_Bound}, assumption \ref{MDP_assumptions} on the boundedness of $r$ and the bound \eqref{R_{v,t}^N_bounds_eq} gives the result. 
    \end{proof}

    \begin{remark} \label{B_bound_remark}
         When studying the convergence behavior of the actor and critic network outputs, we will use this lemma for $h=B$ where $B$ is as defined in \eqref{B_def_eq}. Notice that by assumption \ref{actor_critic_models_assumption} on the boundedness of $\sigma$ and its derivatives and lemma \ref{parameter_bounds_apriori_lemma}, we get $B\in C_b^2(\mathbb{R})$.
    \end{remark}

    Finally, using the same arguments as in section \ref{martingale_init_sec}, for any bounded function $h$ the central limit theorem and uniform integrability give the following bounds.
    \begin{lemma} \label{empirical_measure_init_clt_lemma}
        Let $h \in C_b^{2}$ and $p\in \mathds{N}$. Then there exists a constant $C_{T,p}$ that does not depend on $N$ such that the following $L_p$ bound holds
        \begin{equation*}
            \begin{aligned}
                \mathds{E}\left[ \left| \langle h, v_0^N-v_0^{(0)}\rangle \right|^p\right]+\mathds{E}\left[ \left| \langle h, \mu_0^N-\mu_0^{(0)}\rangle \right|^p\right] &\leq C_{T,p} N^{-\frac{p}{2}}.
            \end{aligned}
        \end{equation*}
    \end{lemma}
    
Lemma \ref{empirical_measures_stationarity_lemma} and lemma \ref{empirical_measure_init_clt_lemma} then yield the statement of proposition \ref{empirical_measures_conv_prop}. 

\subsection{Proof of Theorem \ref{leading_order_conv_th}: Convergence of the Actor and Critic Network} \label{leading_order_conv_sec}

Using the tools from the previous two sections, we are now ready to prove theorem \ref{leading_order_conv_th}. We start by studying the process $|Q_t^N-Q_t|$. The various constants used in the proof may change at each step, but without affecting their independence and uniformity over the parameters specified at first use. Note that using the pre-limit equation \eqref{Q_t^N_eq} and the formula for $Q_t^{(0)}$ in theorem \ref{leading_order_conv_th} we can write 
        \begin{align}
            &Q_t^N(x,a) - Q_t^{(0)}(x,a) =\nonumber\\       &=\sum_{\substack{k_1,k_2,k_3,k_4 \in \{0,1\} \\ k_1+k_2+k_3+k_4\leq 3}} \alpha \int_{0}^t \sum_{(x',a'),(x'',a'') \in \mathcal{X}\times \mathcal{A}} \Bigg[ r(x',a') + \gamma Q_s^{(0)}(x'',a'') - Q_s^{(0)}(x',a') \Bigg]^{k_1}\nonumber\\
            &\hspace{10pt} \Bigg[\gamma \left(Q_s^N(x'',a'') - Q_s^{(0)}(x'',a'')\right) - \left(Q_s^N(x',a') - Q_s^{(0)}(x',a')\right) \Bigg]^{1-k_1}\nonumber\\
            &\hspace{10pt} \Bigg[\langle B_{(x,a),(x',a')}, v_0^{(0)} \rangle \Bigg]^{k_2} \Bigg[\langle B_{(x,a),(x',a')}, v_s^N -v_0^N\rangle + \langle B_{(x,a),(x',a')}, v_0^N -v_0^{(0)}\rangle \Bigg]^{1-k_2}\Bigg[ g_s^{(0)}(x'',a'')\Bigg]^{k_3}\nonumber\\
            &\hspace{10pt}  \Bigg[g_s^N(x'',a'')-g_s^{(0)}(x'',a'') \Bigg]^{1-k_3} \Bigg[ \pi^{g_s^{(0)}}(x',a')\Bigg]^{k_4} \Bigg[\pi^{g_s^N}(x',a')-\pi^{g_s^{(0)}}(x',a') \Bigg]^{1-k_4}p(x'',a''|x',a')ds\nonumber\\
            &\hspace{10pt} +N^{1-\beta}\langle c\sigma(w\cdot (x,a)), v_0^N\rangle + M_t^{1,N}((x,a)) + M_t^{2,N}((x,a))+ M_t^{3,N}((x,a))+ R_{Q,t}^N, \label{Q_t^N-Q_t_formula_eq}
        \end{align}
    which gives the following bound
        \begin{align}
            &\left|Q_t^N(x,a) - Q_t^{(0)}(x,a)\right|^p \leq\nonumber\\           &\leq\sum_{\substack{k_1,k_2,k_3,k_4 \in \{0,1\} \\ k_1+k_2+k_3+k_4\leq 3}} \alpha C_{p} \int_{0}^t \sum_{(x',a'),(x'',a'') \in \mathcal{X}\times \mathcal{A}} \Bigg[\left| r(x',a') + \gamma Q_s^{(0)}(x'',a'') - Q_s^{(0)}(x',a')\right|^p \Bigg]^{k_1}\nonumber\\
            &\hspace{10pt} \Bigg[ \left|\gamma \left(Q_s^N(x'',a'') - Q_s^{(0)}(x'',a'')\right)\right|^p + \left|\left(Q_s^N(x',a') - Q_s^{(0)}(x',a')\right)\right|^p \Bigg]^{1-k_1}\nonumber\\
            &\hspace{10pt} \Bigg[\left|\langle B_{(x,a),(x',a')}, v_0^{(0)} \rangle \right|^p \Bigg]^{k_2} \Bigg[\left|\langle B_{(x,a),(x',a')}, v_s^N -v_0^N\rangle \right|^p +  \left|\langle B_{(x,a),(x',a')}, v_0^N -v_0^{(0)}\rangle \right|^p\Bigg]^{1-k_2}\Bigg[ \left| g_s^{(0)}(x'',a'')\right|^p\Bigg]^{k_3}\nonumber\\
            &\hspace{10pt}  \Bigg[\left|g_s^N(x'',a'')-g_s^{(0)}(x'',a'') \right|^p \Bigg]^{1-k_3} \Bigg[ \left|\pi^{g_s^{(0)}}(x',a')\right|^p\Bigg]^{k_4} \Bigg[\left|\pi^{g_s^N}(x',a')-\pi^{g_s^{(0)}}(x',a')\right|^p \Bigg]^{1-k_4}p(x'',a''|x',a')ds\nonumber\\
            &\hspace{10pt} +C_p \left|N^{1-\beta}\langle c\sigma(w\cdot (x,a)), v_0^N\rangle\right|^p + C_{p}\left|M_t^{1,N}((x,a))\right|^p + C_p\left|M_t^{2,N}((x,a))\right|^p+ C_p\left|M_t^{3,N}((x,a))\right|^p\nonumber\\
            &\hspace{10pt} + C_p\left|R_{Q,t}^N\right|^p,\label{Q_t^N-Q_t_L_p_power_eq}
        \end{align}
    for some constant $C_p$ that depends on $p$ but is independent of $N$ and uniform over $t\in[0,T]$. We note that by
    our boundedness assumption of $Q_t$
    and assumption \ref{MDP_assumptions} on the boundedness of $r$, we have that 
    \begin{equation} \label{leading_order_conv_sec_b1_eq}
        \max_{(x',a'),(x'',a'') \in \mathcal{X} \times \mathcal{A}} | r(x',a') + \gamma Q_s^{(0)}(x'',a'') - Q_s^{(0)}(x',a')| \leq C_T,
    \end{equation}
    where $C_T$ is a constant that is independent of $N$ and uniform over $t\in[0,T]$. Our assumption on the boundedness of $h$ and assumption \ref{actor_critic_models_assumption} on the measure $v_0$ imply the existence of a constant $C$ that is also independent of $N$ and uniform over $t\in[0,T]$ such that
    \begin{equation} \label{leading_order_conv_sec_b2_eq}
        \max_{(x,a),(x',a') \in \mathcal{X}\times\mathcal{A}}\left|\langle B_{(x,a),(x',a')}, v_0^{(0)}\rangle \right| \leq C.
    \end{equation}
    
    We also recall that by their definition as distributions,
    \begin{align}
           \max_{(x'',a'') \in \mathcal{X} \times \mathcal{A}} |g_s^{(0)}(x'',a'')| &\leq 1\nonumber\\
            \max_{(x',a') \in \mathcal{X} \times \mathcal{A}}|\pi^{g_s^{(0)}}(x',a')| &\leq 1\label{leading_order_conv_sec_b3_eq}
    \\
            \max_{(x',a'),(x'',a'') \in \mathcal{X} \times \mathcal{A}}|p(x'',a''|x',a')| &\leq 1,\nonumber
        \end{align}
which also gives    \begin{equation}\label{leading_order_conv_sec_b4_eq}
        \begin{aligned}
            \max_{(x'',a'') \in \mathcal{X} \times \mathcal{A}} \left|g_s^N(x'',a'')-g_s^{(0)}(x'',a'') \right| &\leq  1\\
            \max_{(x',a') \in \mathcal{X} \times \mathcal{A}} \left|\pi^{g_s^N}(x',a')-\pi^{g_s^{(0)}}(x',a') \right| &\leq 1.
        \end{aligned}
    \end{equation}
   
We will be using those bound along with the bounds of lemma \ref{g_s^N_error_bound_lemma} depending on convenience in each case. We now study the terms of the summation in \eqref{Q_t^N-Q_t_formula_eq}. We group the terms as follows:\\

    \textbullet $k_1 = 0$, $k_2,k_3,k_4 \in \{0,1\}$\\

    In this case, only the term $\left|\gamma \left(Q_s^{(0)}(x'',a'') - Q_s^N(x'',a'')\right)\right|^p + \left|\left(Q_s^{(0)}(x',a') - Q_s^N(x',a')\right)\right|^p$ is relevant as by lemma \ref{empirical_measures_stationarity_lemma}, remark \ref{B_bound_remark}, lemma \ref{empirical_measure_init_clt_lemma}, and equations \eqref{leading_order_conv_sec_b2_eq}, \eqref{leading_order_conv_sec_b3_eq} and \eqref{leading_order_conv_sec_b4_eq}, all other factors are uniformly bounded. This implies that all corresponding terms in the summation are bounded in expectation by 
    \begin{equation*}C_{T,p}\int_0^t\mathds{E}\left[\max_{(x',a') \in \mathcal{X}\times \mathcal{A}} |Q_s^N(x',a')-Q_s^{(0)}(x',a')|^p\right]ds,
    \end{equation*}
    where $C_{T,p}$ is a constant that is independent of $N$ and uniform over $t\in[0,T]$.

    \textbullet $k_1 = 1$, $k_2 = 0$, $k_3,k_4 \in \{0,1\}$\\

    In this case we can use lemma \ref{empirical_measures_stationarity_lemma}, remark \ref{B_bound_remark}, lemma \ref{empirical_measure_init_clt_lemma}, and equations \eqref{leading_order_conv_sec_b1_eq}, \eqref{leading_order_conv_sec_b3_eq}, and \eqref{leading_order_conv_sec_b4_eq} to see that the corresponding terms in the summation are bounded in expectation by 
    \begin{equation*}
        C_{T,p}N^{p(\beta-1)} + C_{T,p}N^{-{p}/{2}}.
    \end{equation*}

    \textbullet $k_1 = 1$, $k_2 = 1$, $k_3=0$, $k_4 \in \{0,1\}$\\

    In this case we use lemma \ref{empirical_measures_stationarity_lemma}, remark \ref{B_bound_remark}, lemma \ref{empirical_measure_init_clt_lemma} and equations \eqref{leading_order_conv_sec_b1_eq} and \eqref{leading_order_conv_sec_b2_eq}, as well as using the second bound of lemma \ref{g_s^N_error_bound_lemma} and the second bound in equation \eqref{leading_order_conv_sec_b4_eq} to argue that the corresponding terms are bounded in expectation by 
    \begin{equation*}
        {C_{T,p}}{N^{-p}} + C_{T,p}\int_0^t\mathds{E}\left[\max_{(x',a') \in \mathcal{X}\times \mathcal{A}} |P_s^N(x',a')-P_s^{(0)}(x',a')|^p\right]ds.
    \end{equation*}

    \textbullet $k_1 = 1$, $k_2 = 1$, $k_3=1$, $k_4=0$\\

    In this case we use lemma \ref{empirical_measures_stationarity_lemma}, remark \ref{B_bound_remark}, lemma \ref{empirical_measure_init_clt_lemma}, equations \eqref{leading_order_conv_sec_b1_eq}, \eqref{leading_order_conv_sec_b2_eq},  the third bound of lemma \ref{g_s^N_error_bound_lemma} and the first bound in equation \eqref{leading_order_conv_sec_b3_eq} to argue that the corresponding terms are also bounded in expectation by
    \begin{equation*}
        C_{T,p}{N^{-p}} + C_{T,p}\int_0^t\mathds{E}\left[\max_{(x',a') \in \mathcal{X}\times \mathcal{A}} |P_s^N(x',a')-P_s^{(0)}(x',a')|^p\right]ds.
    \end{equation*}

    We have exhausted all possible cases for $k_1, k_2, k_3, k_4$.
    It remains to bound the martingale terms and the remainder, for which we use lemma \ref{M_bound_lemma} and the $L_1$ bound in \eqref{R_{Q,t}^N_bounds_eq}. Considering only the leading order terms, we get that for $t\in \left[ 0,T\right]$
     \begin{equation} \label{Q_t^N_error_gronwall_setup_eq_large_t}
    \begin{aligned}
        \mathds{E}\left[ \max_{(x,a) \in \mathcal{X}\times\mathcal{A}}|Q_t^N(x,a)-Q_t^{(0)}(x,a)|^p\right] &\leq  C_{T,p}\int_0^t\mathds{E}\left[\max_{(x',a') \in \mathcal{X}\times \mathcal{A}} |Q_s^N(x',a')-Q_s^{(0)}(x',a')|^p\right]ds\\
        &+ C_{T,p}\int_0^t\mathds{E}\left[\max_{(x',a') \in \mathcal{X}\times \mathcal{A}} |P_s^N(x',a')-P_s^{(0)}(x',a')|^p\right]ds\\
        &+C_{T,p}N^{p(\beta-1)}+C_{T,p}N^{p\left(\frac{1}{2}-\beta\right)}.
    \end{aligned}
    \end{equation}

    Note at this point that the term $C_{T,p}N^{p\left(\frac{1}{2}-\beta\right)}$ in the above bound is due to the random initialization and lemma \ref{prelim_initial_condition_L_1_L_2_bounds}, while the error term $C_{T,p}N^{p(\beta-1)}$ is due to the convergence error of the empirical measures $v_t^N$ and $\mu_t^N$ as shown in lemma \ref{empirical_measures_stationarity_lemma}. Similarly, the process $P_t^N - P_t^{(0)}$ is written as 
        \begin{align}
            &P_t^N(x,a) - P_t^{(0)}(x,a) =\nonumber\\
            &=\sum_{\substack{k_1,k_2,k_3,k_4 \in \{0,1\} \\ k_1+k_2+k_3+k_4\leq 3}}  \int_{0}^t  \sum_{(x',a',a'') \in \mathcal{X}\times \mathcal{A}  \times \mathcal{A}} \zeta_s \Bigg[ Q_s^{(0)}(x',a') \Bigg]^{k_1} \Bigg[Q_s^N(x',a') - Q_s^{(0)}(x',a') \Bigg]^{1-k_1}\nonumber\\
            &\hspace{10pt} \Bigg[\langle B_{(x,a),(x',a')}, v_0^{(0)} \rangle + \langle B_{(x,a),(x',a'')}, v_0^{(0)} \rangle \Bigg]^{k_2} \Bigg[\langle B_{(x,a),(x',a')}, v_s^N -v_0^N\rangle + \langle B_{(x,a),(x',a'')}, v_s^N -v_0^N\rangle\nonumber\\
            &\hspace{30pt} + \langle B_{(x,a),(x',a')}, v_0^N -v_0^{(0)}\rangle + \langle B_{(x,a),(x',a'')}, v_0^N -v_0^{(0)}\rangle \Bigg]^{1-k_2}\nonumber\\
            &\hspace{10pt} \Bigg[ f_s^{(0)}(x',a'')\Bigg]^{k_3} \Bigg[f_s^N(x',a'')-f_s^{(0)}(x',a'') \Bigg]^{1-k_3} \Bigg[ \sigma_{\rho_0}^{g_s^{(0)}}(x',a')\Bigg]^{k_4} \Bigg[\sigma_{\rho_0}^{g_s^N}(x',a')-\sigma_{\rho_0}^{g_s^{(0)}}(x',a') \Bigg]^{1-k_4}ds\nonumber\\
            &\hspace{10pt} +N^{1-\beta}\langle b\sigma(u\cdot (x,a)), v_0^N\rangle + \tilde{M}_t^{N}((x,a)) + R_{P,t}^N, \label{P_t^N-P_t_formula_eq}
        \end{align}
    and can be shown to satisfy the same bound, i.e.,
    \begin{equation} \label{P_t^N_error_gronwall_setup_eq_large_t}
    \begin{aligned}
        \mathds{E}\left[ \max_{(x,a) \in \mathcal{X}\times\mathcal{A}}|P_t^N(x,a)-P_t^{(0)}(x,a)|^p\right] &\leq  C_{T,p}\int_0^t\mathds{E}\left[\max_{(x',a') \in \mathcal{X}\times \mathcal{A}} |Q_s^N(x',a')-Q_s^{(0)}(x',a')|^p\right]ds\\
        &+ C_{T,p}\int_0^t\mathds{E}\left[\max_{(x',a') \in \mathcal{X}\times \mathcal{A}} |P_s^N(x',a')-P_s^{(0)}(x',a')|\right]ds\\
        &+C_{T,p}N^{p(\beta-1)}+C_{T,p}N^{p\left(\frac{1}{2}-\beta\right)}
    \end{aligned}
    \end{equation}

    Adding \eqref{Q_t^N_error_gronwall_setup_eq_large_t} and \eqref{P_t^N_error_gronwall_setup_eq_large_t} and using the  Grönwall Inequality we obtain
    \begin{equation} \label{error_bounds_1st_order_eq}
    \begin{aligned}
        \sup_{t \in \left[0,T\right]}\mathds{E}\left[\max_{(x',a') \in \mathcal{X}\times \mathcal{A}} |Q_t^N(x',a')-Q_t^{(0)}(x',a')|^p\right] & \leq C_{T,p}N^{p(\beta-1)} + C_{T,p}N^{p\left(\frac{1}{2}-{\beta}\right)}\\
        \sup_{t \in \left[0,T\right]}\mathds{E}\left[\max_{(x',a') \in \mathcal{X}\times \mathcal{A}} |P_t^N(x',a')-P_t^{(0)}(x',a')|^p\right] & \leq C_{T,p}N^{p(\beta-1)} + C_{T,p}N^{p\left(\frac{1}{2}-{\beta}\right)},
    \end{aligned}
    \end{equation}
    which proves theorem \ref{leading_order_conv_th}. Substitution of those bounds in lemma \ref{g_s^N_error_bound_lemma} proves proposition \ref{leading_order_conv_prop}. Again, the term $C_{T,p}N^{p\left(\frac{1}{2}-\beta\right)}$ in the above bound is due to the random initialization and lemma \ref{prelim_initial_condition_L_1_L_2_bounds}, while the error term $C_{T,p}N^{p(\beta-1)}$ is due to the convergence error of the empirical measures $v_t^N$ and $\mu_t^N$ as shown in lemma \ref{empirical_measures_stationarity_lemma}.

     \section{Asymptotics of the Leading Order Limit Equations} \label{asymptotics_sec}

    \subsection{Critic Convergence} \label{critic_conv_sec}
   We now prove convergence of the limit critic network output to the solution of Bellman equation, which asserts that for every state-action pair $(x,a)\in \mathcal{X}\times\mathcal{A}$, the value function $V^{g_t^{(0)}}$ corresponding to the optimal policy satisfies 
    \begin{equation} \label{V^gt_Bellman_eq}
    V^{g_t^{(0)}}(x,a) = r(x,a) + \gamma \sum_{(x',a')\in \mathcal{X}\times \mathcal{A}}V^{g_t^{(0)}}(x',a')g_t^{(0)}(x',a')p(x'|x,a).
    \end{equation}

    We define the difference
    \begin{equation} \label{phi_t_def_eq}
        \phi_t = Q_t^{(0)}-V^{g_t^{(0)}}.
    \end{equation}

    Using the limit equation (\ref{Q_t_P_t_def_eq}) and the Bellman equation (\ref{V^gt_Bellman_eq}) we obtain
        \begin{align}
            \frac{d}{dt}\phi_t(x,a) &= \frac{d}{dt}Q_t^{(0)}(x,a) - \frac{d}{dt}V^{g_t^{(0)}}(x,a) \nonumber\\
            &= -\alpha \sum_{(x',a')\in \mathcal{X}\times\mathcal{A}} A_{x,a,x',a'}\pi^{g_t^{(0)}}(x',a')\phi_t(x',a')\nonumber\\
            &+ \alpha \gamma \sum_{(x',a') \in \mathcal{X}\times\mathcal{A}} A_{x,a,x',a'}\pi^{g_t^{(0)}}(x',a') \sum_{(x'',a'')\in \mathcal{X}\times\mathcal{A}} \phi_t(x'',a'')g_t^{(0)}(x'',a'')p(x''|x',a')\nonumber\\
            &- \frac{d}{dt}V^{g_t^{(0)}}(x,a)\\
            & = -\alpha A(\pi^{g_t^{(0)}} \odot \phi_t) + \alpha \gamma A(\pi^{g_t^{(0)}} \odot \Gamma_t) - \frac{d}{dt}V^{g_t^{(0)}}(x,a),\label{d_phi_t_eq}
        \end{align}
where $\odot$ denotes the element-wise multiplication and we defined
    \begin{equation} \label{Gamma_t_def_eq}
        \Gamma_t(x',a') = \sum_{(x'',a'') \in \mathcal{X}\times\mathcal{A}} \phi_t(x'',a'')g_t^{(0)}(x'',a'')p(x''|x',a').
    \end{equation}
    
    Before proceeding further, we will now characterize the derivative $\frac{d}{dt}V^{g_t^{(0)}}$ and provide bounds for some of the involved terms.
    \begin{lemma} \label{dV^g_t^{(0)}_lemma}
        For a fixed state-action pair $(x,a)\in \mathcal{X}\times\mathcal{A}$, the time derivative of $V^{g_t^{(0)}}$ satisfies
        \begin{equation}
            \begin{aligned}
               \frac{d}{dt} V^{g_t^{(0)}}(x,a) =  \zeta_t E_t A\phi_t\odot \sigma^{g_t^{(0)}} - \zeta_t E_t AV^{g_t^{(0)}}\odot \sigma^{g_t^{(0)}}-\zeta_tE_tAf_t\odot \Delta_t,
            \end{aligned}
        \end{equation}
    \end{lemma}
    where the matrix $E_t$ and the vector $\Delta_t$ are given by
    \begin{equation} \label{E_t_D_t_def_eq}
        \begin{aligned}
            \Delta_t(x,a) &= \sum_{a''\in \mathcal{A}}Q_t^{(0)}(x,a'')\sigma^{g_t^{(0)}}(x,a'')\\
            E_{t,(x_0,a_0),(x,a)} &= \gamma \sum_{x''\in \mathcal{X}}\sum_{(x',a')\in \mathcal{X}\times\mathcal{A}}\sigma_{x_0}^{g_t^{(0)}}(x',a')V^{g_t^{(0)}}(x',a') (1-\eta_t)\frac{f_t(x',a')}{g_t^{(0)}(x',a')}\times\\
            &\hspace{3cm}\times\left( \mathds{1}\left\{ (x,a)=(x',a')\right\}-f_t(x,a)\right)p(x''|x_0,a_0),
        \end{aligned}
    \end{equation}
      and the matrix $E$ satisfies the bounds
    \begin{equation*}
        \|E_t\|_2 \leq \|E_t\|_F \leq r_{\text{max}} |\mathcal{X}\times \mathcal{A}|\frac{\gamma}{1-\gamma}.
    \end{equation*}
    \begin{proof}
        By the chain rule we have that 
        \begin{equation} \label{V^g_t^{(0)}_chain_rule_eq}
            \begin{aligned}
                \frac{d}{dt}V^{g_t^{(0)}} = \frac{\partial V^{g_t^{(0)}}}{\partial P_t^{(0)}} \cdot \frac{d}{d_t}P_t^{(0)}.
            \end{aligned}
        \end{equation}

 Using the limit equation (\ref{Q_t_P_t_def_eq}) we get
    \begin{equation} \label{dP_t_eq}
        \begin{aligned}
            \frac{d}{d_t}P_t^{(0)}(x,a) &= \sum_{(x',a')\in \mathcal{X}\times\mathcal{A}}\zeta_tQ_t^{(0)}(x',a')\left[ A_{x,a,x',a'}-\sum_{a''\in \mathcal{A}}f_t(x',a'')A_{x,a,x',a''}\right]\sigma^{g_t^{(0)}}(x',a'')\\
            &= \zeta_t A Q_t^{(0)}\odot \sigma^{g_t^{(0)}} - \zeta_tAf_t \odot \Delta_t\\
            &= \zeta_t A\phi_t\odot \sigma^{g_t^{(0)}} - \zeta_t AV^{g_t^{(0)}}\odot \sigma^{g_t^{(0)}}-\zeta_tAf_t\odot \Delta_t.
        \end{aligned}
    \end{equation}

    Using the Policy Gradient Theorem \cite{pgthm}  we compute
    \begin{equation} \label{d_V_d_P_eq_1}
        \begin{aligned}
            \frac{\partial V^{g_t^{(0)}}(x_0)}{\partial P_t^{(0)}}(x,a) &= \sum_{(x',a')\in \mathcal{X}\times \mathcal{A}} \sigma_{x_0}^{g_t^{(0)}}(x',a')V^{g_t^{(0)}}(x',a')\frac{\partial}{\partial P_t^{(0)}(x,a)}\log{g_t^{(0)}(x',a')}\\
            &= \sum_{(x',a')\in \mathcal{X}\times\mathcal{A}}\sigma_{x_0}^{g_t^{(0)}}(x',a')V^{g_t^{(0)}}(x',a') (1-\eta_t)\frac{f_t(x',a')}{g_t^{(0)}(x',a')}\left( \mathds{1}\left\{ (x,a)=(x',a')\right\}-f_t(x,a)\right).
        \end{aligned}
    \end{equation}

    Since the Q-function $V^{g_t^{(0)}}(x_0,a_0)$ is obtained from the value function $V^{g_t^{(0)}}(x_0)$ as
    \begin{equation}
        V^{g_t^{(0)}}(x_0,a_0) = r(x_0,a_0)+\gamma \sum_{x'\in \mathcal{X}}V^{g_t^{(0)}}(x')p(x'|x_0,a_0),
    \end{equation}
    equation (\ref{d_V_d_P_eq_1}) yields 
    \begin{equation} \label{d_V_d_P_eq_2}
        \begin{aligned}
            &\frac{\partial V^{g_t^{(0)}}(x_0, a_0)}{\partial P_t^{(0)}}(x,a) = \gamma \sum_{x''\in \mathcal{X}} \frac{\partial V^{g_t^{(0)}}(x'')}{\partial P_t^{(0)}(x,a)} p(x''|x_0,a_0)\\
            &= \gamma \sum_{x''\in \mathcal{X}}\sum_{(x',a')\in \mathcal{X}\times\mathcal{A}}\sigma_{x_0}^{g_t^{(0)}}(x',a')V^{g_t^{(0)}}(x',a') (1-\eta_t)\frac{f_t(x',a')}{g_t^{(0)}(x',a')}\left( \mathds{1}\left\{ (x,a)=(x',a')\right\}-f_t(x,a)\right)p(x''|x_0,a_0)\\
            &=: E_{t,(x_0,a_0),(x,a)}.
        \end{aligned}
    \end{equation}

    We now derive bounds on the matrix $E_t$ defined above. Notice first that by its definition in (\ref{state_action_value_function_def_eq}) we immediately get the uniform bound
    \begin{equation} \label{V_bound_eq}
        V^{\pi}(x,a) \leq \frac{r_{\text{max}}}{1-\gamma},
    \end{equation}
    where $r_{\text{max}}$ is the maximum reward over the state-action pairs that is bounded by assumption \ref{MDP_assumptions}. 

    Further notice that by their definitions we obtain
    \begin{equation}
        \begin{aligned}
            (1-\eta_t)\frac{f_t(x,a)}{g_t^{(0)}(x,a)} \in (0,1), \quad \mathds{1}\{(x,a)=(x',a')\}-f_t(x,a) \in (-1,1),
        \end{aligned}
    \end{equation}
 which along with equation for $E_t$ (\ref{d_V_d_P_eq_2}) gives the element-wise upper bound.
    \begin{equation}
        \begin{aligned}
            |E_{t,(x_0,a_0),(x,a)}| \leq r_{\text{max}}\frac{\gamma}{1-\gamma}.
        \end{aligned}
    \end{equation}

   The bound on the norms of matrix $E_{t}$ follows.
    \end{proof}

    We now define the process 
    \begin{equation} \label{Y_t_def_eq}
        Y_t = \frac{1}{2}\phi_t^TA^{-1}\phi_t,
    \end{equation}
    whose time derivative is
    \begin{equation} \label{d_Y_t_eq}
        \begin{aligned}
            \frac{d}{d_t}Y_t &= \phi_t^T A^{-1} \frac{d}{d_t}\phi_t\\
            &= -\alpha \phi_t^T \pi^{g_t^{(0)}} \odot \phi_t + \alpha \gamma \phi_t^T \pi^{g_t^{(0)}} \odot \Gamma_t  - \zeta_t \phi_t^TA^{-1}E_tA \phi_t \odot \sigma^{g_t^{(0)}}\\
            &+ \zeta_t \phi_t^TA^{-1}E_tAV^{g_t^{(0)}} \odot \sigma^{g_t^{(0)}} - \zeta_t\phi_tA^{-1}E_tAf_t \odot \Delta_t.
        \end{aligned}
    \end{equation}

    We will now analyze each term on the the right hand side of (\ref{d_Y_t_eq}). For convenience, we denote 
    \begin{equation*}
        U = \frac{\lambda_{\text{max}}(A)}{\lambda_{\text{min}}(A)}\cdot \frac{\gamma}{1-\gamma}|\mathcal{X}\times \mathcal{A}|,
    \end{equation*}
    where $\lambda_{\text{min}}(A)$ and $\lambda_{\text{max}}(A)$ denote the smallest and largest eigenvalue of $A$ respectively. 
    The following bound is derived in section 5.1 in \cite{nac}.
    \begin{equation}
        \begin{aligned}
            |\phi_t^T \pi^{g_t^{(0)}} \odot \Gamma_t| &\leq \phi_t^2 \cdot \pi^{g_t^{(0)}},
        \end{aligned}
    \end{equation}
    where $\phi_t^2 = \phi_t \odot \phi_t$ is the element-wise square of $\phi_t$. We obtain the bound
    \begin{equation} \label{dY_t_bound_1_eq}
        \alpha \gamma \phi_t^T \pi^{g_t^{(0)}} \odot \Gamma_t \leq \alpha\gamma \phi_t^2 \odot \pi^{g_t^{(0)}} 
    \end{equation}

    One can also easily verify
    \begin{equation} \label{dY_t_bound_2_eq}
        \phi_t^T \pi^{g_t^{(0)}} \odot \phi_t = \phi_t^2 \cdot \pi^{g_t^{(0)}}
    \end{equation}

     Notice at this point that 
    \begin{equation*}
        \|A^{-1}\|_2\|E_t\|_2\|A\|_2 \leq \frac{\lambda_{\text{max}}(A)}{\lambda_{\text{min}}(A)}\cdot \frac{\gamma}{1-\gamma}|\mathcal{X}\times \mathcal{A}| =U,
    \end{equation*}
    which follows from lemma \ref{dV^g_t^{(0)}_lemma} and the fact that $A$ is positive definite. We will use this to bound the remaining terms on the right hand side of (\ref{d_Y_t_eq}). Specifically, we bound
        \begin{align}
           \left| \zeta_t \phi_t^TA^{-1}E_tA \phi_t \odot \sigma^{g_t^{(0)}} \right| &\leq \zeta_t \|\phi_t\|_2^2\|A^{-1}\|_2\|A\|_2\|E_t\|_2 \leq U\zeta_t\|\phi_t\|_2^2\nonumber\\
           &\leq U\frac{1}{\lambda_{\text{min}}(A)}\zeta_t Y_t\nonumber\\
           \left|\zeta_t \phi_t^TA^{-1}E_tAV^{g_t^{(0)}} \odot \sigma^{g_t^{(0)}}\right| &\leq \zeta_t \|\phi_t\|_2\|A^{-1}\|_2\|A\|_2\|E_t\|_2 \|V^{g_t^{(0)}} \odot \sigma^{g_t^{(0)}}\|_2\nonumber\\
           &\leq U r_{\text{max}}\frac{\gamma}{1-\gamma} \zeta_t \|\phi_t\|_2 \leq U r_{\text{max}}\frac{\gamma}{1-\gamma} \zeta_t \left( 1+ \|\phi_t\|_2^2\right)\nonumber\\
           &\leq U r_{\text{max}}\frac{\gamma}{1-\gamma} \zeta_t + U r_{\text{max}}\frac{1}{\lambda_{\text{min}}(A)}\frac{\gamma}{1-\gamma} \zeta_t Y_t\nonumber\\
        \left|\zeta_t\phi_tA^{-1}E_tAf_t \odot \Delta_t\right| &\leq \zeta_t \|\phi_t\|_2\|A^{-1}\|_2\|A\|_2\|E_t\|_2 \| f_t \odot \Delta_t\|_2 \leq U\zeta_t\|\phi_t\|_2\|\Delta_t\|_2\label{dY_t_bound_3_eq}\\
        &\leq U\zeta_t\|\phi_t\|_2\left( \sqrt{|\mathcal{X}\times\mathcal{A}|}\|\phi_t\|_2 +  r_{\text{max}}\frac{\gamma}{1-\gamma} \right)\nonumber\\
        &\leq U\zeta_t\sqrt{|\mathcal{X}\times\mathcal{A}|}\|\phi_t\|_2^2 + Ur_{\text{max}}\frac{\gamma}{1-\gamma}\zeta_t (1+\|\phi_t\|_2^2)\nonumber\\
        &\leq Ur_{\text{max}}\frac{\gamma}{1-\gamma}\zeta_t + U\zeta_t\sqrt{|\mathcal{X}\times\mathcal{A}|}\frac{1}{\lambda_{\text{max}}(A)}Y_t,\nonumber
        \end{align}
    where in the last bound we used the definition of $\Delta_t$ which gives the element-wise bound
    \begin{equation*}
        \begin{aligned}
            \Delta_t(x',a') &= \sum_{a''\in \mathcal{A}} Q_t(x',a'')\sigma^{g_t^{(0)}}(x',a'') = \sum_{a''\in \mathcal{A}} \left( \phi_t(x',a'')+V^{g_t^{(0)}}(x',a'') \right) \sigma^{g_t^{(0)}}(x',a'')\\
            & \leq \max_{a''}\phi_t(x',a'') + r_{\text{max}}\frac{\gamma}{1-\gamma}.
        \end{aligned}
    \end{equation*}

    Using \eqref{dY_t_bound_1_eq}, \eqref{dY_t_bound_2_eq}, and \eqref{dY_t_bound_3_eq} in \eqref{d_Y_t_eq} we obtain
    \begin{equation} \label{dY_t_IF_bound_1_eq}
        \begin{aligned}
            \frac{d}{dt}Y_t &\leq -\alpha \left(1- \gamma\right)\phi_t^2 \cdot \pi^{g_t^{(0)}} +U_1\zeta_t Y_t + U_2\zeta_t \\
            &\leq -\alpha \left(1- \gamma\right)C\eta_t^{n_0}\frac{1}{\lambda_{\text{max}}(A)}Y_t +U_1\zeta_t Y_t + U_2\zeta_t,
        \end{aligned}
    \end{equation}
    where
    \begin{equation}
        \begin{aligned}
            U_1 &= U\frac{1}{\lambda_{\text{min}}(A)} + U r_{\text{max}}\frac{1}{\lambda_{\text{min}}(A)}\frac{\gamma}{1-\gamma} +  U \sqrt{|\mathcal{X}\times\mathcal{A}|}\frac{1}{\lambda_{\text{max}}(A)}\\
            U_2 &= 2U r_{\text{max}}\frac{\gamma}{1-\gamma} 
        \end{aligned}
    \end{equation}
    are positive constants that do not depend on $t$, and we used the bound
    \begin{equation}
        \begin{aligned}
            \phi_t^2 \cdot \pi^{g_t^{(0)}} \geq C\eta_t^{n_0} \phi_t^T \phi_t,
        \end{aligned}
    \end{equation}
    which follows from \eqref{mc_stationary_lower_bound_eq} and the definition of $g_t^{(0)}$ in \eqref{actor_policy_learning_rates_limit_eq}, and where $C=\frac{1}{|\mathcal{A}|^{n_0}}$ does not depend on $t$.

    By our assumptions $\zeta_t \to 0$ and $\frac{\zeta_t}{\eta_t^{n_0}}\to 0$, for any fixed $\epsilon>0$ there exist $t_1, t_2 >0$  such that 
    \begin{equation}
    \begin{aligned}
        U_1\zeta_t&<\frac{1}{2} \alpha(1-\gamma)C\eta_t^{n_0}\frac{1}{\lambda_{\text{max}}(A)}, \quad \forall  t>t_1\\
        U_2\zeta_t&<\eta_t^{n_0}, \quad \forall t>t_2.
    \end{aligned}
    \end{equation}

Setting $K = \frac{1}{2} \alpha(1-\gamma)C\frac{1}{\lambda_{\text{max}}(A)}$, which is a positive constant, and 
$t_0 = \max{(t_1,t_2)}$ and considering $t>t_0$, \eqref{dY_t_IF_bound_1_eq} becomes  
\begin{equation}\label{dY_t_IF_bound_2_eq}
        \frac{d}{dt}Y_t \leq -K\eta_t^{n_0}Y_t + \eta_t^{n_0}
    \end{equation}

    We now multiply both sides of \eqref{dY_t_IF_bound_2_eq} with the integrating factor $\exp{\left(K\int_{t_0}^t\eta_s^{n_0}ds\right)}$ to get   \begin{equation}\label{dY_t_IF_bound_3_eq}
        \frac{d}{dt} \left[ Y_t \exp{\left(K\int_{t_0}^t \eta_s^{n_0}ds \right)}\right] \leq \eta_t^{n_0}
    \end{equation}

    Integrating both sides, we obtain \begin{equation}\label{dY_t_IF_bound_4_eq}
        Y_t \leq Y_{t_0}\exp{\left(-K\int_{t_0}^t \eta_s^{n_0}ds \right)}  +  \left(\int_{t_0}^t \eta_s^{n_0}ds \right) \exp{\left(-K\int_{t_0}^t \eta_s^{n_0}ds \right)} 
    \end{equation}

     By our assumptions that  $\frac{\zeta_t}{\eta_t^{n_0}}\to 0$ and $\int_0^\infty \zeta_s ds = \infty$, we also get that  $\int_{t_0}^\infty \eta_t^{n_0} ds = \infty$ and so there is a $t_3>t_0$ such that for $t>t_3$ we have    \begin{equation}\label{dY_t_IF_bound_5_eq}
         \exp{\left(-K\int_{t_0}^t \eta_s^{n_0}ds \right)} \leq \left( \int_{t_0}^t \eta_s^{n_0}ds \right)^{-2}
    \end{equation}

    Moreover, the following inequality holds
    \begin{equation}
        \int_{t_0}^t \eta_s^{n_0}ds \geq (t-t_0)\eta_t^{n_0} 
    \end{equation}

    where we used the fact that $\eta_t$ is decreasing. Using \eqref{dY_t_IF_bound_4_eq} and \eqref{dY_t_IF_bound_5_eq} in \eqref{dY_t_IF_bound_3_eq} we obtain    \begin{equation}\label{dY_t_IF_bound_6_eq}
        Y_t \leq \frac{Y_{t_0}}{(t-t_0)^2\eta_t^{2n_0}} + \frac{1}{(t-t_0)\eta_t^{n_0}}, \quad t>t_3
    \end{equation}

 Since $\eta_t = \Omega(1/t^a)$ for every $a>0$, we obtain $Y_t = O(\eta_t)$ as long as $Y_{t_0}$ is finite.    It thus remains to show that $Y_t$ does not blow up in finite time. Taking absolute values in \eqref{dY_t_IF_bound_1_eq} and since $\eta_t$ and $\zeta_t$ are decreasing and $Y_t$ is nonnegative, we get
    \begin{equation}
        \frac{d}{dt}\left| Y_t \right| \leq \left(2K\eta_0^{n_0} +U_1\zeta_0 \right)\left| Y_t \right| + U_2\zeta_0
    \end{equation}

    By Grönwall's inequality we obtain
    \begin{equation}
        \left| Y_t \right| \leq \left( Y_0+U_2\zeta_0\right)t \exp{\left(2K\eta_0^{n_0}t +U_1\zeta_0t \right)},
    \end{equation}
    and so $Y_{t_0}$ is finite.

    In order to show that $\left|Q_t - V^{f_t^{(0)}}\right|= O(\eta_t)$, all that is left is to show that $\left|V^{g_t^{(0)}} - V^{f_t^{(0)}} \right|= O(\eta_t)$. The triangle inequality then gives the result. This is shown in equations (5.25), (5.26) and (5.27) in \cite{nac} and will not be repeated here.

    \subsection{Actor Convergence}
    We now show that the actor converges to a stationary point, i.e. $\|\nabla_PJ(f_t^{(0)})\|\rightarrow 0$ as $t\rightarrow\infty$. Our approach is very similar to section 5.2 of \cite{nac}, with some of the steps in the proof being simpler as the output of the critic network is not clipped in the update equations of the actor model parameters and, consequently, in the limit equation \eqref{leading_error_terms_def_eq} of $P_t^{(0)}$. We present the main arguments for completeness.

    Using the transformation $Y_t = A^{-1}P_t^{(0)}$ and following the same steps as in the derivation of equation (5.31) of \cite{nac}, equation \eqref{leading_error_terms_def_eq} for $P_t^{(0)}$ gives
\begin{equation}\label{actor_conv_eq1}
        \frac{dY_t}{d_t}(x,a) = \zeta_t \sigma_{\rho_0}^{g_t^{(0)}}\left[ Q_t^{(0)}(x,a) - \sum_{a'\in\mathcal{A}}Q_t^{(0)}(x,a')f_t^{(0)}(x,a')\right]
    \end{equation}

    By lemma 5.1 of \cite{nac} we know that the policy gradient theorem yields
\begin{equation}\label{actor_conv_eq2}
        \frac{\partial J(f_t^{(0)})}{\partial P_t^{(0)}(x,a)} = \sigma_{\rho_0}^{f_t^{(0)}}(x,a)\left( V^{f_t^{(0)}}(x,a)-V^{f_t^{(0)}}(x)\right)= \sigma_{\rho_0}^{f_t^{(0)}}(x,a)\left( V^{f_t^{(0)}}(x,a)-\sum_{a'\in \mathcal{A}}V^{f_t^{(0)}}(x,a')f_t^{(0)}(x,a')\right).
    \end{equation}

    From equations \eqref{actor_conv_eq1} and \eqref{actor_conv_eq2} and the result of section \ref{critic_conv_sec} we get
    \begin{equation}
    \begin{aligned}
        &\left|\frac{dY_t}{dt}(x,a) - \zeta_t\frac{\partial J(f_t^{(0)})}{\partial P_t^{(0)}(x,a)}\right| \leq \left| \zeta_t \left( \sigma_{\rho_0}^{g_t^{(0)}} - \sigma_{\rho_0}^{f_t^{(0)}}  \right) \left[ Q_t^{(0)}(x,a) - \sum_{a'\in\mathcal{A}}Q_t^{(0)}(x,a')f_t^{(0)}(x,a')\right]\right| \\
        &\hspace{20pt}+ \left| \zeta_t  \sigma_{\rho_0}^{f_t^{(0)}} \left[ \left(Q_t^{(0)}(x,a) -V^{f_t^{(0)}}(x,a) \right) - \sum_{a'\in\mathcal{A}}\left(Q_t^{(0)}(x,a') - V^{f_t^{(0)}}(x,a')\right)f_t^{(0)}(x,a')\right]\right|\\
         &\hspace{20pt}\leq C \zeta_t \eta_t,
    \end{aligned}
    \end{equation}
    where we used the fact that since $\left| Q_t^{(0)} - V^{f_t^{(0)}}\right| \to 0$ and $V^{f_t^{(0)}}$ is bounded, $Q_t^{(0)}$ will be bounded too.
    
    By the chain rule, and since $A$ is positive definite, we get   \begin{equation}\label{actor_conv_eq3}
    \begin{aligned}
        \frac{d}{dt}J(f_t^{(0)}) &= \frac{\partial J(f_t^{(0)})}{\partial P_t^{(0)}(x,a)} \cdot \frac{dP_t^{(0)}}{dt} = \frac{\partial J(f_t^{(0)})}{\partial P_t^{(0)}(x,a)} \cdot A\frac{dY_t}{dt} \\
        &= \frac{\partial J(f_t^{(0)})}{\partial P_t^{(0)}(x,a)} \cdot A\zeta_t \frac{\partial J(f_t^{(0)})}{\partial P_t^{(0)}(x,a)} + A \frac{\partial J(f_t^{(0)})}{\partial P_t^{(0)}(x,a)} \cdot \left(\frac{dY_t}{dt}(x,a) - \zeta_t\frac{\partial J(f_t^{(0)})}{\partial P_t^{(0)}(x,a)}\right)\\
        &\geq \zeta_t \lambda_{\text{min}}(A)\left\|\frac{\partial J(f_t^{(0)})}{\partial P_t^{(0)}(x,a)}\right\|^2 - C \lambda_{\text{max}}(A)\zeta_t \eta_t,
    \end{aligned}
    \end{equation}
    where the constant $C$ is independent of $t$ and we used the fact that $\frac{\partial J(f_t^{(0)})}{\partial P_t^{(0)}(x,a)}$ is bounded by \eqref{actor_conv_eq2}. This bound is similar to the bound (5.34) in \cite{nac} and following the same subsequent steps one can show
    \begin{equation}
        \left\| \frac{\partial J(f_t^{(0)})}{\partial P_t^{(0)}(x,a)} \right\| \to 0.
    \end{equation}

    Since $ J(f_t^{(0)})$ is bounded, convergence to a stationary point follows.

    \subsection{Uniqueness of the Solution}
    To conclude this section, we show that the limit equations \eqref{Q_t_P_t_def_eq} admit at most one solution. This will be the first step in an inductive argument to prove uniqueness of the solutions to all the expansion term equations in theorem \ref{main_result_theorem}. We only provide an outline of the proof structure, as it follows similar arguments to those for the proof of the convergence bounds. More details can be found in section \ref{uniqueness_sec_1}, where uniqueness of the first order error terms is shown using the same arguments. We assume the existence of two distinct tuples $(Q_t, P_t, f_t, g_t, v_t, \mu_t, \pi^{g_t}, \sigma_{\rho_0}^{g_t})$ and $(\tilde Q_t, \tilde P_t,\tilde f_t,\tilde g_t,\tilde v_t, \tilde\mu_t,\tilde \pi^{g_t},\tilde \sigma_{\rho_0}^{g_t})$ that both satisfy the limit equations as suggested in section \ref{leading_result_sec}. Clearly, $\mu_t = \tilde \mu_t = \mu_0^{(0)}$ and $v_t = \tilde v_t =v_0^{(0)}$ as the initialization distributions of the parameters are given. Moreover, the differences $\left|f_t - \tilde f_t\right|$ $\left|g_t - \tilde g_t\right|$, $\left|\pi^{g_t} - \tilde \pi^{g_t}\right|$, $\left|\sigma_{\rho_0}^{g_t} - \tilde \sigma_{\rho_0}^{g_t}\right|$ can all be bound in terms of $\left|P_t - \tilde P_t \right|$, (i.e. satisfy a Lipschitz condition), which can be shown using similar arguments as in section \ref{leading_order_conv_chap}. From equations \eqref{Q_t_P_t_def_eq} we then obtain
    \begin{equation}
        \begin{aligned}
            \left| Q_t - \tilde Q_t\right|+\left| P_t - \tilde P_t\right| &\leq C\int_0^t \left(\left|Q_s - \tilde Q_s \right| +\left|P_s - \tilde P_s \right|\right)ds.
        \end{aligned}
    \end{equation}

    Using the Grönwall inequality we obtain $Q_t = \tilde Q_t$ and $P_t = \tilde P_t$, which then implies equality of the two solution tuples.

    \section{First Order Error Term Convergence} \label{first_order_error_sec}

    In this section we prove propositions \ref{first_order_error_conv_prop}, \ref{empirical_measures_leading_order_error_conv_prop},  and theorem \ref{actor_critic_networks_leading_error_conv_thm}. We recall that the error terms $P_t^{1,N}$ and $Q_t^{1,N}$ are defined as
    \begin{equation*}
        P_t^{1,N} = N^\phi \left( P_t^N-P_t^{(0)} \right), \quad, Q_t^{1,N} = N^\phi \left( Q_t^N - Q_t^{(0)}\right), 
    \end{equation*}
    where $\phi = \min\left\{1-\beta, \beta-\frac{1}{2}\right\}$, i.e $\phi = 1-\beta$ if $\beta>\frac{3}{4}$ and $\phi = \beta - \frac{1}{2}$ if $\beta \in \left(\frac{1}{2}, \frac{3}{4} \right]$. As in the previous section, we will first derive bounds on the errors of the actor model output, the policy, the stationary measures and the empirical measures. In the case $\beta>\frac{3}{4}$, the rest of the proof will be similar to the analysis in section \ref{leading_order_conv_sec}. In the case  $\beta \in \left(\frac{1}{2}, \frac{3}{4} \right]$, notice that scaling by a factor of $N^{\beta-\frac{1}{2}}$ causes the initial condition in the pre-limit equations \eqref{Q_t^N_eq} and \eqref{P_t^N_eq}, $Q_0^N$ and $P_0^N$ respectively, to be of constant order, and follow a normal distribution. As a result, convergence in $L_p$ (Definition \ref{L_p_conv_def}) cannot be achieved, but we show convergence in distribution to an ODE with random initial condition that captures the randomness induced by the initialization of the parameters. In the case $\beta \in \left( \frac{3}{4},1\right)$ our convergence is in $L_p$ to an ODE with zero initial condition, not yet capturing that randomness, which suggests the subsequent study of higher order error terms. Throughout this section, we also assume $T<\infty$ fixed. The proof of the lemmas in section \ref{first_order_error_tersms_sec} can be found in section \ref{first_order_proofs_app} of the appendix.

    \subsection{Analysis of the First Order Actor Model, Policy, Stationary Distribution and Empirical Measure Errors } \label{first_order_error_tersms_sec}

    In analogy to section \ref{actor_model_policy_stationary_measure_conv_sec}, we now derive bounds on the convergence of the error terms $f_t^{1,N}$ and $g_t^{1,N}$,  defined in \eqref{leading_error_terms_def_eq},  to the processes  $f_t^{(1)}$ and $g_t^{(1)}$ respectively, defined in \eqref{leading_order_error_terms_limits_eq}.  The proofs of the lemmas \ref{g_t^N_bound_lemma2} and \ref{stationary_measures_leading_order_error_conv_lemma}  are deferred to Appendix \ref{first_order_proofs_app}.
    \begin{lemma} \label{g_t^N_bound_lemma2}
        Let $p\in \mathds{N}$. The following bounds hold for the convergence of the first order actor model and policy errors
        \begin{equation*}
            \begin{aligned}
                \left|f_t^{1,N}(x,a)  - {f}_t^{(1)}(x,a) \right| &\leq C_T \max_{(x',a') \in \mathcal{X}\times \mathcal{A}}|P_t^{1,N}(x',a')-P_t^{(1)}(x',a')|  +O_{L_p}(N^{-\phi})\\
                \left|g_t^{1,N}(x,a)  - g_t^{(1)}(x,a) \right| &\leq (1-\eta_t)\left|f_t^{1,N}(x,a)  - f_t^{(1)}(x,a) \right| +O(N^{\phi-1})\\
                \mathds{E}\left[\max_{(x,a)\in \mathcal{X}\times\mathcal{A}}\left|f_t^{1,N}(x,a)  - {f}_t^{(1)}(x,a) \right|^p\right] &\leq C_{T,p} \mathds{E}\left[\max_{(x',a') \in \mathcal{X}\times \mathcal{A}}|P_t^{1,N}(x',a')-P_t^{(1)}(x',a')|^p\right]  +C_{T,p}N^{-p\phi}\\
                \mathds{E}\left[\max_{(x,a)\in \mathcal{X}\times\mathcal{A}}\left|g_t^{1,N}(x,a)  - g_t^{(1)}(x,a) \right|^p\right] &\leq C_{T,p} \mathds{E}\left[\max_{(x',a') \in \mathcal{X}\times \mathcal{A}}|P_t^{1,N}(x',a')-P_t^{(1)}(x',a')|^p\right]  +C_{T,p}N^{-p\phi},
            \end{aligned}
        \end{equation*}
        where the last two bounds hold for $N>N_{T,p}$ large enough.
    \end{lemma}

     In analogy to section \ref{actor_model_policy_stationary_measure_conv_sec} again, we now also derive bounds on the convergence of the error terms $\pi^{g_t^{1,N}}$ and $\sigma_{\rho_0}^{g_t^{1,N}}$,  which are also defined in \eqref{leading_error_terms_def_eq},  to the processes $\pi^{g_t^{(1)}}$ and $\sigma_{\rho_0}^{g_t^{(1)}}$ respectively, defined in \eqref{leading_order_error_terms_limits_eq}.
    \begin{lemma} \label{stationary_measures_leading_order_error_conv_lemma}
    Let $T<\infty$ and $p\in \mathds{N}$. There are constants $C_{T,p}$ and $N_{T,p}$ independent of $N$ such that the stationary measure error terms satisfy the following $L_p$ bounds for $N>N_{T,p}$
    \begin{equation*}
        \begin{aligned}
            \mathds{E}\left[ \max_{{(x,a)}\in \mathcal{X}\times\mathcal{A}}|\pi^{g_t^{1,N}}(x,a) - \pi^{g_t^{(1)}}(x,a)|^p\right] & \leq C_{T,p}N^{-p\phi} + \mathds{E}\left[ \max_{(x',a') \in \mathcal{X}\times \mathcal{A}}|P_t^{1,N}(x',a')-P_t^{(1)}(x',a')|^p \right]\\
            \mathds{E}\left[ \max_{{(x,a)}\in \mathcal{X}\times\mathcal{A}}|\sigma_{\rho_0}^{g_t^{1,N}}(x,a) - \sigma_{\rho_0}^{g_t^{(1)}}(x,a)|^p\right] & \leq C_{T,p}N^{-p\phi} + \mathds{E}\left[ \max_{(x',a') \in \mathcal{X}\times \mathcal{A}}|P_t^{1,N}(x',a')-P_t^{(1)}(x',a')|^p \right].
        \end{aligned}
    \end{equation*}
        
    \end{lemma}

    We will now show convergence of the error of the empirical measure process, while also providing convergence rates. We start by studying the scaled pre-limit equation for the process $v_t^N$ in \eqref{v_t^N_eq}. The following pre-limit equation holds, details about the derivation can be found in section \ref{prelim_sec_app} of the Appendix. Similar to the pre-limit equations \eqref{Q_t^N_eq} and \eqref{P_t^N_eq}, it consists of a drift term, martingale terms, and a small remainder.        \begin{align}\label{v_t^N_error_eq}
            N^{1-\beta}\left(\langle h, v_{t}^N \rangle - \langle h,v_0^N \rangle\right) & = \alpha \int_0^t \sum_{(x',a'),(x'',a'') \in \mathcal{X}\times\mathcal{A}} \left(r(x',a')+\gamma Q_s^N(x'',a'')- Q_s^N(x',a')\right)\times\nonumber\\
            &\qquad\times\langle C_{(x',a')}^h, v_s^N \rangle\pi^{g_s^N}(x',a')g_{s}^N(x'',a'')p(x''|x',a')ds\nonumber\\
            &+M_{h,t}^{1,N}+M_{h,t}^{2,N}+M_{h,t}^{3,N} + R_{v}^{1,N},
        \end{align}
         where the martingale terms are defined as 
    \begin{equation} \label{M_ht^N_eq}
        \begin{aligned}
            M_{h,t}^{1,N} &= \frac{\alpha}{N} \sum_{k=0}^{\lfloor NT\rfloor -1} r(x_k, a_k)\langle C_{(x_k,a_k)}^h, v_k^N \rangle - \frac{\alpha}{N} \sum_{(x',a') \in \mathcal{X} \times \mathcal{A}}\sum_{k=0}^{\lfloor NT\rfloor -1} r(x',a')\langle C_{x',a'}^h, v_k^N \rangle \pi^{g_s^N}(x',a')\\
            M_{h,t}^{2,N} &= \frac{\alpha}{N} \sum_{k=0}^{\lfloor NT\rfloor -1} Q_k^N(x_{k+1}, a_{k+1})\langle C_{(x_k,a_k)}^h, v_k^N \rangle \\
            &\hspace{20pt}- \frac{\alpha}{N} \sum_{(x',a'
            ), (x'',a'') \in \mathcal{X} \times \mathcal{A}}\sum_{k=0}^{\lfloor NT\rfloor -1}  Q_k^N(x'', a'')\langle C_{x',a'}^h, v_k^N \rangle \pi^{g_s^N}(x',a')g_{s}^N(x'',a'')p(x''|x',a'))\\
             M_{h,t}^{3,N} &= -\frac{\alpha}{N} \sum_{k=0}^{\lfloor NT\rfloor -1} Q_k(x_k, a_k)\langle C_{(x_k,a_k)}^h, v_k^N \rangle - \frac{\alpha}{N} \sum_{(x',a') \in \mathcal{X} \times \mathcal{A}}\sum_{k=0}^{\lfloor NT\rfloor -1} Q_k(x',a')\langle C_{x',a'}^h, v_k^N \rangle \pi^{g_s^N}(x',a'),
        \end{aligned}
    \end{equation}
    and the remainder $R_{v}^{1,N}$ differs from the remainder in \ref{v_t^N_eq} but can still be bounded as
    \begin{equation} \label{R_v_bound_eq}
        \begin{aligned}
            \left| R_v^{1,N}\right| \leq C_TN^{-\beta},
        \end{aligned}
    \end{equation}
    for some uniform constant $C_T$ independent of $N$. The martingale terms can also be bounded in $L_1$. The proof of the following lemma is similar to the proof of lemma \ref{M_bound_lemma} which can be found in section \ref{mc_sec_app} of the Appendix.
    \begin{lemma}  \label{M_ht_L_1_bound_lemma}
         Let $T<\infty$ and $p \in \mathds{N}$. For any fixed $h\in C_b^3(\mathbb{R})$, the martingale terms in the pre-limit equation \eqref{v_t^N_error_eq} satisfy 
    \begin{equation*} 
    \begin{aligned}
        \sup_{t\in [0,T]}\mathds{E}\left[\left|M_{h,t}^{1,N}\right|^p\right] + \sup_{t\in [0,T]}\mathds{E}\left[\left|M_{h,t}^{2,N}\right|^p\right]+ \sup_{t\in [0,T]}\mathds{E}\left[\left|M_{h,t}^{3,N}\right|^p\right]& \leq C_{T,p}N^{-p/2}
    \end{aligned}
    \end{equation*}

    \end{lemma}                                                   

    The empirical measure error process $\mu_t^{1,N}$ is handled analogously. Using a similar analysis as in section \ref{leading_order_conv_sec}, one can show the result of proposition \ref{empirical_measures_leading_order_error_conv_prop}. Notice at this point that this holds true in particular for $h=B$ as defined in \eqref{C_xi^f_eq} since by lemma \ref{parameter_bounds_apriori_lemma} and assumptions \ref{MDP_assumptions} and \ref{actor_critic_models_assumption} on the finiteness of $\mathcal{X}\times \mathcal{A}$ and the differentiability of $\sigma$ we get $B\in C_b^3(\mathbb{R})$.
   
    \subsection{Convergence of the First Order Actor and Critic Network Output Error Term for $\beta \in \left(\frac{3}{4},1\right)$} \label{first_order_error_tersms_sec_2}

    In analogy to section \ref{leading_order_conv_chap}, we now show convergence of the first order error 
    terms $Q_t^{1,N}$ and $P_t^{1,N}$ for $\beta \in \left( \frac{3}{4},1\right)$. Notice that 
    for those values of $\beta$, as we increase the scaling parameter $\phi$ in equations 
    \eqref{first_order_network_output_error_terms} and \eqref{leading_error_terms_def_eq} the empirical 
    measure error processes $v_t^{1,N}$ will reach constant order before the scaled initializations 
    $Q_0^N$ and $P_0^N$ do, implied in \eqref{first_order_network_output_error_terms} by equations 
    \eqref{Q_t^N_eq} and \eqref{P_t^N_eq} respectively. This implies that for $\phi=1-\beta$, the scaling 
    that causes $v_t^{1,N}$ and $\mu_t^{1,N}$ to be of constant order. The scaled initializations $N^\phi 
    Q_0^N$ and $N^\phi P_0^N$ are still vanishing, so we expect convergence to a non random ODE with zero initial condition, similar to the leading order limit terms $Q_t$ and $P_t$. As in section 
    \ref{leading_order_conv_chap}, we start by providing the pre-limit equations of $Q_t^{1,N}$ and $P_t^{1,N}$, which follow directly
    from equations \eqref{Q_t^N-Q_t_formula_eq} and \eqref{P_t^N-P_t_formula_eq} and the definitions of the first order error terms in \eqref{leading_error_terms_def_eq}.
    \begin{equation} \label{Q_t^1N_eq}
        \begin{aligned}
            &Q_t^{1,N}(x,a)=\\
            &=\sum_{\substack{k_1,k_2,k_3,k_4 \in \{0,1\} \\ k_1+k_2+k_3+k_4\geq 1}} \alpha N^{\left(1-\sum_{i=1}^4 k_i\right) \phi} \int_{0}^t \sum_{(x',a'),(x'',a'') \in \mathcal{X}\times \mathcal{A}} \Bigg[ r(x',a') + \gamma Q_s^{(0)}(x'',a'') - Q_s^{(0)}(x',a') \Bigg]^{1-k_1}\\
            &\hspace{10pt} \Bigg[\gamma Q_s^{1,N}(x'',a'') - Q_s^{1,N}(x',a') \Bigg]^{k_1}\Bigg[\langle B_{(x,a),(x',a')}, v_0^{(0)} \rangle \Bigg]^{1-k_2} \Bigg[\langle B_{(x,a),(x',a')}, v_s^{1,N}\rangle \Bigg]^{k_2}\\
            &\hspace{10pt} \Bigg[ g_s^{(0)}(x'',a'')\Bigg]^{1-k_3} \Bigg[g_s^{1,N}(x'',a'') \Bigg]^{k_3} \Bigg[ \pi^{g_s^{(0)}}(x',a')\Bigg]^{1-k_4} \Bigg[\pi^{g_s^{1,N}}(x',a') \Bigg]^{k_4}p(x'',a''|x',a')ds\\
            &\hspace{10pt} +N^{1-\beta+\phi}\langle c\sigma(w\cdot (x,a)), v_0^N\rangle + N^\phi M_t^{1,N}((x,a)) + N^\phi M_t^{2,N}((x,a))+ N^{\phi} M_t^{3,N}((x,a))+ N^{\phi}R_{Q,t}^N
        \end{aligned}
    \end{equation}
    \begin{equation} \label{P_t^1N_eq}
        \begin{aligned}
            &P_t^{1,N}(x,a) =\\
            &=\sum_{\substack{k_1,k_2,k_3,k_4 \in \{0,1\} \\ k_1+k_2+k_3+k_4\geq 1}}  N^{\left(1-\sum_{i=1}^4 k_i\right) \phi} \int_{0}^t  \sum_{x',a',a'' \in \mathcal{X}\times \mathcal{A}  \times \mathcal{A}} \zeta_s \Bigg[ Q_s^{(0)}(x',a') \Bigg]^{1-k_1} \Bigg[Q_s^{1,N}(x',a') \Bigg]^{k_1}\\
            &\hspace{10pt} \Bigg[\langle B_{(x,a),(x',a')}, \mu_0^{(0)} \rangle + \langle B_{(x,a),(x',a'')}, \mu_0^{(0)} \rangle \Bigg]^{1-k_2} \Bigg[\langle B_{(x,a),(x',a')}, \mu_s^{1,N} \rangle + \langle B_{(x,a),(x',a'')}, \mu_s^{1,N} \rangle \Bigg]^{k_2}\\
            &\hspace{10pt} \Bigg[ f_s^{(0)}(x',a'')\Bigg]^{1-k_3} \Bigg[f_s^{1,N}(x',a'')\Bigg]^{k_3} \Bigg[ \sigma_{\rho_0}^{g_s^{(0)}}(x',a')\Bigg]^{1-k_4} \Bigg[\sigma_{\rho_0}^{g_s^{1,N}}(x',a') \Bigg]^{k_4}ds\\
            &\hspace{10pt} +N^{1-\beta+\phi}\langle b\sigma(u\cdot (x,a)), \mu_0^N\rangle + N^\phi \tilde{M}_t^{N}((x,a)) + N^{\phi}R_{P,t}^N.
        \end{aligned}
    \end{equation}

    Notice here that each summand in the summations consists of an integral of a product of leading order limit factors (those raised to the power of $1-k_i$) and first order error factors (those raised to the power of $k_i$). We recall that each leading order limit factor is uniformly bounded, and that each error factor in bounded in $L_p$ for every $p\in \mathbb{N}$. Hölder's Inequality then ensures that every integral in the summation is bounded in $L_p$, which implies that all the integrals corresponding to $k_1+k_2+k_3+k_4 \geq 2$ are in $O_{L_p}(N^{-\phi})$ and will thus not affect $L_p$ convergence. Moreover, for both equations by lemma \ref{M_bound_lemma}, the bounds \eqref{R_{Q,t}^N_bounds_eq} and \eqref{R_{P,t}^N_bounds_eq} and equations \eqref{initialization_dist_conv_eq} immediately imply that all the terms outside the summations vanish in $L_p$. Specifically, for $Q_t^{1,N}$ and considering only the vanishing terms of leading order, equation \eqref{Q_t^1N_eq} can then be rewritten as
        \begin{align}
             Q_t^{1,N}(x,a) &= \alpha \int_{0}^t \sum_{(x',a'),(x'',a'') \in \mathcal{X}\times \mathcal{A}} \left(\gamma Q_s^{1,N}(x'',a'') - Q_s^{1,N}(x',a') \right)
             \langle B_{(x,a),(x',a')}, v_0^{(0)} \rangle\nonumber\\
            &\hspace{40pt} g_s^{(0)}(x'',a'') \pi^{g_s^{(0)}}(x',a') p(x'',a''|x',a')ds\nonumber\\
            &+ \alpha \int_{0}^t \sum_{(x',a'),(x'',a'') \in \mathcal{X}\times \mathcal{A}} \left(r(x',a') + \gamma Q_s^{(0)}(x'',a'') - Q_s^{(0)}(x',a') \right)
             \langle B_{(x,a),(x',a')}, v_s^{1,N} \rangle\nonumber\\
            &\hspace{40pt} g_s^{(0)}(x'',a'') \pi^{g_s^{(0)}}(x',a') p(x'',a''|x',a')ds\nonumber\\
            &+ \alpha \int_{0}^t \sum_{(x',a'),(x'',a'') \in \mathcal{X}\times \mathcal{A}} \left(r(x',a') + \gamma Q_s^{(0)}(x'',a'') - Q_s^{(0)}(x',a') \right)
             \langle B_{(x,a),(x',a')}, v_0^{(0)} \rangle\nonumber\\
            &\hspace{40pt} g_s^{1,N}(x'',a'') \pi^{g_s^{(0)}}(x',a') p(x'',a''|x',a')ds\nonumber\\
            &+ \alpha \int_{0}^t \sum_{(x',a'),(x'',a'') \in \mathcal{X}\times \mathcal{A}} \left(r(x',a') + \gamma Q_s^{(0)}(x'',a'') - Q_s^{(0)}(x',a') \right)
             \langle B_{(x,a),(x',a')}, v_0^{(0)} \rangle\nonumber\\
            &\hspace{40pt}  g_s^{(0)}(x'',a'') \pi^{g_s^{1,N}}(x',a') p(x'',a''|x',a')ds\nonumber\\
            &+ O_{L_p}(N^{\frac{1}{2}-\beta +\phi }).\label{Q_t^{(1)}N_new_eq}
        \end{align}

    Subtracting $Q_t^{(1)}$ as defined in \eqref{Q_t^{(1)}_def_eq} we obtain
    \begin{equation} \label{Q_t^{(1)}N-Q_t^{(1)}_eq}
        \begin{aligned}
             \Big|Q_t^{1,N}&(x,a) - Q_t^{(1)}(x,a)\Big|^p\\
             &\leq \alpha C_p \int_{0}^t \sum_{(x',a'),(x'',a'') \in \mathcal{X}\times \mathcal{A}} \left(\left|d Q_s^{1,N}(x'',a'')-\gamma Q_s^{(1)}(x'',a'')\right|^p + \left|Q_s^{1,N}(x',a')-Q_s^{(1)}(x',a') \right|^p\right]
             \\
            &\hspace{40pt} \left|\langle B_{(x,a),(x',a')}, v_0^{(0)} \rangle \right|^p \left|g_s^{(0)}(x'',a'')\right|^p\left| \pi^{g_s^{(0)}}(x',a') \right|^p p(x'',a''|x',a')ds\\
            &+ \alpha C_p \int_{0}^t \sum_{(x',a'),(x'',a'') \in \mathcal{X}\times \mathcal{A}} \left|r(x',a') + \gamma Q_s^{(0)}(x'',a'') - Q_s^{(0)}(x',a') \right|^p
             \left|\langle B_{(x,a),(x',a')}, v_s^{1,N}-v_s^{(1)} \rangle\right|^p\\
            &\hspace{40pt} \left|g_s^{(0)}(x'',a'')\right|^p \left| \pi^{g_s^{(0)}}(x',a') \right|^p p(x'',a''|x',a')ds\\
            &+ \alpha C_p\int_{0}^t \sum_{(x',a'),(x'',a'') \in \mathcal{X}\times \mathcal{A}} \left|r(x',a') + \gamma Q_s^{(0)}(x'',a'') - Q_s^{(0)}(x',a') \right|^p
             \left|\langle B_{(x,a),(x',a')}, v_0^{(0)} \rangle\right|^p\\
            &\hspace{40pt} \left|g_s^{1,N}(x'',a'') -g_s^{(1)}(x'',a'')\right|^p \left|\pi^{g_s^{(0)}}(x',a') \right|^pp(x'',a''|x',a')ds\\
            &+ \alpha C_p \int_{0}^t \sum_{(x',a'),(x'',a'') \in \mathcal{X}\times \mathcal{A}} \left|r(x',a') + \gamma Q_s^{(0)}(x'',a'') - Q_s^{(0)}(x',a') \right|^p
             \left|\langle B_{(x,a),(x',a')}, v_0^{(0)} \rangle\right|^p\\
            &\hspace{40pt}  \left|g_s^{(0)}(x'',a'')\right|^p \left| \pi^{g_s^{1,N}}(x',a') - \pi^{g_s^{(1)}}(x',a') \right|^p p(x'',a''|x',a')ds\\
            &+ O_{L_p}(N^{\frac{1}{2}-\beta +\phi })
        \end{aligned}
    \end{equation}

    With an analysis similar to section \ref{leading_order_conv_sec} and using lemmas \ref{g_t^N_bound_lemma2} and \ref{stationary_measures_leading_order_error_conv_lemma} one can then show
    \begin{equation} \label{Q_t^{(1)}N_error_gronwall_setup_eq}
    \begin{aligned}
        \mathds{E}\left[ \max_{(x,a) \in \mathcal{X}\times\mathcal{A}}|Q_t^{1,N}(x,a)-Q_t^{(1)}(x,a)|^p\right] &\leq  C_{T,p}\int_0^t\mathds{E}\left[\max_{(x',a') \in \mathcal{X}\times \mathcal{A}} |Q_s^{1,N}(x',a')-Q_s^{(1)}(x',a')|^p\right]ds\\
        &+ C_{T,p}\int_0^t\mathds{E}\left[\max_{(x',a') \in \mathcal{X}\times \mathcal{A}} |P_s^{1,N}(x',a')-P_s^{(1)}(x',a')|^p\right]ds\\
        &+C_{T,p}N^{-p\phi}+C_{T,p}N^{p\left(\frac{1}{2}-\beta+\phi\right)}
    \end{aligned}
    \end{equation}

    The same bound can be derived analogously for $P_t^{1,N}$, i.e
    \begin{align}
        \mathds{E}\left[ \max_{(x,a) \in \mathcal{X}\times\mathcal{A}}|P_t^{1,N}(x,a)-P_t^{(1)}(x,a)|^p\right] &\leq  C_{T,p}\int_0^t\mathds{E}\left[\max_{(x',a') \in \mathcal{X}\times \mathcal{A}} |Q_s^{1,N}(x',a')-Q_s^{(1)}(x',a')|^p\right]ds\nonumber\\
        &+ C_{T,p}\int_0^t\mathds{E}\left[\max_{(x',a') \in \mathcal{X}\times \mathcal{A}} |P_s^{1,N}(x',a')-P_s^{(1)}(x',a')|^p\right]ds\nonumber\\
        &+C_{T,p}N^{-p\phi}+C_{T,p}N^{p\left(\frac{1}{2}-\beta+\phi\right)}.\label{P_t^{(1)}N_error_gronwall_setup_eq}
    \end{align}
    
       As in section \ref{leading_order_conv_sec}, adding \eqref{Q_t^{(1)}N_error_gronwall_setup_eq} and \eqref{P_t^{(1)}N_error_gronwall_setup_eq} and using the Grönwall inequality gives
    \begin{equation} \label{error_bounds_1st_order_error_eq}
    \begin{aligned}
        \sup_{t \in \left[0,T\right]}\mathds{E}\left[\max_{(x',a') \in \mathcal{X}\times \mathcal{A}} |Q_t^{1,N}(x',a')-Q_t^{(1)}(x',a')|^p\right] & \leq C_{T,p}N^{p(\beta-1)} + C_{T,p}N^{p\left(\frac{3}{2}-{2\beta}\right)}\\
        \sup_{t \in \left[0,T\right]}\mathds{E}\left[\max_{(x',a') \in \mathcal{X}\times \mathcal{A}} |P_t^{1,N}(x',a')-P_t^{(1)}(x',a')|^p\right] & \leq C_{T,p}N^{p(\beta-1)} + C_{T,p}N^{p\left(\frac{3}{2}-{2\beta}\right)},
    \end{aligned}
    \end{equation}
    where we substituted $\phi = 1-\beta$. The second part of theorem \ref{actor_critic_networks_leading_error_conv_thm} then follows. The second parts of propositions \ref{first_order_error_conv_prop} and \ref{empirical_measures_leading_order_error_conv_prop} then follow from theorem \ref{actor_critic_networks_leading_error_conv_thm} and lemmas \ref{g_t^N_bound_lemma2} and \ref{stationary_measures_leading_order_error_conv_lemma}. Note that similar to \eqref{error_bounds_1st_order_eq}, the term $C_{T,p}N^{p(\beta-1)}$ corresponds to the convergence rate of the first order empirical measure error terms $v_t^{1,N}$ and $\mu_t^{1,N}$ as determined in proposition \ref{empirical_measures_leading_order_error_conv_prop}, while the error term $C_{T,p}N^{p\left(\frac{3}{2}-{2\beta}\right)}$ corresponds to the random initialization and lemma \ref{prelim_initial_condition_convergence}.

    \subsection{Convergence of the First Order Actor and Critic Network Output Error Term for $\beta \in \left(\frac{1}{2},\frac{3}{4}\right]$}\label{first_order_error_tersms_sec_3}

    We now study the case $\beta \in \left(\frac{1}{2},\frac{3}{4}\right]$. In contrast to the case $\beta \in \left(\frac{3}{4},1 \right)$, the scaled initializations $N^\phi Q_0^N$ and $N^\phi P_0^N$ do not vanish but converge in distribution to a Gaussian vector. Hence, the Grönwall argument used previously does not apply here. Instead, we will prove weak convergence in the appropriate space through relative compactness.
    \label{leading_order_relcomp_sec}
    \subsubsection{Compact Containment}
     Throughout this subsection, $C_{T,p}$ denotes a constant that may change at each step but remains independent of $N$. We also consider $N>N_{T,p}$ where $N_{T,p}$ is sufficiently large for all the following bounds to hold, and can be chosen independent of $N$ too. From the convergence rates in Theorem \ref{leading_order_conv_th} we get,
    \begin{equation} \label{Q_t^{(1)}N_P_t^{(1)}N_L_p_bounds}
        \begin{aligned}
            \sup_{0\leq t\leq T}\mathds{E}\left[ \max_{(x,a)\in \mathcal{X}\times \mathcal{A}}\left|Q_t^{1,N}(x,a)\right|^p \right] + \sup_{0\leq t\leq T}\mathds{E}\left[ \max_{(x,a)\in \mathcal{X}\times \mathcal{A}}\left|P_t^{1,N}(x,a)\right|^p\right] \leq C_{T,p},
        \end{aligned}
    \end{equation}
    while the convergence rates in proposition \ref{leading_order_conv_prop} imply
    \begin{equation} \label{f_t^{(1)}N_g_t^{(1)}N_L_p_bounds}
        \begin{aligned}
            \sup_{0\leq t\leq T}\mathds{E}\left[ \max_{(x,a)\in \mathcal{X}\times \mathcal{A}}\left|f_t^{1,N}(x,a)\right|^p\right] + \sup_{0\leq t\leq T}\mathds{E}\left[ \max_{(x,a)\in \mathcal{X}\times \mathcal{A}}\left|g_t^{1,N}(x,a)\right|^p\right] \leq C_{T,p},
        \end{aligned}
    \end{equation}
    \begin{equation} \label{stationary_measures_errors_L_p_bounds}
        \begin{aligned}
            \sup_{0\leq t\leq T}\mathds{E}\left[ \max_{(x,a)\in \mathcal{X}\times \mathcal{A}}\left|{\pi}^{g_t^{1,N}}(x,a)\right|^p\right] +
            \sup_{0\leq t\leq T}\mathds{E}\left[ \max_{(x,a)\in \mathcal{X}\times \mathcal{A}}\left|{\sigma}_{\rho_0}^{g_t^{1,N}}(x,a)\right|^p\right] \leq C_{T,p},
        \end{aligned}
    \end{equation}
    and the convergence rates in Proposition \ref{empirical_measures_conv_prop} imply that for every $h \in C_b^2(\mathds{R})$ we have
\begin{equation}\label{empirical_measure_errors_L_p_bounds}
        \begin{aligned}
            \sup_{0\leq t\leq T}\mathds{E}\left[\left|\langle h, v_t^{1,N}\rangle \right|^p\right] + \sup_{0\leq t\leq T}\mathds{E}\left[\left|\langle h, \mu_t^{1,N}\rangle \right|^p\right] &\leq C_{T,p}. 
        \end{aligned}
    \end{equation}

    Using \eqref{Q_t^{(1)}N_P_t^{(1)}N_L_p_bounds}, \eqref{f_t^{(1)}N_g_t^{(1)}N_L_p_bounds}, \eqref{stationary_measures_errors_L_p_bounds} and \eqref{empirical_measure_errors_L_p_bounds} along with the Markov Inequality directly gives compact containment, as stated in the following lemma. \begin{lemma}\label{leading_error_terms_compact containment}
        For every $\eta>0$ there is a compact subset $\mathcal{K}$ of $E$ and a $N_{\eta,T,p} \in \mathds{N}$ such that whenever $N>N_{\eta,T,p}$
        \begin{equation*}
            \sup_{N>N_0, 0\leq t\leq T}\mathds{P}\left[ (Q_t^{1,N}, P_t^{1,N}, v_t^{1,N}, \mu_t^{1,N}, f_t^{1,N}, g_t^{1,N}, {\pi}^{g_t^{1,N}}, {\sigma}_{\rho_0}^{g_t^{1,N}}) \notin K \right] \leq \eta.
        \end{equation*}
    \end{lemma}

    \subsubsection{Regularity of the leading order pre-limit and limit terms} \label{leading_lim_prelim_reg_sec}

    In order to study the regularity of the various stochastic processes that appeared so far, we will first derive continuity bounds for the processes $Q_t^{(0)}, P_t^{(0)}, f_t^{(0)}, g_t^{(0)}, v_t^{(0)}, \mu_t^{(0)}, \pi^{g_t^{(0)}}$ and $\sigma_{\rho_0}^{g_t^{(0)}}$. Regularity of their corresponding pre-limit processes $Q_t^N, P_t^N, f_t^N, g_t^N, v_t^N, \mu_t^N, \pi^{g_t^N}$ and $\sigma_{\rho_0}^{g_t^N}$ then follows from the convergence rates in theorem \ref{leading_order_conv_th} and propositions  \ref{leading_order_conv_prop}, and \ref{empirical_measures_conv_prop}. We will make frequent use of the Maximum-Length Inequality for integrals
    \begin{equation}\label{ML-Inequality}
        \int_{a}^b f(s)ds \leq |b-a| \cdot \sup_{t\in [a,b]}|f(t)|.
    \end{equation}

    We start by studying the Critic and Actor network outputs.
    \begin{lemma} \label{Q_t_P_t_reg_lemma}
    For any $t\in [0,T]$ and $\delta \in [0,T-t]$ there exists a uniform constant $C_T$ independent of $N$ such that the limit Critic and Actor network outputs $Q_t$ and $P_t$ satisfy the following bounds
    \begin{equation*}
        \begin{aligned}
            \max_{(x,a)\in \mathcal{X}\times\mathcal{A}}\left|Q_{t+\delta}^{(0)}(x,a)-Q_t^{(0)}(x,a) \right| +\max_{(x,a)\in \mathcal{X}\times\mathcal{A}}\left|P_{t+\delta}^{(0)}(x,a)-P_t^{(0)}(x,a) \right| \leq C_T\cdot \delta.
        \end{aligned}
    \end{equation*}

    Moreover, the pre-limit Critic and Actor network outputs $Q_t^N$ and $P_t^N$ satisfy
    \begin{equation*}
        \begin{aligned}
            \max_{(x,a)\in \mathcal{X}\times\mathcal{A}}\left|Q_{t+\delta}^N(x,a)-Q_t^N(x,a) \right| +\max_{(x,a)\in \mathcal{X}\times\mathcal{A}}\left|P_{t+\delta}^N(x,a)-P_t^N(x,a) \right| \leq C_T\cdot \delta + O_{L_p}(N^{-\phi}).
        \end{aligned}
    \end{equation*}
    \end{lemma}
    \begin{proof}
        The result for $Q^{(0)}$ and $P^{(0)}$ follows immediately from equations \eqref{Q_t_P_t_def_eq}, the boundedness of the limit processes involved, and the integral inequality \eqref{ML-Inequality}. To show the result for $Q^{N}$ we apply the triangle inequality
        \begin{equation*}
            \begin{aligned}
                \left|Q_{t+\delta}^N(x,a) - Q_t^N(x,a) \right| \leq \left|Q_{t+\delta}^N(x,a) - Q_{t+\delta}^{(0)}(x,a) \right|+\left|Q_{t+\delta}^{(0)}(x,a) - Q_t^{(0)}(x,a) \right|+\left|Q_{t}^{(0)}(x,a) - Q_t^N(x,a) \right|,
            \end{aligned}
        \end{equation*}
        and use the first result and the convergence rates in theorem \ref{leading_order_conv_th}. The bound for $P^{N}$ follows similarly.
    \end{proof}
    \begin{lemma} \label{f_t_g_t_reg_lemma}
        For any $t\in [0,T]$ and $\delta \in [0,T-t]$ there exists a uniform constant $C_T$ independent of $N$ such that the limit Actor model and policy $f_t$ and $g_t$ satisfy the following bounds
    \begin{equation*}
        \begin{aligned}
            \max_{x,a\in \mathcal{X}\times\mathcal{A}}\left|f_{t+\delta}^{(0)}(x,a)-f_t^{(0)}(x,a) \right| +\max_{x,a\in \mathcal{X}\times\mathcal{A}}\left|g_{t+\delta}^{(0)}(x,a)-g_t^{(0)}(x,a) \right|\leq C_T\cdot \delta.\\
        \end{aligned}
    \end{equation*}

    Moreover, the prelimit Actor model and policy $f_t^N$ and $g_t^N$ satisfy
    \begin{equation*}
        \begin{aligned}
            \max_{x,a\in \mathcal{X}\times\mathcal{A}}\left|f_{t+\delta}^N(x,a)-f_t^N(x,a) \right| +
            \max_{x,a\in \mathcal{X}\times\mathcal{A}}\left|g_{t+\delta}^N(x,a)-g_t^N(x,a) \right| \leq C_T\cdot \delta + O_{L_p}(N^{-\phi}).
        \end{aligned}
    \end{equation*}
        
    \end{lemma}

    \begin{proof}
        The result for $f^{(0)}$ follows immediately from the second bound in lemma \ref{Q_t_P_t_reg_lemma} and the Lipschitz continuity of the softmax function. For the  result for $g^{(0)}$ we need a continuity bound for the limit exploration policy $\eta_t$. Using similar arguments to lemma \ref{eta_t^N_convergence_error_bound} it is easy to show
        \begin{equation} \label{eta_reg_bound_eq}
            \left|\eta_{t+\delta}-\eta_t \right| \leq \delta.
        \end{equation}
        
        The regularity result then for $g^{(0)}$ follows from the first result, the bound \eqref{eta_reg_bound_eq} and the formula of $g_t$ in \eqref{actor_policy_learning_rates_limit_eq}.
        The regularity  results for $f^{N}$ and $g^{N}$ can be shown using a similar argument to lemma \ref{Q_t_P_t_reg_lemma}, using a triangle inequality and the bounds from proposition \ref{leading_order_conv_prop} for each.
    \end{proof}
    \begin{lemma} \label{v_t_mu_t_reg_lemma}
        For any fixed $h\in C_b^2(\mathbb{R})$, $t\in [0,T]$ and $\delta \in [0,T-t]$ there exists a uniform constant $C_T$ independent of $N$ such that the pre-limit empirical measures $v_t^N$ and $\mu_t^N$ satisfy the following bounds
    \begin{equation*}
        \begin{aligned}
            \left|\langle h, v_{t+\delta}^N\rangle - \langle h, v_t^N\rangle \right| +\left|\langle h, \mu_{t+\delta}^N\rangle - \langle h, \mu_t^N\rangle \right| = O_{L_p}(N^{\beta-1}).
        \end{aligned}
    \end{equation*}    
    \end{lemma}
    \begin{proof}
        Since the limit empirical measures $v_t$ and $\mu_t$ are constant over time, the result follows immediately using a triangle inequality argument similar to lemma \ref{Q_t_P_t_reg_lemma} and then using the convergence rates of proposition \ref{empirical_measures_conv_prop}.
    \end{proof}
    \begin{lemma} \label{pi_t_sigma_t_reg_lemma}
    For any $t\in [0,T]$ and $\delta \in [0,T-t]$ there exists a uniform constant $C_T$ independent of $N$ such that the limit stationary measures $\pi^{g_t^{(0)}}$ and $\sigma_{\rho_0}^{g_t^{(0)}}$ satisfy the following bounds
    \begin{equation*}
        \begin{aligned}
            \max_{x,a\in \mathcal{X}\times\mathcal{A}} \left|\pi^{g_{t+\delta}^{(0)}}(x,a)-\pi^{g_{t}^{(0)}}(x,a)\right|+\max_{x,a\in \mathcal{X}\times\mathcal{A}} \left|\sigma_{\rho_0}^{g_{t+\delta}^{(0)}}(x,a)-\sigma_{\rho_0}^{g_{t}^{(0)}}(x,a)\right| \leq C_T\cdot \delta.
        \end{aligned}
    \end{equation*}

    Moreover, the pre-limit stationary measures $\pi^{g_t^N}$ and $\sigma_{\rho_0}^{g_t^N}$ satisfy
    \begin{equation*}
        \begin{aligned}
            \max_{x,a\in \mathcal{X}\times\mathcal{A}} \left|\pi^{g_{t+\delta}^N}(x,a)-\pi^{g_{t}^N}(x,a)\right|+
            \max_{x,a\in \mathcal{X}\times\mathcal{A}} \left|\sigma_{\rho_0}^{g_{t+\delta}^N}(x,a)-\sigma_{\rho_0}^{g_{t}^N}(x,a)\right| \leq C_T\cdot \delta + O_{L_p}(N^{-\phi}).
        \end{aligned}
    \end{equation*}
    \end{lemma}
    \begin{proof}
        The first two results follow from lemma \ref{f_t_g_t_reg_lemma} and equation \eqref{measure_lipschitz_assumption_eq}, which is based on the Lipschitz assumption of the stationary measures. For the last two results we use the same triangle Inequality argument as in lemma \ref{Q_t_P_t_reg_lemma} along with the bounds of proposition \ref{leading_order_conv_prop}.
    \end{proof}
    \subsubsection{Regularity of the first order pre-limit processes}
        \begin{lemma} \label{Q_t^{(1)}n_P_t^{(1)}N_reg_lemma}
        There exist uniform constants $C_T$ and $N_T$ independent of $N$, such that the processes $Q_t^{1,N}$ and $P_t^{1,N}$ satisfy the following bounds for $N>N_T$ 
            \begin{align*}
                \max_{(x,a)\in \mathcal{X}\times\mathcal{A}}\left| Q_{t+\delta}^{1,N}(x,a) - Q_t^{1,N}(x,a)\right| &\leq \delta \cdot O_{L_p}(1)+O_{L_p}(N^{\phi-\frac{1}{2}})\\
                \max_{(x,a)\in \mathcal{X}\times\mathcal{A}}\left| P_{t+\delta}^{1,N}(x,a) - P_t^{1,N}(x,a)\right| &\leq \delta \cdot O_{L_p}(1)+O_{L_p}(N^{\phi-\frac{1}{2}}).
            \end{align*}
    \end{lemma}
    \begin{proof}
        Using equation \eqref{Q_t^1N_eq}, we can write
        \begin{align}
            &Q_{t+\delta}^{1,N}(x,a) - Q_t^{1,N}(x,a)=\nonumber\\       &=\sum_{\substack{k_1,k_2,k_3,k_4 \in \{0,1\} \\ k_1+k_2+k_3+k_4\geq 1}} \alpha N^{\left(1-\sum_{i=1}^4 k_i\right) \phi} \int_{t}^{t+\delta} \sum_{(x',a'),(x'',a'') \in \mathcal{X}\times \mathcal{A}} \Bigg[ r(x',a') + \gamma Q_s^{(0)}(x'',a'') - Q_s^{(0)}(x',a') \Bigg]^{1-k_1}\nonumber\\
            &\hspace{10pt} \Bigg[d Q_s^{1,N}(x'',a'') - Q_s^{1,N}(x',a') \Bigg]^{k_1}\Bigg[\langle B_{(x,a),(x',a')}, v_0^{(0)} \rangle \Bigg]^{1-k_2} \Bigg[\langle B_{(x,a),(x',a')}, v_s^{1,N}\rangle \Bigg]^{k_2}\nonumber\\
            &\hspace{10pt} \Bigg[ g_s^{(0)}(x'',a'')\Bigg]^{1-k_3} \Bigg[g_s^{1,N}(x'',a'') \Bigg]^{k_3} \Bigg[ \pi^{g_s^{(0)}}(x',a')\Bigg]^{1-k_4} \Bigg[\pi^{g_s^{1,N}}(x',a') \Bigg]^{k_4}p(x'',a''|x',a')ds\nonumber\\
            &\hspace{10pt} + N^\phi \left(M_{t+\delta}^{1,N}((x,a)) - M_{t}^{1,N}((x,a))\right) +N^\phi \left(M_{t+\delta}^{2,N}((x,a)) - M_{t}^{2,N}((x,a))\right)\nonumber\\
            &+ N^\phi \left(M_{t+\delta}^{3,N}((x,a)) - M_{t}^{3,N}((x,a))\right)+ N^{\phi}\left(R_{Q,t+\delta}^N-R_{Q,t}^N\right).\label{Q_t^{(1)}N_reg_eq}
        \end{align}

    All terms outside of the summation can be bounded by $C_{T}N^{\phi-\frac{1}{2}}$ by lemma \ref{M_bound_lemma} and the bounds \eqref{R_{Q,t}^N_bounds_eq}. Moreover, by using Hölder's inequality and the bounds \eqref{Q_t^{(1)}N_P_t^{(1)}N_L_p_bounds}, \eqref{f_t^{(1)}N_g_t^{(1)}N_L_p_bounds}, \eqref{stationary_measures_errors_L_p_bounds}, and \eqref{empirical_measure_errors_L_p_bounds},  we can bound each integrand in the summation in $L_p$ by a constant $C_{T,p}$. The integral inequality \eqref{ML-Inequality} then gives the first bound. The result for $P_t^{1,N}$ can be derived analogously.
    \end{proof}
    \begin{lemma} \label{stationary_distribution_error_reg_lemma}
    For any fixed $h\in C_b^2(\mathbb{R})$, the processes $v_t^{1,N}$ and $\mu_t^{1,N}$ satisfy 
    \begin{equation}
        \begin{aligned}
             \left| \langle h, v_{t+\delta}^{1,N}\rangle - \langle h, v_t^{1,N}\rangle \right| +\left| \langle h, \mu_{t+\delta}^{1,N}\rangle - \langle h, \mu_t^{1,N}\rangle \right| \leq \delta \cdot O_{L_p}(N^{\phi -1+\beta})+O_{L_p}(N^{\phi-\frac{1}{2}}).
        \end{aligned}
    \end{equation}
   \end{lemma}
   \begin{proof}
       We use equation \eqref{v_t^N_error_eq} to write
        \begin{align}
            \langle h, v_{t+\delta}^{1,N}\rangle &- \langle h, v_t^{1,N}\rangle \nonumber\\
            &= \alpha N^{\phi-1+\beta} \int_t^{t+\delta} \sum_{(x',a'),(x'',a'') \in \mathcal{X}\times\mathcal{A}} \left(r(x',a')+\gamma Q_s^N(x'',a'')- Q_s^N(x',a')\right)\nonumber\\
            &\hspace{30pt}\langle C_{(x',a')}^h, v_s^N \rangle\pi^{g_s^N}(x',a')g_{s}^N(x'',a'')p(x''|x',a')ds\nonumber\\
            &+N^{\phi-1+\beta}\left(M_{h,t+\delta}^{1,N}- M_{h,t}^{1,N} \right)+N^{\phi-1+\beta}\left(M_{h,t+\delta}^{2,N}- M_{h,t}^{2,N} \right)+N^{\phi-1+\beta}\left(M_{h,t+\delta}^{3,N}- M_{h,t}^{3,N} \right)\nonumber\\
            &+ N^{\phi-1+\beta}\left( R_{v, t+\delta}^{1,N} - R_{v, t}^{1,N}\right).\label{v_t^N_reg_eq}
        \end{align}

     By lemma \ref{M_ht_L_1_bound_lemma} and the bounds \eqref{R_v_bound_eq} we see that the terms outside of the integral are in $O_{L_p}(N^{\phi-\frac{1}{2}})$ and thus small. Moreover, by using Hölder's inequality and the bounds \eqref{Q_t^{(1)}N_P_t^{(1)}N_L_p_bounds}, \eqref{f_t^{(1)}N_g_t^{(1)}N_L_p_bounds}, \eqref{stationary_measures_errors_L_p_bounds}, and \eqref{empirical_measure_errors_L_p_bounds},  we conclude that the integrand is of order $O_{L_p}(1)$. The integral inequality \eqref{ML-Inequality} then gives the first result. The result for $\mu_t^{1,N}$ can be shown analogously.
   \end{proof}
   \begin{lemma} \label{f_t^{(1)}N_g_t^{(1)}N_reg_lemma}
        There exist uniform constants $C_T$ and $N_T$ independent of $N$, such that the processes $Q_t^{1,N}$ and $P_t^{1,N}$ satisfy the following bounds for $N>N_T$ 
        \begin{equation}
            \begin{aligned}
                \max_{(x,a)\in \mathcal{X}\times\mathcal{A}}\left| f_{t+\delta}^{1,N}(x,a) - f_t^{1,N}(x,a)\right| \leq \delta \cdot O_{L_p}(1) + O_{L_p}(N^{-\phi})\\
                \max_{(x,a)\in \mathcal{X}\times\mathcal{A}}\left| g_{t+\delta}^{1,N}(x,a) - g_t^{1,N}(x,a)\right| \leq \delta \cdot O_{L_p}(1) +O_{L_p}(N^{-\phi}).
            \end{aligned}
        \end{equation}
    \end{lemma}
    \begin{proof}
        The first order Taylor expansion of the softmax function gives
   \begin{equation} \label{f_t^{(1)}N_reg_taylors_eq}
       \begin{aligned}
           f_{t+\delta}^{1,N}(x,a) &= \sum_{a'\in \mathcal{A}}f_{t+\delta}^{N,*}(x,a)\left(\mathds{1}\{ a=a'\} - f_{t+\delta}^{N,*}(x,a') \right)P_{t+\delta}^{1,N}(x,a')\\
           f_{t}^{1,N}(x,a) &= \sum_{a'\in \mathcal{A}}f_{t}^{N,*}(x,a)\left(\mathds{1}\{ a=a'\} - f_{t}^{N,*}(x,a') \right)P_{t}^{1,N}(x,a'),
       \end{aligned}
   \end{equation}
   where $f_{t}^{N,*}(x,a)$ and $f_{t+\delta}^{N,*}(x,a)$ lie on the segments connecting $f_t^N(x,a)$ with $f_t^{(0)}(x,a)$ and $f_{t+\delta}^N(x,a)$ with $f_{t+\delta}^0(x,a)$ respectively, for every $(x,a)\in \mathcal{X}\times \mathcal{A}$. From \eqref{f_t^{(1)}N_reg_taylors_eq} we now get
   \begin{equation} \label{f_t^{(1)}N_reg_eq}
       \begin{aligned}
           f_{t+\delta}^{1,N}(x,a)&-f_t^{1,N}(x,a) =  \sum_{\substack{k_1,k_2,k_3 \in \{0,1\} \\ k_1+k_2+k_3\geq 1}} \sum_{a'\in \mathcal{A}} \left[ f_{t+\delta}^{N,*}(x,a)-f_t^{N,*}(x,a)\right]^{k_1} \left[ f_t^{N,*}(x,a)\right]^{1-k_1}\\
           &\hspace{60pt} \left[ f_{t}^{N,*}(x,a')-f_{t+\delta}^{N,*}(x,a')\right]^{k_2} \left[ \mathds{1}\{a=a'\} - f_t^{N,*}(x,a')\right]^{1-k_2}\\
           &\hspace{60pt} \left[ P_{t+\delta}^{1,N}(x,a')-P_t^{1,N}(x,a')\right]^{k_3} \left[ P_t^{1,N}(x,a')\right]^{1-k_3}.
       \end{aligned}
   \end{equation}

   Now by the triangle inequality, lemma \ref{f_t_g_t_reg_lemma} and proposition \ref{leading_order_conv_prop} we get that for any $(x,a)\in \mathcal{X}\times\mathcal{A}$,
       \begin{align}
           \left| f_{t+\delta}^{N,*}(x,a)-f_t^{N,*}(x,a)\right| &\leq \left| f_{t+\delta}^{N,*}(x,a)-f_{t+\delta}^{(0)}(x,a)\right| + \left|f_{t+\delta}^{(0)}(x,a) - f_t^{(0)}(x,a) \right| + \left| f_t^{(0)}(x,a) - f_t^{N,*}(x,a)\right|\nonumber\\
           &\leq \left| f_{t+\delta}^{N}(x,a)-f_{t+\delta}^{(0)}(x,a)\right| + \left|f_{t+\delta}^{(0)}(x,a) - f_t^{(0)}(x,a) \right| + \left| f_t^{(0)}(x,a) - f_t^{N}(x,a)\right|\nonumber\\
           &\leq C_T\cdot \delta + O_{L_p}(N^{-\phi}).\label{f_t^{(1)}N_reg_bound1_eq}
       \end{align}

   Using  Hölder's inequality along with the bound \eqref{f_t^{(1)}N_reg_bound1_eq}, lemma \ref{Q_t^{(1)}n_P_t^{(1)}N_reg_lemma}, the $L_p$ bound for $P_t^{1,N}$ in \eqref{Q_t^{(1)}N_P_t^{(1)}N_L_p_bounds}, and the boundedness of softmax outputs, the first result follows. 
   We now write
       \begin{align}
           g_t^{1,N}(x,a) &= N^{\phi}\left( g_t^N(x,a)-g_t^{(0)}(x,a)\right)\nonumber\\
           &= N^{\phi} \left( \frac{\eta_t^N}{|\mathcal{A}|}+(1-\eta_t^N)f_t^N(x,a) - \frac{\eta_t}{|\mathcal{A}|}-(1-\eta_t)f_t^{(0)}(x,a) \right)\nonumber\\
            &= N^{\phi} \left( \frac{\eta_t^N-\eta_t}{|\mathcal{A}|}-\eta_t(f_t^N(x,a) - f_t^{(0)}(x,a)) - (\eta_t^N-\eta_t)f_t^N(x,a)+ (f_t^N(x,a)-f_t^{(0)}(x,a)) \right)\nonumber\\
            &= (1-\eta_t)f_t^{1,N}(x,a)+O(N^{1-\phi}),\nonumber
       \end{align}
   where we used lemmas \ref{eta_t^N_convergence_error_bound} and the boundedness of Softmax outputs. The second result then follows from the first result and the bound \eqref{eta_reg_bound_eq}. 
    \end{proof}

    \begin{lemma}\label{stationary_measure_errors_reg_lemma}
       There exist uniform constants $C_T$ and $N_T$ independent of $N$, such that the processes $Q_t^{1,N}$ and $P_t^{1,N}$ satisfy the following bounds for $N>N_T$ 
        \begin{equation}
            \begin{aligned}
                \max_{(x,a)\in \mathcal{X}\times\mathcal{A}}\left| \pi^{g_{t+\delta}^{1,N}} - \pi^{g_{t}^{1,N}}\right|+\max_{(x,a)\in \mathcal{X}\times\mathcal{A}}\left| \sigma_{\rho_0}^{g_{t+\delta}^{1,N}} - \sigma_{\rho_0}^{g_{t}^{1,N}}\right|  \leq \delta\cdot O_{L_p}(1)+O_{L_p}(N^{-\phi}).
            \end{aligned}
        \end{equation}
   \end{lemma}
   \begin{proof}
       We prove the result for $\pi^{g_t^{1,N}}$. The result for $\sigma_{\rho_0}^{g_t^{1,N}}$ can be shown analogously. By our analysis in section \ref{first_order_error_tersms_sec} and in particular in the proof of lemma \ref{stationary_measures_leading_order_error_conv_lemma}, we know that
       \begin{equation}
           \begin{aligned}
               \pi^{g_{t+\delta}^{1,N}} &= -\pi^{g_{t+\delta}^{(0)}} \mathbb{G}_{t+\delta}^{1,N}\left( I-\mathbb{P}_{t+\delta}^N+W_{t+\delta}^N \right)^{-1}\\
               \pi^{g_{t}^{1,N}} &= -\pi^{g_{t}^{(0)}}  \mathbb{G}_{t}^{1,N}\left( I-\mathbb{P}_{t}^N+W_{t}^N \right)^{-1}.
           \end{aligned}
       \end{equation}

       Where $\mathbb{G}_t^{1,N}((x,a),(x',a')) = g_t^{1,N}(x',a')p(x'|x,a)$. This gives
       \begin{equation} \label{pi_reg_eq1}
           \begin{aligned}
             \pi^{g_{t+\delta}^{1,N}}-\pi^{g_{t}^{1,N}}=&\sum_{\substack{k_1,k_2,k_3 \in \{0,1\} \\ k_1+k_2+k_3\geq 1}} \left[\pi^{g_t^{(0)}}-\pi^{g_{t+\delta}^{(0)}} \right]^{k_1} \left[-\pi^{g_t^{(0)}} \right]^{1-k_1}\left[ \mathbb{G}_{t+\delta}^{1,N}-\mathbb{G}_t^{1,N}\right]^{k_2}\left[\mathbb{G}_t^{1,N} \right]^{1-k_2}\\
             &\hspace{40pt} \left[\left( I-\mathbb{P}_{t+\delta}^N+W_{t+\delta}^N\right)^{-1} - \left(I-\mathbb{P}_{t}^N+W_{t}^N \right)^{-1} \right]^{k_3}\left[\left(I-\mathbb{P}_{t}^N+W_{t}^N \right)^{-1} \right]^{1-k_3}
           \end{aligned}
       \end{equation}

       We will now bound each factor in \eqref{pi_reg_eq1}. We start with noticing that by lemmas \ref{f_t_g_t_reg_lemma} and \ref{pi_t_sigma_t_reg_lemma} 
\begin{equation}\label{matrices_reg_bounds_eq1}
           \begin{aligned}
               \max_{(x,a),(x',a')\in \mathcal{X}\times\mathcal{A}} \left|\mathbb{P}_{t+\delta}^N((x,a),(x',a')) - \mathbb{P}_t^N((x,a),(x',a'))\right| &= C_T \cdot \delta + O_{L_p}(N^{-\phi}) \\
               \max_{(x,a),(x',a')\in \mathcal{X}\times\mathcal{A}} \left|W_{t+\delta}((x,a),(x',a')) - W_t((x,a),(x',a'))\right| &= C_T \cdot \delta + O_{L_p}(N^{-\phi}).
           \end{aligned}
       \end{equation}

       We now apply the matrix inversion error bound \eqref{matrix_inv_error_bound_eq} using the matrices 
       \begin{equation}
           \begin{aligned}
               A &= -\mathbb{P}_t^N+W_{\pi^{g_t^N}}\\
               \delta A &= -\mathbb{P}_{t+\delta}^N+\mathbb{P}_t^N + W_{t+\delta}^N-W_{\pi^{g_t^N}}.
           \end{aligned}
       \end{equation}

       Using \eqref{matrices_reg_bounds_eq1} and following the same steps as in section \ref{first_order_error_tersms_sec}, we can show
\begin{equation}\label{matrices_reg_bounds_eq2}
           \begin{aligned}
               &\| \left( I-\mathbb{P}_{t+\delta}^N+W_{t+\delta}^N\right)^{-1} - \left(I-\mathbb{P}_{t}^N+W_{t}^N \right)^{-1} \|_{\infty}\\
               &\leq  C\|  \left(I-\mathbb{P}_{t}^N+W_{t}^N \right)^{-1} \|_{\infty}^2 \cdot \left(C_T\cdot \delta + O_{L_p}(N^{-\phi})\right) +O(\delta^2) + O_{L_p}(N^{-2\phi})\\
               & \leq C_T \cdot \delta + O_{L_p}(N^{-\phi}).
            \end{aligned}
       \end{equation}

       Since by our analysis in section \ref{first_order_error_tersms_sec} we also have 
       \begin{equation} \label{matrices_reg_bounds_eq3}
           \begin{aligned}
               \left(I-\mathbb{P}_{t}^N+W_{t}^N \right)^{-1} \|_{\infty} \leq   \left(I-\mathbb{P}_{t}+W_{t}\right)^{-1} \|_{\infty} + O_{L_p}(N^{-\phi}).
           \end{aligned}
       \end{equation}

       Using  Hölder's inequality in \eqref{pi_reg_eq1} along with the bounds \eqref{matrices_reg_bounds_eq2}, \eqref{matrices_reg_bounds_eq3}, \eqref{f_t^{(1)}N_g_t^{(1)}N_L_p_bounds} and lemmas \ref{pi_t_sigma_t_reg_lemma} and \ref{f_t^{(1)}N_g_t^{(1)}N_reg_lemma} gives the first result. The second can be shown analogously.
    \end{proof}

    \subsubsection{Uniqueness of the Solution} \label{uniqueness_sec_1}

    We will now show that there can only be one tuple of processes $(Q_t^{(1)}, P_t^{(1)}, v_t^{(1)}, \mu_t^{(1)}, f_t^{(1)}, g_t^{(1)}, {\pi}^{g_t^{(1)}}, {\sigma}_{\rho_0}^{g_t^{(1)}})$ satisfying equations \eqref{Q_t^{(1)}_def_eq}, \eqref{P_t^{(1)}_def_eq}, and \eqref{leading_order_error_terms_limits_eq}. For this purpose, assume a second tuple of processes $(\tilde{Q}_t^{(1)}, \tilde P_t^{(1)},\tilde v_t^{(1)}, \tilde \mu_t^{(1)}, \tilde f_t^{(1)},\tilde  g_t^{(1)}, \tilde {\pi}^{g_t^{(1)}},\tilde {\sigma}_{\rho_0}^{g_t^{(1)}})$ to satisfy the same equations. By the uniqueness of the leading order limit equation solutions to \eqref{Q_t_P_t_def_eq} and  \eqref{actor_policy_learning_rates_limit_eq} we immediately see that $v_t^{(1)} = \tilde v_t^{(1)}$ and $\mu_t^{(1)} = \tilde \mu_t^{(1)}$. Moreover, from equations \eqref{Q_t^{(1)}_def_eq} and \eqref{P_t^{(1)}_def_eq} we get
    \begin{equation} \label{diff_Q1_eq}
        \begin{aligned}
            Q_t^{(1)}(x,a) &- \tilde Q_t^{(1)}(x,a) = Q_0^{(1)} (x,a) - \tilde Q_0^{(1)}(x,a) \\
            &+ \alpha \int_{0}^t \sum_{(x',a'),(x'',a'') \in \mathcal{X}\times \mathcal{A} } \left[\gamma\left( Q_s^{(1)}(x'',a'') - \tilde Q_s^{(1)}(x'',a'') \right) - \left( Q_s^{(1)}(x',a') - \tilde  Q_s^{(1)}(x',a') \right) \right]
             \\
            &\hspace{40pt} \langle B_{(x,a),(x',a')}, v_0^{(0)} \rangle g_s^{(0)}(x'',a'') \pi^{g_s^{(0)}}(x',a') p(x'',a''|x',a')ds\\
            &+ \alpha \int_{0}^t \sum_{(x',a'),(x'',a'') \in \mathcal{X}\times \mathcal{A} } \left(r(x',a') + \gamma Q^{(0)}_s(x'',a'') - Q^{(0)}_s(x',a') \right)
             \langle B_{(x,a),(x',a')}, v_0^{(0)} \rangle\\
            &\hspace{40pt} \left(g^{(1)}_s(x'',a'') -\tilde g^{(1)}_s(x'',a'')\right)\pi^{g_s^{(0)}}(x',a') p(x'',a''|x',a')ds\\
            &+ \alpha \int_{0}^t \sum_{(x',a'),(x'',a'') \in \mathcal{X}\times \mathcal{A} } \left(r(x',a') + \gamma Q^{(0)}_s(x'',a'') - Q^{(0)}_s(x',a') \right)
             \langle B_{(x,a),(x',a')}, v_0^{(0)} \rangle\\
            &\hspace{40pt}  g_s^{(0)}(x'',a'')\left( \pi^{g_s^{(1)}}(x',a') -\tilde \pi^{g_s^{(1)}}(x',a') \right) p(x'',a''|x',a')ds,
        \end{aligned}
    \end{equation}   
    \begin{equation} \label{diff_P1_eq}
        \begin{aligned}
            P_t^{(1)}(x,a) &- \tilde P_t^{(1)}(x,a) = P_0^{(1)}(x,a) - \tilde P_0^{(1)}(x,a)\\
            &+  \int_{0}^t \sum_{(x',a',x'') \in \mathcal{X}\times \mathcal{A} \times \mathcal{X}} \zeta_s \left(Q^{(1)}_s(x,a) - \tilde Q^{(1)}_s(x,a) \right)
             \left(\langle B_{(x,a),(x',a')}, v_0^{(0)} \rangle - \langle B_{(x,a),(x',a'')}, v_0^{(0)}\rangle  \right)\\
            &\hspace{40pt} f_s^{(0)}(x',a'')\sigma_{\rho_0}^{g_s^{(0)}}(x',a')ds\\
            &+\int_{0}^t \sum_{(x',a',x'') \in \mathcal{X}\times \mathcal{A} \times \mathcal{X}} \zeta_s Q^{(0)}_s(x,a)
             \left(\langle B_{(x,a),(x',a')}, v_0^{(0)} \rangle - \langle B_{(x,a),(x',a'')}, v_0^{(0)}\rangle  \right)\\
            &\hspace{40pt}  \left( f_s^{(1)}(x',a'') - \tilde f_s^{(1)}(x',a'')\right)\sigma_{\rho_0}^{g_s^{(0)}}(x',a')ds\\
            &+ \int_{0}^t \sum_{(x',a',x'') \in \mathcal{X}\times \mathcal{A} \times \mathcal{X}} \zeta_sQ_s^{(0)}(x,a)
             \left(\langle B_{(x,a),(x',a')}, v_0^{(0)} \rangle - \langle B_{(x,a),(x',a'')}, v_0^{(0)}\rangle  \right)\\
            &\hspace{40pt}  f_s^{(0)}(x',a'') \left(\sigma_{\rho_0}^{g_s^{(1)}}(x',a') - \tilde \sigma_{\rho_0}^{g_s^{(1)}}(x',a')\right)ds,
        \end{aligned}
    \end{equation}

    Using equations \eqref{leading_order_error_terms_limits_eq} we can bound the differences$ f^{(1)}_s -\tilde f^{(1)}_s$, $g^{(1)}_s -\tilde g^{(1)}_s$, $ \pi^{g_s^{(1)}} -\tilde \pi^{g_s^{(1)}}$ and $\sigma_{\rho_0}^{g_s^{(1)}} - \tilde \sigma_{\rho_0}^{g_s^{(1)}}$ in terms of $P_t^{(1)} - \tilde P_t^{(1)}$, so from equations \eqref{diff_Q1_eq} and \eqref{diff_P1_eq} we obtain
    \begin{equation} \label{diff_Q1_eq2}
        \begin{aligned}
            \max_{(x,a)\in \mathcal{X}\times\mathcal{A}}&\left|Q_t^{(1)}(x,a) - \tilde Q_t^{(1)}(x,a)\right| \leq  \max_{(x,a)\in \mathcal{X}\times\mathcal{A}}\left|Q_0^{(1)} (x,a) - \tilde Q_0^{(1)}(x,a)\right| \\
            &+ C \int_{0}^t \max_{(x,a)\in \mathcal{X}\times\mathcal{A}}\left|Q_s^{(1)}(x,a) - \tilde Q_s^{(1)}(x,a)\right|ds\\
            &+ \left( C+C\sup_{\tau \in [0,t]} \max_{(x,a)\in \mathcal{X}\times\mathcal{A}}\left|Q_{\tau}^{(0)}(x,a)\right|\right) \int_{0}^t \max_{(x,a)\in \mathcal{X}\times\mathcal{A}}\left|P_s^{(1)}(x,a) - \tilde P_s^{(1)}(x,a)\right|ds\\
        \end{aligned}
    \end{equation}
       \begin{equation} \label{diff_P1_eq2}
        \begin{aligned}
            \max_{(x,a)\in \mathcal{X}\times\mathcal{A}}&\left|P_t^{(1)}(x,a) - \tilde P_t^{(1)}(x,a)\right| \leq  \max_{(x,a)\in \mathcal{X}\times\mathcal{A}}\left|P_0^{(1)} (x,a) - \tilde P_0^{(1)}(x,a)\right| \\
            &+ C \int_{0}^t \max_{(x,a)\in \mathcal{X}\times\mathcal{A}}\left|Q_s^{(1)}(x,a) - \tilde Q_s^{(1)}(x,a)\right|ds\\
            &+ \left( C+C\sup_{\tau \in [0,t]} \max_{(x,a)\in \mathcal{X}\times\mathcal{A}}\left|Q_{\tau}^{(0)}(x,a)\right|\right) \int_{0}^t \max_{(x,a)\in \mathcal{X}\times\mathcal{A}}\left|P_s^{(1)}(x,a) - \tilde P_s^{(1)}(x,a)\right|ds\\
        \end{aligned}
    \end{equation}

    By the convergence of the critic network to the solution of the Bellman Equation as shown in section \ref{critic_conv_sec}, the quantity $\sup_{\tau \in [0,t]} \max_{(x,a)\in \mathcal{X}\times\mathcal{A}}\left|Q_{\tau}^{(0)}(x,a)\right|$ can be uniformly bounded, and so adding \ref{diff_Q1_eq2} and \eqref{diff_P1_eq2} and using the Grönwall inequality we obtain
    \begin{equation}
    \begin{aligned}
        \max_{(x,a) }\left|Q_t^{(1)}(x,a) - \tilde Q_t^{(1)}(x,a)\right|& \leq C \left( \max_{(x,a) }\left|Q_0^{(1)}(x,a) - \tilde Q_0^{(1)}(x,a)\right|+ \max_{(x,a) }\left|P_0^{(1)}(x,a) - \tilde P_0^{(1)}(x,a)\right| \right)\\
        \max_{(x,a) }\left|P_t^{(1)}(x,a) - \tilde P_t^{(1)}(x,a)\right|  &\leq C \left( \max_{(x,a) }\left|Q_0^{(1)}(x,a) - \tilde Q_0^{(1)}(x,a)\right|+ \max_{(x,a) }\left|P_0^{(1)}(x,a) - \tilde P_0^{(1)}(x,a)\right| \right)
    \end{aligned}
    \end{equation}

The latter now gives uniqueness for same initial conditions.

    \subsubsection{Proof of Convergence} \label{first_order_conv_small_beta_sec}

    By Theorem 8.6 and Remark 8.7 in \cite{mpcc}, relative compactness for the processes $Q_t^{1,N}$, $P_t^{1,N}$, $v_t^{1,N}$, $\mu_t^{1,N}$, $f_t^{1,N}$, $g_t^{1,N}$, ${\pi}^{g_t^{1,N}}$, and ${\sigma}_{\rho_0}^{g_t^{1,N}}$ follows from the compact containment and regularity results from section \ref{leading_order_relcomp_sec}, which implies the existent of a converging subsequence. Convergence follows from the uniqueness of limit points shown in section \ref{uniqueness_sec_1}. Hence the process $(Q_t^{1,N}, P_t^{1,N}, v_t^{1,N}, \mu_t^{1,N}, f_t^{1,N}, g_t^{1,N}, {\pi}^{g_t^{1,N}}, {\sigma}_{\rho_0}^{g_t^{1,N}})$ converges weakly to a unique limit point. For the processes $Q_t^{1,N}$ and $P_t^{1,N}$, since integration is a continuous operator, this implies that we can take the limits of the involved processes on both sides of the equations \eqref{Q_t^1N_eq} and \eqref{P_t^1N_eq}. Since the martingale and remainder terms vanish in the large $N$ limit, they do not contribute to the limits $Q_t^{(1)}$ and $P_t^{(1)}$ of  $Q_t^{1,N}$ and $P_t^{1,N}$ respectively. We obtain that $Q_t^{(1)}$ and $P_t^{(1)}$ will satisfy their corresponding equations in \eqref{leading_order_error_terms_limits_eq}. Weak convergence of $f_t^{1,N}$ and $g_t^{1,N}$ to $f_t^{(1)}$ and $g_t^{(1)}$ as defined in \eqref{leading_order_error_terms_limits_eq} then follows from the third and fourth bound of lemma \ref{g_t^N_bound_lemma2}. Weak convergence of $\pi^{g_t^{1,N}}$ and $\sigma_{\rho_0}^{g_t^{1,N}}$ to $\pi^{g_t^{(1)}}$ and $\sigma_{\rho_0}^{g_t^{(1)}}$ respectively, as defined in \eqref{leading_order_error_terms_limits_eq}, follows from lemma \ref{stationary_measures_leading_order_error_conv_lemma}.

    Lastly, for the processes $v_t^{1,N}$ we can take the limits on both sides of \eqref{v_t^N_error_eq} as the integral is a continuous operator, and since the martingale and remainder terms vanish by the bound \eqref{R_v_bound_eq} and lemma \ref{M_ht_L_1_bound_lemma}, the limit $v_t^{(1)}$ satisfies its equation in \eqref{leading_order_error_terms_limits_eq}. The process $\mu_t^{1,N}$ is handled analogously.

    Given the intermediate results mentioned above, a formal proof for the weak convergence of the processes $Q_t^{1,N}$, $P_t^{1,N}$, $v_t^{1,N}$ and $\mu_t^{1,N}$ to their corresponding limits as defined in \eqref{leading_order_error_terms_limits_eq} can be derived by using the same steps as in lemma 4.20 of \cite{nac}, see also \cite{lln,clt,normeff}, which we refer to for more details.

    \section{Conclusion}

    In conclusion, our work provides a fine-grain analysis of the neural actor-critic algorithm for properly scaled shallow neural networks. We view the actor and critic neural networks as stochastic processes and we characterize their  asymptotic behavior as the number of hidden units $N$ grow to infinity. The derived asymptotic  expansion of the actor and critic neural networks (\ref{expansion_form_eq}) provides a bias-variance decomposition of both actor and critic to leading order in $N$. 
    
    This asymptotic expansion highlights the fundamental importance of scaling—that is, the appropriate initialization of parameters—when neural networks are employed as estimators. Unlike the neural network regression problem \cite{lln, clt, normeff, normdeep}, where the bias diminishes rapidly enough in training time for its effect on the convergence error to be negligible, our results reveal a distinct bias–variance trade-off in the actor–critic framework. We conjecture that this behavior arises from the algorithm’s online learning nature and/or the coupled dynamics of the actor and critic networks, which induce a more persistent bias component.

    In addition, the mathematical proof of convergence leads to natural choices for the hyperparameters of interest, like the learning rates and exploration rates. Under these choices convergence of the algorithm with statistically good behavior is ensured.

    There are many interesting future directions one could take from here. For starters, it would be interesting to study deep neural networks, as opposed to shallow neural networks. This was done for the regression problem in \cite{normdeep} and it was found that different layers have different sensitivity regarding scalings. It is interesting to explore similar kind of questions in the neural actor-critic setup. In addition, in this paper we studied the behavior for the standard SGD algorithm. It is interesting to study what happens for other widely used optimization algorithms like RMSProp or Adam.




    \begin{appendix}
        \section{Evolution of the pre-limit processes} \label{prelim_sec_app}

 \begin{proof}[Proof of lemma \ref{parameter_bounds_apriori_lemma}.]
    The first two bounds can be proven in the same way as the first two bounds of lemma 4.2 in \cite{nac}. We provide a proof for the third bound. We observe that
        \begin{equation*}
            |B_{k+1}^i-B_k^i| \leq N^{-\beta}\zeta_k^N \max_{\xi \in \mathcal{X}\times \mathcal{A}}|Q_k^N(\xi)|\left(\max_{\xi \in \mathcal{X}\times \mathcal{A}}|\sigma(U_k^i\cdot \xi)| +|\mathcal{A}|\max_{\xi \in \mathcal{X}\times \mathcal{A}}|f_k^N( \xi)||\sigma(U_k^i\cdot \xi)|\right) \leq \frac{C_T}{N}, 
        \end{equation*}
        where we used the bounds
        \begin{equation*}
            \begin{aligned}
                \max_{\xi \in \mathcal{X}\times \mathcal{A}}|Q_k^N(\xi)| &\leq C_TN^{1-\beta}, \quad
                \max_{\xi \in \mathcal{X}\times \mathcal{A}}|f_k^N( \xi)| &\leq 1, \quad
                \max_{\xi \in \mathcal{X}\times \mathcal{A}}|\sigma(U_k^i\cdot \xi)| &\leq 1.\quad \zeta_k^N = \frac{1}{N^{2-2\beta}\left(1+\frac{k}{N} \right)}
            \end{aligned}
        \end{equation*}

        The first of which follows from the definition \eqref{critic_network_eq} and the first bound of this lemma, the second follows from the uniform bound of the softmax function, and the third follows from assumption \ref{actor_critic_models_assumption}. Similarly, we have
        \begin{equation*}
            \|U_{k+1}^i-U_k^i\| \leq N^{-\beta}\zeta_k^N \max_{\xi \in \mathcal{X}\times \mathcal{A}}|Q_k^N(\xi)|\left(\max_{\xi \in \mathcal{X}\times \mathcal{A}}|\sigma(U_k^i\cdot \xi)| \|\xi\| +|\mathcal{A}|\max_{\xi \in \mathcal{X}\times \mathcal{A}}|f_k^N( \xi)||\sigma(U_k^i\cdot \xi)\|\xi\| |\right) \leq \frac{C_T}{N}, 
        \end{equation*}
        where we used the same bounds along with the assumption \ref{MDP_assumptions} on the finiteness of $\mathcal{X}\times \mathcal{A}$ which implies $\max_{\xi \in \mathcal{X}\times \mathcal{A}}\|\xi\| < \infty$.
        \end{proof}

    We will now consider the evolution of the network outputs $Q_k^N(\xi), P_k^N(\xi)$. Using a Taylor expansion we have for the critic network
        \begin{align*}
            Q_{k+1}^N&(\xi)-Q_k^N(\xi) = \frac{1}{N^\beta}\sum_{i=1}^N C_{k+1}^i \sigma(W_{k+1}^i\cdot \xi) - \frac{1}{N^\beta}\sum_{i=1}^NC_k^i \sigma(W_k^i \cdot \xi)\\
            &= \frac{1}{N^\beta}\sum_{i=1}^N  \left(  (C_{k+1}^i - C_k^i) \sigma(W_{k+1}^i\cdot \xi)+C_k^i\left(  \sigma(W_{k+1}^i\cdot \xi) -  \sigma(W_{k}^i\cdot \xi)\right) \right)\\
            &= \frac{1}{N^\beta}\sum_{i=1}^N \bigg(  (C_{k+1}^i - C_k^i) \left[ \sigma(W_{k}^i\cdot \xi)+ \sigma'(W_{k}^{i,*}\cdot \xi)(W_{k+1}^i -W_k^i)\cdot \xi \right] \\
            & \hspace{1.5cm} + C_k^i \left[ \sigma'(W_{k}^{i}\cdot \xi) (W_{k+1}^i -W_k^i)\cdot \xi +\frac{1}{2} \sigma'(W_{k}^{i,**}\cdot \xi)\left((W_{k+1}^i -W_k^i)\cdot \xi \right)^2 \right] \bigg)\\
            &=\frac{\alpha}{N^{2}} \left( r(\xi_k) +\gamma Q_k^N(\xi_{k+1}) - Q_k^N(\xi_k)\right) \sum_{i=1}^N \sigma(W_k^i \cdot \xi_k)\sigma(W_k^i \cdot \xi)\\
            &+ \frac{\alpha}{N^{2}} \left( r(\xi_k) +\gamma Q_k^N(\xi_{k+1}) - Q_k^N(\xi_k)\right) \sum_{i=1}^N (C_k^i)^2 \sigma'(W_k^i \cdot \xi_k)\sigma'(W_k^i \cdot \xi) \xi_k \cdot \xi\\
            &+ \frac{1}{N^{2+\beta}} R_k^N(\xi),
        \end{align*}
    where $W_{k}^{i,*}, W_{k}^{i,**}$ are elements between $W_{k}^{i}$ and $W_{k+1}^{i}$, and
    \begin{align}
        R_k^N(\xi) &= \sum_{i=1}^N \frac{\alpha^2}{N^{2-2\beta}}\left( r(\xi_k) +\gamma Q_k^N(\xi_{k+1}) - Q_k^N(\xi_k)\right)^2 C_k^i\sigma(W_k^i \cdot \xi_k)\sigma'(W_k^{i,*} \cdot \xi)  \sigma'(W_k^i \cdot \xi_k) \xi_k \cdot \xi\nonumber\\
        &+ \frac{1}{2}\sum_{i=1}^N\frac{\alpha^2}{N^{2-2\beta}}\left( r(\xi_k) +\gamma Q_k^N(\xi_{k+1}) - Q_k^N(\xi_k)\right)^2 (C_k^i)^3 \sigma''(W_k^{i,**} \cdot \xi)  \left(\sigma'(W_k^i \cdot \xi_k) \xi_k \cdot \xi\right)^2.\label{R_k^i_eq}
    \end{align}

    Using the empirical measure $v_k^N$ the above can be written as
    \begin{equation} \label{Q_k^N_diff_eq}
        Q_{k+1}^N(\xi)-Q_k^N(\xi) = \frac{\alpha}{N} \left( r(\xi_k) +\gamma Q_k^N(\xi_{k+1}) - Q_k^N(\xi_k)\right) \langle B_{\xi, \xi_k}, v_k^N \rangle\\
        +\frac{1}{N^{2+\beta}}R_k(\xi)
    \end{equation}
    where $B$ is defined as in \eqref{B_def_eq}. Using this we write via a telescoping sum series
        \begin{align}
            Q_t^N(\xi) &= Q_0^N(\xi) + \sum_{k=0}^{\lfloor Nt \rfloor-1} (Q_{k+1}^N(\xi)-Q_k^N(\xi))\label{Q_t^N_prelim_for_gronwall}\\
            &= Q_0^N(\xi) + \frac{\alpha}{N} \sum_{k=0}^{\lfloor Nt \rfloor-1} \left( r(\xi_k) +\gamma Q_k^N(\xi_{k+1}) - Q_k^N(\xi_k)\right) \langle B_{\xi, \xi_k}, v_k^N \rangle + \frac{1}{N^{2+\beta}}\sum_{k=0}^{\lfloor Nt\rfloor -1} R_k(\xi)\nonumber
        \end{align}

    We decompose the second term on the right hand side of the formula above into a drift and martingale component and use Riemann integration to obtain
    \begin{equation}\label{Q_t^N_eq}
        \begin{aligned}
            Q_t^N(\xi) &= \frac{\alpha}{N} \sum_{k=0}^{\lfloor Nt \rfloor} \sum_{\xi',( x'', a'') \in \mathcal{X}\times \mathcal{A}} \left( r(\xi') +\gamma Q_k^N(x'',a'') - Q_k^N(\xi')\right) \langle B_{\xi, \xi'}, v_k^N \rangle \pi^{g_k^N}(\xi')g_{k}^N(x'',a'')p(x''|\xi')\\
            &\hspace{20pt}+N^{1-\beta}\langle c\sigma(w\cdot \xi), v_0^N\rangle + M_t^{1,N}(\xi) + M_t^{2,N}(\xi)+ M_t^{3,N}(\xi)+ \frac{1}{N^{2+\beta}}\sum_{k=0}^{\lfloor Nt\rfloor -1} R_k(\xi)\\
            &= \alpha \int_{0}^{t} \sum_{\xi',( x'', a'') \in \mathcal{X}\times \mathcal{A}} \left( r(\xi') +\gamma Q_s^N(x'',a'') - Q_s^N(\xi')\right) \langle B_{\xi, \xi'}, v_s^N \rangle \pi^{g_s^N}(\xi')g_{s}^N(x'',a'')p(x''|\xi')ds\\
            &\hspace{20pt}+N^{1-\beta}\langle c\sigma(w\cdot \xi), v_0^N\rangle + M_t^{1,N}(\xi) + M_t^{2,N}(\xi)+ M_t^{3,N}(\xi)+ V_t^N(\xi)+ \frac{1}{N^{2+\beta}}\sum_{k=0}^{\lfloor Nt\rfloor -1} R_k(\xi),
        \end{aligned}
    \end{equation}
    where the martingale terms are
        \begin{align}
            M_t^{1,N}(\xi) &= \frac{\alpha}{N}\sum_{k=0}^{\lfloor Nt \rfloor-1} r(\xi_k) \langle B_{\xi, \xi_k}, v_k^N \rangle - \frac{\alpha}{N}\sum_{k=0}^{\lfloor Nt \rfloor-1} \sum_{\xi' \in \mathcal{X}\times\mathcal{A}} r(\xi') \langle B_{\xi, \xi'}, v_k^N \rangle \pi^{g_k^N}(\xi')\nonumber\\
            M_t^{2,N}(\xi) &= \frac{\alpha}{N}\sum_{k=0}^{\lfloor Nt \rfloor-1}\gamma Q_k^N(\xi_{k+1}) \langle B_{\xi, \xi_k}, v_k^N \rangle \label{M_t^N_eq_app}\\
            &\hspace{20pt}- \frac{\alpha}{N}\sum_{k=0}^{\lfloor Nt \rfloor-1} \sum_{(\xi', x'',a'') \in \mathcal{X}\times\mathcal{A}\times \mathcal{X}\times \mathcal{A}}\gamma Q_k^N(x'',a'') \langle B_{\xi, \xi'}, v_k^N \rangle \pi^{g_s^N}(\xi')g_{s}^N(x'',a'')p(x''|\xi')\nonumber\\
            M_t^{3,N}(\xi) &= -\frac{\alpha}{N}\sum_{k=0}^{\lfloor Nt \rfloor-1} Q_k^N(\xi_k) \langle B_{\xi, \xi_k}, v_k^N \rangle + \frac{\alpha}{N}\sum_{k=0}^{\lfloor Nt \rfloor-1} \sum_{\xi' \in \mathcal{X}\times\mathcal{A}} Q_k^N(\xi') \langle B_{\xi, \xi'}, v_k^N \rangle \pi^{g_k^N}(\xi').
        \end{align}

The Riemann integral remainder term is
    \begin{equation} \label{V_t^N_eq}
        V_t^N(\xi) = -\alpha\int_{\frac{\lfloor Nt \rfloor}{N}}^t  \sum_{\xi',( x'', a'') \in \mathcal{X}\times \mathcal{A}} \left( r(\xi') +\gamma Q_s^N(x'',a'') - Q_s^N(\xi')\right) \langle B_{\xi, \xi'}, v_s^N \rangle \pi^{g_s^N}(\xi')g_{s}^N(z,a'')p(x''|\xi')ds.
    \end{equation}

    Using the integral inequality \eqref{ML-Inequality} and the bound \ref{Q_P_L_4_Bound} we see that there is a uniform constant $C_T$ independent of $N$ such that (recall definition \ref{O_L_p_def}) 
    \begin{equation}
    \begin{aligned}
        \sup_{t\in[0,T]}\max_{\xi \in \mathcal{X}\times \mathcal{A}}\left|V_t^N(\xi)\right| \leq C_TN^{-\beta},\quad
        \max_{\xi \in \mathcal{X}\times \mathcal{A}}\left|V_t^N(\xi)\right| = O_{L_p}(N^{-1})
    \end{aligned}
    \end{equation}

    Similarly, for the remainder terms $R_k(\xi)$ we observe that by the bounds from equation \eqref{Q_P_absolute_bound_eq} and lemma \ref{Q_P_L_4_Bound} we see that there is a uniform constant $C_T$ independent of $N$ such that 
    \begin{equation} \label{R_k^N_bound_eq}
        \begin{aligned}
            \max_{\xi\in \mathcal{X}\times\mathcal{A}}\left| R_k^N(\xi) \right| &\leq C_TN,\quad
           \max_{\xi\in \mathcal{X}\times\mathcal{A}}\left| R_k^N(\xi) \right| = O_{L_p}(N^{2\beta-1})
        \end{aligned}
    \end{equation}

    The bounds \eqref{R_{Q,t}^N_bounds_eq} follow.  Similarly the actor network output can be written as
         \begin{align*}
            P_{k+1}^N&(\xi)-P_k^N(\xi) = \frac{1}{N^\beta}\sum_{i=1}^N B_{k+1}^i \sigma(U_{k+1}^i\cdot \xi) - \frac{1}{N^\beta}\sum_{i=1}^NB_k^i \sigma(U_k^i \cdot \xi)\\
            &= \frac{1}{N^\beta}\sum_{i=1}^N  \left(  (B_{k+1}^i - B_k^i) \sigma(U_{k+1}^i\cdot \xi)+B_k^i\left(  \sigma(U_{k+1}^i\cdot \xi) -  \sigma(U_{k}^i\cdot \xi)\right) \right)\\
            &= \frac{1}{N^\beta}\sum_{i=1}^N \bigg(  (B_{k+1}^i - B_k^i) \left[ \sigma(U_{k}^i\cdot \xi)+ \sigma'(U_{k}^{i,*}\cdot \xi)(U_{k+1}^i -U_k^i)\cdot \xi \right] \\
            & \hspace{1.5cm} + B_k^i \left[ \sigma'(U_{k}^{i}\cdot \xi) (U_{k+1}^i -U_k^i)\cdot \xi +\frac{1}{2} \sigma'(U_{k}^{i,**}\cdot \xi)\left((U_{k+1}^i -U_k^i)\cdot \xi \right)^2 \right] \bigg)\\
            &=\frac{1}{N^{2}} \left( \frac{1}{1+\frac{k}{N}} Q_k^N(\tilde{\xi}_k)\right) \sum_{i=1}^N \left(  \sigma(U_k^i \cdot \tilde{\xi}_k) - f_k^N(\tilde{x}_k, a')\sigma(U_k^i\cdot(\tilde{x}_k,a'))\right)\sigma(W_k^i \cdot \xi)\\
            &+ \frac{1}{N^{2}} \left( \frac{1}{1+\frac{k}{N}} Q_k^N(\tilde{\xi}_k)\right) \sum_{i=1}^N (B_k^i)^2\left(  \sigma'(U_k^i \cdot \tilde{\xi}_k)\tilde{\xi}_k\cdot \xi - \sum_{a' \in \mathcal{A}}f_k^N(\tilde{x}_k, a')\sigma'(U_k^i\cdot(\tilde{x}_k,a')) (\tilde{x},a')\cdot \xi\right)\sigma'(U_k^i \cdot \xi)\\
            &+ \frac{1}{N^{2+\beta}} \tilde{R}_k(\xi),
        \end{align*}
    where $U_{k}^{i,*}, U_{k}^{i,**}$ are elements between $U_{k}^{i}$ and $U_{k+1}^{i}$, and
    \begin{equation} \label{R_tilde_k^i_eq}
    \begin{aligned}
        \tilde{R}_k(\xi) &= \sum_{i=1}^N\left( \frac{1}{1+\frac{k}{N}} Q_k^N(\tilde{\xi}_k)\right)^2 B_k^i \left(  \sigma(U_k^i \cdot \tilde{\xi}_k) - \sum_{a' \in \mathcal{A}}f_k^N(\tilde{x}_k, a')\sigma(U_k^i\cdot(\tilde{x}_k,a'))\right)\\
        &\hspace{20pt}\left(  \sigma'(U_k^i \cdot \tilde{\xi}_k)\tilde{\xi}_k\cdot \xi - \sum_{a' \in \mathcal{A}}f_k^N(\tilde{x}_k, a')\sigma'(U_k^i\cdot(\tilde{x}_k,a')) (\tilde{x},a')\cdot \xi\right)  \sigma'(U_k^{i,*} \cdot \xi)\\
        &+\frac{1}{2}\sum_{i=1}^N \left( \frac{1}{1+\frac{k}{N}} Q_k^N(\tilde{\xi}_k)\right)^2 (B_k^i)^3 \sigma''(U_k^{i,**} \cdot \xi) \\
        &\hspace{20pt}\left(  \sigma'(U_k^i \cdot \tilde{\xi}_k)\tilde{\xi}_k\cdot \xi - \sum_{a' \in \mathcal{A}}f_k^N(\tilde{x}_k, a')\sigma'(U_k^i\cdot(\tilde{x}_k,a')) (\tilde{x},a')\cdot \xi\right)^2.
    \end{aligned}
    \end{equation}

    Using the empirical measure $\mu_k^N$ the above can be written as
    \begin{equation} \label{P_k^N_diff_eq}
    \begin{aligned}
        P_{k+1}^N(\xi)-P_k^N(\xi) &= \frac{1}{N} \left( \frac{1}{1+\frac{k}{N}} Q_k^N(\tilde{\xi}_k)\right) \left( \langle B_{\xi, \xi_k}, \mu_k^N \rangle + \sum_{a' \in \mathcal{A}}f_k^N(\tilde{x}_k,a') \langle B_{\xi, (\tilde{x},a')}, \mu_k^N \rangle \right)\\
        &+\frac{1}{N^{2+\beta}}\tilde{R}_k(\xi).
    \end{aligned}
    \end{equation}
    where $B$ is as defined as in \eqref{B_def_eq}.
    
    Using this we write
    \begin{equation} \label{P_t^N_prelim_for_gronwall}
        \begin{aligned}
            P_t^N(\xi) &= P_0^N(\xi) + \sum_{k=0}^{\lfloor Nt \rfloor-1} (P_{k+1}^N(\xi)-P_k^N(\xi))\\
            &= P_0^N(\xi) + \frac{\alpha}{N} \sum_{k=0}^{\lfloor Nt \rfloor-1}  \left( \frac{1}{1+\frac{k}{N}} Q_k^N(\tilde{\xi}_k)\right) \left( \langle B_{\xi, \xi_k}, \mu_k^N \rangle + \sum_{a' \in \mathcal{A}}f_k^N(\tilde{x}_k,a') \langle B_{\xi, (\tilde{x},a')}, \mu_k^N \rangle \right)\\
            &\hspace{20pt}+ \frac{1}{N^{2+\beta}}\sum_{i=0}^{\lfloor Nt\rfloor -1} \tilde{R}_k(\xi).
        \end{aligned}
    \end{equation}

    We decompose the second term on the right hand side of the formula above into a drift and martingale component and use Riemann integration to obtain
    \begin{equation} \label{P_t^N_eq}
        \begin{aligned}
            P_t^N(\xi)&= \frac{\alpha}{N} \sum_{k=0}^{\lfloor Nt \rfloor-1} \sum_{({x}', {a}', a'') \in \mathcal{X}\times\mathcal{A}\times \mathcal{A}}  \left( \frac{1}{1+\frac{k}{N}} Q_k^N({x}', {a}')\right) \Bigg( \langle B_{\xi, ({x}', {a}')}, \mu_k^N \rangle \\
            &\hspace{40pt}+  \langle B_{\xi, ({x}',{a}'')}, \mu_k^N \rangle \Bigg) f_k^N({x}',{a}'')\sigma_{\rho_0}^{g_k^N}({x}',{a}')\\
            &\hspace{20pt}+ N^{1-\beta}\langle b\sigma(u\cdot {\xi}), \mu_0^N\rangle + \tilde{M}_t^N(\xi) +  \frac{1}{N^{2+\beta}}\sum_{i=0}^{\lfloor Nt\rfloor -1} \tilde{R}_k(\xi)\\
            &=\int_{0}^{t} \sum_{({x}', {a}', a'') \in \mathcal{X}\times\mathcal{A}\times \mathcal{A}}  \left( \frac{1}{1+s} Q_s^N({x}', {a}')\right) \Bigg( \langle B_{\xi, ({x}', {a}')}, \mu_s^N \rangle \\
            &\hspace{40pt}+  \langle B_{\xi, ({x}',{a}'')}, \mu_s^N \rangle \Bigg) f_s^N({x}',{a}'')\sigma_{\rho_0}^{g_s^N}({x}',{a}')ds\\
            &\hspace{20pt}+ N^{1-\beta}\langle b\sigma(u\cdot {\xi}), \mu_0^N\rangle + \tilde{M}_t^N(\xi) + \tilde{V}_t^N(\xi) + \frac{1}{N^{2+\beta}}\sum_{i=0}^{\lfloor Nt\rfloor -1} \tilde{R}_k(\xi),
        \end{aligned}
    \end{equation}
     where the martingale term is
        \begin{align}
            \tilde{M}_t^{N}(\xi) &= \frac{1}{N}\sum_{k=0}^{\lfloor Nt \rfloor-1} \left( \frac{1}{1+\frac{k}{N}} Q_k^N(\tilde{\xi}_k)\right) \left( \langle B_{\xi, \tilde{\xi}_k}, \mu_k^N \rangle + \sum_{a' \in \mathcal{A}}f_k^N(\tilde{x}_k,a') \langle B_{\xi, (\tilde{x}_k,a')}, \mu_k^N \rangle \right)\nonumber\\
            &- \frac{1}{N}\sum_{k=0}^{\lfloor Nt \rfloor-1} \sum_{({x}', {a}',a'') \in \mathcal{X}\times\mathcal{A}\times \mathcal{A}}  \left( \frac{1}{1+\frac{k}{N}} Q_k^N({x}', {a}')\right) \Bigg( \langle B_{\xi, ({x}', {a}')}, \mu_k^N \rangle\nonumber\\
            &\hspace{40pt}+  \langle B_{\xi, ({x}',{a}'')}, \mu_k^N \rangle \Bigg) f_k^N({x}',{a}'')\sigma_{\rho_0}^{g_k^N}({x}',{a}').\label{M_tilde_t^N_eq_app}
        \end{align}
 
The Riemann integral remainder term is   
    \begin{align}
        \tilde{V}_t^N(\xi)& = -\int_{\frac{\lfloor Nt \rfloor}{N}}^t \sum_{({x}', {a}') \in \mathcal{X}\times\mathcal{A}}  \left( \frac{1}{1+s} Q_s^N({x}', {a}')\right) \Bigg( \langle B_{\xi, ({x}', {a}')}, \mu_s^N \rangle \nonumber\\
        &\hspace{40pt}+ \sum_{{a}'' \in \mathcal{A}}f_s^N({x}',{a}'') \langle B_{\xi, ({x}',{a}'')}, \mu_s^N \rangle \Bigg) \sigma_{\rho_0}^{g_s^N}({x}',{a}')ds.\label{V_tilde_t^N_eq}
    \end{align}
    
    The bounds \eqref{R_{P,t}^N_bounds_eq} are derived similarly to the bounds \eqref{R_{Q,t}^N_bounds_eq}.
    We are ready to prove lemma \ref{Q_P_L_4_Bound}.

    \begin{proof}[Proof of lemma \ref{Q_P_L_4_Bound}.]
    We prove the result for the critic network output $Q_t^N$. The proof for the actor network output $P_t^N$ can be derived similarly. The constant $C_T<\infty$ may change from line to line. We first notice that by the prelimit equation \eqref{Q_k^N_diff_eq} and the bound \eqref{R_k^N_bound_eq} we get
        \begin{equation*}
            \begin{aligned}
                \max_{\xi\in \mathcal{X}\times \mathcal{A}}\left|Q_{k+1}^N(\xi) \right| \leq  \max_{\xi\in \mathcal{X}\times \mathcal{A}}\left|Q_{k}^N(\xi) \right| + \frac{C_T}{N}\max_{\xi\in \mathcal{X}\times \mathcal{A}}\left|Q_{k}^N(\xi) \right| +\frac{C_T}{N}.
            \end{aligned}
        \end{equation*}

        Raising to the power of $p$ and ignoring lower order terms (which for $N$ large enough is justified as one can simply increase the constant $C_T$)
         \begin{equation*}
            \begin{aligned}
                \max_{\xi\in \mathcal{X}\times \mathcal{A}}\left|Q_{k+1}^N(\xi) \right|^p \leq  \max_{\xi\in \mathcal{X}\times \mathcal{A}}\left|Q_{k}^N(\xi) \right|^p\left(1+\frac{C_T}{N}\right) +\frac{C_T}{N},
            \end{aligned}
        \end{equation*}
        which then implies 
        \begin{equation*}
            \begin{aligned}
                \max_{\xi\in \mathcal{X}\times \mathcal{A}}\left|Q_{k+1}^N(\xi) \right|^p - \max_{\xi\in \mathcal{X}\times \mathcal{A}}\left|Q_{k}^N(\xi) \right|^p &\leq \frac{C_T}{N} \max_{\xi\in \mathcal{X}\times \mathcal{A}}\left|Q_{k}^N(\xi) \right|^p + \frac{C_T}{N},
            \end{aligned}
        \end{equation*}
        which gives the telescoping series
        \begin{equation*}
            \begin{aligned}
                \max_{\xi\in \mathcal{X}\times \mathcal{A}}\left|Q_{k}^N(\xi) \right|^p &= \max_{\xi\in \mathcal{X}\times \mathcal{A}}\left|Q_{0}^N(\xi) \right|^p + \sum_{j=0}^{k-1} \left( \max_{\xi\in \mathcal{X}\times \mathcal{A}}\left|Q_{j+1}^N(\xi) \right|^p - \max_{\xi\in \mathcal{X}\times \mathcal{A}}\left|Q_{j}^N(\xi) \right|^p\right)\\
                &\leq \max_{\xi\in \mathcal{X}\times \mathcal{A}}\left|Q_{0}^N(\xi) \right|^p + \sum_{j=0}^{k-1}\left( \frac{C_T}{N} \max_{\xi\in \mathcal{X}\times \mathcal{A}}\left|Q_{j+1}^N(\xi) \right|^p +\frac{C_T}{N}\right)\\
                &= \max_{\xi\in \mathcal{X}\times \mathcal{A}}\left|Q_{0}^N(\xi) \right|^p + \frac{C_T}{N} \sum_{j=0}^{k-1} \max_{\xi\in \mathcal{X}\times \mathcal{A}}\left|Q_{j+1}^N(\xi) \right|^p +C_T.
            \end{aligned}
        \end{equation*}

        Taking expectation gives
        \begin{equation*}
            \begin{aligned}
                \mathds{E}\left[ \max_{\xi\in \mathcal{X}\times \mathcal{A}}\left|Q_{k}^N(\xi) \right|^p\right] \leq \mathds{E}\left[ \max_{\xi\in \mathcal{X}\times \mathcal{A}}\left|Q_{0}^N(\xi) \right|^p \right] +\frac{C_T}{N}\sum_{j=1}^N\mathds{E}\left[\max_{\xi\in \mathcal{X}\times \mathcal{A}}\left|Q_{j}^N(\xi) \right|^p \right] +C_T.
            \end{aligned}
        \end{equation*}

        The result then follows from assumption \ref{actor_critic_models_assumption} and the discrete Gronwall inequality.
\end{proof}

    We now study the evolution of the empirical measures $v_k^N$ and $\mu_k^N$. By Taylor expansion we have that for any $h \in C_b^2(\mathds{R}^{1+d})$
         \begin{align*}
            \langle h, &v_{k+1}^N \rangle - \langle h, v_{k}^N \rangle = \frac{1}{N}\sum_{i=1}h(C_{k+1}^i, W_{k+1}^i) -\frac{1}{N}\sum_{i=1}h(C_{k}^i, W_{k}^i)\\
            &= \frac{1}{N} \sum_{i=1}^N \partial_c h(C_k^i, W_k^i) (C_{k+1}^i-C_k^i) + \frac{1}{N} \sum_{i=1}^N \partial_w h(C_k^i, W_k^i)(W_{k+1}^i-W_k^i)\\
            &+ \frac{1}{2N} \sum_{i=1}^N \partial_c^2 h(\bar{C_k^{i}}, \bar{W_k^{i}}) (C_{k+1}^i-C_k^i)^2\\
            &+ \frac{1}{2N} \sum_{i=1}^N (C_{k+1}^i-C_k^i)\partial_{cw}^2 h(\hat{C_k^{i}}, \hat{W_k^{i}}) (W_{k+1}^i-W_k^i)\\
            &+ \frac{1}{2N} \sum_{i=1}^N (W_{k+1}^i-W_k^i)^T\partial_{w}^2 h(\tilde{C_k^{i}}, \tilde{W_k^{i}}) (W_{k+1}^i-W_k^i),
        \end{align*}
    for points $(\bar{C_k^{i}}, \bar{W_k^{i}}), (\hat{C_k^{i}}, \hat{W_k^{i}}), (\tilde{C_k^{i}}, \tilde{W_k^{i}})$ on the segment connecting $(C_{k}^i, W_{k}^i)$ and $(C_{k+1}^i, W_{k+1}^i)$.

    Using the parameter update equations \eqref{param_update} and lemma \ref{parameter_bounds_apriori_lemma} we have
    \begin{equation} \label{v_k^N_diff_eq}
        \begin{aligned}
            \langle h, v_{k+1}^N \rangle &- \langle h, v_{k}^N \rangle= \frac{\alpha}{N^{3-\beta}} \sum_{i=1}^N \partial_c h(C_k^i, W_k^i)\left(r(\xi_k)+\gamma Q_k^N(\xi_{k+1})- Q_k^N(\xi_k)\right)\sigma(W_k^i\cdot \xi_k)\\
            &+ \frac{\alpha}{N^{3-\beta}} \sum_{i=1}^N \partial_w h(C_k^i, W_k^i)\left(r(\xi_k)+\gamma Q_k^N(\xi_{k+1})- Q_k^N(\xi_k)\right)C_k^i\sigma'(W_k^i\cdot\xi_k)\cdot \xi_k+ O(N^{-2}).
        \end{aligned}
    \end{equation}

    We then get 
    \begin{equation} \label{v_t^N_eq_app}
        \begin{aligned}
            \langle h, v_{t}^N \rangle &- \langle h,v_0^N \rangle = \sum_{k=0}^{\lfloor NT\rfloor -1}(\langle h, v_{k+1}^N \rangle - \langle h,v_k^N \rangle)\\
            &= \frac{\alpha}{N^{2-\beta}} \sum_{k=0}^{\lfloor NT\rfloor -1} \left(r(\xi_k)+\gamma Q_k^N(\xi_{k+1})- Q_k^N(\xi_k)\right) \langle C_{\xi_k}^h, v_k^N \rangle+ O(N^{-1}),
        \end{aligned}
    \end{equation}
    where
    \begin{equation} \label{C_xi^f_eq}
        C_{\xi}^h = \sigma(w\cdot \xi)\partial_c h(c,v) + c\sigma'(w\cdot \xi)\partial_w h(c,w) \xi.
    \end{equation}

    Similarly for the measure $\mu_k^N$ we have that for any $\tilde{h} \in C_b^2(\mathds{R}^{1+d})$
    \begin{equation}
        \begin{aligned}
            \langle \tilde{h}, &\mu_{k+1}^N \rangle - \langle \tilde{h}, \mu_{k}^N \rangle = \frac{1}{N}\sum_{i=1}\tilde{h}(B_{k+1}^i, U_{k+1}^i) -\frac{1}{N}\sum_{i=1}\tilde{h}(B_{k}^i, U_{k}^i)\\
            &= \frac{1}{N} \sum_{i=1}^N \partial_b \tilde{h}(B_k^i, U_k^i) (B_{k+1}^i-B_k^i) + \frac{1}{N} \sum_{i=1}^N \partial_u \tilde{h}(B_k^i, U_k^i)(U_{k+1}^i-U_k^i)\\
            &+ \frac{1}{2N} \sum_{i=1}^N \partial_b^2 \tilde{h}(\bar{B_k^{i}}, \bar{U_k^{i}}) (B_{k+1}^i-B_k^i)^2\\
            &+ \frac{1}{2N} \sum_{i=1}^N (B_{k+1}^i-B_k^i)\partial_{bu}^2 \tilde{h}(\hat{B_k^{i}}, \hat{U_k^{i}}) (U_{k+1}^i-U_k^i)\\
            &+ \frac{1}{2N} \sum_{i=1}^N (U_{k+1}^i-U_k^i)^T\partial_{u}^2 \tilde{h}(\tilde{b_k^{i}}, \tilde{U_k^{i}}) (U_{k+1}^i-U_k^i),
        \end{aligned}
    \end{equation}
    for points $(\bar{B_k^{i}}, \bar{U_k^{i}}), (\hat{B_k^{i}}, \hat{U_k^{i}}), (\tilde{B_k^{i}}, \tilde{U_k^{i}})$ on the segment connecting $(B_{k}^i, U_{k}^i)$ and $(B_{k+1}^i, U_{k+1}^i)$.

    Using the parameter update equations \eqref{param_update} and lemma \ref{parameter_bounds_apriori_lemma} we have
        \begin{align}
            &\langle \tilde{h}, \mu_{k+1}^N \rangle - \langle \tilde{h}, \mu_{k}^N \rangle\nonumber\\
            &= \frac{\alpha}{N^{3-\beta}} \sum_{i=1}^N \partial_b \tilde{h}(B_k^i, U_k^i)\frac{1}{1+\frac{k}{N}} Q_k^N(\tilde{\xi_k}) \left( \sigma(U_k^i \cdot \tilde{\xi_k}) - \sum_{a'' \in \mathcal{A}} f_k^N(\tilde{x_k},a'')\sigma(U_k^i \cdot (\tilde{x_k}, a'')) \right)\nonumber\\
            &+ \frac{\alpha}{N^{3-\beta}} \sum_{i=1}^N \partial_u \tilde{h}(B_k^i, U_k^i)\frac{1}{1+\frac{k}{N}} Q_k^N(\tilde{\xi_k}) \left(B_k^i \sigma'(U_k^i \cdot \tilde{\xi_k})\tilde{\xi_k} - \sum_{a'' \in \mathcal{A}} f_k^N(\tilde{x_k},a'')B_k^i\sigma'(U_k^i \cdot (\tilde{x_k}, a''))(\tilde{x_k}, a'') \right)\nonumber\\
            &+ O(N^{-2}).\label{mu_k^N_diff_eq}
        \end{align}

    We then get 
    \begin{equation} \label{mu_t^N_eq}
        \begin{aligned}
            &\langle \tilde{h}, \mu_{t}^N \rangle - \langle \tilde{h},\mu_0^N \rangle = \sum_{k=0}^{\lfloor NT\rfloor -1}(\langle \tilde{h}, \mu_{k+1}^N \rangle - \langle \tilde{h},\mu_k^N \rangle)\\
            &= \frac{\alpha}{N^{2-\beta}} \sum_{k=0}^{\lfloor NT\rfloor -1} \frac{1}{1+\frac{k}{N}} Q_k^N(\tilde{\xi_k}) \left(\langle C_{\xi_k}^{\tilde{h}}, \mu_k^N \rangle -\sum_{a' \in \mathcal{A}}f_k^N(\tilde{x}_k, a')\langle C_{(\tilde{x}_k, a')}^{\tilde{h}}, \mu_k^N \rangle\right)+ O(N^{-1}).
        \end{aligned}
    \end{equation}

    We conclude this section with two helpful bounds on $B$ and $C^h$ defined in \eqref{B_def_eq} and \eqref{C^f_eq} respectively, that we will use in the next section. The proof of the following lemmas follow directly from \eqref{v_k^N_diff_eq} and \eqref{mu_k^N_diff_eq}, and the fact that if $h$ is in $C_b^3(\mathbb{R})$, then $C^h$ is in $C_b^2(\mathbb{R})$, as well as the fact that $B$ is in $C_b^3(\mathbb{R})$ by lemma \ref{parameter_bounds_apriori_lemma} and assumptions \ref{MDP_assumptions} and \ref{actor_critic_models_assumption} on the finiteness of $\mathcal{X}\times\mathcal{A}$ and the differentiability of $\sigma$.
    
    \begin{lemma} \label{B_diff_bound_lemma}
        For all $k\leq NT$ there exists a uniform constant $C_T$ that depends on $T$ only such that
        \begin{equation*}
            \begin{aligned}
                \max_{\xi, \xi'\in \mathcal{X}\times \mathcal{A}} \left| \langle B_{\xi, \xi'}, v_{k+1}^N \rangle - \langle B_{\xi, \xi'}, v_{k}^N \rangle \right| +\max_{\xi, \xi' \in \mathcal{X}\times \mathcal{A}} \left|\langle B_{\xi, \xi'}, \mu_{k+1}^N \rangle - \langle B_{\xi, \xi'}, \mu_{k}^N \rangle\right|&\leq \frac{C_T}{N}
            \end{aligned}
        \end{equation*}
    \end{lemma}
    \begin{lemma} \label{C_diff_bound_lemma}
        For any fixed $h\in C_b^3(\mathbb{R})$ and all $k\leq NT$, there exists a uniform constant $C_T$ that depends on $T$ only such that
        \begin{equation*}
            \begin{aligned}
                 \left| \langle C_{\xi_{k+1}}^{h}, v_{k+1}^N \rangle - \langle  C_{\xi_{k}}^{h}, v_{k}^N \rangle \right|+\left|\langle C_{\xi_{k+1}}^{h}, \mu_{k+1}^N \rangle - \langle C_{\xi_k}^{h}, \mu_{k}^N \rangle\right| &\leq \frac{C_T}{N}
            \end{aligned}
        \end{equation*}

        In particular, this is true for $h=B$ as defined in \eqref{B_def_eq}.
    \end{lemma}

    \section{Proofs of lemmas \ref{g_t^N_bound_lemma2} and \ref{stationary_measures_leading_order_error_conv_lemma}} \label{first_order_proofs_app}

 \begin{proof}[Proof of lemma \ref{g_t^N_bound_lemma2}.]   
If we consider the first order Taylor expansion of the softmax function around $f_t^{(0)}(x,a)$ we get \begin{align}\label{softmax_taylor_2_eq}
            &f_t^N(x,a)-f_t^{(0)}(x,a) = \text{Softmax}(P_t^N(x,a)) - \text{Softmax}(P_t^{(0)}(x,a))\nonumber\\
            &=  \sum_{a' \in \mathcal{A}} \text{Softmax}(P_t^{(0)}(x,a))  \left( \mathds{1}\{a=a'\} - \text{Softmax}(P_t^{(0)}(x,a'))\right) \left( P_t^N(x,a') - P_t^{(0)}(x,a')\right)\nonumber\\
            &+R_{Tf},
        \end{align}
    where the remainder term $R_{Tf}$ satisfies 
    \begin{equation*}
        |R_{Tf}| \leq C \max_{(x,a) \in \mathcal{X}\times \mathcal{A}}|P_t^{(0)}(x,a)-P_t^N(x,a)|^2,
    \end{equation*}
    for some uniform constant $C<\infty$ that can be chosen to depend on the size of the action space $\mathcal{A}$ only. This implies 
    \begin{equation} \label{f_t^1N_prelim_eq}
        \begin{aligned}
            f_t^{1,N}(x,a) = \sum_{a' \in \mathcal{A}} f_t^{(0)}(x,a) \left( \mathds{1}\{a=a'\} - f_t^{(0)}(x,a')\right) P_t^{1,N}(x,a') +O_{L_p}(N^{-\phi}),
        \end{aligned}
    \end{equation}
    where we used the convergence rates of theorem \ref{leading_order_conv_th}. Hence
    \begin{equation*}
        \left|f_t^{1,N}(x,a) - f_t^{(1)}(x,a) \right| \leq C \max_{(x',a') \in \mathcal{X}\times \mathcal{A}}|P_t^{1,N}(x',a')-P_t^{(1)}(x',a')|  +O_{L_p}(N^{-\phi}),
    \end{equation*}
    which is the first result. The third result follows by taking maximum and expectations.

    Let us now study the process $\left|g_t^{1,N}(x,a)  - g_t^{(1)}(x,a) \right|$. Using the definitions \eqref{g_k^N_eq} and \eqref{discrete_to_cont_eq} and lemma \ref{eta_t^N_convergence_error_bound} we get
    \begin{equation} \label{g_tilde_gronwall_ineq}
    \begin{aligned}
        g_t^{1,N}(x,a) - {g}_t^{(1)}(x,a) &= (1-\eta_t)\left(f_t^{1,N}(x,a)- f_t^{(1)}(x,a)\right) \\
        &+N^{\phi}(\eta_t - \eta_t^N)(f_t^N(x,a)-f_t^{(0)}(x,a)) +O(N^{\phi-1}),
    \end{aligned}
    \end{equation}
    which implies 
    \begin{equation}
        \left| {g}_t^{1,N}(x,a) - {g}_t^{(1)}(x,a) \right| \leq (1-\eta_t) \left| f_t^{1,N}(x,a)- f_t^{(1)}(x,a)\right| + O(N^{\phi-1}),
    \end{equation}
    which is the second result. Since $\phi \leq \frac{1}{2}$, taking maximum and expectations gives the fourth result.
\end{proof}

\begin{proof}[Proof of lemma \ref{stationary_measures_leading_order_error_conv_lemma}.]   
    The stationary probability measures $\pi^{g_t^N}$, $\sigma_{\rho_0}^{g_t^N}$,$\pi^{g_t^{(0)}}$, $\sigma_{\rho_0}^{g_t^{(0)}}$ satisfy the equations
    \begin{equation} \label{stationary_measures_matrix_eq}
        \begin{aligned}
        \pi^{g_t^N}\left(\mathds{P}_t^N -W_{\pi^{g_t^N}}\right) &= 0, \hspace{30pt} \sigma_{\rho_0}^{g_t^N}\left(\Pi_t^N-W_{\sigma_{\rho_0}^{g_t^N}}\right) &&= 0\\
            \pi^{g_t^{(0)}}\left(\mathds{P}_t^{(0)} -W_{\pi^{g_t^{(0)}}}\right) &= 0, \hspace{30pt} \sigma_{\rho_0}^{g_t^{(0)}}\left(\Pi_t-W_{\sigma_{\rho_0}^{g_t^{(0)}}}\right) &&= 0,
        \end{aligned}
    \end{equation}
    where we used matrix notation and the definition of $W_v$ in Definition \ref{W_v_def}. We will now show the result for the measure ${\pi}^{g_t^{1,N}}$. The result for ${\sigma}_{\rho_0}^{g_t^{1,N}}$ can be derived analogously. Using \eqref{stationary_measures_matrix_eq} we can write 
    \begin{equation*}
            \pi^{g_t^N}\left(\mathds{P}_t^N -W_{\pi^{g_t^N}}\right) - \pi^{g_t^{(0)}}\left(\mathds{P}_t^{(0)} -W_{\pi^{g_t^{(0)}}}\right) = 0,
    \end{equation*}
    which gives 
    \begin{equation} \label{pi_tilde_t^N_eq}
        \begin{aligned}
            \left(\pi^{g_t^N}-\pi^{g_t^{(0)}} \right)\left(\mathds{P}_t^N -W_{\pi^{g_t^N}}\right) &=- \pi^{g_t^{(0)}}\left(\mathds{P}_t^N-\mathds{P}_t^{(0)} -W_{\pi^{g_t^N}}+W_{\pi^{g_t^{(0)}}}\right) \\
            \implies   \pi^{g_t^{1,N}}\left(\mathds{P}_t^N -W_{\pi^{g_t^N}}\right)=&- \pi^{g_t^{(0)}}\left(N^{\phi}(\mathds{P}_t^N-\mathds{P}_t^{(0)}) -N^{\phi}(W_{\pi^{g_t^N}}-W_{\pi^{g_t^{(0)}}})\right)\\
            \implies  \pi^{g_t^{1,N}}\left(\mathds{P}_t^N -W_{\pi^{g_t^N}}\right)=&- \pi^{g_t^{(0)}}\left(\mathds{P}_t^{(1)}+r_t^N -W_{\pi^{g_t^{1,N}}}\right),
        \end{aligned}
    \end{equation}
    where the matrix $\mathds{P}_t^{(1)}$ has entries
    
    \begin{equation*}
        \mathds{P}^{(1)}_{t,(x,a),(x',a')} = (1-\eta_t)f_t^{(1)}(x,a)p(x'|x,a),
    \end{equation*}
    
    and the remainder matrix
    \begin{equation*}
        \begin{aligned}
            r_t^N = N^{\phi}(\mathds{P}_t^N-\mathds{P}_t^{(0)})-\mathds{P}_t^{(1)}
        \end{aligned}
    \end{equation*} 
      has all columns equal, i.e $r^N_{t, (x,a),(x',a')} := r^N_t(x,a)$ and satisfies the bound
    \begin{equation} \label{r_t^N_gronwall_bound}
        \mathds{E}\left[ \max_{(x,a)\in \mathcal{X}\times\mathcal{A}}|r_t^N(x,a)|^p\right] \leq C_{T,p}N^{-p\phi} + \mathds{E}\left[ \max_{(x',a') \in \mathcal{X}\times \mathcal{A}}|P_t^{1,N}(x',a')-P_t^{(1)}(x',a')|^p \right],
    \end{equation}
    and $W_{\pi^{g_t^{1,N}}} $ follows Definition \ref{W_v_def}. Since $\pi^{g_t^{(0)}}W_t^{1,N} = \pi^{g_t^{1,N}}$, we can use proposition 11.1 of \cite{snell} to argue \eqref{pi_tilde_t^N_eq} has a unique solution that can be written as 
    \begin{equation} \label{pi_t^N_sol_eq}
      \pi^{g_t^{1,N}} = -\pi^{g_t^{(0)}}\left(\mathds{P}_t^{(1)} +r_t^N \right) \left(I-\mathds{P}_t^N + W_{\pi^{g_t^N}}\right)^{-1}.
    \end{equation}

    By the same proposition, $\pi^{g_t^{(1)}}$ is well defined, and using its definition along with \eqref{pi_t^N_sol_eq} gives
        \begin{align}
            \pi^{g_t^{1,N}} - \pi^{g_t^{(1)}}& = -\pi^{g_t^{(0)}}r_t^N\left( I-\mathds{P}_t^N + W_{\pi^{g_t^N}} \right)^{-1}\nonumber\\
            &\quad-\pi^{g_t^{(0)}} \mathds{P}_t^{(1)} \left[ \left( I-\mathds{P}_t^N + W_{\pi^{g_t^N}} \right)^{-1} - \left( I-\mathds{P}_t^{(0)} + W_{\pi^{g_t^{(0)}}} \right)^{-1}\right].\label{pi_tilde_error_eq}
        \end{align}

    We can bound $\left( I-\mathds{P}_t^N + W_{\pi^{g_t^N}} \right)^{-1} - \left( I-\mathds{P}_t^{(0)} + W_{\pi^{g_t^{(0)}}} \right)^{-1}$ by using the matrix inversion error bound (see e.g. \cite{mat_inv})
    \begin{equation} \label{matrix_inv_error_bound_nopower_eq}
        \| (A+\delta A)^{-1} - A^{-1} \|_2 \leq \|A^{-1}\|_2^2 \cdot \|  \delta A\|_2 + O(\|\delta A\|_2^2),
    \end{equation}
which we raise to the power of $p$ to obtain
     \begin{equation} \label{matrix_inv_error_bound_eq}
        \| (A+\delta A)^{-1} - A^{-1} \|^p_2 \leq C_p\|A^{-1}\|_2^{2p} \cdot \|  \delta A\|_2^p + C_p O(\|\delta A\|_2^{2p}),
    \end{equation}
    for $A = I -\mathds{P}_t^{(0)}+W_t$ and $\delta A = -\mathds{P}_t^N+\mathds{P}_t^{(0)}+W_{\pi^{g_t^N}}-W_t$. Using this in \eqref{pi_tilde_error_eq} gives
    \begin{equation} \label{tilde_pi_error_bound_pre_eq}
        \begin{aligned}
            &\|\pi^{g_t^{1,N}} - \pi^{g_t^{(1)}}\|_2^p \leq C_p\|\pi^{g_t^{(0)}}\|_2^p\|r_t^N\|_2^p \Bigg[ \|\left( I-\mathds{P}_t^{(0)} + W_{\pi^{g_t^{(0)}}}\right)^{-1}\|_2^p+\\
            &\hspace{10pt}\|\left( I-\mathds{P}_t^{(0)} + W_{\pi^{g_t^{(0)}}} \right)^{-1}\|_2^{2p} \left( \|\mathds{P}_t^N-\mathds{P}_t^{(0)}\|_2^p+\| W_{\pi^{g_t^N}}-W_{\pi^{g_t^{(0)}}}\|_2^p \right)\\
            &\hspace{10pt}+O\left( \|\mathds{P}_t^N-\mathds{P}_t^{(0)}\|_2^{2p}+\| W_{\pi^{g_t^N}}-W_{\pi^{g_t^{(0)}}}\|_2^{2p}\right) \Bigg] \\
            &+C_p\|\pi^{g_t^{(0)}}\|_2^p \| \mathds{P}_t^{(1)}\|_2^p \Bigg[ \left( I-\mathds{P}_t^{(0)} + W_{\pi^{g_t^{(0)}}} \right)^{-1}\|_2^{2p} \left( \|\mathds{P}_t^N-\mathds{P}_t^{(0)}\|_2^p+\| W_{\pi^{g_t^N}}-W_{\pi^{g_t^{(0)}}}\|_2^p \right)\\
            &+O\left( \|\mathds{P}_t^N-\mathds{P}_t^{(0)}\|_2^{2p}+\| W_{\pi^{g_t^N}}-W_{\pi^{g_t^{(0)}}}\|_2^{2p}\right)\Bigg].
        \end{aligned}
    \end{equation}

    By the equivalence of the 2- and maximum-norms for matrices we get
    \begin{equation} \label{r_t^N_L2_bound_eq}
        \begin{aligned}
            &\mathds{E}\left[ \|r_t^N\|_2^p\right] \leq C_p\mathds{E}\left[ \max_{(x,a)\in \mathcal{X}\times\mathcal{A}}\left|r_t^N(x,a)\right|^p\right] \leq C_{T,p}N^{-p\phi} + C_{T,p}\mathds{E}\left[ \max_{(x',a') \in \mathcal{X}\times \mathcal{A}}\left|P_t^{1,N}(x',a')-P_t^{(1)}(x',a')\right|^p \right]\\
            &\|\mathds{P}_t^N-\mathds{P}_t^{(0)}\|_2^p+\| W_{\pi^{g_t^N}}-W_{\pi^{g_t^{(0)}}}\|_2^p \leq C_p  \max_{(x,a), (x',a') \in \mathcal{X}\times\mathcal{A}}\left|{P^N}_{t,(x,a).(x',a')}\right|^p+ \\
            &\hspace{40pt}+ C_p \max_{(x,a), (x',a') \in \mathcal{X}\times\mathcal{A}}\left|{P}^{(0)}_{t,(x,a).(x',a')}\right|^p +C_p \max_{(x,a), (x',a') \in \mathcal{X}\times\mathcal{A}}\left|{W_{\pi^{g_t^N}, (x,a),(x',a')}}\right|^p \\
            & \hspace{40pt} + C_p\max_{(x,a), (x',a') \in \mathcal{X}\times\mathcal{A}}\left|W_{t,(x,a), (x',a'))}\right|^p\\
            &\hspace{40pt}\leq C_p.
        \end{aligned}
    \end{equation}

    Moreover, by the equivalence of the 2- and maximum-norms for matrices and the power-mean inequality we get 
    \begin{align}
        &\mathds{E}\left[ \|\mathds{P}_t^N-\mathds{P}_t^{(0)}\|_2^p\right] \leq C_p\mathds{E}\left[ \max_{(x,a),(x',a') \in \mathcal{X}\times\mathcal{A}}{\left|\mathds{P}_{t,(x,a), (x',a')}^N-\mathds{P}_{t,(x,a), (x',a')}^{(0)}\right|^p}\right]\nonumber\\
        &\leq C_p\mathds{E}\left[N^{-p\phi}\left(\max_{(x,a),(x',a') \in \mathcal{X}\times\mathcal{A}}{\left||\mathds{P}_{t,(x,a), (x',a')}^{(1)}\right|^p} +\max_{(x,a),(x',a') \in \mathcal{X}\times\mathcal{A}}{\left|r_{t,(x,a), (x',a')}^N\right|^p}\right) \right]\nonumber\\
        &\leq  C_{T,p}N^{-p\phi} + C_{T,p}\mathds{E}\left[ \max_{(x',a') \in \mathcal{X}\times \mathcal{A}}\left|P_t^{1,N}(x',a')-P_t^{(1)}(x',a')\right|^p \right].\label{diff_P_diff_W_L1_bound}
  \end{align}
\begin{align*}      
        \mathds{E}\left[ \|W_{\pi^{g_t^N}}-W_{\pi^{g_t^{(0)}}}\|_2^p\right] &\leq C_p\mathds{E}\left[ \max_{(x,a),(x',a') \in \mathcal{X}\times\mathcal{A}}{\left|W_{\pi^{g_t^N}, (x,a),(x',a')}-W_{t, (x,a),(x',a')} \right|^p}\right]\\
        &\leq C_p \mathds{E}\left[\max_{(x',a') \in \mathcal{X}\times \mathcal{A}}{\left|\pi^{g_t^N}(x',a')-\pi^{g_t^{(0)}}(x',a')\right|^p} \right] \\
        &\leq C_{T,p}N^{-p\phi},
    \end{align*}
    which also yields the $L_2$ bounds
    \begin{equation} \label{diff_P_diff_W_L2_bound}
        \begin{aligned}
            \mathds{E}\left[ \|\mathds{P}_t^N-\mathds{P}_t^{(0)}\|_2^{2p}\right] 
            &\leq C_{T,p}N^{-2p\phi} + C_{T,p}\mathds{E}\left[ \max_{(x',a') \in \mathcal{X}\times \mathcal{A}}\left|P_t^{1,N}(x',a')-P_t^{(1)}(x',a')\right|^{2p} \right]\\
            \mathds{E}\left[ \|W_{\pi^{g_t^N}}-W_{\pi^{g_t^{(0)}}}\|_2^{2p}\right] &\leq C_{T,p}N^{-2p\phi}.
        \end{aligned}
    \end{equation}

    The bounds \eqref{r_t^N_L2_bound_eq}, \eqref{diff_P_diff_W_L1_bound} and \eqref{diff_P_diff_W_L2_bound} along with the error bound \eqref{tilde_pi_error_bound_pre_eq} and the equivalence of the matrix 2- and maximum-norms give the first bound of the lemma. Using the same arguments one can also derive the second bound.
\end{proof}

\section{Convergence of Higher Order Error Terms} \label{higher_order_error_sec}
    \subsection{The Inductive Hypothesis on the higher order correction terms}

    Throughout this section we assume $\beta \in \left(\frac{2n-1}{2n}, 1 \right]$ for a fixed $n\in \mathds{N}$. We claim the following convergence behavior for the  processes $Q_t^N$, $P_t^N$, $v_t^N$, $\mu_t^N$, $f_t^N$, $g_t^N$, $\pi^{g_t^N}$, and $\sigma_{\rho_0}^{g_t^N}$.

    \begin{hypothesis}\label{limit_terms_hyp}
        If $\beta \in \left(\frac{2n-1}{2n}, 1 \right]$, then 
        the following expansions hold

        \begin{equation*}
        \begin{aligned}
            Q_t^N &= Q_t^{(0)} + N^{\beta-1}Q_t^{(1)} + N^{2\beta-2}Q_t^{(2)}+...+N^{(n-1)(\beta-1)}Q_t^{(n-1)}+\hat{Q}_t^N\\
            P_t^N &= P_t^{(0)} + N^{\beta-1}P_t^{(1)} + N^{2\beta-2}P_t^{(2)}+...+N^{(n-1)(\beta-1)}P_t^{(n-1)}+\hat{P}_t^N\\
             v_t^N &= v_0^{(0)} + N^{\beta-1}v_t^{(1)} + N^{2\beta-2}v_t^{(2)}+...+N^{(n-1)(\beta-1)}v_t^{(n-1)}+\hat{v}_t^N\\
            \mu_t^N &= \mu_0^{(0)} + N^{\beta-1}\mu_t^{(1)} + N^{2\beta-2}\mu_t^{(2)}+...+N^{(n-1)(\beta-1)}\mu_t^{(n-1)}+\hat{\mu}_t^N\\
             f_t^N &= f_t^{(0)} + N^{\beta-1}f_t^{(1)} + N^{2\beta-2}f_t^{(2)}+...+N^{(n-1)(\beta-1)}f_t^{(n-1)}+\hat{f}_t^N\\
            P_t^N &= g_t^{(0)} + N^{\beta-1}g_t^{(1)} + N^{2\beta-2}g_t^{(2)}+...+N^{(n-1)(\beta-1)}g_t^{(n-1)}+\hat{g}_t^N\\
             \pi^{g_t^N} &= \pi^{g_t^{(0)}} + N^{\beta-1}\pi^{g_t^{(1)}} + N^{2\beta-2}\pi^{g_t^{(2)}}+...+N^{(n-1)(\beta-1)}\pi^{g_t^{(n-1)}}+\hat{\pi}^{g_t^N}\\
           \sigma_{\rho_0}^{g_t^N} &= \sigma_{\rho_0}^{g_t^{(0)}} + N^{\beta-1}\sigma_{\rho_0}^{g_t^{(1)}} + N^{2\beta-2}\sigma_{\rho_0}^{g_t^{(2)}}+...+N^{(n-1)(\beta-1)}\sigma_{\rho_0}^{g_t^{(n-1)}}+\hat{\sigma}_{\rho_0}^{g_t^N},
        \end{aligned}    
        \end{equation*}
        where the remainder terms $\hat{f}_t^N,\hat{g}_t^N, \hat{v}_t^N, \hat{\mu}_t^N, \hat{\pi}^{g_t^N}$ and $\hat{\sigma}_{\rho_0}^{g_t^N}$ are in $O_{L_p}(N^{n(\beta-1)})$, the remainder terms $\hat{Q}_t^N$ and $\hat{P}_t^N$ are in $O_{L_p}(N^{(n-1)(\beta-1)-\phi_n})$ with $\phi_n=\min{\left\{1-\beta, \beta - \frac{1}{2} +(n-1)(\beta-1)\right\}}$, and the involved functions are defined as in equations \eqref{actor_model_higher_order_error_terms_limits_eq}, \eqref{stationary_measure_higher_order_error_terms_limits_eq}, \eqref{empirical_measure_higher_order_error_terms_limits_eq} and \eqref{Q_P_higher_order_error_terms_limits_eq}. 
        
        We will prove hypothesis \ref{limit_terms_hyp} inductively on $n$. Throughout this section, we assume that the statement is true for $\beta \in \left(\frac{2k-1}{2k},1\right)$ for any $k\leq n$, and we will show that it is true for $\beta \in \left( \frac{2n+1}{2n+2},1\right)$. Notice further that propositions \ref{higher_order_error_terms_conv_prop}, \ref{empirical_measures_intermediate_terms_conv_prop} and theorem \ref{Q_P_intermediate_error_terms_conv_th} follow from hypothesis \ref{limit_terms_hyp}.

        In subsections \ref{Q_t^{(n)}N_P_t^{(n)}N_prelim_sec}-\ref{stat_measure_nth_order_prelim_sec} we use hypothesis \ref{limit_terms_hyp} to derive the formulas for the terms of the corresponding asymptotic expansions. To do so we appropriate match powers of $N$. Hypothesis \ref{limit_terms_hyp} is shown to hold in subsection \ref{sec:proofInductiveHypothesis}.

    \end{hypothesis}

     \subsection{Analysis of the Higher Order Actor and Critic Network Output errors} \label{Q_t^{(n)}N_P_t^{(n)}N_prelim_sec}

    We start by deriving pre-limit equations for the processes $\hat{Q}_t^N, \hat{P}_t^N, \hat{f}_t^N,\hat{g}_t^N, \hat{v}_t^N, \hat{\mu}_t^N, \hat{\pi}^{g_t^N}, \hat{\sigma}_{\rho_0}^{g_t^N}$. To derive a pre-limit equation for $\hat{Q}_t^N$, we plug the expansion of $Q_t^N$ as defined in hypothesis \ref{limit_terms_hyp} into both sides of the pre-limit equation \eqref{Q_t^N_eq}. 
    \begin{align}
        &Q_t^{(0)}(x,a) +N^{\beta-1}Q_t^{(1)}(x,a)+N^{2\beta-2}Q_t^{(2)}(x,a)+N^{(n-1)(\beta-1)}Q_t^{(n-1)}(x,a)+\hat{Q}_t^N(x,a)  \nonumber\\
        &=\alpha \int_{0}^{t} \sum_{(x',a', x'', a'') \in \mathcal{X}\times \mathcal{A}\times \mathcal{X}\times\mathcal{A}} \Big[ r(\xi') +\gamma \left(Q_s^{(0)}(x'',a'') +N^{\beta-1}Q_s^{(1)}(x'',a'')+...+\hat{Q}_s^N(x'',a'')\right)\nonumber\\
        &\hspace{40pt}- \left(Q_s^{(0)}(x',a') +N^{\beta-1}Q_s^{(1)}(x',a')+...+\hat{Q}_s^N(x',a')\right)\Big]\nonumber\\
        &\hspace{20pt}\langle B_{(x,a), (x',a')}, v_0^{(0)} + N^{\beta-1}v_s^{(1)} +...\hat{v}_s^N\rangle \left[ g_s^{(0)}(x'',a'') + N^{\beta-1}g_s^{(1)}(x'',a'') +...+\hat{g}_s^N(x'',a'')\right] \nonumber\\
        &\hspace{20pt}\left[ \pi^{g_s^{(0)}}(x',a') + N^{\beta-1}\pi^{g_s^{(1)}}(x',a') +...+\hat{\pi}^{g_s^N}(x',a')\right]p(x''|x',a')ds\nonumber\\
        &\hspace{20pt}+N^{1-\beta}\langle c\sigma(w\cdot (x,a)), v_0^N\rangle + M_t^{1,N}(x,a) + M_t^{2,N}(x,a)+ M_t^{3,N}(x,a)+R_{Q,t}^N.\label{Q_t^N_hat_prelim_eq_1}
    \end{align}
    
Rearranging the latter with respect to powers of $N$ gives  
    {\small    \begin{align}
            &Q_t^{(0)}(x,a) +N^{\beta-1}Q_t^{(1)}(x,a)+N^{2\beta-2}Q_t^{(2)}(x,a)+N^{(n-1)(\beta-1)}Q_t^{(n-1)}(x,a)+\hat{Q}_t^N(x,a) \nonumber\\
            &=\sum_{k=0}^{n-1} N^{k(\beta-1)} \sum_{\substack{m_1,m_2,m_3,m_4 \in \{0,1,...,k\} \\ m_1+m_2+m_3+m_4=k}} \alpha \int_{0}^t \sum_{(x',a'),(x'',a'')\in \mathcal{X}\times \mathcal{A}} \left[r(x',a')\cdot \mathds{1}\left\{m_1=0 \right\} +\gamma Q_s^{(m_1)}(x'',a'')-Q_s^{(m_1)}(x',a') \right] \nonumber\\
                &\hspace{60pt}\langle B_{(x,a),(x',a')}, v_s^{(m_2)}\rangle g_s^{(m_3)}(x'',a'') \pi^{g_s^{(m_4)}}(x'',a'')p(x''|x',a')ds\nonumber\\&+N^{n(\beta-1)}\sum_{\substack{m_1,m_2,m_3,m_4 \in \{0,1,...,n-1\} \\ m_1+m_2+m_3+m_4=n}} \alpha \int_{0}^t \sum_{(x',a'),(x'',a'')\in \mathcal{X}\times \mathcal{A}} \left[r(x',a')\cdot \mathds{1}\left\{m_1=0 \right\} +\gamma Q_s^{(m_1)}(x'',a'')-Q_s^{(m_1)}(x',a') \right]\nonumber \\
                &\hspace{60pt}\langle B_{(x,a),(x',a')}, v_s^{(m_2)}\rangle g_s^{(m_3)}(x'',a'') \pi^{g_s^{(m_4)}}(x'',a'')p(x''|x',a')ds\nonumber\\
            &+\sum_{k=n+1}^{4n-4} N^{k(\beta-1)} \sum_{\substack{m_1,m_2,m_3,m_4 \in \{0,1,...,n-1\} \\ m_1+m_2+m_3+m_4=k}} \alpha \int_{0}^t \sum_{(x',a'),(x'',a'')\in \mathcal{X}\times \mathcal{A}} \left[r(x',a')\cdot \mathds{1}\left\{m_1=0 \right\} +\gamma Q_s^{(m_1)}(x'',a'')-Q_s^{(m_1)}(x',a') \right]\nonumber \\
                &\hspace{60pt}\langle B_{(x,a),(x',a')}, v_s^{(m_2)}\rangle g_s^{(m_3)}(x'',a'') \pi^{g_s^{(m_4)}}(x'',a'')p(x''|x',a')ds\nonumber\\
            &+  \sum_{k=1}^{3n-3}N^{k(\beta-1)}\sum_{\substack{k_1,k_2,k_3,k_4 \in \{0,1\} \\ k_1+k_2+k_3+k_4\geq 1}} \quad \sum_{\substack{m_1,m_2,m_3,m_4 \in \{0,1,...,n-1\} \\ k_1m_1+k_2m_2+k_3m_3+k_4m_4=k}}\nonumber\\
            &\hspace{20pt}\alpha \int_{0}^t \sum_{(x',a'),(x'',a'')\in \mathcal{X}\times \mathcal{A}} \left[ \gamma \hat{Q}_s^N(x'',a'') - \hat{Q}_s^N(x',a')\right]^{k_1}\Big[r(x',a')\cdot \mathds{1}\left\{m_1=0 \right\} +\gamma Q_s^{(m_1)}(x'',a'')-Q_s^{(m_1)}(x',a') \Big]^{1-k_1}\nonumber \\
                &\hspace{40pt}\left[ \langle B_{(x,a),(x',a')}, \tilde{v}_s^N\rangle\right]^{k_2}\left[ \langle B_{(x,a),(x',a')}, v_s^{(m_2)}\rangle\right]^{1-k_2}\left[\hat{g}_s^N(x'',a'') \right]^{k_3} \left[g_s^{(m_3)}(x'',a'') \right]^{1-k_3}\nonumber\\
                &\hspace{40pt} \left[\hat{\pi}^{g_s^N}(x'',a'') \right]^{k_4}\left[\pi^{g_s^{(m_4)}}(x'',a'')\right]^{1-k_4}p(x''|x',a')ds\nonumber\\
            &+ \sum_{\substack{k_1,k_2,k_3,k_4 \in \{0,1\} \\ k_1+k_2+k_3+k_4\geq 1}} \alpha \int_{0}^t \sum_{(x',a'),(x'',a'')\in \mathcal{X}\times \mathcal{A}} \left[ \gamma\ \hat{Q}_s^N(x'',a'') - \hat{Q}_s^N(x',a')\right]^{k_1}\Big[r(x',a')+\gamma Q_s^{(0)}(x'',a'')-Q_s^{(0)}(x',a') \Big]^{1-k_1} \nonumber\\
                &\hspace{40pt}\left[ \langle B_{(x,a),(x',a')}, \hat{v}_s^N\rangle\right]^{k_2}\left[ \langle B_{(x,a),(x',a')}, v_0^{(0)}\rangle\right]^{1-k_2}\left[\hat{g}_s^N(x'',a'') \right]^{k_3} \left[g_s^{(0)}(x'',a'') \right]^{1-k_3}\nonumber\\
                &\hspace{40pt} \left[\hat{\pi}^{g_s^N}(x'',a'') \right]^{k_4}\left[\pi^{g_s^{(0)}}(x'',a'')\right]^{1-k_4}p(x''|x',a')ds\nonumber\\
                &\hspace{20pt}+N^{1-\beta}\langle c\sigma(w\cdot (x,a)), v_0^N\rangle + M_t^{1,N}(x,a) + M_t^{2,N}(x,a)+ M_t^{3,N}(x,a)+R_{Q,t}^N.\label{Q_t^N_hat_prelim_eq_2}
        \end{align}}

    We note that by hypothesis \ref{limit_terms_hyp} on the boundedness of the limit terms we have 
     {\small   \begin{align}
            &\sum_{k=n+1}^{4n-4} N^{k(\beta-1)} \sum_{\substack{m_1,m_2,m_3,m_4 \in \{0,1,...,n-1\} \\ m_1+m_2+m_3+m_4=k}} \alpha \int_{0}^t \sum_{(x',a'),(x'',a'')\in \mathcal{X}\times \mathcal{A}} \left[r(x',a')\cdot \mathds{1}\left\{m_1=0 \right\} +\gamma Q_s^{(m_1)}(x'',a'')-Q_s^{(m_1)}(x',a') \right] \nonumber\\
            &\hspace{60pt}\langle B_{(x,a),(x',a')}, v_s^{(m_2)}\rangle g_s^{(m_3)}(x'',a'') \pi^{g_s^{(m_4)}}(x'',a'')p(x''|x',a')ds\nonumber\\
            &\hspace{20pt} = O_{L_p}(N^{(n+1)\label{Q_t^N_hat_helper_eq1}(\beta-1)}),
        \end{align}}
    and when also using the $L_p$ bounds of the processes $\tilde{Q}_t^N$, $\tilde{v}_t^N$, $\tilde{g}_t^N$, $\tilde{\pi}^{g_t^N}$ along with Hölder's inequality for the product of random variables we obtain
       {\small \begin{align}
            &  \sum_{k=1}^{3n-3}N^{k(\beta-1)}\sum_{\substack{k_1,k_2,k_3,k_4 \in \{0,1\} \\ k_1+k_2+k_3+k_4\geq 1}} \quad \sum_{\substack{m_1,m_2,m_3,m_4 \in \{0,1,...,n-1\} \\ k_1m_1+k_2m_2+k_3m_3+k_4m_4=k}}\nonumber\\
            &\alpha \int_{0}^t \sum_{(x',a'),(x'',a'')\in \mathcal{X}\times \mathcal{A}} \left[ \gamma\hat{Q}_s^N(x'',a'') - \hat{Q}_s^N(x',a')\right]^{k_1}\Big[r(x',a')\cdot \mathds{1}\left\{m_1=0 \right\} +\gamma Q_s^{(m_1)}(x'',a'')-Q_s^{(m_1)}(x',a') \Big]^{1-k_1}\nonumber \\
                &\hspace{40pt}\left[ \langle B_{(x,a),(x',a')}, \hat{v}_s^N\rangle\right]^{k_2}\left[ \langle B_{(x,a),(x',a')}, v_s^{(m_2)}\rangle\right]^{1-k_2}\left[\hat{g}_s^N(x'',a'') \right]^{k_3} \left[g_s^{(m_3)}(x'',a'') \right]^{1-k_3}\nonumber\\
                &\hspace{40pt} \left[\hat{\pi}^{g_s^N}(x'',a'') \right]^{k_4}\left[\pi^{g_s^{(m_4)}}(x'',a'')\right]^{1-k_4}p(x''|x',a')ds\nonumber\\
                &\hspace{20pt}  = O_{L_p}(N^{n(\beta-1)-\phi_n}).\label{Q_t^N_hat_helper_eq2}
        \end{align}}

    By hypothesis \ref{limit_terms_hyp} and equation \eqref{Q_P_higher_order_error_terms_limits_eq} we also have
{\small \begin{align}
            &Q_t^{(0)}(x,a) +N^{\beta-1}Q_t^{(1)}(x,a)+N^{2\beta-2}Q_t^{(2)}(x,a)+N^{(n-1)(\beta-1)}Q_t^{(n-1)}(x,a) \nonumber\\
             &=\sum_{k=0}^{n-1} N^{k(\beta-1)} \sum_{\substack{m_1,m_2,m_3,m_4 \in \{0,1,...,k\} \\ m_1+m_2+m_3+m_4=k}} \alpha \int_{0}^t \sum_{(x',a'),(x'',a'')\in \mathcal{X}\times \mathcal{A}} \left[r(x',a')\cdot \mathds{1}\left\{m_1=0 \right\} +\gamma Q_s^{(m_1)}(x'',a'')-Q_s^{(m_1)}(x',a') \right]\nonumber \\
                &\hspace{60pt}\langle B_{(x,a),(x',a')}, v_s^{(m_2)}\rangle g_s^{(m_3)}(x'',a'') \pi^{g_s^{(m_4)}}(x'',a'')p(x''|x',a')ds.\label{Q_t^N_hat_helper_eq3}
        \end{align}}
    
    Using \eqref{Q_t^N_hat_helper_eq1}, \eqref{Q_t^N_hat_helper_eq2}, and \eqref{Q_t^N_hat_helper_eq3} along  with lemma \ref{M_bound_lemma} and the bound \eqref{R_{Q,t}^N_bounds_eq} in \eqref{Q_t^N_hat_prelim_eq_2} we obtain 
    {\small        \begin{align}
            &\hat{Q}_t^N(x,a)=\nonumber\\
            &= N^{n(\beta-1)}\sum_{\substack{m_1,m_2,m_3,m_4 \in \{0,1,...,n-1\} \\ m_1+m_2+m_3+m_4=n}} \alpha \int_{0}^t \sum_{(x',a'),(x'',a'')\in \mathcal{X}\times \mathcal{A}} \left[r(x',a')\cdot \mathds{1}\left\{m_1=0 \right\} +\gamma Q_s^{(m_1)}(x'',a'')-Q_s^{(m_1)}(x',a') \right]\nonumber \\
                &\hspace{60pt}\langle B_{(x,a),(x',a')}, v_s^{(m_2)}\rangle g_s^{(m_3)}(x'',a'') \pi^{g_s^{(m_4)}}(x'',a'')p(x''|x',a')ds\nonumber\\
            &+ \sum_{\substack{k_1,k_2,k_3,k_4 \in \{0,1\} \\ k_1+k_2+k_3+k_4\geq 1}} \alpha \int_{0}^t \sum_{(x',a'),(x'',a'')\in \mathcal{X}\times \mathcal{A}} \left[ \gamma\hat{Q}_s^N(x'',a'') - \hat{Q}_s^N(x',a')\right]^{k_1}\Big[r(x',a')+\gamma Q_s^{(0)}(x'',a'')-Q_s^{(0)}(x',a') \Big]^{1-k_1} \nonumber\\
                &\hspace{40pt}\left[ \langle B_{(x,a),(x',a')}, \hat{v}_s^N\rangle\right]^{k_2}\left[ \langle B_{(x,a),(x',a')}, v_0^{(0)}\rangle\right]^{1-k_2}\left[\hat{g}_s^N(x'',a'') \right]^{k_3} \left[g_s^{(0)}(x'',a'') \right]^{1-k_3}\nonumber\\
                &\hspace{40pt} \left[\hat{\pi}^{g_s^N}(x'',a'') \right]^{k_4}\left[\pi^{g_s^{(0)}}(x'',a'')\right]^{1-k_4}p(x''|x',a')ds\\
                &\hspace{20pt}+N^{1-\beta}\langle c\sigma(w\cdot (x,a)), v_0^N\rangle + O_{L_p}(N^{n(\beta-1)-\phi_n}).\label{Q_t^N_hat_prelim_eq_3}
        \end{align}}

    In the same way, when considering $\hat{P}_t^N$, one can obtain
{\allowdisplaybreaks\begin{align}\label{P_t^N_hat_prelim_eq_3}
            &\hat{P}_t^N(x,a)= N^{n(\beta-1)}\sum_{\substack{m_1,m_2,m_3,m_4 \in \{0,1,...,n-1\} \nonumber\\ m_1+m_2+m_3+m_4=n}} \alpha \int_{0}^t \sum_{(x',a',a'')\in \mathcal{X}\times \mathcal{A}\times \mathcal{A}} \frac{1}{1+s} Q_s^{(m_1)}(x',a') \nonumber\\
                &\hspace{60pt}\left(\langle B_{(x,a),(x',a')}, \mu_s^{(m_2)}\rangle+\langle B_{(x,a),(x',a'')}, \mu_s^{(m_2)}\rangle \right)f_s^{(m_3)}(x',a'') \sigma_{\rho_0}^{g_s^{(m_4)}}(x',a')ds\nonumber\\
            &+ \sum_{\substack{k_1,k_2,k_3,k_4 \in \{0,1\}\nonumber \\ k_1+k_2+k_3+k_4\geq 1}}  \int_{0}^t \sum_{(x',a',a'')\in \mathcal{X}\times \mathcal{A}\times \mathcal{A}}\frac{1}{1+s} \left[ \hat{Q}_s^N(x',a')\right]^{k_1}\Big[Q_s^{(0)}(x',a') \Big]^{1-k_1} \nonumber\\
                &\hspace{40pt}\left[ \langle B_{(x,a),(x',a')}, \hat{\mu}_s^{N}\rangle+\langle B_{(x,a),(x',a'')}, \hat{\mu}_s^{N}\rangle\right]^{k_2}\left[ \langle B_{(x,a),(x',a')}, \mu_0^{(0)}\rangle+\langle B_{(x,a),(x',a'')}, \mu_0^{(0)}\rangle\right]^{1-k_2}\nonumber\\
                &\hspace{40pt}\left[\hat{f}_s^N(x',a'') \right]^{k_3} \left[f_s^{(0)}(x',a'') \right]^{1-k_3} \left[\hat{\sigma}_{\rho_0}^{g_s^N}(x',a') \right]^{k_4}\left[\sigma_{\rho_0}^{g_s^{(0)}}(x',a')\right]^{1-k_4}\nonumber\\
                &\hspace{20pt}+N^{1-\beta}\langle b\sigma(u\cdot (x,a)), \mu_0^N\rangle + O_{L_p}(N^{n(\beta-1)-\phi_n}).
        \end{align}
        
We now set 
\begin{equation}\label{rescaled_higher_order_prelim_terms_def_eq}
            \begin{aligned}
                Q_t^{n,N}&=N^{(n-1)(1-\beta)+\phi_n}\hat{Q}_t^N, &&\quad P_t^{n,N}=N^{(n-1)(1-\beta)+\phi_n}\hat{P}_t^N\\
                f_t^{n,N}&=N^{(n-1)(1-\beta)+\phi_n}\hat{f}_t^N, &&\quad g_t^{n,N}=N^{(n-1)(1-\beta)+\phi_n}\hat{g}_t^N\\
                \pi^{g_t^{n,N}}&=N^{(n-1)(1-\beta)+\phi_n}\hat{\pi}^{g_t^{N}}, &&\quad \sigma_{\rho_0}^{g_t^{n,N}}=N^{(n-1)(1-\beta)+\phi_n}\hat{\sigma}_{\rho_0}^{g_t^{N}}\\
                v_t^{n,N}&=N^{(n-1)(1-\beta)+\phi_n}\hat{v}_t^N, &&\quad \mu_t^{n,N}=N^{(n-1)(1-\beta)+\phi_n}\hat{\mu}_t^N,
            \end{aligned}
        \end{equation}
        so that from \eqref{Q_t^N_hat_prelim_eq_3} and \eqref{P_t^N_hat_prelim_eq_3} we get
       {\small  \begin{align}
            &Q_t^{n,N}(x,a)=\nonumber\\
            &= N^{\phi_n-1+\beta}\sum_{\substack{m_1,m_2,m_3,m_4 \in \{0,1,...,n-1\} \\ m_1+m_2+m_3+m_4=n}} \alpha \int_{0}^t \sum_{(x',a'),(x'',a'')\in \mathcal{X}\times \mathcal{A}} \left[r(x',a')\cdot \mathds{1}\left\{m_1=0 \right\} +\gamma Q_s^{(m_1)}(x'',a'')-Q_s^{(m_1)}(x',a') \right] \nonumber\\
                &\hspace{60pt}\langle B_{(x,a),(x',a')}, v_s^{(m_2)}\rangle g_s^{(m_3)}(x'',a'') \pi^{g_s^{(m_4)}}(x'',a'')p(x''|x',a')ds\nonumber\\
            &+ \int_{0}^t \sum_{(x',a'),(x'',a'')\in \mathcal{X}\times \mathcal{A}} \left[\gamma Q_s^{n,N}(x'',a'') - Q_s^{n,N}(x',a')\right]\langle B_{(x,a),(x',a')},  v_0^{(0)}\rangle g_s^{(0)}(x'',a'') \pi^{g_s^{(0)}}(x'',a'')p(x''|x',a')ds\nonumber\\
             &+ \int_{0}^t \sum_{(x',a'),(x'',a'')\in \mathcal{X}\times \mathcal{A}} \Big[r(x',a')+\gamma Q_s^{(0)}(x'',a'')-Q_s^{(0)}(x',a') \Big]\langle B_{(x,a),(x',a')}, v_s^{n,N}\rangle g_s^{(0)}(x'',a'') \pi^{g_s^{(0)}}(x'',a'')p(x''|x',a')ds\nonumber\\
             &+ \int_{0}^t \sum_{(x',a'),(x'',a'')\in \mathcal{X}\times \mathcal{A}} \Big[r(x',a')+\gamma Q_s^{(0)}(x'',a'')-Q_s^{(0)}(x',a') \Big]\langle B_{(x,a),(x',a')}, v_0^{(0)}\rangle g_s^{n,N}(x'',a'') \pi^{g_s^{(0)}}(x'',a'')p(x''|x',a')ds\nonumber\\
            &+ \int_{0}^t \sum_{(x',a'),(x'',a'')\in \mathcal{X}\times \mathcal{A}} \Big[r(x',a')+\gamma Q_s^{(0)}(x'',a'')-Q_s^{(0)}(x',a') \Big]\langle B_{(x,a),(x',a')}, v_0^{(0)}\rangle g_s^{(0)}(x'',a'') \pi^{g_s^{n,N}}(x'',a'')p(x''|x',a')ds\nonumber\\
                &\hspace{20pt}+N^{n(1-\beta)+\phi_n}\langle c\sigma(w\cdot (x,a)), v_0^N\rangle +O_{L_p}(N^{\beta-1}),\label{Q_t^nN_prelim_eq}
        \end{align}}
    and
{\small        \begin{align}
            &{P}_t^{n,N}(x,a)=\nonumber\\
            &= N^{\phi_n-1+\beta}\sum_{\substack{m_1,m_2,m_3,m_4 \in \{0,1,...,n-1\} \nonumber\\ m_1+m_2+m_3+m_4=n}} \alpha \int_{0}^t \sum_{(x',a',a'')\in \mathcal{X}\times \mathcal{A}\times \mathcal{A}} \frac{1}{1+s} Q_s^{(m_1)}(x',a') \left(\langle B_{(x,a),(x',a')}, \mu_s^{(m_2)}\rangle+\langle B_{(x,a),(x',a'')}, \mu_s^{(m_2)}\rangle \right)\nonumber\\
                &\hspace{60pt}f_s^{(m_3)}(x',a'') \sigma_{\rho_0}^{g_s^{(m_4)}}(x',a')ds\nonumber\\
            & +\int_{0}^t \sum_{(x',a',a'')\in \mathcal{X}\times \mathcal{A}\times \mathcal{A}}\frac{1}{1+s}  Q_s^{n,N}(x',a')\left[ \langle B_{(x,a),(x',a')}, \mu_0^{(0)}\rangle+\langle B_{(x,a),(x',a'')}, \mu_0^{(0)}\rangle\right]f_s^{(0)}(x',a'')\sigma_{\rho_0}^{g_s^{(0)}}(x',a')ds\nonumber\\
            & +\int_{0}^t \sum_{(x',a',a'')\in \mathcal{X}\times \mathcal{A}\times \mathcal{A}}\frac{1}{1+s}  Q_s^{(0)}(x',a')\left[ \langle B_{(x,a),(x',a')}, \mu_s^{n,N}\rangle+\langle B_{(x,a),(x',a'')}, \mu_s^{n,N}\rangle\right]f_s^{(0)}(x',a'')\sigma_{\rho_0}^{g_s^{(0)}}(x',a')ds\nonumber\\
            & +\int_{0}^t \sum_{(x',a',a'')\in \mathcal{X}\times \mathcal{A}\times \mathcal{A}}\frac{1}{1+s}  Q_s^{(0)}(x',a')\left[ \langle B_{(x,a),(x',a')}, \mu_0^{(0)}\rangle+\langle B_{(x,a),(x',a'')}, \mu_0^{(0)}\rangle\right]f_s^{n,N}(x',a'')\sigma_{\rho_0}^{g_s^{(0)}}(x',a')ds\nonumber\\
            & +\int_{0}^t \sum_{(x',a',a'')\in \mathcal{X}\times \mathcal{A}\times \mathcal{A}}\frac{1}{1+s}  Q_s^{(0)}(x',a')\left[ \langle B_{(x,a),(x',a')}, \mu_0^{(0)}\rangle+\langle B_{(x,a),(x',a'')}, \mu_0^{(0)}\rangle\right]f_s^{(0)}(x',a'')\sigma_{\rho_0}^{g_s^{n,N}}(x',a')ds\nonumber\\
                &\hspace{20pt}+N^{n(\beta-1)+\phi_n}\langle b\sigma(u\cdot (x,a)), \mu_0^N\rangle +O_{L_p}(N^{\beta-1}).\label{P_t^nN_prelim_eq}
        \end{align}}

    Note that all the terms corresponding to $m_1+m_2+m_3+m_4>1$ are implied in the error terms $O_{L_p}(N^{\beta-1};N)$.

    \subsection{Analysis of the higher order Empirical Measure errors assuming hypothesis \ref{limit_terms_hyp}} \label{v_tnN_mu_t^{(n)}N_prelim_sec}

     We now also derive pre-limit equations for higher order error terms of the empirical measure processes $v_t^N$ and $\mu_t^N$. We start with the process $v_t^N$,  for which we plug the expansion from hypothesis \ref{limit_terms_hyp} into the pre-limit equation \eqref{v_t^N_error_eq} and use lemmas \ref{empirical_measure_init_clt_lemma}, \ref{M_ht_L_1_bound_lemma} and the bound \eqref{R_v_bound_eq} to obtain
    {\small   \begin{align}
            &\langle h, N^{\beta-1}v_t^{(1)}+N^{2\beta-2}v_t^{(2)}+...+N^{(n-1)(\beta-1)}v_t^{(n-1)}+\hat{v}_t^N \rangle=\nonumber \\
            &= \sum_{k=1}^{n-1}N^{k(\beta-1)}\sum_{\substack{m_1,m_2,m_3,m_4 \in \{0,1,...,k-1\} \\ m_1+m_2+m_3+m_4=k-1}} \alpha \int_0^t \sum_{(x',a'),(x'',a'') \in \mathcal{X}\times\mathcal{A}} \left(r(x',a')\mathds{1}\{m_1=0\}+\gamma Q_s^{(m_1)}(x'',a'')- Q_s^{(m_1)}(x',a')\right)\nonumber\\
            &\hspace{60pt}\langle C_{(x',a')}^h, v_s^{(m_2)} \rangle g_{s}^{m_3}(x'',a'')\pi^{g_s^{(m_4)}}(x',a')p(x''|x',a')ds\nonumber\\
            &+ N^{n(\beta-1)}\sum_{\substack{m_1,m_2,m_3,m_4 \in \{0,1,...,n-1\} \\ m_1+m_2+m_3+m_4=n-1}} \alpha \int_0^t \sum_{(x',a'),(x'',a'') \in \mathcal{X}\times\mathcal{A}} \left(r(x',a')\mathds{1}\{m_1=0\}+\gamma Q_s^{(m_1)}(x'',a'')- Q_s^{(m_1)}(x',a')\right)\nonumber\\
            &\hspace{60pt}\langle C_{(x',a')}^h, v_s^{(m_2)} \rangle g_{s}^{m_3}(x'',a'')\pi^{g_s^{(m_4)}}(x',a')p(x''|x',a')ds\nonumber\\
            &+ \sum_{k=n+1}^{4n-3}N^{k(\beta-1)}\sum_{\substack{m_1,m_2,m_3,m_4 \in \{0,1,...,n-1\} \\ m_1+m_2+m_3+m_4=k-1}} \alpha \int_0^t \sum_{(x',a'),(x'',a'') \in \mathcal{X}\times\mathcal{A}} \left(r(x',a')\mathds{1}\{m_1=0\}+\gamma Q_s^{(m_1)}(x'',a'')- Q_s^{(m_1)}(x',a')\right)\nonumber\\
            &\hspace{60pt}\langle C_{(x',a')}^h, v_s^{(m_2)} \rangle g_{s}^{m_3}(x'',a'')\pi^{g_s^{(m_4)}}(x',a')p(x''|x',a')ds\nonumber\\
            &+\sum_{k=1}^{3n-2}N^{k(\beta-1)} \sum_{\substack{k_1,k_2,k_3,k_4 \in \{0,1\} \\ k_1+k_2+k_3+k_4\geq 1}} \sum_{\substack{m_1,m_2,m_3,m_4 \in \{0,1,...,n-1\} \\ \sum_{i=1}^4 (1-k_i)m_i=k-1}} \alpha \int_0^t \sum_{(x',a'),(x'',a'')\in \mathcal{X}\times\mathcal{A}} \left[ \gamma\hat{Q}_s^N(x'',a'') - \hat{Q}_s^N(x',a')\right]^{k_1}\nonumber\\
            &\hspace{60pt} \left[r(x',a')\mathds{1}\{m_1=0\}+\gamma Q_s^{(m_1)}(x'',a'')- Q_s^{(m_1)}(x',a') \right]^{1-k_1} \left[\langle C_{(x',a')}^h, \hat{v}_s^{N} \rangle \right]^{k_2}\left[ \langle C_{(x',a')}^h, v_s^{(m_2)} \rangle\right]^{1-k_2}\nonumber\\
            &\hspace{60pt} \left[\hat{g}_{s}^{N}(x'',a'') \right]^{k_3}\left[ g_{s}^{m_3}(x'',a'')\right]^{1-k_3} \left[ \hat{\pi}^{g_s^{N}}(x',a')\right]^{k_4} \left[\pi^{g_s^{(m_4)}}(x',a') \right]^{1-k_4}p(x''|x',a')ds\nonumber\\
            &+ O_{L_p}(N^{-\frac{1}{2}}).\label{tilde_v_t^N_prelim_eq1}
        \end{align}}

    Notice that by our inductive hypothesis \ref{limit_terms_hyp} we have
      {\small  \begin{align}
            &\langle h, N^{\beta-1}v_t^{(1)}+N^{2\beta-2}v_t^{(2)}+...+N^{(n-1)(\beta-1)}v_t^{(n-1)} \rangle=\nonumber \\
            &= \sum_{k=1}^{n-1}N^{k(\beta-1)}\sum_{\substack{m_1,m_2,m_3,m_4 \in \{0,1,...,k-1\} \\ m_1+m_2+m_3+m_4=k-1}} \alpha \int_0^t \sum_{(x',a'),(x'',a'') \in \mathcal{X}\times\mathcal{A}} \left(r(x',a')\mathds{1}\{m_1=0\}+\gamma Q_s^{(m_1)}(x'',a'')- Q_s^{(m_1)}(x',a')\right)\nonumber\\
            &\hspace{60pt}\langle C_{(x',a')}^h, v_s^{(m_2)} \rangle g_{s}^{m_3}(x'',a'')\pi^{g_s^{(m_4)}}(x',a')p(x''|x',a')ds,\label{tilde_v_t^N_prelim_eq2}
        \end{align}}
    and we can also obtain the bounds 
      {\small   \begin{align}
            & \sum_{k=n+1}^{4n-3}N^{k(\beta-1)}\sum_{\substack{m_1,m_2,m_3,m_4 \in \{0,1,...,n-1\} \\ m_1+m_2+m_3+m_4=k-1}} \alpha \int_0^t \sum_{(x',a'),(x'',a'') \in \mathcal{X}\times\mathcal{A}} \left(r(x',a')\mathds{1}\{m_1=0\}+\gamma Q_s^{(m_1)}(x'',a'')- Q_s^{(m_1)}(x',a')\right)\nonumber\\
            &\hspace{60pt}\langle C_{(x',a')}^h, v_s^{(m_2)} \rangle g_{s}^{m_3}(x'',a'')\pi^{g_s^{(m_4)}}(x',a')p(x''|x',a')ds\nonumber\\
            &=O_{L_p}(N^{(n+1)(\beta-1)}),\nonumber\\
            &\sum_{k=1}^{3n-2}N^{k(\beta-1)} \sum_{\substack{k_1,k_2,k_3,k_4 \in \{0,1\} \\ k_1+k_2+k_3+k_4\geq 1}} \sum_{\substack{m_1,m_2,m_3,m_4 \in \{0,1,...,n-1\} \\ \sum_{i=1}^4 (1-k_i)m_i=k-1}} \alpha \int_0^t \sum_{(x',a'),(x'',a'')\in \mathcal{X}\times\mathcal{A}} \left[ \gamma\hat{Q}_s^N(x'',a'') - \hat{Q}_s^N(x',a')\right]^{k_1}\nonumber\\
            &\hspace{60pt} \left[r(x',a')\mathds{1}\{m_1=0\}+\gamma Q_s^{(m_1)}(x'',a'')- Q_s^{(m_1)}(x',a') \right]^{1-k_1} \left[\langle C_{(x',a')}^h, \hat{v}_s^{N} \rangle \right]^{k_2}\left[ \langle C_{(x',a')}^h, v_s^{(m_2)} \rangle\right]^{1-k_2}\nonumber\\
            &\hspace{60pt} \left[\hat{g}_{s}^{N}(x'',a'') \right]^{k_3}\left[ g_{s}^{m_3}(x'',a'')\right]^{1-k_3} \left[ \hat{\pi}^{g_s^{N}}(x',a')\right]^{k_4} \left[\pi^{g_s^{(m_4)}}(x',a') \right]^{1-k_4}p(x''|x',a')ds\nonumber\\
            &= O_{L_p}(N^{n(\beta-1)-\phi_n}).\label{tilde_v_t^N_prelim_eq3}
        \end{align}}

    Using \eqref{tilde_v_t^N_prelim_eq2} and \eqref{tilde_v_t^N_prelim_eq3} in \eqref{tilde_v_t^N_prelim_eq1} gives 
       {\small  \begin{align}
            &\langle h, \tilde{v}_t^N \rangle=\nonumber \\
            &= N^{n(\beta-1)}\sum_{\substack{m_1,m_2,m_3,m_4 \in \{0,1,...,n-1\} \\ m_1+m_2+m_3+m_4=n-1}} \alpha \int_0^t \sum_{(x',a'),(x'',a'') \in \mathcal{X}\times\mathcal{A}} \left(r(x',a')\mathds{1}\{m_1=0\}+\gamma Q_s^{(m_1)}(x'',a'')- Q_s^{(m_1)}(x',a')\right)\nonumber\\
            &\hspace{60pt}\langle C_{(x',a')}^h, v_s^{(m_2)} \rangle g_{s}^{m_3}(x'',a'')\pi^{g_s^{(m_4)}}(x',a')p(x''|x',a')ds\nonumber\\
            & + O_{L_p}(N^{n(\beta-1)-\phi_n}).\label{tilde_v_t^N_prelim_eq}
        \end{align}}

    Similarly one can obtain
   {\small         \begin{align}
            &\langle h, \tilde{\mu}_t^N \rangle=\nonumber \\
            &= N^{n(\beta-1)}\sum_{\substack{m_1,m_2,m_3,m_4 \in \{0,1,...,n-1\} \\ m_1+m_2+m_3+m_4=n-1}} \alpha \int_0^t &\sum_{(x',a'),(x'',a'')\in \mathcal{X}\times \mathcal{A}} \frac{1}{1+s} Q_s^{(m_1)}(x,a) \left( \langle C_{(x',a')}^h, \mu_s^{(m_2)} \rangle+\langle C_{(x',a'')}^h, \mu_s^{(m_2)} \rangle\right)\nonumber\\
                & \hspace{60pt} f_s^{(m_3)}(x',a'') \sigma_{\rho_0}^{g_s^{(m_4)}}(x',a')ds\nonumber\\
            & + O_{L_p}(N^{n(\beta-1)-\phi_n})\label{tilde_mu_t^N_prelim_eq}
        \end{align}}

    Using the rescaled prelimit terms defined in \eqref{rescaled_higher_order_prelim_terms_def_eq}, equations \eqref{tilde_v_t^N_prelim_eq} and \eqref{tilde_mu_t^N_prelim_eq}, we can conclude from equations \eqref{Q_t^N_hat_prelim_eq_3} and \eqref{P_t^N_hat_prelim_eq_3} that 
   {\small  \begin{align}
            &\langle h, v_t^{n,N} \rangle=\nonumber \\
            &= N^{\phi_n-1+\beta}\sum_{\substack{m_1,m_2,m_3,m_4 \in \{0,1,...,n-1\} \\ m_1+m_2+m_3+m_4=n-1}} \alpha \int_0^t \sum_{(x',a'),(x'',a'') \in \mathcal{X}\times\mathcal{A}} \left(r(x',a')\mathds{1}\{m_1=0\}+\gamma Q_s^{(m_1)}(x'',a'')- Q_s^{(m_1)}(x',a')\right)\nonumber\\
            &\hspace{60pt}\langle C_{(x',a')}^h, v_s^{(m_2)} \rangle g_{s}^{m_3}(x'',a'')\pi^{g_s^{(m_4)}}(x',a')p(x''|x',a')ds\nonumber\\
            &+ O_{L_p}(N^{(\beta-1)}),\label{v_t^{(n)}N_prelim_eq}
        \end{align}}
    and
 {\small        \begin{align}
            &\langle h, {\mu}_t^{n,N} \rangle =\nonumber\\
            &= N^{\phi_n-1+\beta}\sum_{\substack{m_1,m_2,m_3,m_4 \in \{0,1,...,n-1\} \\ m_1+m_2+m_3+m_4=n-1}} \alpha \int_0^t &\sum_{(x',a'),(x'',a'')\in \mathcal{X}\times \mathcal{A}} \frac{1}{1+s} Q_s^{(m_1)}(x,a) \left( \langle C_{(x',a')}^h, \mu_s^{(m_2)} \rangle+\langle C_{(x',a'')}^h, \mu_s^{(m_2)} \rangle\right)\nonumber\\
                & \hspace{60pt} f_s^{(m_3)}(x',a'') \sigma_{\rho_0}^{g_s^{(m_4)}}(x',a')ds\nonumber\\
            &+ O_{L_p}(N^{(\beta-1)}).\label{mu_t^{(n)}N_prelim_eq}
        \end{align}}

    The induction step for hypothesis \ref{limit_terms_hyp} regarding the processes $v_t^N$ and $\mu_t^N$ then follows, and we get the following convergence rates:
    \begin{align}    \langle h, v_t^{n,N}\rangle + \langle h, \mu_t^{n,N} \rangle &= O_{L_p}(N^{\phi_n-1+\beta}), \quad \text{if } \beta \in \left(\frac{2n-1}{2n}, \frac{2n+1}{2n+2}\right)\nonumber\\
        \langle h, v_t^{n,N}-v_t^{(n)}\rangle + \langle h, \mu_t^{n,N}-\mu_t^{(n)} \rangle &= O_{L_p}(N^{\beta -1}), \quad \text{if } \beta = \frac{2n+1}{2n+2}.\label{higher_order_empirical_measure_conv_rates}
    \end{align}

    \subsection{Analysis of the higher order Actor model and Policy errors} \label{f_t^nN_g_t^nN_prelim_sec}

    In this section we will show convergence in $L_p$ of $f_t^{n,N}$ and $g_t^{n,N}$ defined in \ref{rescaled_higher_order_prelim_terms_def_eq} to $f_t^{(n)}$ and $g_t^{(n)}$ respectively. The expansion of $f_t^N$ in \ref{limit_terms_hyp} can be obtained via Taylor expansion. Specifically, considering $\beta \in \left(\frac{2n-1}{2n},1 \right)$ the $n-$th order Taylor expansion of $f_t^N(x,a)$ around $f_t^{(0)}(x,a)$ is
    \begin{equation} \label{f_t^N_n-1-th_order_exp_eq}
        \begin{aligned}
            f_t^N(x,a) &= f_t^{(0)}(x,a) + \sum_{k=1}^{n}N^{k(\beta-1)} D^k(f_t^{(0)}(x,a)) : \left[P_t^{(1)}(x,a) \right]^{\otimes k} \\
            &+ N^{(n+1)(\beta-1)} D^{(n+1)}(f_t^{(0)}(x,a)) : \left[P_t^{1,*} (x,a) \right]^{\otimes n},
        \end{aligned}
    \end{equation}
    where we used tensor notation and assumed $x$ fixed and $a$ being the index. The involved quantities and operations are defined as in equation \eqref{actor_model_higher_order_error_terms_limits_eq}. Moreover, $P_t^{1,*} (x,a)$ lies on the segment connecting $P_t^{(1)} (x,a)$ and $P_t^{1,N}(x,a)$ for every $(x,a)\in \mathcal{X}\times \mathcal{A}$. Notice that using the convergence rates in proposition \ref{first_order_error_conv_prop} and by Hölder's inequality for products of random variables we immediately get
    \begin{equation} \label{actor_model_taylor_remainder_bound_eq}
        \begin{aligned}
            D^{(n)}(f_t^{(0)}(x,a)) : \left[P_t^{1,*} (x,a) \right]^{\otimes n} = O_{L_p}(N^{n(\beta-1)}).
        \end{aligned}
    \end{equation}

    Combining hypothesis \ref{limit_terms_hyp}, equation \eqref{rescaled_higher_order_prelim_terms_def_eq} and the expansion \eqref{f_t^N_n-1-th_order_exp_eq} we get
    
    \begin{equation}
        \begin{aligned}
            \max_{(x,a)\in \mathcal{X}\times\mathcal{A}}\left| f_t^{n,N}(x,a)-f_t^{(n)}(x,a) \right|= O_{L_p}\left( n^{\beta-1}\right),
        \end{aligned}
    \end{equation}
    where we defined
    \begin{equation}
        f_t^{(n)}(x,a) = D^n(f_t^{(0)}(x,a)) : \left[P_t^{(1)}(x,a) \right]^{\otimes n}.
    \end{equation}

    The inductive step of hypothesis \ref{limit_terms_hyp} for $f_t^N$ follows. For the policy $g_t^N$, we now use equations \eqref{g_k^N_eq} and \eqref{discrete_to_cont_eq} along with lemma \eqref{eta_t^N_convergence_error_bound} to obtain
    \begin{equation} \label{g_t^N_higher_order_exp_wrt_f_eq}
        \begin{aligned}
            g_t^N = \frac{\eta_t}{|\mathcal{A}|} + (1-\eta_t)f_t^N +O(N^{-1})
        \end{aligned}
    \end{equation}

    Since the $O(N^{-1})$ error term is small, by the uniqueness of the expansion the inductive step for $g_t^N$ also follows, and we get 
    \begin{equation}
        \begin{aligned}
            \max_{x,a\in \mathcal{X}\times\mathcal{A}}\left| g_t^{n,N}(x,a)-g_t^{(n)}(x,a) \right|= O_{L_p}\left( n^{\beta-1}\right).
        \end{aligned}
    \end{equation}

    The induction step for hypothesis \ref{limit_terms_hyp} regarding the process $g_t^N$ then also follows, and we obtain the following convergence rates
    \begin{equation} \label{higher_order_actor_model_policy_conv_rates}
    \begin{aligned}
        \max_{(x,a)\in \mathcal{X}\times \mathcal{A}}f_t^{n,N}(x,a) + \max_{(x,a)\in \mathcal{X}\times \mathcal{A}}g_t^{n,N}(x,a) = O_{L_p}(N^{\phi_n-1+\beta}) \quad \text{if } \beta \in \left(\frac{2n-1}{2n}, \frac{2n+1}{2n+2}\right)\\
        \max_{(x,a)\in \mathcal{X}\times \mathcal{A}}\left|f_t^{n,N}(x,a)-f_t^{(n)}(x,a)\right| + \max_{(x,a)\in \mathcal{X}\times \mathcal{A}}\left|g_t^{n,N}(x,a)-g_t^{(n)}(x,a)\right| = O_{L_p}(N^{\beta -1}) \quad \text{if } \beta = \frac{2n+1}{2n+2}.
    \end{aligned}
    \end{equation}

    \subsection{Analysis of the higher order stationary measure errors terms}\label{stat_measure_nth_order_prelim_sec}

    In this section we prove the inductive step of hypothesis \ref{limit_terms_hyp} for the expansion of the sationary measures $\pi^{g_t^N}$ and $\sigma_{\rho_0}^{g_t^N}$. We will only present the proof for $\pi^{g_t^N}$ as the proof for $\sigma_{\rho_0}^{g_t^N}$ is similar. As in section \ref{first_order_error_tersms_sec} we make use of the face that $\pi^{g_t^N}$ is the unique measure satisfying
    \begin{equation} \label{stationary_measure_conditions_eq}
        \begin{aligned}
            \pi^{g_t^N} \mathbb{P}_t^N &= \pi^{g_t^N} W_{\pi^{g_t^N}}=\pi^{g_t^N}.
        \end{aligned}
    \end{equation}

    In order to determine the $n-$th term of the expansion for $\pi^{g_t^N}$, we will derive an expansion for both sides of  \eqref{stationary_measure_conditions_eq}. We do this using the expansions up to the $n-1$-th term that hypothesis \ref{limit_terms_hyp} suggests.

    We start by noticing that by its definition in \eqref{stationary_measure_transition_matrices_def_eq}, and using hypothesis \ref{limit_terms_hyp}, we can expand $\mathbb{P}_t^N$ 
    \begin{equation} \label{mathbbP_t^N_expansion_eq}
        \begin{aligned}
            \mathbb{P}_t^N = \mathbb{P}_t^{(0)} +N^{\beta-1}\mathbb{P}_t^{(1)}+...+N^{(n-1)(\beta-1)}\mathbb{P}_t^{(n-1)}+\hat{\mathbb{P}}_t^N,
        \end{aligned}
    \end{equation}
    where
    \begin{equation} \label{mathbbP_t^N_expansion_terms_eq}
        \begin{aligned}
            \mathbb{P}^i_{t, (x,a),(x',a')} &= g_t^i(x',a')p(x'|x,a) \quad 0\leq i \leq n-1\\
            \hat{\mathbb{P}}^N_{t, (x,a),(x',a')} &= \hat{g}_t^N(x',a')p(x'|x,a).
        \end{aligned}
    \end{equation}

    Similarly using the expansion for $\pi^{g_t^N}$ in hypothesis \ref{limit_terms_hyp}, we derive the expansion
    \begin{equation} \label{W^N_expansion_terms_eq}
        \begin{aligned}
            W_{\pi^{g_t^N}} = W_{\pi^{g_t^{(0)}}} +N^{\beta-1}W_{\pi^{g_t^{(1)}}} +...+N^{(n-1)(\beta-1)}W_{\pi^{g_t^{(n-1)}}} + W_{\hat{\pi}^{g_t^N}}.
        \end{aligned}
    \end{equation}

    Using again the expansion of $\pi^{g_t^N}$, we can now derive the expansions for the products $\pi^{g_t^N}\mathbb{P}_t^N$ and $\pi^{g_t^N}W_{\pi^{g_t^N}}$. Specifically, we obtain
    \begin{equation} \label{pi^NP_t^N_expansion_eq_1}
        \begin{aligned}
            \pi^{g_t^N}\mathbb{P}_t^N &= \sum_{k=0}^{n-1}N^{k(\beta-1)}\sum_{j=0}^k \pi^{g_t^{(j)}}\mathbb{P}_t^{k-j}+ N^{n(\beta-1)}\sum_{j=1}^{n-1} \pi^{g_t^{(j)}}\mathbb{P}_t^{(n-j)}+ \sum_{k=n+1}^{2n-2} \sum_{j=k-n+1}^{n-1}N^{k(\beta-1)}\pi^{g_t^{(j)}}\mathbb{P}_t^{k-j}\\
            &+\hat{\pi}^{g_t^N}\mathbb{P}_t^{(0)}+ \pi^{g_t^{(0)}}\tilde{\mathbb{P}}_t^N+ \sum_{k=1}^{n-1}N^{k(\beta-1)}\left[\hat{\pi}^{g_t^N}\mathbb{P}_t^k+ \pi^{g_t^k}\hat{\mathbb{P}}_t^N \right]+\hat{\pi}^{g_t^N}\hat{\mathbb{P}}_t^N,
        \end{aligned}
    \end{equation}
    where by hypothesis \ref{limit_terms_hyp} we have 
    \begin{equation}
        \begin{aligned}
            \sum_{k=n+1}^{2n-2} \sum_{j=k-n+1}^{n-1}N^{k(\beta-1)}\pi^{g_t^{(j)}}\mathbb{P}_t^{k-j} = O_{L_p}(N^{(n+1)(\beta-1)})\\
            \sum_{k=1}^{n-1}N^{k(\beta-1)}\left[\tilde{\pi}^{g_t^N}\mathbb{P}_t^k+ \pi^{g_t^k}\tilde{\mathbb{P}}_t^N \right] = O_{L_p}(N^{n(\beta-1)-\phi_n})\\
            \hat{\pi}^{g_t^N}\hat{\mathbb{P}}_t^N = O_{L_p}(N^{(2n-2)(\beta-1)-2\phi_n}),
        \end{aligned}
    \end{equation}
    which implies 
    \begin{equation} \label{pi^NP_t^N_expansion_eq}
        \begin{aligned}
            \pi^{g_t^N}\mathbb{P}_t^N &= \sum_{k=0}^{n-1}N^{k(\beta-1)}\sum_{j=1}^k \pi^{g_t^{(j)}}\mathbb{P}_t^{k-j}+ N^{n(\beta-1)}\sum_{j=1}^{n-1} \pi^{g_t^{(j)}}\mathbb{P}_t^{(n-j)}+\hat{\pi}^{g_t^N}\mathbb{P}_t^{(0)}+ \pi^{g_t^{(0)}}\hat{\mathbb{P}}_t^N\\
            &+ O_{L_p}(N^{n(\beta-1)-\phi_n}).
        \end{aligned}
    \end{equation}

In the same way we can obtain 
    \begin{equation} \label{pi^NW_expansion_eq}
        \begin{aligned}
            \pi^{g_t^N}W_{\pi^{g_t^N}} &= \sum_{k=0}^{n-1}N^{k(\beta-1)}\sum_{j=0}^k \pi^{g_t^{(j)}}W_{\pi^{g_t^{(k-j)}}}+ N^{n(\beta-1)}\sum_{j=1}^{n-1} \pi^{g_t^{(j)}}W_{\pi^{g_t^{(n-j)}}}+\hat{\pi}^{g_t^N}W_{\pi^{g_t^{(0)}}}+ \pi^{g_t^{(0)}}W_{\hat{\pi}^{g_t^N}}\\
            &+ O_{L_p}(N^{n(\beta-1)-\phi_n})
        \end{aligned}
    \end{equation}

Since $\pi^{g_t^N}\mathbb{P}_t^N - \pi^{g_t^N}W_{\pi^{g_t^N}}=0$, we subtract \eqref{pi^NW_expansion_eq} from \eqref{pi^NP_t^N_expansion_eq} and require all expansion terms be equal to $0$. Notice that for $0\leq k\leq n-1$, the $k-$th expansion term of their difference is
    \begin{equation*}
        \begin{aligned}
            \sum_{j=0}^k \pi^{g_t^{(j)}}\mathbb{P}_t^{k-j} - \sum_{j=0}^k \pi^{g_t^{(j)}}W_{\pi^{g_t^{(k-j)}}} = -\pi^{g_t^k}-\sum_{j=0}^k \pi^{g_t^{(j)}}\mathbb{P}_t^{k-j} - \sum_{j=1}^k \pi^{g_t^{(j)}}W_{\pi^{g_t^{(k-j)}}}=0,
        \end{aligned}
    \end{equation*}
    which is precisely the formula of $\pi^{g_t^k}$ in \eqref{stationary_measure_higher_order_error_terms_limits_eq} and which is suggested in hypothesis \ref{limit_terms_hyp}. Setting the n-th term in the expansion of the difference $\pi^{g_t^N}\mathbb{P}_t^N - \pi^{g_t^N}W_{\pi^{g_t^N}}$ equal to $0$ and considering the rescaled error term $\pi^{g_t^{n,N}}$ in \eqref{rescaled_higher_order_prelim_terms_def_eq}
    \begin{equation*}
        \begin{aligned}
            \mathbb{P}_t^{n,N} &= N^{(n-1)(1-\beta)+\phi_n}\hat{\mathbb{P}}_t^N,
        \end{aligned}
    \end{equation*} 
    we get
\begin{equation}
        \begin{aligned}
            \pi^{g_t^{(0)}}W_{\pi^{g_t^{n,N}}} &=N^{\phi_n-1+\beta} \sum_{j=1}^{n-1} \pi^{g_t^{(j)}}\mathbb{P}_t^{(n-j)} - N^{\phi_n-1+\beta}\sum_{j=1}^{n-1} \pi^{g_t^{(j)}}W_{\pi^{g_t^{(n-j)}}}  +{\pi}^{g_t^{n,N}}\mathbb{P}_t^{(0)}+ \pi^{g_t^{(0)}}{\mathbb{P}}_t^{n,N}- {\pi}^{g_t^{n,N}}W_{\pi^{g_t^{(0)}}}.
        \end{aligned}
    \end{equation}

    Using the fact that $\pi^{g_t^{(0)}}W_{\pi^{g_t^{n,N}}} = \pi^{g_t^{n,N}}$ we obtain
    \begin{equation} \label{pi_t^nN_prelim_eq}
        \begin{aligned}
            \pi^{g_t^{n,N}} = \left[ \pi^{g_t^{(0)}}{\mathbb{P}}_t^{n,N}+ N^{\phi_n-1+\beta}\sum_{j=1}^{n-1} \pi^{g_t^{(j)}}\mathbb{P}_t^{(n-j)} - N^{\phi_n-1+\beta}\sum_{j=1}^{n-1} \pi^{g_t^{(j)}}W_{\pi^{g_t^{(n-j)}}}\right] \left( I-\mathbb{P}_t^{(0)}+W_{\pi^{g_t^{(0)}}}\right)^{-1}.
        \end{aligned}
    \end{equation}

    Analogously, one can also derive
    \begin{equation} \label{sigma_t^nN_prelim_eq}
        \begin{aligned}
            \sigma_{\rho_0}^{g_t^{n,N}} = \left[ \sigma_{\rho_0}^{g_t^{(0)}}{\Pi}_t^{n,N}+ N^{\phi_n-1+\beta}\sum_{j=1}^{n-1} \sigma_{\rho_0}^{g_t^{(j)}}\Pi_t^{(n-j)} - N^{\phi_n-1+\beta}\sum_{j=1}^{n-1} \sigma_{\rho_0}^{g_t^{(j)}}W_{\sigma_{\rho_0}^{g_t^{(n-j)}}}\right] \left( I-\mathbb{P}_t^{(0)}+W_{\sigma_{\rho_0}^{g_t^{(0)}}}\right)^{-1}.
        \end{aligned}
    \end{equation}

    Notice that by their definitions, the results of section \ref{f_t^nN_g_t^nN_prelim_sec}, and the equations \eqref{stationary_measure_higher_order_error_terms_limits_eq}, the elements of the error matrices  $\mathbb{P}_t^{n,N}$ and $\Pi_t^{n,N}$ are in $O_{L_p}(N^{\phi_n-1+\beta})$. Moreover, the inductive step of hypothesis \ref{limit_terms_hyp} for the processes $\pi^{g_t^N}$ and $\sigma_{\rho_0}^{g_t^N}$ is satisfied, and we get the following convergence rates:
    \begin{equation} \label{higher_order_stationary_measures_conv_rates}
    \begin{aligned}
        \max_{(x,a)\in \mathcal{X}\times \mathcal{A}}\pi^{g_t^{n,N}}(x,a) + \max_{(x,a)\in \mathcal{X}\times \mathcal{A}}\sigma_{\rho_0}^{g_t^{n,N}}(x,a) = O_{L_p}(N^{\phi_n-1+\beta}) \quad \text{if } \beta \in \left(\frac{2n-1}{2n}, \frac{2n+1}{2n+2}\right)\\
        \max_{(x,a)\in \mathcal{X}\times \mathcal{A}}\left|\pi^{g_t^{n,N}}(x,a)-\pi^{g_t^{(n)}}(x,a)\right| + \max_{(x,a)\in \mathcal{X}\times \mathcal{A}}\left|\sigma_{\rho_0}^{g_t^{n,N}}(x,a)-\sigma_{\rho_0}^{g_t^{(n)}}(x,a)\right| = O_{L_p}(N^{\beta -1}) \quad \text{if } \beta = \frac{2n+1}{2n+2}.
    \end{aligned}
    \end{equation}

    \subsection{Proof of Hypothesis \ref{limit_terms_hyp}}\label{sec:proofInductiveHypothesis}
    We already argued for the inductive step on hypothesis \ref{limit_terms_hyp} for the processes $f_t^{n,N}$, $g_t^{n,N}$, $v_t^{n,N}$, $\mu_t^{n,N}$, $\pi^{g_t^{n,N}}$, $\sigma_{\rho_0}^{g_t^{n,N}}$  for $\beta \in \left(\frac{2n-1}{2n},1\right)$ in the previous sections. 
In this section we prove the inductive step in the hypothesis \ref{limit_terms_hyp}  for the processes  $Q_t^{n,N}$ and $P_t^{n,N}$ for the different values of $\beta$.

    Notice that subtracting $Q_t^{(n)}$ as defined in \eqref{Q_P_higher_order_error_terms_limits_eq} from \eqref{Q_t^nN_prelim_eq} and using lemma \ref{prelim_initial_condition_L_1_L_2_bounds} gives
        \begin{align}
            &Q_t^{n,N}(x,a)-Q_t^{(n)}(x,a) \label{Q_t^{(n)}N-Q_t^N_eq}\\
            &= \int_{0}^t \sum_{(x',a'),(x'',a'')\in \mathcal{X}\times \mathcal{A}} \Big[ \gamma \left( Q_s^{n,N}(x'',a'')-Q_s^{(n)}(x'',a'')\right) - \left(Q_s^{n,N}(x',a')- Q_s^{(n)}(x',a')\right)\Big]\nonumber\\
            &\hspace{60pt}\langle B_{(x,a),(x',a')},  v_0^{(0)}\rangle g_s^{(0)}(x'',a'') \pi^{g_s^{(0)}}(x'',a'')p(x''|x',a')ds\nonumber\\
             &+ \int_{0}^t \sum_{(x',a'),(x'',a'')\in \mathcal{X}\times \mathcal{A}} \Big[r(x',a')+\gamma Q_s^{(0)}(x'',a'')-Q_s^{(0)}(x',a') \Big]\nonumber\\
             &\hspace{60pt} \langle B_{(x,a),(x',a')}, v_s^{n,N}-v_s^{(n)}\rangle g_s^{(0)}(x'',a'') \pi^{g_s^{(0)}}(x'',a'')p(x''|x',a')ds\nonumber\\
            &+ \int_{0}^t \sum_{(x',a'),(x'',a'')\in \mathcal{X}\times 
/\mathcal{A}} \Big[r(x',a')+\gamma Q_s^{(0)}(x'',a'')-Q_s^{(0)}(x',a') \Big]\nonumber\\
             &\hspace{60pt} \langle B_{(x,a),(x',a')}, v_0^{(0)}\rangle \left[ g_s^{n,N}(x'',a'')-g_s^{(n)}(x'',a'')\right] \pi^{g_s^{(0)}}(x'',a'')p(x''|x',a')ds\nonumber\\
            &+ \int_{0}^t \sum_{(x',a'),(x'',a'')\in \mathcal{X}\times \mathcal{A}} \Big[r(x',a')+\gamma Q_s^{(0)}(x'',a'')-Q_s^{(0)}(x',a') \Big]\nonumber\\
            &\hspace{60pt} \langle B_{(x,a),(x',a')}, v_0^{(0)}\rangle g_s^{(0)}(x'',a'') \left[\pi^{g_s^{n,N}}(x'',a'')-\pi^{g_s^{(n)}}(x'',a'')\right]p(x''|x',a')ds\nonumber\\
                &\hspace{20pt} +O_{L_p}(N^{\beta-1})  + O_{L_p}(N^{-\frac{1}{2}+(n+1)(1-\beta)}).
        \end{align}

    Using the convergence rates \eqref{higher_order_empirical_measure_conv_rates}, \eqref{higher_order_actor_model_policy_conv_rates}, \eqref{higher_order_stationary_measures_conv_rates} and \eqref{Q_t^{(n)}N_conv_rates} previously derived in this section, we obtain
    \begin{equation}
        \begin{aligned}
            \max_{x,a\in \mathcal{X}\times\mathcal{A}} \left|Q_t^{n,N}(x,a)-Q_t^{(n)}(x,a) \right|^p &\leq C_T \int_0^t \max_{x,a\in \mathcal{X}\times\mathcal{A}} \left|Q_s^{n,N}(x,a)-Q_s^{(n)}(x,a) \right|^p  +C_{T,p}N^{p(\beta-1)} \\
            &+ C_{T,p} N^{-\frac{1}{2}+(n+1)(1-\beta)}.
        \end{aligned}
    \end{equation}

    Grönwall's inequality then gives
    \begin{equation} \label{Q_t^{(n)}N_conv_rates}
        \begin{aligned}
            \max_{x,a\in \mathcal{X}\times\mathcal{A}} \left|Q_t^{n,N}(x,a)-Q_t^{(n)}(x,a) \right| = O_{L_p}(N^{\beta-1}) +  O_{L_p}(N^{-\frac{1}{2}+(n+1)(1-\beta)}),
        \end{aligned}
    \end{equation}
    which by the form of $Q_t^{(n)}$ proves the inductive step of hypothesis \ref{limit_terms_hyp} for $Q_t^N$. 

    We now consider the process $P_t^{n,N}$. Subtracting $P_t^{(n)}$ as defined in \eqref{Q_P_higher_order_error_terms_limits_eq} from \eqref{P_t^nN_prelim_eq} gives
\begin{equation}\label{P_t^{(n)}N-P_t^{(n)}_prelim_eq}
        \begin{aligned}
            &\hat{P}_t^{n,N}(x,a) - P_t^{(n)}(x,a)\\
            & =\int_{0}^t \sum_{(x',a',a'')\in \mathcal{X}\times \mathcal{A}\times \mathcal{A}}\frac{1}{1+s}  \left(Q_s^{n,N}(x',a')-Q_s^{(n)}(x',a')\right)\\
            &\hspace{60pt}\left[ \langle B_{(x,a),(x',a')}, \mu_0^{(0)}\rangle+\langle B_{(x,a),(x',a'')}, \mu_0^{(0)}\rangle\right]f_s^{(0)}(x',a'')\sigma_{\rho_0}^{g_s^{(0)}}(x',a')ds\\
            & +\int_{0}^t \sum_{(x',a',a'')\in \mathcal{X}\times \mathcal{A}\times \mathcal{A}}\frac{1}{1+s}  Q_s^{(0)}(x',a')\\
            &\hspace{60pt}\left[ \langle B_{(x,a),(x',a')}, \mu_s^{n,N}-\mu_s^{(n)}\rangle+\langle B_{(x,a),(x',a'')}, \mu_s^{n,N}\rangle\right]f_s^{(0)}(x',a'')\sigma_{\rho_0}^{g_s^{(0)}}(x',a')ds\\
            & +\int_{0}^t \sum_{(x',a',a'')\in \mathcal{X}\times \mathcal{A}\times \mathcal{A}}\frac{1}{1+s}  Q_s^{(0)}(x',a')\\
            &\hspace{60pt}\left[ \langle B_{(x,a),(x',a')}, \mu_0^{(0)}\rangle+\langle B_{(x,a),(x',a'')}, \mu_0^{(0)}\rangle\right]\left(f_s^{n,N}(x',a'')-f_s^{(n)}(x',a'') \right)\sigma_{\rho_0}^{g_s^{(0)}}(x',a')ds\\
            & +\int_{0}^t \sum_{(x',a',a'')\in \mathcal{X}\times \mathcal{A}\times \mathcal{A}}\frac{1}{1+s}  Q_s^{(0)}(x',a')\\
            &\hspace{60pt}\left[ \langle B_{(x,a),(x',a')}, \mu_0^{(0)}\rangle+\langle B_{(x,a),(x',a'')}, \mu_0^{(0)}\rangle\right]f_s^{(0)}(x',a'')\left( \sigma_{\rho_0}^{g_s^{n,N}}(x',a')- \sigma_{\rho_0}^{g_s^{(n)}}(x',a') \right)ds\\
                &\hspace{20pt} +O_{L_p}(N^{\beta-1}) +  O_{L_p}(N^{-\frac{1}{2}+(n+1)(1-\beta)}).
        \end{aligned}
    \end{equation}

    Using the convergence rates \eqref{higher_order_empirical_measure_conv_rates}, \eqref{higher_order_actor_model_policy_conv_rates}, \eqref{higher_order_stationary_measures_conv_rates} and \eqref{Q_t^{(n)}N_conv_rates} previously derived in this section, we obtain
    \begin{equation} \label{P_t^{(n)}N_conv_rates}
        \begin{aligned}
            \max_{x,a\in \mathcal{X}\times\mathcal{A}} \left|P_t^{n,N}(x,a)-P_t^{(n)}(x,a) \right| = O_{L_p}(N^{\beta-1}) +  O_{L_p}(N^{-\frac{1}{2}+(n+1)(1-\beta)}),
        \end{aligned}
    \end{equation}
    which by the form of $P_t^{(n)}$ proves the inductive step of hypothesis \ref{limit_terms_hyp} for $P_t^N$. 

    \subsection{Relative Compactness for $\beta \in \left(\frac{2n-1}{2n}, \frac{2n+1}{2n+2}\right]$}

    In this section we consider $\beta \in \left(\frac{2n-1}{2n}, \frac{2n+1}{2n+2}\right]$. The expansions in hypothesis \ref{limit_terms_hyp} hold and we show relative compactness for the rescaled error terms $Q_t^{n,N}$, $P_t^{n,N}$, $f_t^{n,N}$, $g_t^{n,N}$, $v_t^{n,N}$, $\mu_t^{n,N}$, $\pi^{g_t^{n,N}}$ and $\sigma_{\rho_0}^{g_t^{n,N}}$.
    
    Compact containment of those processes follows immediately from the convergence rates in hypothesis \ref{limit_terms_hyp}, which gives that they are all in $O_{L_P}(1)$.

    We now start by inductively showing continuity for the limit processes $Q_t^{i}$, $P_t^{i}$, $f_t^{i}$, $g_t^{i}$, $v_t^{i}$, $\mu_t^{i}$, $\pi^{g_t^{i}}$ and $\sigma_{\rho_0}^{g_t^{i}}$ for $i <n$. The results are summarized in the following lemma.
    \begin{lemma} \label{higher_order_limit_continuity_lemma}
        For $\beta \in \left(\frac{2n-1}{2n}, \frac{2n+1}{2n+2}\right]$ and any $i<n$ and $t\in [0,T)$, there is a constant $C_T$ independent of $N$ such that the following bounds hold for any $\delta \in [t,T-t]$
        \begin{equation*}
        \begin{aligned}
            \max_{(x,a)\in \mathcal{X}\times\mathcal{A}}\left|Q_{t+\delta}^{(i)}(x,a) - Q_t^{(i)}(x,a) \right| +\max_{(x,a)\in \mathcal{X}\times\mathcal{A}}\left|P_{t+\delta}^{(i)}(x,a) - P_t^{(i)}(x,a) \right|&\leq C_T \cdot \delta\\
            \max_{(x,a)\in \mathcal{X}\times\mathcal{A}}\left|f_{t+\delta}^{(i)}(x,a) - f_t^{(i)}(x,a) \right|+\max_{(x,a)\in \mathcal{X}\times\mathcal{A}}\left|g_{t+\delta}^{(i)}(x,a) - g_t^{(i)}(x,a) \right| &\leq C_T \cdot \delta\\
            \max_{(x,a)\in \mathcal{X}\times\mathcal{A}}\left|\pi^{g_{t+\delta}^{(i)}}(x,a) - \pi^{g_t^{(i)}}(x,a) \right| +\max_{(x,a)\in \mathcal{X}\times\mathcal{A}}\left|\sigma_{\rho_0}^{g_{t+\delta}^{(i)}}(x,a) - \sigma_{\rho_0}^{g_t^{(i)}}(x,a) \right|&\leq C_T \cdot \delta
        \end{aligned}
        \end{equation*}

        Moreover, for any fixed $h\in C_b^3(\mathds{R})$ the following bounds also hold
        \begin{equation*}
            \begin{aligned}
                \left| \langle h, v_{t+\delta}^{(i)} \rangle - \langle h, v_t^{(i)} \rangle \right|+\left| \langle h, \mu_{t+\delta}^{(i)} \rangle - \langle h, \mu_t^{(i)} \rangle \right| &\leq C_T\cdot \delta
            \end{aligned}
        \end{equation*}
        \end{lemma}

        \begin{proof}
            The bounds for $i=0$ have been shown in section \ref{leading_lim_prelim_reg_sec}, which also implies boundedness for the corresponding processes. We now first notice that this also implies that all integrands in \eqref{Q_t^{(1)}_def_eq}, \eqref{P_t^{(1)}_def_eq}, and the equations for $v_t^{(1)}$ and $\mu_t^{(1)}$ in \eqref{leading_order_error_terms_limits_eq} are bounded, and so the result for $Q_t^{(1)}$, $P_t^{(1)}$, $v_t^{(1)}$ and $\mu_t^{(1)}$ follows from the integral inequality \eqref{ML-Inequality}. The processes $f_t^{(1)}$, $g_t^{(1)}$ satisfy the lemma's bounds as sums and products of $f_t^{(0)}(x,a)$ and $P_t^{(1)}(x',a')$ for  $(x,a),(x',a') \in \mathcal{X}\times\mathcal{A}$. Similarly, $\pi^{g_t^{(1)}}$ and $\sigma_{\rho_0}^{g_t^{(1)}}$ satisfy the lemma's bounds as sums and products of $g_t^{(0)}(x,a)$ , $g_t^{(1)}(x',a')$,$\pi^{g_t^{(0)}}(x'',a'')$ and $\sigma_{\rho_0}^{g_t^{(0)}}(x''',a''')$  for  $(x,a), (x',a'), (x'',a''), (x''',a''') \in \mathcal{X}\times\mathcal{A}$. The remaining bounds can be shown inductively. 
            
            We assume that the bounds hold for $i\leq k-1<n$. This yields boundedness for the limit processes $Q_t^{(i)}$, $P_t^{(i)}$, $f_t^{(i)}$, $g_t^{(i)}$, $v_t^{(i)}$, $\mu_t^{(i)}$, $\pi^{g_t^{(i)}}$ and $\sigma_{\rho_0}^{g_t^{(i)}}$ for $i \leq k-1$. All integrands in the formulas for $Q_t^{(i)}$, $P_t^k$, $v_t^k$ and $\mu_t^k$ as derived in section \ref{main_results_sec} are bounded, and so their corresponding results follow from the integral inequality \eqref{ML-Inequality}. The processes $f_t^{(i)}$, $g_t^{(i)}$ satisfy the lemma's bounds as sums and products of $f_t^{(0)}(x,a)$ and $P_t^{(1)}(x',a')$ for  $(x,a),(x',a') \in \mathcal{X}\times\mathcal{A}$. Lastly, the processes $\pi^{g_t^{(i)}}$ and $\sigma_{\rho_0}^{g_t^{(i)}}$ satisfy the lemma's bounds as sums and products of $\pi^{g_t^{(i)}}$, $\sigma_{\rho_0}^{g_t^{(i)}}$, and $g_t^{(i)}$ for all $i\leq k-1$ and evaluated at  state-action pairs.
            The result of the lemma then follows.
        \end{proof}

        We are now ready to show regularity for the processes for the rescaled error terms $Q_t^{n,N}$, $P_t^{n,N}$, $f_t^{n,N}$, $g_t^{n,N}$, $v_t^{n,N}$, $\mu_t^{n,N}$, $\pi^{g_t^{n,N}}$ and $\sigma_{\rho_0}^{g_t^{n,N}}$.
        From the pre-limit equations \eqref{v_t^{(n)}N_prelim_eq} and \eqref{mu_t^{(n)}N_prelim_eq} and the result of lemma \ref{higher_order_limit_continuity_lemma} we immediately get the following regularity result for $v_t^{n,N}$ and $\mu_t^{n,N}$.
        \begin{lemma} \label{v_t^{(n)}N_mu_t^{(n)}N_reg_lemma}
        For any $h\in C_b^3(\mathds{R})$ the processes $v_t^{n,N}$ and $\mu_t^{n,N}$ satisfy the following bounds
        \begin{equation}
            \begin{aligned}
                \left| \langle h,v_{t+\delta}^{n,N} \rangle - \langle h, v_t^{n,N}\rangle \right| \leq \delta \cdot O_{L_p}(N^{\phi_n-1+\beta}) + O_{L_p}(N^{\beta-1})\\
                \left| \langle h,\mu_{t+\delta}^{n,N} \rangle - \langle h, \mu_t^{n,N}\rangle \right| \leq \delta \cdot O_{L_p}(N^{\phi_n-1+\beta}) + O_{L_p}(N^{\beta-1}).
            \end{aligned}
        \end{equation}
    \end{lemma}

From our analysis in section \ref{f_t^nN_g_t^nN_prelim_sec} and lemma \ref{higher_order_limit_continuity_lemma} the next regularity result for $f_t^{n,N}$ and $g_t^{n,N}$ follows.
    \begin{lemma} \label{f_t^{(n)}N_g_t^{(n)}N_reg_lemma}
        The processes $f_t^{n,N}$ and $g_t^{n,N}$ satisfy the following bounds
        \begin{equation}
            \begin{aligned}
                \max_{(x,a)\in \mathcal{X}\times\mathcal{A}}\left|f_{t+\delta}^{n,N}(x,a)  - f_t^{n,N}(x,a) \right| \leq \delta \cdot O_{L_p}(N^{\phi_n-1+\beta}) + O_{L_p}(N^{\beta-1})\\
                \max_{(x,a)\in \mathcal{X}\times\mathcal{A}}\left|g_{t+\delta}^{n,N}(x,a)  - g_t^{n,N}(x,a) \right| \leq \delta \cdot O_{L_p}(N^{\phi_n-1+\beta}) + O_{L_p}(N^{\beta-1}).
            \end{aligned}
        \end{equation}
    \end{lemma}

    From the pre-limit equations \eqref{pi_t^nN_prelim_eq} and \eqref{sigma_t^nN_prelim_eq} and lemma \ref{higher_order_limit_continuity_lemma} we obtain the following regularity result for $\pi^{g_t^{n,N}}$ and $\sigma_{\rho_0}^{g_t^{n,N}}$.
    \begin{lemma} \label{pi_t^{(n)}N_sigma_t^{(n)}N_reg_lemma}
        The processes $\pi^{g_t^{n,N}}$ and $\sigma_{\rho_0}^{g_t^{n,N}}$ satisfy the following bounds
        \begin{equation}
            \begin{aligned}
                \max_{(x,a)\in \mathcal{X}\times\mathcal{A}}\left|\pi^{g_{t+\delta}^{n,N}}(x,a)  - \pi^{g_t^{n,N}}(x,a) \right| \leq \delta \cdot O_{L_p}(N^{\phi_n-1+\beta}) + O_{L_p}(N^{\beta-1})\\
                \max_{(x,a)\in \mathcal{X}\times\mathcal{A}}\left|\sigma_{\rho_0}^{g_{t+\delta}^{n,N}}(x,a)  - \sigma_{\rho_0}^{g_t^{n,N}}(x,a) \right| \leq \delta \cdot O_{L_p}(N^{\phi_n-1+\beta}) + O_{L_p}(N^{\beta-1}).
            \end{aligned}
        \end{equation}
    \end{lemma}

    Lastly, we consider the pre-limit equations \eqref{Q_t^nN_prelim_eq} and  \eqref{P_t^nN_prelim_eq}.  From the compact containment of the rescaled $n$-th order limit terms, as argued in the beginning of this section, and using the result from lemma \ref{higher_order_limit_continuity_lemma} on the boundedness of the limit processes up to order $n-1$, we see that all integrands in the pre-limit equations \eqref{Q_t^nN_prelim_eq} and  \eqref{P_t^nN_prelim_eq} are in $O_{L_p}(1;N)$  lemma \eqref{prelim_initial_condition_L_1_L_2_bounds} and the integral inequality \eqref{ML-Inequality} then give the following bounds

    \begin{lemma} \label{Q_t^{(n)}N_P_t^{(n)}N_reg_lemma}
        The processes $Qt^{n,N}$ and $P_t^{n,N}$ satisfy the following bounds
        \begin{equation}
            \begin{aligned}
                \max_{(x,a)\in \mathcal{X}\times\mathcal{A}}\left|Q_{t+\delta}^{n,N}(x,a)  - Q_t^{n,N}(x,a) \right| \leq \delta \cdot O_{L_p}(N^{\phi_n-1+\beta}) + O_{L_p}(N^{\beta-1}) + O_{L_p}(N^{\frac{1}{2}-n(1-\beta)})\\
                \max_{(x,a)\in \mathcal{X}\times\mathcal{A}}\left|P_{t+\delta}^{n,N}(x,a)  - P_t^{n,N}(x,a) \right| \leq \delta \cdot O_{L_p}(N^{\phi_n-1+\beta}) + O_{L_p}(N^{\beta-1}) + O_{L_p}(N^{\frac{1}{2}-n(1-\beta)}).
            \end{aligned}
        \end{equation}
    \end{lemma}

    Combining compact containment and regularity, relative compactness of the processes $Q_t^{n,N}$, $P_t^{n,N}$, $f_t^{n,N}$, $g_t^{n,N}$, $v_t^{n,N}$, $\mu_t^{n,N}$, $\pi^{g_t^{n,N}}$ and $\sigma_{\rho_0}^{g_t^{n,N}}$ follows.

    \subsection{Uniqueness of the solution}

    The existence of at most one solution to the equations \eqref{actor_model_higher_order_error_terms_limits_eq}, \eqref{stationary_measure_higher_order_error_terms_limits_eq}, \eqref{empirical_measure_higher_order_error_terms_limits_eq} and \eqref{Q_P_higher_order_error_terms_limits_eq}. Can be shown in a way analogous to section \ref{uniqueness_sec_1} using an inductive argument. We assume the error terms $(Q_t^{(m)}, P_t^{ (m)}, v_t^{ (m)}, \mu_t^{ (m)}, f_t^{ (m)}, g_t^{ (m)}, {\pi}^{g_t^{ (m)}}, {\sigma}_{\rho_0}^{g_t^{ (m)}})$ up to order $m\leq n-1$ to admit at most one solution, and we assume $(Q_t^{(n)}, P_t^{ (n)}, v_t^{ (n)}, \mu_t^{ (n)}, f_t^{ (n)}, g_t^{ (n)}, {\pi}^{g_t^{ (n)}}, {\sigma}_{\rho_0}^{g_t^{ (n)}})$ and $(\tilde{Q}_t^{ (n)}, \tilde P_t^{ (n)},\tilde v_t^{ (n)}, \tilde \mu_t^{ (n)}, \tilde f_t^{ (n)},\tilde  g_t^{ (n)}, \tilde {\pi}^{g_t^{ (n)}},\tilde {\sigma}_{\rho_0}^{g_t^{ (n)}})$ to be two solutions to the $n-th$ order error term equations. By the uniqueness of the leading order limit equation solutions and the inductive hypothesis we immediately get that $f_t^{ (n)}= \tilde f_t^{ (n)}$,  $g_t^{ (n)}= \tilde g_t^{ (n)}$, ${\pi}^{g_t^{ (n)}} = \tilde {\pi}^{g_t^{ (n)}}$, and ${\sigma}_{\rho_0}^{g_t^{ (n)}}) = \tilde  {\sigma}_{\rho_0}^{g_t^{ (n)}})$, $ v_t^{ (n)}= \tilde v_t^{ (n)}$, $\mu_t^{ (n)} = \tilde \mu_t^{ (n)}$. From \eqref{Q_P_higher_order_error_terms_limits_eq} we obtain
    \begin{align}
            &Q_t^{(n)}(x,a) - \tilde Q_t^{(n)}(x,a) = Q_0^{(n)}(x,a) - \tilde Q_0^{(n)}(x,a) \nonumber\\
            &+   \alpha \int_{0}^t \sum_{(x',a'),(x'',a'')\in \mathcal{X}\times \mathcal{A}} \left[  \gamma \left(Q_s^{(n)}(x'',a'') - \tilde Q_s^{(n)}(x'',a'')\right) -\left( Q_s^{(n)}(x',a') - \tilde Q_s^{(n)}(x',a') \right) \right] \nonumber\\
                &\hspace{60pt}\times\langle B_{(x,a),(x',a')}, v_0^{(0)}\rangle g_s^{(0)}(x'',a'') \pi^{g_s^{(0)}}(x'',a'')p(x''|x',a')ds\nonumber\\
        &P_t^{(n)}(x,a) - \tilde P_t^{(n)}(x,a) = P_0^{(n)}(x,a) - \tilde P_0^{(n)}(x,a) \\
            &+  \int_{0}^t \sum_{(x',a',a'')\in \mathcal{X}\times \mathcal{A}\times \mathcal{A}} \frac{1}{1+s} \left( Q_s^{(n)}(x',a') -\tilde  Q_s^{(n)}(x',a')\right)\left(\langle B_{(x,a),(x',a')}, \mu_s^{(0)}\rangle+\langle B_{(x,a),(x',a'')}, \mu_s^{(0)}\rangle \right)\nonumber\\
                &\hspace{60pt}\times f_s^{(0)}(x',a'') \sigma_{\rho_0}^{g_s^{(0)}}(x',a')ds.\nonumber
        \end{align}

Assuming now that $Q_0^{(n)} (x,a) = \tilde Q_0^{(n)}$ and $P_0^{(n)} (x,a) = \tilde P_0^{(n)}$, we get 
\begin{align}
        &\max_{(x,a)\in \mathcal{X}\times\mathcal{A}}\left|Q_t^{(n)}(x,a) - \tilde Q_t^{(n)}(x,a)\right| +\max_{(x,a)\in \mathcal{X}\times\mathcal{A}}\left|P_t^{(n)}(x,a) - \tilde P_t^{ (n)}(x,a)\right|\leq  \nonumber \\
            &\leq C \int_{0}^t \left[\max_{(x,a)\in \mathcal{X}\times\mathcal{A}}\left|Q_s^{(n)}(x,a) - \tilde Q_s^{(n)}(x,a)\right|+\max_{(x,a)\in \mathcal{X}\times\mathcal{A}}\left|P_s^{(n)}(x,a) - \tilde P_s^{(n)}(x,a)\right|\right]ds,\nonumber
\end{align}
and uniqueness follows by the standard the Grönwall inequality.

    \subsection{Proof of Theorem \ref{higher_order_conv_th}}

    Given the result of the results in this section, to prove theorem \ref{higher_order_conv_th} we use similar arguments as in section \ref{first_order_conv_small_beta_sec}. Relative compactness which follows from compact containment and regularity,  shown in the previous sections, ensure that the process $(Q_t^{n,N},P_t^{n,N}, f_t^{n,N}, g_t^{n,N}, v_t^{n,N},\mu_t^{n,N}, \pi^{g_t^{n,N}}, \sigma_{\rho_0}^{g_t^{n,N}})$ converges weakly. Moreover, since the integral operator is continuous, we may substitute all the processes with their limits ( for which uniqueness has also been established) in equations \eqref{Q_t^nN_prelim_eq} and \eqref{P_t^nN_prelim_eq}. Then theorem \ref{higher_order_conv_th} follows.

        \section{Slowly varying Functionals on slowly varying Inhomogeneous Markov Chains} \label{mc_sec_app}
        In this section we will study the fluctuations of slowly varying functionals, averaged over an inhomogeneous but slowly varying Markov chain, around their mean. Our goal is to bound the martingale terms \eqref{M_t^N_eq} and \eqref{M_ht^N_eq}. Specifically, we will consider quantities of the form
        \begin{equation} \label{M_t^N_general_eq}
            M_{k_0,K}^N = \sum_{k=k_0}^{K}h_k^N(s_k) - \sum_{k=k_0}^K \sum_{s\in \mathcal{S}}h_k^N(s)\pi_k^N(s),
        \end{equation}
        that satisfy the following properties:
        \begin{assumption} \label{MC_bound_assumption}
            We assume that :
            \begin{enumerate}
            \item $(s_k)_k$ is a sample trajectory of an inhomogeneous Markov chain $\mathcal{M}$ with state space $\mathcal{S}$ and transition matrix $\mathds{P}_k^N$ at step $k$. 
            \item The state space $\mathcal{S}$ has finite size.
            \item The homogeneous Markov chains $(\mathcal{M}, \mathds{P}_k^N)$ with the transition matrix being $\mathds{P}_k^N$ throughout each step are ergodic and have unique stationary distributions $\pi_k^N$. In particular, we can define $\pi_k^N$ to be the probability measure satisfying
            \begin{equation*}
                \pi_k^N = \pi_k^N\mathbb{P}_k^N
            \end{equation*}
            in matrix notation.
            \item There is a constant $\epsilon>0$ such that for every $s,s' \in \mathcal{S}$ and $k_0\leq k \leq K$, the transition probabilities $\mathds{P}_k^N(s\rightarrow s')$ satisfy
                \begin{equation*}
                     \mathds{P}_k^N(s\rightarrow s') \geq \epsilon >0,
                \end{equation*}
                whenever $\mathds{P}_k^N(s\rightarrow s')>0$, i.e transition probabilities that are positive are so uniformly over $N$
            \item There exists a uniform constant $C<\infty$ such that 
                \begin{equation*}
                    \begin{aligned}
                        \max_{s,s'\in \mathcal{S},k_0\leq k\leq K-1} \left| \mathds{P}_{k+1}^N(s\rightarrow s') - \mathds{P}_{k}^N(s\rightarrow s') \right| \leq CN^{-\beta}\\
                        \max_{s\in \mathcal{S}, k_o \leq k \leq K-1} \left| \pi_{k+1}^N(s) - \pi_k^N(s) \right| \leq C_TN^{-\beta}
                    \end{aligned}
                \end{equation*}
            \item There is a uniform constant $C_p<\infty$ depending on $p$ only such that the functions $h_k^N$ satisfy
            \begin{equation*}
                \begin{aligned}
                \mathds{E}\left[ \max_{s\in \mathcal{S}, k_0\leq k\leq K} \left|h_k^N(s)\right|^p \right] &\leq C_p\\
                    \mathds{E}\left[ \max_{s\in \mathcal{S}, k_0\leq k\leq K-1} \left|h_{k+1}^N(s)-h_k^N(s)\right|^p \right] &\leq {C_p}N^{-p}
                \end{aligned}
            \end{equation*}
        \end{enumerate}
        \end{assumption}

        \begin{lemma} \label{stationary_lower_bound_lemma}
            There exists a uniform constant $C \in (0,1)$ such that the stationary distributions $\pi_k^N$ satisfy the lower bound
            \begin{equation*}
                \min_{s\in \mathcal{S},k_0\leq k\leq K} \pi_k^N(s) \geq C.
            \end{equation*}

            Moreover, with $n$ denoting the matrix power, there is a constant $n_0$ such that the fluctuations around the stationary distribution satisfy
            \begin{equation*}
                \max_{s,s' \in \mathcal{S}, k_0\leq k\leq K}\left| \mathds{P}_k^{N,n}(s \rightarrow s') - \pi_k^N(s') \right| \leq (1-C)^{\frac{n}{n_0}}.
            \end{equation*}
        \end{lemma}
        \begin{proof}
            The unimportant constants $C$ and $n_0$ may change at every step but remain uniform. For every fixed $k$, the ergodicity assumption in \ref{MC_bound_assumption} implies that for every $s,s' \in \mathcal{S}$ there exists a $n_0=n_0(s, s',k,N)$ such that $\mathds{P}_k^{N,n_0}(s\rightarrow s')>0$. We can take $n_0$ large enough to be uniform by observing that for any $k,N$ the implied path of length $n_0$ and positive probability from $s$ to $s'$ can be chosen to be acyclic and of length at most $d_{s}=|\mathcal{S}|$ by a pigeon hole argument. Indeed, if the path is longer than $d_{s}$, we consider the first $\lfloor \frac{d_{s}}{2} \rfloor$ and the last $\lfloor \frac{d_{s}}{2} \rfloor+1$ states. Two of them need to be equal, and the path skipping all the states between those still has positive probability. Consequently, choosing $n_0 = \text{lcm}(1,...,d_{s})$, where lcm denotes the least common multiple, gives
            \begin{equation} \label{mc_transition_lower_bound_eq}
            \mathds{P}_k^{N,n_0}(s,s')>0 \hspace{10pt} \forall s,s' \in \mathcal{S}, \hspace{5pt} k_0\leq k\leq K,N \in \mathds{N}.
        \end{equation}

        By assumption \ref{MC_bound_assumption}, this implies
        \begin{equation} \label{mc_transition_lower_bound_eq2}
            \mathds{P}_k^{N,n_0}(s,s')\geq\epsilon^{n_0}>0 \hspace{10pt} \forall s,s' \in \mathcal{S}, \hspace{5pt} k_0\leq k\leq K,N \in \mathds{N}.
        \end{equation}

        The stationary distribution now satisfies
    \begin{equation} \label{mc_stationary_lower_bound_eq}
    \begin{aligned}
        \min_{k_0\leq k\leq K}\pi_k^N(s') &= \min_{k_0\leq k\leq K} \sum_{s \in \mathcal{S}}\pi_k^N(s) \mathds{P}_k^{N,n_0}(s\rightarrow s')\\
        &\geq \epsilon^{n_0} \min_{k_0\leq k\leq K} \sum_{x,a \in \mathcal{X}\times \mathcal{A}}\pi_k^N(x,a)  \\
        &\geq \epsilon^{n_0},
    \end{aligned}
    \end{equation}
    where we used the fact that $\pi_k^N$ is a probability measure for every $k, N$. The first part of the lemma follows. The bounds of the fluctuations around the corresponding stationary distributions follow from Theorem 16.2.4 in \cite{mcss}.
 \end{proof}

  We now determine solutions to the Poisson equations corresponding to the chain $\mathcal{M}$. The proof of the following lemma is the same as the proof of lemma 4.7 in \cite{odeconv} and we do not repeat it here.
\begin{lemma}\label{Poisson_Equation_lemma}
     The Poisson equations corresponding to the chain $\mathcal{M}$ are
    \begin{equation*}
        \begin{aligned}
            v_{k,s}^{N}(s') - \mathds{P}_k v_{k,s}^{N}(s') &= \mathds{1}\left\{ s' =s \right\} - \pi_k^N(s).
        \end{aligned}
    \end{equation*}

    A solution for each of them is
    \begin{equation*}
        \begin{aligned}
            v_{k, s}^{N}(s') &:= \sum_{n\geq 0} \left[ \mathds{P}_k^{N,n}(s', s)-\pi_k^N(s)\right],
        \end{aligned}
    \end{equation*}
    and there exists a uniform positive constant $C$, such that the solutions satisfy
    \begin{equation*}
        \begin{aligned}
            \max_{s, s'\in \mathcal{S}, k_0\leq k\leq K} |v_{k,s}^{N}(s')| &\leq C.
        \end{aligned}
    \end{equation*}
         
     \end{lemma}

With the help of the solutions to the Poisson equations, we will now break the sum
    \begin{equation} \label{PE_sums_to_bound}
        \begin{aligned}
            &\sum_{k=k_0}^{K}\left[ \mathds{1}\left\{ s_k=s\right\} - \pi_k^N(s)\right],
        \end{aligned}
    \end{equation}
into martingale and remainder terms. Our goal throughout lemmas \ref{PE_e_2_bounds} and \ref{PE_e_1_bounds} is to bound the above sum in $L_p$.    Using the solution to the Poisson equations from lemma \ref{Poisson_Equation_lemma}, we can write
    \begin{equation} \label{po_eq_break_1_eq}
        \begin{aligned}
            \sum_{k=k_0}^{K}\Big[ \mathds{1}&\left\{ s_k=s\right\} - \pi_k^N(s)\Big] = \sum_{k=k_0}^{K} \left[ v_{k,s}^{N}(s_k) -\mathbb{P}_k^Nv_{k,s}^{N}(s_k) \right]\\
            &= v_{k_0,s}^{N}(s_0) -\mathbb{P}_{k_0}^Nv_{k_0,s}^{N}(s_0)+\sum_{k=k_0+1}^{K}\left[ v_{k,s}^{N}(s_k) -\mathbb{P}_k^Nv_{k,s}^{N}(s_{k-1})+\mathbb{P}_k^Nv_{k,s}^{N}(s_{k-1})- \mathbb{P}_k^Nv_{k,s}^{N}(s_k) \right]\\
            &= v_{k_0,s}^{N}(s_0) -\psi_{k_0,s}^{N}(s_0)+\sum_{k=k_0+1}^{K}\left[ v_{k,s}^{N}(s_k) -\mathbb{P}_k^Nv_{k,s}^{N}(s_{k-1})+\psi_{k,s}^{N}(s_{k-1})- \psi_{k,s}^{N}(s_k) \right],
        \end{aligned}
    \end{equation}
    where we set $\psi_{k,s}^{N}(\cdot) = \mathbb{P}_k^Nv_{k,s}^{N}(\cdot)$.
We now write the right hand side of \eqref{po_eq_break_1_eq} as
    \begin{equation*}
        \begin{aligned}
            \sum_{k=k_0}^{K}\Big[ \mathds{1}\left\{ s_k=s \right\}- \pi_k^N(s)\Big] = \sum_{k=k_0+1}^{K}\epsilon_{k,s}^{N,1} + \sum_{k=k_0}^{K-1}\epsilon_{k,s}^{N,2}+\rho_{k_0,K,s}^{N},
        \end{aligned}
    \end{equation*}
    where
    \begin{equation*}
        \begin{aligned}
            \epsilon_{k,s}^{N,1}&= \left[ v_{k,s}^{N}({s}_{k})-\mathds{P}_k^{N}v_{k,s}^{N}(s_{k-1})\right], \quad
            \epsilon_{k,s}^{N,2}= \left[ \psi_{k+1,s}^{N}({s}_{k})-\psi_{k,s}^{N}(s_k)\right]\\
            \rho_{k_0,K,s}^{N} &= v_{k_0,s}^{N}({s}_{k_0})+ \psi_{K,s}^{N}(s_{K})
        \end{aligned}
    \end{equation*}

    We will now derive bounds for each term separately. By lemma \ref{Poisson_Equation_lemma} and the fact that the matrices $\mathds{P}_k^N$ are stochastic for all $k,N$, we immediately get for a constant $0<C<\infty$,
    \begin{equation} \label{PE_ro_terms_bound}
        \begin{aligned}
            \max_{s \in \mathcal{S}} \left|\rho_{k_0,K,s}^{N}\right| &\leq C,
        \end{aligned}
    \end{equation}
 
    \begin{lemma}\label{PE_e_2_bounds}
    There exist a uniform constant $C$ and a $N_0 \in \mathbb{N}$ that does not depend on $k_0$ and $K$ such that the following $L_{\infty}$ bound holds for $N\geq N_0$
    \begin{equation*}
        \begin{aligned}
            \max_{s\in \mathcal{S}}\left|\sum_{k=k_0}^{K-1}\epsilon_{k,s}^{N,2}\right| &\leq {C}(K-k_0)N^{-\frac{1}{2}}.
        \end{aligned}
    \end{equation*}
    \end{lemma}

    \begin{proof}
        Let $\theta \in \mathds{R}^+$ be a positive constant whose exact value will be determined later. We can write
        \begin{equation} \label{e_k^2_bound_eq1}
            \begin{aligned}
                \sum_{k=k_0}^{K-1}&\epsilon_{k,s}^{N,2} = \sum_{k=k_0}^{K-1}\left[ \psi_{k+1,s}^{N}({s}_{k})-\psi_{k,s}^{N}(s_k)\right]\\
                &= \sum_{k=k_0}^{K-1}\left[ \sum_{s' \in \mathcal{S}}\mathds{P}_{k+1}^N(s', s_k) v_{k+1,s}^{N}({s}_{k})-\sum_{s' \in \mathcal{S}}\mathds{P}_k^N(s',s_k) v_{k,s}^{N}(s_k)\right]\\
                &= \sum_{k=k_0}^{K-1}\left[\sum_{s' \in \mathcal{S}}\mathds{P}_{k+1}^N(s', s_k) \sum_{n\geq 0} \left[ \mathds{P}_{k+1}^{N,n}(s_k, s)-\pi_{k+1}^N(s)\right]-\sum_{s' \in \mathcal{S}}\mathds{P}_{k}^N(s', s_k) \sum_{n\geq 0} \left[ \mathds{P}_k^{N,n}(s_k, s)-\pi_k^N(s)\right]\right]\\
                &=  \sum_{k=k_0}^{K-1}\left[ \sum_{n=1}^{\infty}\left( \mathds{P}_{k+1}^{N,n}(s_k,s) - \pi_{k+1}^N(s) \right) - \sum_{n=1}^{\infty}\left( \mathds{P}_{k}^{M,n}(s_k,s) - \pi_k^N(s) \right)\right]\\
                &=  \sum_{k=k_0}^{K-1}\left[ \sum_{n=1}^{\lfloor N^{\theta}\rfloor-1}\left( \mathds{P}_{k+1}^{N,n}(s_k,s) - \pi_{k+1}^N(s) \right) - \sum_{n=1}^{\lfloor N^{\theta}\rfloor-1}\left( \mathds{P}_{k}^{M,n}(s_k,s) - \pi_k^N(s) \right)\right]\\
                &+\sum_{k=k_0}^{K-1}\left[ \sum_{n=\lfloor N^{\theta}\rfloor}^{\infty}\left( \mathds{P}_{k+1}^{N,n}(s_k,s) - \pi_{k+1}^N(s) \right) - \sum_{n=\lfloor N^{\theta}\rfloor }^{\infty}\left( \mathds{P}_{k}^{M,n}(s_k,s) - \pi_k^N(s) \right)\right].
            \end{aligned}
        \end{equation}

    Let us now bound the tail of the sum in \eqref{e_k^2_bound_eq1} by using lemma \ref{stationary_lower_bound_lemma}
    \begin{equation} \label{e_k^2_tail_bound}
        \begin{aligned}
            &\left| \sum_{k=k_0}^{K-1}\left[ \sum_{n=\lfloor N^{\theta}\rfloor}^{\infty}\left( \mathds{P}_{k+1}^{N,n}(s_k,s) - \pi_{k+1}^N(s) \right) - \sum_{n=\lfloor N^{\theta}\rfloor }^{\infty}\left( \mathds{P}_{k}^{N,n}(s_k,s) - \pi_k^N(s) \right)\right]\right|\\
            &\hspace{20pt}\leq \sum_{k=k_0}^{K-1}\sum_{n=\lfloor N^{\theta}\rfloor}^{\infty} 2\left(1-C \right)^\frac{n}{n_0}\\
            &\hspace{20pt}\leq 2 (K-k_0)\left(1-C \right)^\frac{{\lfloor N^{\theta}\rfloor}}{n_0}\sum_{n=0}^{\infty}\left(1-C \right)^{\frac{n}{n_0}}\\
            &\hspace{20pt}\leq 2(K-k_0) \left( \frac{1}{1-(1-C)^{\frac{1}{n_0}}} \right)\left(1-C \right)^\frac{{\lfloor N^{\theta}\rfloor}}{n_0},
        \end{aligned}
    \end{equation}
    which decays exponentially with $N$ and can thus be bounded by $C(K-k_0)N^{-\frac{1}{2}}$ for any choice of $\theta>0$ and some constant $C>0$, as long as $N>N_0$ for some $N_0\in \mathbb{N}$. We could have chosen any power of $N$ here, the choice $N^{-\frac{1}{2}}$ is made for convenience. At this step, it is important to notice that $N_0$ can be chosen to be independent of $k_0$ and $K$. This is used throughout the remainder of this section to argue for uniformity with respect to those variables.

    We also derive the bound
    \begin{equation} \label{e_k^2_nontail_bound_eq}
        \begin{aligned}
            &\left| \sum_{k=k_0}^{K-1}\left[ \sum_{n=1}^{\lfloor N^{\theta}\rfloor-1}\left( \mathds{P}_{k+1}^{N,n}(s_k,s) - \pi_{k+1}^N(s) \right) - \sum_{n=1}^{\lfloor N^{\theta}\rfloor-1}\left( \mathds{P}_{k}^{N,n}(s_k,s) - \pi_k^N(s) \right)\right]\right|\\
            &\leq (K-k_0)N^{\theta}\left[ \max_{k_0\leq k\leq  K-1, n\leq \lfloor N^{\theta}\rfloor} \left| \mathds{P}_{k+1}^{N,n}(s_k,s)-\mathds{P}_{k}^{N,n}(s_k,s) \right| +\max_{k_0\leq k\leq  K-1, n\leq \lfloor N^{\theta}\rfloor} \left| \pi_{k+1}^N(s)-\pi_k^N(s)\right|\right]\\
            &\hspace{20pt} \leq C(K-k_0)N^{2\theta-\beta},
        \end{aligned}
    \end{equation}
    where we used assumption \ref{MC_bound_assumption} and the observation that we can factorize
    \begin{equation}
        \begin{aligned}
            \mathds{P}_{k+1}^{N,n}-\mathds{P}_{k}^{N,n} = \left( \mathds{P}_{k+1}^N - \mathds{P}_{k}^N \right) \sum_{i=0}^{n-1} \mathds{P}_{k+1}^{N,i}\mathds{P}_{k}^{N,n-1-i},
        \end{aligned}
    \end{equation}
    which by the sub-additivity and multiplicativity of the matrix maximum norm gives
    \begin{equation}
    \begin{aligned}
        \max_{s, s' \in \mathcal{S}} &| \mathds{P}_{k+1}^{N,n}(s, s')-\mathds{P}_{k}^{N,n}(s, s')| \\
        &\leq \max_{s, s' \in \mathcal{S}} | \mathds{P}_{k+1}^{N}(s, s')-\mathds{P}_{k}^{N}(s, s')|\sum_{i=0}^{n-1} \left(\max_{s, s' \in \mathcal{S}}|\mathds{P}_{k+1}^{N,i}(s, s')| \cdot \max_{s, s' \in \mathcal{S}}|\mathds{P}_{k}^{N,n-1-i}(s, s')|\right)\\
        &\leq CN^{\theta-\beta}.
    \end{aligned}
    \end{equation}
    
    We used the fact that $\mathds{P}_k^{N,i}$ and $\mathds{P}_{k-1}^{N,i}$ are stochastic matrices for every $i \in \{1,...,n-1\}$. Combining the bounds of \eqref{e_k^2_tail_bound} and \eqref{e_k^2_nontail_bound_eq} for the specific choice $\theta = \frac{\beta}{2}-\frac{1}{4} >0$ and using those in \eqref{e_k^2_bound_eq1} gives the result.
    \end{proof}

    \begin{lemma}\label{PE_e_1_bounds}
        There exists a uniform constant $C_p$ that only depends on $p$ and, particularly, that does not depend on $k_0$ and $K$ such $L_p$ bounds hold for $p\geq 2$
        \begin{equation*}
            \begin{aligned}
                \max_{s \in \mathcal{S}}\mathds{E}\left[ \left|\sum_{k=k_0+1}^{K}\epsilon_{k,s}^{N,1} \right|^p \right]&\leq {C_p}{(K-k_0)^{p/2}}.
            \end{aligned}
        \end{equation*}
    \end{lemma}

    \begin{proof}
        We start by noticing that $\left\{Z_{n,k_0,s}^N = \sum_{k=k_0+1}^{n}\epsilon_{k,s}^{N,1}, \mathcal{F}_{n-1}\right\}_{n\geq k_0+1}$ is a martingale. Indeed, this follows from the increments having zero expectation, i.e.
        \begin{equation*}
            \mathds{E}\left[ v_{k,s}^{N,1}(s_{k})|\mathcal{F}_{k-1}\right] = \mathds{P}_{k}^N v_{k,\xi}^{N}(s_{k-1}),
        \end{equation*}
    as well as the fact that $\mathds{P}_k^N$ is a stochastic matrix for every $k,N$ and the uniform bound on $v_{k,\xi}^{N,1}$ from lemma \ref{Poisson_Equation_lemma}, which ensures that $Z_{n,k_0,\xi}^N$ is integrable for every fixed $n$.

By lemma \ref{Poisson_Equation_lemma}, and assumption \ref{MC_bound_assumption} on the finiteness of $\mathcal{S}$, the quadratic variation satisfies 
    \begin{equation*}
        \begin{aligned}
            \left[ Z_{K,k_0,s}^N \right] &=  \sum_{k=k_0}^{K} \left| \mathds{P}_k^N v_{k,s}^{N}(s_{k}) - v_{k,s}^{N}(s_{k+1})\right|^2 \leq C(K-k_0).
        \end{aligned}
    \end{equation*}

Consequently, 
    \begin{equation*}
         \mathds{E}\left[\left[ Z_{K,k_0,s}^N \right]^{\frac{p}{2}} \right] \leq \left(C(K-k_0)\right)^{p/2}.
    \end{equation*}

    The $L_p$ bounds then follows from the Burkholder-Davis-Gundy Inequality for martingales.
    \end{proof}

    The bounds \eqref{PE_ro_terms_bound} and lemmas \eqref{PE_e_2_bounds} and \eqref{PE_e_1_bounds} allow us to finally derive $L_p$ bounds for the quantities defined in \eqref{PE_sums_to_bound}. Particularly, we get the following bounds, that hold for $N>N_0$ for some $N_0 \in \mathbb{N}$
    \begin{equation} \label{PE_sums_bounds_eq}
        \begin{aligned}
            &\mathds{E}\left[\left|\sum_{k=k_0}^{K}\left[ \mathds{1}\left\{ s_k=s\right\} - \pi_k^N(s)\right]\right|^p\right] \leq {C_p}{(K-k_0)^{p/2}} +C_p + C_p\left(K-k_0\right)^pN^{-p/2}.
        \end{aligned}
    \end{equation}

    We are now ready to study bound the quantity \eqref{M_t^N_general_eq} for functionals that are not varying, i.e. do not depend on the step $k$, which is made precise in the following lemma. As we will use it in our further analysis for functions of different order with respect to $N$, we introduce the parameter $\theta$, that will be generic for now and take the values $\theta=0$ or $\theta=p$ later on, when using the result of this lemma.
    \begin{lemma}\label{mc_clt1}
    Let $h^N :\mathcal{S}\rightarrow \mathbb{R}$ be a sequence of functions for which there is a uniform constant $C'_p$ depending on $p$ only such that 
    \begin{equation*}
        \begin{aligned}
            \mathds{E}\left[ \max_{s\in\mathcal{S}} \left|h(s)\right|^p\right] \leq C'_pN^{-\theta}.
        \end{aligned}
    \end{equation*}
    
    Then there exists a uniform constant $C_p<\infty$ depending on $p$ only and a $N_0\in \mathbb{N}$, such that for all $N>N_0$ the following $L_p$ bounds hold:  
    \begin{equation*}
        \mathds{E}\left[ \left|  \sum_{k=k_0}^{K} \left(h^N(s_k) - \sum_{s \in \mathcal{S}}h^N(s)\pi_k^N(s)\right)\right|^{p}\right] \leq C_p(K-k_0)^{\frac{p}{2}}N^{-\theta} +C_pN^{-\theta}+C_p(K-k_0)^{p}N^{-\frac{p}{2}-\theta}.
    \end{equation*}  
    \end{lemma}
    \begin{proof}
        We write
        \begin{equation*}
        \begin{aligned}
             \sum_{k=k_0}^{K}\Big(h^N(s_k) &- \sum_{s\in \mathcal{S}}h^N(s)\pi_k^N(s)\Big) = \sum_{s \in \mathcal{S}} \left( {h^N(s)}\sum_{k=k_0}^{K}\left( \mathds{1}\{s_k=s\} - h^N(s)\pi_k^N(s)\right)\right)\\
            &=\sum_{s \in \mathcal{S}} h^N(s) \left( \sum_{k=k_0}^{K} \mathds{1}\{s_k=s\} -\pi_k^N(s)\right).
        \end{aligned}
        \end{equation*}

        The bounds follow from applying the Cauchy-Schwarz inequality along with the bounds \eqref{PE_sums_bounds_eq} and assumption \ref{MC_bound_assumption} on the finiteness of $\mathcal{S}$.
    \end{proof}

    \begin{lemma} \label{mc_clt2}
        Consider the quantity
        \begin{equation*}
        \begin{aligned}
           M_{k_0,K}^N &=  \sum_{k=k_0}^{K} h_{k}^N(s_k) - \sum_{k=k_0}^{K}\sum_{s \in \mathcal{S}} h_{k}^N(s)\pi_k^N(s),
        \end{aligned}
        \end{equation*}
        where the quantities involved satisfy assumptions \ref{MC_bound_assumption}. There exist uniform constants $C_p$ depending on $p$ only, and a constant $N_0 \in \mathbb{N}$ such that for all $N >N_0$ the following $L_p$ bounds hold
        \begin{equation*}
        \begin{aligned}
             \mathds{E} \left[ |M_{k_0,K}^N|^p \right] &\leq   C_p(K-k_0)^{\frac{p}{2}} + C_p + C_p(K-k_0)^{p}N^{-\frac{p}{2}}\\
                &+ C_p(K-k_0)^{\frac{3p}{2}}N^{-p} +C_p(K-k_0)^pN^{-p}+C_p(K-k_0)^{2p}N^{-\frac{3p}{2}}.
        \end{aligned}
        \end{equation*}

    \end{lemma}

    \begin{proof} 
        We first notice that by lemma \ref{mc_clt1}, for $\theta=p$, as $N>N_0$ 
        \begin{equation*}
            \begin{aligned}
                \mathds{E}&\left[ \left|  \sum_{m=k}^{K} \left(h_{k}^N(s_m)-h_{k-1}^N(s_m) \right) - \sum_{m=k}^{K} \sum_{s \in \mathcal{S}} \left(h_{k}^N(s)-h_{k-1}^N(s) \right)\pi_m^N(s)\right|^{p}\right]\\
                &\leq C_p(K-m)^{\frac{p}{2}}N^{-p} +C_pN^{-p}+C_p(K-m)^{p}N^{-\frac{3p}{2}}.
            \end{aligned}
        \end{equation*}

        Moreover, by lemma \ref{mc_clt1}, for $\theta=0$, we get 
        \begin{equation*}
        \begin{aligned}
            \mathds{E}\left[ \left|  \sum_{k=k_0}^{K} \left(h_{0}^N(s_k) - \sum_{s \in \mathcal{S}}h_{0}^N(s)\pi_k^N(s) \right) \right|^p \right] &\leq  C_p(K-k_0)^{\frac{p}{2}} +C_p+C_p(K-k_0)^{p}N^{-\frac{p}{2}}.
        \end{aligned}
        \end{equation*}
        
        We now write
            \begin{align*}
                M_{k_0,K}^N &=  \sum_{k=k_0}^{K} \left[ h_{k}^N(s_k) - \sum_{s \in \mathcal{S}} h_{k}^N(s)\pi_k^N(s) \right] \\
                &= \left( \sum_{k=k_0}^{K} h_{0}^N(s_k) + \sum_{k=k_0+1}^{K}\sum_{m=k}^{K}\left( h_{k}^N(s_m) - h_{k-1}^N(s_m) \right)\right)\\
                &-\left( \sum_{k=k_0}^{K} \sum_{s \in \mathcal{S}}h_{0}^N(s)\pi_k^N(s) + \sum_{k=1}^{K}\sum_{m=k}^{K} \sum_{s \in \mathcal{S}}\left( h_{k}^N(s) - h_{k-1}^N(s) \right)\pi_m^N(s)\right)\\
                &= \left( \sum_{k=k_0}^{K} h_{0}^N(s_k)-\sum_{k=k_0}^{K} \sum_{s \in \mathcal{S}}h_{0}^N(s)\pi_k^N(s) \right)\\
                &+ \sum_{k=k_0+1}^{K} \left[ \sum_{m=k}^{K}\left( h_{k}^N(s_m) - h_{k-1}^N(s_m)\right) -  
                \sum_{m=k}^{K} \sum_{s \in \mathcal{S}}\left( h_{k}^N(s) - h_{k-1}^N(s) \right)\pi_m^N(s) \right].
            \end{align*}

        Using the power-mean inequality we then obtain for $N>N_0$
        \begin{equation*}
            \begin{aligned}
                &\mathds{E}\left[|M_{k_0,K}^N|^p\right] \leq C_p\mathds{E}\left[ \left|  \sum_{k=k_0}^{K} h_{0}^N(s_k)-\sum_{k=k_0}^{K} \sum_{s \in \mathcal{S}}h_{0}^N(s)\pi_k^N(s) \right|^p \right]\\
                &+C_p(K-k_0)^{p-1}\sum_{k =k_0+ 1}^{K} \mathds{E}\left[ \left| \sum_{m=k}^{K}\left( h_{k}^N(s_m) - h_{k-1}^N(s_m)\right) -  
                \sum_{m=k}^{K} \sum_{s \in \mathcal{S}}\left( h_{k}^N(s) - h_{k-1}^N(s) \right)\pi_m^N(s) \right|^p \right]\\
                &\leq   C_p(K-k_0)^{\frac{p}{2}} + C_p + C_p(K-k_0)^{p}N^{-\frac{p}{2}}\\
                &\hspace{10pt} + C_p(K-k_0)^{\frac{3p}{2}}N^{-p} +C_p(K-k_0)^pN^{-p}+C_p(K-k_0)^{2p}N^{-\frac{3p}{2}}.
            \end{aligned}
        \end{equation*}
    \end{proof}

    We are now ready to proof lemmas \ref{M_bound_lemma} and \ref{M_ht_L_1_bound_lemma} using the results of this section. 

\begin{proof} [Proof of lemma \ref{M_bound_lemma}.]
    We start by studying the term $M_t^{1,N}(\xi^*)$ for some fixed $\xi^*=(x^*,a^*)$ as defined in \eqref{M_t^N_eq}. We will use the results of this section. Specifically, we will show that Assumption \ref{MC_bound_assumption} is satisfied for 
    \begin{equation*}
        \begin{aligned}
            k_0&=0, \quad K=\lfloor Nt \rfloor -1, \quad  \mathcal{S} = \mathcal{X}\times\mathcal{A}, \quad s_k = (x_k,a_k),\quad
            h_k^N(x,a) = r(x,a)\langle B_{(x^*,a^*),(x,a)}, v_k^N\rangle\\ 
            \pi_k^N &= \pi^{g_k^N}, \quad
            \mathbb{P}_k^N((x,a)\to (x',a'))= g_k^N(x',a')p(x'|x,a).
        \end{aligned}
    \end{equation*}

    Assumptions 1-3 in \ref{MC_bound_assumption} are satisfied by assumptions \ref{MDP_assumptions} and \ref{Markov_Chain_ergodicity_assumption}. Assumption 4 in \ref{MC_bound_assumption} is satisfied by the definition of $\eta_t^N$ and the formula \eqref{g_k^N_eq} for $g_k^N$. It remains to show that assumptions 5 and 6 in \ref{MC_bound_assumption} hold. We start with assumption 5 and we notice that for arbitrary $k,x,a$, 
    \begin{equation}
        \begin{aligned}
            g_{k+1}^N(x,a) - g_{k}^N(x,a) & = \frac{\eta_{k+1}^N-\eta_{k}^N}{|\mathcal{A}|} + (1-\eta_{k+1}^N)\left(f_{k+1}^N(x,a)-f_{k}^N(x,a) \right)+\left(\eta_{k+1}^N-\eta_{k}^N\right)f_{k}^N(x,a).
        \end{aligned}
    \end{equation}

   Using the formula for $\eta_k^N$ in \eqref{learning_rates_def_eq} and the Lipschitz continuity of the softmax function along with the equation \eqref{Q_k^N_diff_eq} and the bounds \eqref{Q_P_absolute_bound_eq}, we obtain
    \begin{equation} \label{g_k^N_diff_bound_eq}
        \left| g_{k+1}^N(x,a) - g_{k}^N(x,a) \right| \leq C_TN^{-\beta}
    \end{equation}

    Since by the definition of $\mathbb{P}_k^N$ and equation \eqref{measure_lipschitz_assumption_eq} we have
    \begin{equation} \label{P_pi_diff_bound_eq}
        \begin{aligned}
        \left| \pi^{g_{k+1}}(x,a)-\pi^{g_k^N}(x,a)\right| &\leq \left| g_{k+1}^N(x,a) - g_{k}^N(x,a) \right|\\
            \left| \mathbb{P}_{k+1}((x,a)\to(x',a'))-\mathbb{P}_k((x,a)\to(x',a'))\right| &\leq \left| g_{k+1}^N(x,a) - g_{k}^N(x,a) \right|.
        \end{aligned}
    \end{equation}

    Assumption 5 in \ref{MC_bound_assumption} follows. We now also argue that assumption 6 holds. The first part of assumption 6 follows immediately from assumption \ref{MDP_assumptions} on the boundedness of $r$ and the bound \eqref{leading_order_conv_sec_b2_eq}. For the second part we write
    \begin{equation}
        \begin{aligned}
            h_{k+1}^N(x,a)-h_k^N(x,a) = r(x,a)\left( \langle B_{(x^*,a^*),(x,a)}, v_{k+1}^N \rangle - \langle B_{(x^*,a^*),(x,a)}, v_{k}^N \rangle  \right).
        \end{aligned}
    \end{equation}

The desired bound then follows from  lemma \ref{B_diff_bound_lemma} and assumption \ref{MDP_assumptions} on the boundedness of $r$ and the finiteness of $\mathcal{X}\times \mathcal{A}$, that allows us to take the maximum over $\xi^*$

    We now consider $M_t^{3,N}$ as defined in \eqref{M_t^N_eq}. The only difference to the study of $M_t^{1,N}$ is that we need to re-define 
    \begin{equation*}
        h_k^N(x,a) = Q_k^N(x,a)\langle B_{(x^*,a^*),(x,a)}, v_k^N\rangle,
    \end{equation*}
    and show that this choice of $h_k^N$ also satisfies assumption 6 in \ref{MC_bound_assumption}. The first part follows immediately from the $L_p$ bounds of lemma \ref{Q_P_L_4_Bound} and the bound \eqref{leading_order_conv_sec_b2_eq}. For the second part we write 
    \begin{equation*}
    \begin{aligned}
        h_{k+1}^N(x,a) - h_k^N(x,a) &= Q_{k+1}^N(x,a) \left( \langle B_{(x^*,a^*),(x,a)}, v_{k+1}^N \rangle - \langle B_{(x^*,a^*),(x,a)}, v_{k}^N \rangle\right) \\
        &- \left( Q_{k+1}^N(x,a)-Q_k^N(x,a)\right) \langle B_{(x^*,a^*),(x,a)}, v_{k}^N \rangle.
    \end{aligned}
    \end{equation*}

    The result then follows from lemmas \ref{Q_P_L_4_Bound} and \ref{B_diff_bound_lemma}, equation \eqref{leading_order_conv_sec_b2_eq}, and the observation that
    \begin{equation*}
        Q_{k+1}^N(x,a)-Q_k^N(x,a) = O_{L_p}(N^{-1}),
    \end{equation*}
    which follows from equation \eqref{Q_k^N_diff_eq}, the bounds \eqref{leading_order_conv_sec_b2_eq} and \eqref{R_k^N_bound_eq}, assumption \ref{MDP_assumptions} on the boundedness of $r$, and lemma \ref{Q_P_L_4_Bound}.

    The argument for $\tilde{M}_t^N$ defined in \eqref{M_tilde_t^N_eq} is analogous. Lastly, we study the term $M_t^{2,N}$, defined in \eqref{M_t^N_eq}, which is also of the form \eqref{M_t^N_general_eq} with 
     \begin{equation*}
        \begin{aligned}
            k_0&=0, \quad K=\lfloor Nt \rfloor -1, \quad  \mathcal{S} = \mathcal{X}\times\mathcal{A}\times \mathcal{X}\times \mathcal{A}, \quad
            s_k = (x_k,a_k, x_{k+1},a_{k+1})\\
            h_k^N(x,a,x',a') &= Q_k^N(x',a')\langle B_{(x^*,a^*),(x,a)}, v_k^N\rangle, \quad 
            \pi_k^N(x,a,x',a') = g_k^N(x',a')p(x'|x,a)\pi^{g_k^N}(x,a)\\
            \mathbb{P}_k^N((x,a,x',a')&\to (x'',a'',x''',a'''))= \mathds{1}\{(x'',a'')=(x',a')\}g_k^N(x''',a''')p(x'''|x',a').
        \end{aligned}
    \end{equation*}

    Notice that the implied Markov chains are just the $2$-nd order Markov chains induced from $(\mathcal{M}, g_k^N)$, so that assumptions 1-4 in \ref{MC_bound_assumption} are still satisfied. We will now show that assumption 5 is satisfied.

    Notice that 
    \begin{equation}
        \begin{aligned}
           & \Big|\mathbb{P}_{k+1}^N((x,a,x',a')\to (x'',a'',x''',a'''))-\mathbb{P}_k^N((x,a,x',a')\to (x'',a'',x''',a'''))\Big| \\
            &\quad= \Big|\mathds{1}\{(x'',a'')=(x',a')\}\left(g_{k+1}^N(x''',a''') - g_k^N(x''',a''')\right)p(x'''|x',a')\Big|\\
            &\quad\leq \left|g_{k+1}^N(x''',a''') - g_k^N(x''',a''') \right|\\
            &\Big|g_{k+1}^N(x',a')p(x'|x,a)\pi^{g_{k+1}^N}(x,a)-g_k^N(x',a')p(x'|x,a)\pi^{g_k^N}(x,a)\Big|\\
            &\quad=p(x',a')\left|g_{k+1}^N(x',a')\left( \pi^{g_{k+1}^N}(x,a) - \pi^{g_{k}^N}(x,a)\right) + \pi^{g_{k}^N}(x,a) \left(g_{k+1}^N(x',a') - g_{k}^N(x',a') \right)\right|.
        \end{aligned}
    \end{equation}

    Assumption 5 then follows from the bounds \eqref{g_k^N_diff_bound_eq}, \eqref{P_pi_diff_bound_eq} and the boundedness of $g_k^N$ by its definition in \eqref{g_k^N_eq}. The argument for assumption 6 is similar to the argument for the term $M_t^{3,N}$.

    The terms $M_t^{1,N}$, $M_t^{2,N}$, $M_t^{3,N}$ and $\tilde{M}_t^N$ thus are all of the form \eqref{M_t^N_general_eq} and satisfy assumption \ref{MC_bound_assumption}. Their bounds in lemma \ref{M_bound_lemma} follow from lemma \ref{mc_clt2}.
    \end{proof}

    The proof of lemma \ref{M_ht_L_1_bound_lemma} is analogous. The difference is that one needs to use the bound in lemma \ref{C_bound_lemma} and lemma \ref{C_diff_bound_lemma} in place of the bound \eqref{leading_order_conv_sec_b2_eq} and lemma \ref{B_diff_bound_lemma} respectively. Details are omitted due to the similarity of the argument.

   }

    \end{appendix}

\printbibliography

\end{document}